\documentclass[letterpaper]{article} 
\usepackage{aaai21}  
\usepackage{times}  
\usepackage{helvet} 
\usepackage{courier}  
\usepackage[hyphens]{url}  
\usepackage{graphicx} 
\usepackage{natbib}  
\usepackage{caption} 
\usepackage[utf8]{inputenc} 
\usepackage[T1]{fontenc}    
\usepackage{hyperref}       
\usepackage{url}            
\usepackage{booktabs}       
\usepackage{amsfonts}       
\usepackage{nicefrac}       

\usepackage{microtype}
\usepackage{graphicx, subcaption}
\usepackage{booktabs} 
\usepackage{enumitem}
\usepackage[noend]{algpseudocode}
\usepackage{algorithm}

\usepackage{amsthm}
\usepackage{mathtools}
\usepackage{xspace}

\DeclarePairedDelimiter{\ceil}{\lceil}{\rceil}
\usepackage{hyperref}
\usepackage{chngpage}

\DeclareMathOperator*{\argmin}{arg\,min}
\def \OPTname {\texttt{FLOW\textsuperscript{2}}\xspace}
\def \HPOname {\texttt{CFO}\xspace}
\def \HPOnameO {\texttt{CFO-0}\xspace}

\newcommand{\mt}{\mathsf{T}}

\def \bc {\mathbf{h}}

\def \mu {\delta}
\def \bu {\mathbf{u}}

\def \bs {\mathbf{s}}

\def \tS {\tilde S}

\def \by {\mathbf{y}}
\def \bz {\mathbf{z}}
\def \bbE {\mathbb{E}}

\def \bx {\mathbf{x}}

\def \bu {\mathbf{u}}

\def \bbS {\mathbb{S}}

\def \cC {\mathcal{H}}

\def \Proj {\text{Proj}}
\def \cX {\mathcal{X}}
\def \cU {\mathcal{U}}
\def \deltalb {\delta_{\text{lower}}}

\def \sign {\text{sign}}
\def \Vol {\text{Vol}}

\usepackage{pifont}

\usepackage{todonotes}

\newtheorem{fact}{Fact}

\newtheorem{condition}{Condition}

\usepackage{amssymb}
\usepackage{bm}

\newtheorem{remark}{Remark}
\newtheorem{lemma}{Lemma}
\newtheorem{theorem}{Theorem}

\newtheorem{definition}{Definition}
\newtheorem{proposition}{Proposition}

\newcommand{\compilehidecomments}{true}
\ifthenelse{ \equal{\compilehidecomments}{true} }{%
\newcommand{\Outline}[1]{{\color{blue}{\text{} #1}}}
\newcommand{\note}[1]{{\color{red}{\text{} #1}}}
}{
	\newcommand{\Outline}[1]{}
	\newcommand{\note}[1]{}
}
\newcommand{\aaai}[1]{\noindent{\textcolor{blue}{#1}}}

\newcommand{\revision}[1]{#1}
\newcommand{\techreport}[1]{#1}

\newcommand{\papertext}[1] 

\title{Frugal Optimization for Cost-related Hyperparameters}

\author{%
   Qingyun Wu\textsuperscript{\rm 1}, Chi Wang\textsuperscript{\rm 2}, Silu Huang\textsuperscript{\rm 2} 
 \\
}

\affiliations{
 \textsuperscript{\rm 1}Microsoft Research, New York City, \textsuperscript{\rm 2}Microsoft Research, Redmond \\
    \{Qingyun.Wu, Wang.Chi, Silu.Huang\}@microsoft.com
}

\begin{document}

\maketitle

\begin{abstract}
The increasing demand for democratizing machine learning algorithms calls for hyperparameter optimization (HPO) solutions at low cost. Many machine learning algorithms have hyperparameters which can cause a large variation in the training cost. But this effect is largely ignored in existing HPO methods, which are incapable to properly control cost during the optimization process. To address this problem, we develop a new cost-frugal HPO solution. The core of our solution is a simple but new randomized direct-search method, for which we prove a convergence rate of $O(\frac{\sqrt{d}}{\sqrt{K}})$ and an \(O(d\epsilon^{-2})\)-approximation guarantee on the total cost. We provide strong empirical results in comparison with state-of-the-art HPO methods on large AutoML benchmarks. 
\end{abstract}

\section{Introduction}
Machine learning algorithms usually involve a number of hyperparameters that have a large impact on model accuracy and need to be set appropriately for each task~\cite{melis2018on}. For practitioners to easily and confidently apply generic ML algorithms, methods that can automatically tune these hyperparameters at low cost are needed. It motivates research in efficient hyperparameter optimization (HPO)~\cite{falkner2018}. 
HPO is generally considered as a black-box function optimization problem where evaluating the black-box function is expensive as training and validating a model can be time-consuming.
Further, this evaluation cost can be directly affected by a subset of hyperparameters.
For example, in gradient boosted trees, the variation of the number of trees and the depth per tree can result in a large variation on training and validation time. In such scenarios, sample efficiency cannot be directly translated to cost frugality.
Unfortunately, there has not been a generic HPO formulation that considers the existence of such cost-related hyperparameters, as previous work is mostly focused on the case where the evaluation cost is constant or some special case of variable cost (detailed in Related Work Section)). In this paper, we provide such a formulation to fill this gap and propose a cost-frugal solution. 


\subsection{Problem formulation} Let $f(\bx)$ denote the validation loss of the concerned machine learning algorithm under the given training dataset\footnote{We assume the concerned ML algorithm, training datasets and validation methods are all fixed.} and hyperparameter configuration $\bx$. Let $g(\bx)$ denote the evaluation cost\footnote{Throughout this paper, evaluation cost refers to the computation/time needed for training and validating an ML algorithm with the given training and validation dataset.}  incurred when obtaining $f(\bx)$. In general, $g(\bx)$ is not constant in the hyperparameter configuration space$\cX$. Instead, 
there may exist $\bx_1, \bx_2 \in \cX$ such that $g(\bx_1) \gg g(\bx_2)$. 
The goal of an HPO algorithm $\pi$ is to find $\bx^* = \argmin_{\bx \in \cX} f(\bx)$.
The total cost incurred during this loss optimization process is $G(\pi) = \sum_{i=1}^{K_{\pi}^*} g(\bx_i)$, where $K_{\pi}^{*}$ is the number of function evaluations involved when $\bx^*$ is found by algorithm $\pi$.  In this paper, we formulate the cost-frugal HPO problem as follows: the HPO algorithm $\pi$ needs to find $\bx^*$ while keeping its total cost $G(\pi)$ small.
When cost is indeed constant with respect to $\bx$, $G(\pi)$ naturally degenerates to the number of iterations for convergence times a constant, and thus $\pi$ is naturally cost-frugal as long as it has a good convergence property. In the case where cost-related hyperparameters exist, this formulation enables a characterization of the HPO algorithm $\pi$’s cost behavior.


\subsection{Contribution} In this paper, we take a fresh and unique path of addressing this problem based on \textit{randomized direct search}, and develop a frugal optimization method \HPOname\ for generic cost-related hyperparameters. Our solution is designed toward both small iteration number before convergence and bounded cost per iteration, which lead to low total evaluation cost under mild conditions. 
Specifically, \HPOname\ is built upon our newly proposed randomized direct search method \OPTname, which can have an approximately optimal total cost when minimizing loss.  
Our solution is backed by both theoretical and empirical studies.

On the theoretical side, we first prove that \OPTname\ enjoys a convergence rate of $O(\frac{\sqrt{d}}{\sqrt{K}})$ even in the non-convex case under a common smoothness condition. This convergence result is of independent interest in the theory of derivative-free optimization. Then we prove that due to \OPTname's unique update rule, when it is combined with a low-cost initialization, the cost in any iteration of \OPTname can be upper bounded under reasonable conditions, and the total cost for obtaining an $\epsilon$-approximation of the loss is bounded by $O(d\epsilon^{-2})$ times the optimal configuration's cost by expectation. 
To the best of our knowledge, such theoretical bound on cost does not exist in any HPO literature. 

On the empirical side, we perform extensive evaluations using a latest AutoML benchmark~\cite{Gijsbers2019benchmark} which contains large scale classification tasks. We also enrich it with datasets from a regression benchmark~\cite{Olson2017PMLB} to test regression tasks. Compared to existing random search algorithm and four variations of Bayesian optimization, \HPOname\ shows better anytime performance and better final performance in tuning a popular library XGBoost~\cite{kdd:chen2016xgboost} and deep neural networks on most of the tasks with a significant margin. 

\subsection{Related work} \label{sec:related_work}
The most dominating search strategy for HPO is Bayesian optimization (BO), which uses probabilistic models to approximate the blackbox function to optimize. 
Notably effective BO-based HPO methods include Gaussian process (GP)~\cite{snoek2012practical}, tree Parzen estimator (TPE)~\cite{bergstra2011algorithms} and sequential model-based algorithm configuration method (SMAC)~\cite{hutter2011}.  
However, classical BO-based methods are mainly designed for minimizing the total number of function evaluations, which does not necessarily lead to low evaluation cost.  Some recent work studies ways to control cost in HPO using multi-fidelity optimization.
FABOLAS~\cite{klein2017fast} introduces dataset size as an additional degree of freedom in Bayesian optimization. Hyperband~\cite{ICLR:li2017hyperband} and BOHB~\cite{falkner2018} try to reduce cost by allocating gradually increasing `budgets' in the search process. 
The notion of budget can correspond to either sample size or the number of iterations for iterative training algorithms. These solutions assume the evaluation cost to be equal or similar for each fixed `budget', which is not necessarily true when there exist cost-related hyperparameters. These solutions also require a predefined `maximal budget' and assume the optimal configuration is found at the maximal budget. So the notion of budget is not suitable for modeling even a single cost-related hyperparameter whose optimal value is not necessarily at maximum, e.g., the number K in K-nearest-neighbor algorithm. The same is true for two other multi-fidelity methods BOCA~\cite{kandasamy2017multi} and BHPT~\cite{lu2019BHPT}.  

The only existing method for generic cost-related hyperparameters is Gaussian process with expected improvement per second (GPEIPS)~\cite{snoek2012practical}. It models the evaluation cost using another Gaussian process, and heuristically adds the estimated cost into the acquisition function. Although this method is applicable to generic cost-related hyperparameters, there is no theoretical guarantee on either the optimization of loss or cost. 



\section{Method}

In this section we first present our proposed randomized direct search method \OPTname; then we present theoretical guarantees about both convergence and the cost of \OPTname; at last we present our cost-frugal HPO method \HPOname, which is built upon \OPTname\ and several practical adjustments. 

\subsection{\OPTname} \label{sec_flow2}

\begin{algorithm}
\caption{\OPTname}\label{alg:flow2}
\begin{algorithmic}[1]
\State \textbf{Inputs:} Stepsize $\mu>0$, initial point $\bx_0$ which has a low cost, i.e., small $g(\bx_0)$, and number of iterations $K$.
\State \textbf{Initialization:} Obtain $f(\bx_0)$  
\For{$k = 0, 1, 2, \ldots, K-1$} 
    \State Sample $\bu_k$ uniformly at random from unit sphere $\bbS$
    \If{$f(\bx_{k} + \mu \bu_k) < f(\bx_{k})$} 
    $\bx_{k+1} = \bx_{k} + \mu \bu_k$ 
    \ElsIf{$f(\bx_{k} - \mu \bu_k) < f(\bx_{k})$} 
     \State $\bx_{k+1} = \bx_{k} - \mu \bu_k$
    \Else
        \space $\bx_{k+1} = \bx_{k}$
    \EndIf
\EndFor
\end{algorithmic}
\end{algorithm}

Our algorithm \OPTname\ is presented in Algorithm~\ref{alg:flow2}, taking the step size $\mu$, the number of iterations $K$, and the initial value $\bx_0$ as the input. Let $\bx_k$ denote the incumbent configuration at iteration $k$. \OPTname\ proceeds as follows: at each iteration $k$, we sample a vector $\bu_k$ uniformly at random from a $(d-1)$-unit sphere $\bbS$. Then we compare loss $f(\bx_{k} + \mu \bu_k)$ with $f(\bx_{k})$, if $f(\bx_{k} + \mu \bu_k) <  f(\bx_{k})$, then $\bx_{k+1}$ is updated as $\bx_{k} + \mu \bu_k$; otherwise we compare $f(\bx_{k} - \mu \bu_k)$ with $f(\bx_k)$. If the loss is decreased, then $\bx_{k+1}$ is updated as $\bx_{k} - \mu \bu_k$, otherwise $\bx_{k+1}$ stays at $\bx_k$. 

\OPTname is a simple but carefully designed randomized direct search method. At a high-level, \OPTname starts from a low-cost region and gradually moves towards low-loss region, attempting to avoid evaluations on high-cost and high-loss hyperparameters. \revision{Although the algorithm does not explicitly use the cost information except a requirement of a low-cost initial point, it does implicitly leverage relations between the configurations and their corresponding cost, and the loss and cost of each configuration. 
} \OPTname has some very nice properties including bounded number of iterations and bounded cost per iteration, which together make \OPTname cost-frugal. We remark that existing randomized search methods fail to bound the total cost as we will demonstrate analytically in this section and empirically in the Section~\ref{sec:exp}. In the following, we will elaborate more on \OPTname's cost frugality, along with the comparison with existing methods. 

\textbf{Insights about the frugality of \OPTname.} We first recognize that the exact analytic form of $g(\bx)$ is usually not completely present, mainly because hyperparameters can have complex interactions with the training and inference of an ML algorithm. $g(\bx)$ can be observed after evaluating $\bx$ (i.e., after the training and validation of the ML algorithm using $\bx$), and in some cases estimated after enough observations. 
It is highly non-trivial to incorporate of the observed or estimated cost information into the search process without affecting the convergence of loss or increasing the total cost needed before finding the optimal point. 
For example, simply downweighting the priority of high-cost configurations (as used in GPEIPS) may make an HPO algorithm spend a long time evaluating many cheap but high-loss configurations. 
Therefore, \OPTname\ is designed not to require the exact analytic form of cost function or estimation from observations. Instead, \OPTname\ specializes its new configuration proposal and incumbent configuration update rules while doing a randomized direct search. Perhaps surprisingly, our algorithm without explicitly using any observed or estimated cost outperforms previous work using that information. \revision{Our algorithm does implicitly leverage some properties of the cost function such as Lipschitz continuity, which are detailed in the cost analysis of \OPTname subsection.} 
When \OPTname's specialized update rules are combined with those properties, it can restrict its search in a subspace of the original search space where the cost of each point is not unnecessarily high, without violation of its convergence property. 

Here let us briefly explain how our update rules are specialized for the sake of cost frugality. Let $\bx_k^{\text{best}}$ denote the currently best configuration (i.e., having the lowest loss) found in \OPTname\ at iteration $k$. Since at every iteration \OPTname\ first checks whether the proposed new points $\bx_{k} \pm \mu \bu_k$ can decrease loss before updating $\bx_{k+1}$ to $\bx_{k} \pm \mu \bu_k$ or to $\bx_{k}$, it guarantees at any iteration $k$, ({\em Property 1}) $\bx_{k} = \bx^{\text{best}}_{k}$ is always true; ({\em Property 2}) the next incumbent configuration $\bx_{k+1}$ is always in a neighboring area of $\bx_{k}$; ({\em Property 3}) except the initialization step, evaluations of new configurations are only invoked in line 5 or line 6 of Algorithm~\ref{alg:flow2}, which correspond to a cost of at most $g(\bx_{k} + \mu \bu_k) + g(\bx_{k} - \mu \bu_k)$. 
Property 1 and Property 2 can help us establish a bound on the difference between the cost of the new incumbent configuration $g(\bx_{k+1})$ and the cost of the currently best point $g(\bx^{\text{best}}_{k})$. If the starting point $\bx_0$ is initialized at the low-cost region\footnote{We call a region low-cost region if $g(\bx) < g(\bx^*)$ for all $\bx$ in that region, and high-cost region otherwise.}, we can further prove that $g(\bx_{k+1})$ 
will not be too much larger than $g(\bx^*)$. 
Combining the above conclusions with Property 3, we bound the cost incurred in each iteration of \OPTname. Then with the convergence guarantee, we can finally bound the total cost incurred in \OPTname (detailed in Proposition~\ref{prop:cost_diff_add} and Theorem~\ref{thm:FLOWcost_twophase_add}).

\textbf{Comparison with existing algorithms.}  
Compared to our method, commonly used Bayesian optimization methods in HPO cannot guarantee Property 2 and 3 introduced in the above paragraph and thus can hardly have a theoretical guarantee on the total cost. Our method is closely related to directional direct search methods~\cite{Kolda2003} and zeroth-order optimization~\cite{nesterov2017random}. Directional direct search methods can guarantee Property 1 and 2 
but they usually cannot guarantee Property 3 and do not have a good convergence result. Our method inherits advantages of zeroth-order optimization~\cite{nesterov2017random} in terms of being able to use randomized function evaluations to approximate gradient information (directional derivative in our case) for general black-box functions, which is the key for our good convergence rate guarantee. However, unlike our method, zeroth-order optimization methods usually cannot guarantee Property 1. Because of the reasons above, it is difficult for both types of the aforementioned methods to establish a bound on the cost similar to \OPTname. 
Our method also shares a similar spirit with hypergradient descent techniques~\cite{pmlr-v37-maclaurin15,10.5555/3045390.3045469} in the sense that we are both trying to use approximated gradient information (with respect to hyperparameters) to guide the search. However, the application of hypergradient descent techniques depends on how the training algorithm works while our method is applicable to general black-box functions. 

\subsection{Convergence of \OPTname} \label{sec:flow2_convergence}

\textbf{Insights about convergence.} From the convergence perspective, \OPTname\ is built upon the insight that if $f(\bx)$ is differentiable\footnote{For non-differentiable functions we can use smoothing techniques, such as Gaussian smoothing~\cite{nesterov2017random} to make a close differentiable approximation of the original objective function. 
} and $\mu$ is small, $\frac{f(\bx+ \mu \bu) - f(\bx)}{\mu}$ can be considered as an approximation to the directional derivative of loss function on direction $\bu$, i.e., $f'_{\bu}(\bx)$. By moving toward the directions where the approximated directional derivative
$\frac{f(\bx+ \mu \bu) - f(\bx)}{\mu}\approx f'_{\bu}(\bx)$ 
is negative, it is likely that we can move toward regions that can decrease the loss. Following this intuition, we establish rigorous theoretical guarantees for the convergence of \OPTname in both non-convex and convex cases under a L-smoothness condition. \techreport{All proof details are provided in Appendix~\ref{sec:appendix_proof_convergence}.}

\begin{condition}[L-smoothness] \label{assumption_smooth} 
Differentiable function $f$ is L-smooth if for some non-negative constant $L$, $\forall \bx, \by \in \mathcal{X}$, $ |f(\by) - f(\bx) - \nabla f(\bx)^\mt (\by - \bx) )| \leq \frac{L}{2} \lVert \by - \bx \rVert^2$, where $\nabla f(\bx)$ denotes the gradient of $f$ at $\bx$.
\end{condition}

\begin{proposition} \label{proposition:flow_gradient_prop}
Under Condition~\ref{assumption_smooth}, \OPTname\ guarantees $f(\bx_k) - \bbE
[f(\bx_{k+1})|\bx_k]   \geq \mu  \cdot c_d \lVert \nabla f(\bx_k) \rVert_2 - \frac{L\mu^2}{2} $, in which $c_d= \frac{2\Gamma(\frac d 2)}{(d-1)\Gamma(\frac{d-1}{2})\sqrt\pi}$ and $\Gamma(\cdot)$ denotes the Gamma function. 

\end{proposition}
This proposition provides a lower bound on the expected decrease of loss for every iteration in \OPTname, i.e., $f(\bx_k) - \bbE[f(\bx_{k+1})|\bx_k]$, where expectation is taken over the randomly sampled directions $\bu_k$. 


\paragraph{Proof idea} The main challenge in our proof lies in the fact that 
while $\bu_k$ is sampled uniformly from the unit hypersphere, the update condition (Line 5 and 6, which are also designed with the purpose of cost control) filters certain directions, and complicates the computation of the expectation.
The main idea of our proof is to partition the unit sphere $\bbS$ into different regions
according to the value of the directional derivative.
For the regions where the directional derivative along the sampled direction $\bu_k$ has large absolute value,
it can be shown that 
our moving direction is close to the gradient descent direction using the L-smoothness condition, which leads to large decrease in loss. 
We prove that even if the loss decrease for $\bu$ in other regions is 0, the overall expectation of loss decrease is close to the expectation of absolute directional derivative over the unit sphere, which equals to $c_d \lVert f(\bx_k) \rVert_{2}$. 

In the following two theorems we characterize the rate of convergence to the global optimal point $\bx^*$ in the convex case and to first-order stationary points in the non-convex case.
\begin{theorem}[Convergence of \OPTname\ in the convex case] \label{theorem:flow_convergence_convex}
If $f$ is convex and satisfies Condition~\ref{assumption_smooth}, $\bbE[f(\bx_K)] - f(\bx^*) \leq  e^{-\frac{\mu c_d K}{R}} r_0 +  \frac{L\mu R}{2c_d}$, in which $c_d$ is defined as in Proposition~\ref{proposition:flow_gradient_prop}, $r_0 = f(\bx_0) - f(\bx^*)$ and $R = \max_{k \in [ K]} \lVert \bx_k - \bx^* \rVert_2$. If $ \mu \propto \frac{1}{\sqrt{K}}$, $  \bbE[f(\bx_K)] - f(\bx^*) \leq  e^{-\frac{c_d \sqrt{K}}{R}} r_0 +  \frac{L R}{2c_d \sqrt{K}}$.

\end{theorem}
\begin{theorem}[Convergence of \OPTname\ in the non-convex case]
\label{theorem:flow_convergence}
Under Condition~\ref{assumption_smooth},  $\min_{k \in [K]} \bbE[\lVert \nabla f(\bx_k) \rVert_2 ] \leq  \frac{r_0  + \frac{1}{2}L K \mu^2}{ c_d (K-1) \mu }$, in which 
$\frac{1}{c_d} = O(\sqrt{d})$, and $r_0 = f(\bx_0) -  f(\bx^*)$.  By letting $\mu \propto \frac{1}{\sqrt{K}}$, $\min_{k \in [K]  } \bbE[ \lVert \nabla f(\bx_k) \rVert_2 ]= O(\frac{\sqrt{d}}{\sqrt{K}})$.
\end{theorem}

The proof of the above two theorems is based on Proposition \ref{proposition:flow_gradient_prop}.
\techreport{We provide a comparison with the convergence properties of zeroth-order optimization and Bayesian optimization methods in Remark~\ref{remark:convergence_compare_zo} and \ref{remark:convergence_compare_bo} of Appendix~\ref{sec:appendix_proof_convergence}.} It is obvious that when the cost is constant, a good convergence of loss directly translates to a good bound on the total cost. However this is not necessarily true when cost-related hyperparameters exist, in which case a naive upper bound for $G(\pi)$ can be as large as $ K^*_{\pi} \max_{\bx \in \cX} g(\bx)$, recalling that $K^*_{\pi}$ denotes the number of iterations used by $\pi$ to find $\bx^*$. 
\subsection{Cost analysis of \OPTname} \label{sec:flow2_cost_ana}
In this section, we provide a rigorous analysis of the cost behavior of \OPTname. 
\techreport{All proof details are provided in Appendix~\ref{sec:appendix_proof_cost}. }

\begin{condition}[Lipschitz continuity of cost function $g(\bx)$] \label{assu:cost_Lipsch}
 $\forall$ $\bx_1, \bx_2 \in \cX$, $| g(\bx_1) - g(\bx_2)| \leq U \times z(\bx_1- \bx_2) $,
 in which $U$ is the Lipschitz constant, and $z(\bx_1 - \bx_2)$ is a particular distance function of $\bx_1 - \bx_2$. For example $z(\cdot)$ can be the Euclidean norm on the cost-related subset of the dimensions. 
\end{condition}
Condition~\ref{assu:cost_Lipsch} recognizes the existence of a certain degree of continuity in the cost function. Using the gradient boosted trees example, let $\bx = (x_1, x_2, x_3)$, where the three dimensions represent the tree number, max tree depth and learning rate respectively. This assumption implies that the difference between $g((50, 10, l_1))$ and $g((51, 11, l_2))$, where $l_1$ and $l_2$ are two settings of the learning rate, should not be too large.

Using the notations defined in Condition~\ref{assu:cost_Lipsch}, we define $D \coloneqq U \times \max_{\bu \in \bbS} z(\mu \bu)$, which is intuitively the largest distance between the points in two consecutive iterations of \OPTname\ considering the fact that the new point is sampled from the $(d-1)$-unit sphere surrounding the incumbent point. Let $\tilde \bx^*$ denotes a locally optimal point of $f$.

\begin{condition} [Local monotonicity between cost and loss] \label{assu:cost_loss_monotonicity}
 $\forall$ $\bx_1,\bx_2 \in \cX$, if $2D+ g(\tilde \bx^*) \geq g(\bx_1) > g(\bx_2) \geq g(\tilde \bx^*)$, then $f(\bx_1) \geq f(\bx_2)$.
\end{condition}

Condition~\ref{assu:cost_loss_monotonicity} states that when the cost surpasses a locally optimal point $\tilde \bx^*$'s cost, i.e., $g(\bx) \geq g(\tilde \bx^*)$, 
with the increase of the evaluation cost in a local range, the loss does not decrease. Intuitively, for most ML models, when the model's complexity\footnote{Model complexity is usually positively correlated with evaluation cost.} is increased beyond a suitable size, the model's performance would not improve with the increase on the model's complexity due to overfitting.  
\techreport{Empirical justification for this assumption is provided in Remark~\ref{remark:monotonicity} of Appendix~\ref{sec:appendix_cost}.}

\begin{proposition}[Bounded cost change in \OPTname] \label{prop:cost_diff_add}
If Condition~\ref{assu:cost_Lipsch} is true, then $g(\bx_{k+1}) \leq g(\bx_k) +  D $, $\forall k$. 
\end{proposition}

\begin{proposition}[Bounded cost for any function evaluation of \OPTname] \label{prop:cost_bound_add}
Under Condition~\ref{assu:cost_Lipsch} and \ref{assu:cost_loss_monotonicity}, if $g(\bx_0) < g(\tilde \bx^*)$, then $g(\bx_k) \leq g(\tilde \bx^*) +  D $, $\forall k$.
\end{proposition}
Proposition~\ref{prop:cost_bound_add} asserts that the cost of each evaluation is always within a constant away from the evaluation cost of the locally optimal point. The high-level idea is that \OPTname\ will only move when there is a decrease in the validation loss and thus the search procedure would not use much more than the locally optimal point's evaluation cost once it enters the locally monotonic area defined in Condition~\ref{assu:cost_loss_monotonicity}. 

Let $\tilde G(\OPTname)$ denotes the expected total evaluation cost for \OPTname\ to approach a first-order stationary point $f(\tilde \bx^*)$ within distance $\epsilon$, and $K^*$ as the expected number of iterations taken by \OPTname\ until convergence. 
According to Theorem~\ref{theorem:flow_convergence}, $K^* = O(\frac{d}{\epsilon^2})$.

\begin{theorem} [Expected total evaluation cost of \OPTname] \label{thm:FLOWcost_twophase_add}
Under Condition~\ref{assu:cost_Lipsch} and Condition~\ref{assu:cost_loss_monotonicity}, if $K^* \leq \ceil{\frac{\gamma}{D}}$, $\tilde G(\OPTname) \leq K^*(g(\tilde \bx^*) + g(\bx_0)) + 2K^* D$; else,
$\tilde G(\OPTname) \leq 2K^* g(\tilde \bx^*) + 4K^*D  -(\frac{\gamma}{D} -1)\gamma $, 
in which $\gamma=g(\tilde \bx^*) - g(\bx_0) > 0$.
\end{theorem}

Theorem~\ref{thm:FLOWcost_twophase_add}
shows that the total evaluation cost of \OPTname\ depends on the number of iterations $K^*$, the maximal change of cost between two iterations $D$, and the evaluation cost gap $\gamma$ between the initial point $\bx_0$ and $\tilde \bx^*$. 
From this result we can see that as long as the initial point is a low-cost point, i.e., $\gamma >0$, and the step size is not too large such that $D \leq g(\bx^*)$, the evaluation cost is always bounded by $\tilde G(\OPTname) \leq 4K^* \cdot (g(\tilde \bx^*) + g(\bx_0) + D) = O(d\epsilon^{-2}) g(\tilde \bx^*) $. Notice that $g(\tilde\bx^*)$ is the minimal cost to spend on evaluating the locally optimal point $\tilde\bx^*$. Our result suggests that the total cost for obtaining an $\epsilon$-approximation of the loss is bounded by $O(d\epsilon^{-2})$ times that minimal cost by expectation. When $g$ is a constant, our result degenerates to the bound on the number of iterations. 
We have not seen 
cost bounds of similar generality in existing work. 


\begin{remark}
To the best of our knowledge, the only theoretical analysis for HPO problems that considers cost appears in Hyperband~\cite{ICLR:li2017hyperband}. They derived a theoretical bound on the loss with respect to the input budget. 
However, as introduced in Section~\ref{sec:related_work}, their notion of `budget' is not suitable for modeling generic cost-related hyperparameters.  
\end{remark}

\begin{remark}
Theorem~\ref{thm:FLOWcost_twophase_add} holds as long as Lipschitz continuity (Condition~\ref{assu:cost_Lipsch}) and local monotonicity (Condition~\ref{assu:cost_loss_monotonicity}) are satisfied. It does not rely on the smoothness condition. So the cost analysis has its value independent of the convergence analysis. This general bound can be further reduced when the computational complexity of the ML algorithm with respect to the hyperparameters is partially known (in $\Theta(\cdot)$ notation instead of the exact analytic form). 
\techreport{In Remark~\ref{remark:factorized_cost} of Appendix~\ref{sec:appendix_cost_factorized}, we consider a particular form (i.e., factorized form) of cost function and provide an improved bound for such cost function. 
We also provide guidance on how to use simple invertible transformation to realize such cost function form given the computational complexity of the ML algorithm with respect to the hyperparameters. }
\end{remark}

\papertext{Due to the space limit, all proof details and how the bound on the total cost can be further reduced are deferred to the technical report of this paper~\cite{wu2020cost}.}

\begin{figure*} [h]
\centering 
\begin{subfigure}{0.26\paperwidth}
\includegraphics[width=1.0\columnwidth]{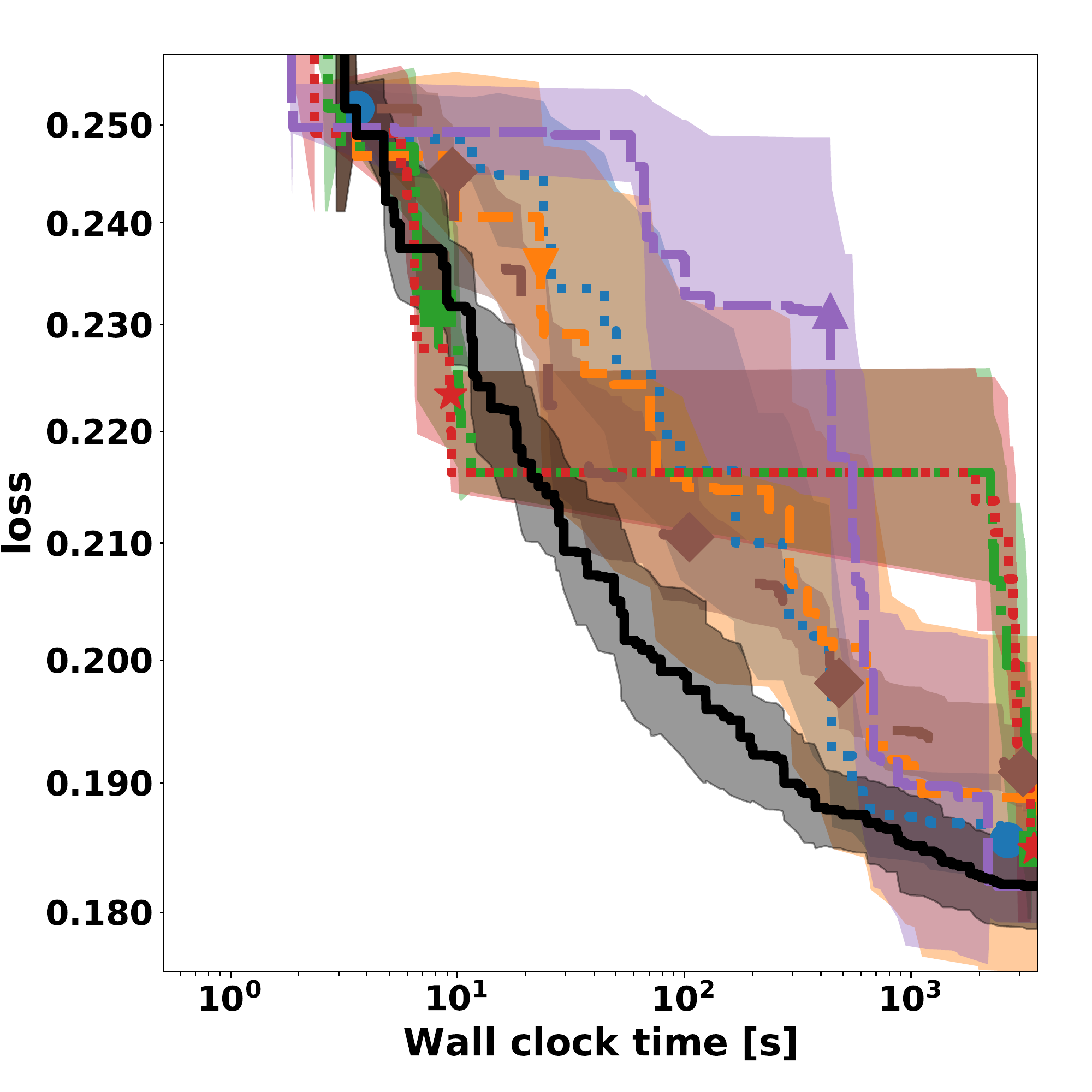}%
\caption{christine, loss = 1 - auc}%
\end{subfigure}\hfill
\begin{subfigure}{0.26\paperwidth}
\includegraphics[width=1.0\columnwidth]{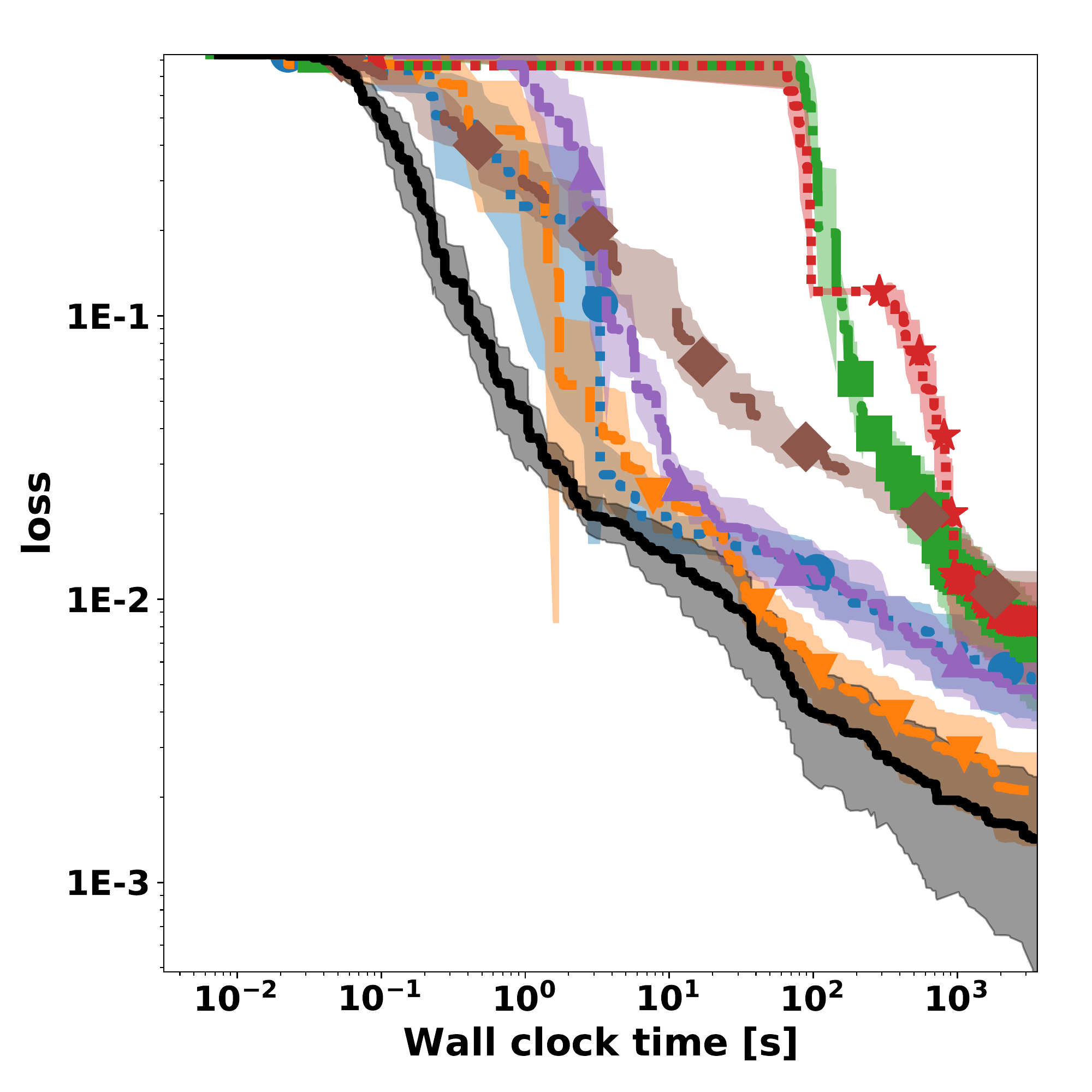}%
\caption{car, loss = log-loss}%
\end{subfigure}\hfill%
\begin{subfigure}{0.26\paperwidth}
\includegraphics[width=1.0\columnwidth]{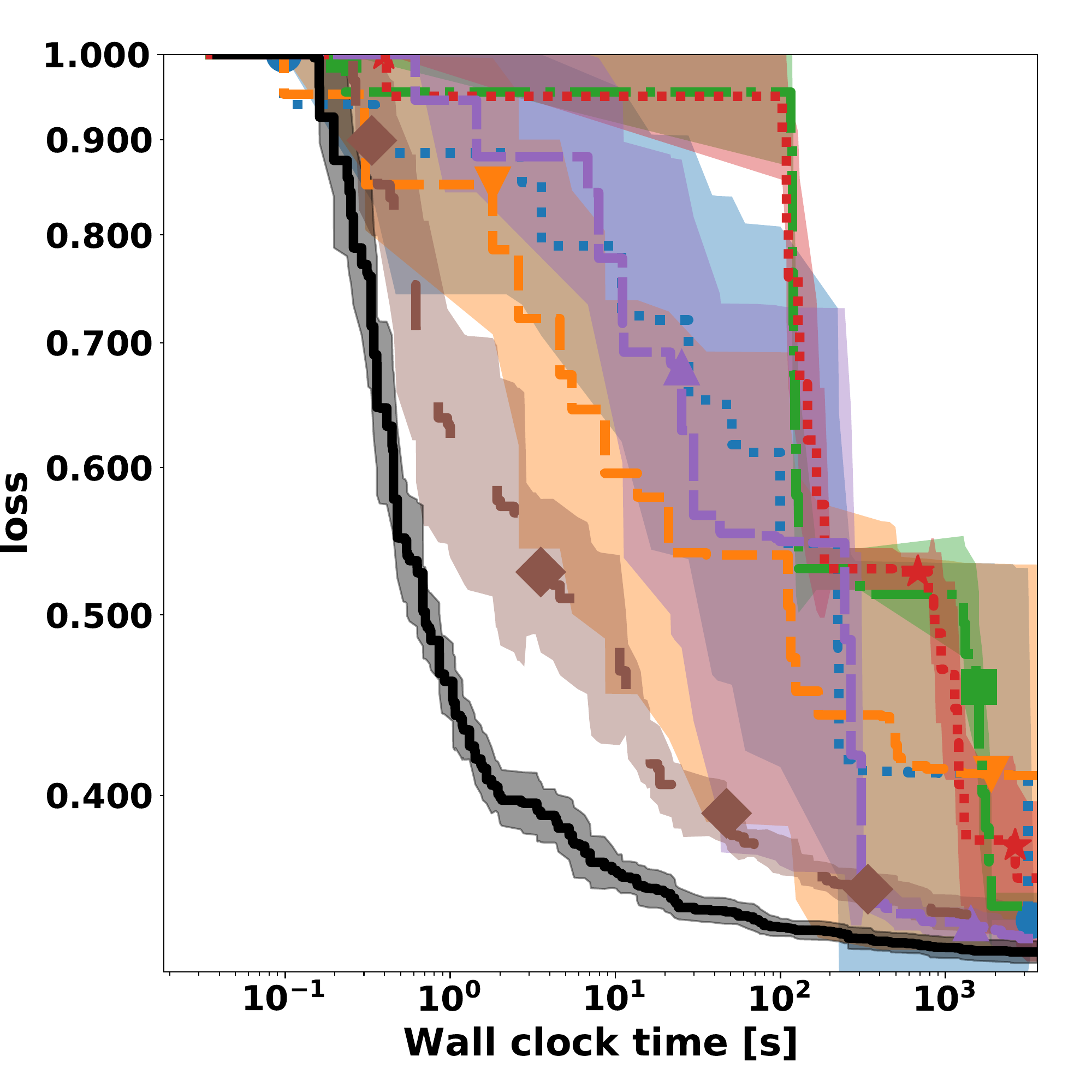}%
\caption{house\_16H, loss = 1 - r2}%
\end{subfigure}\hfill%
\begin{subfigure}{0.26\paperwidth}
\includegraphics[width=\columnwidth]{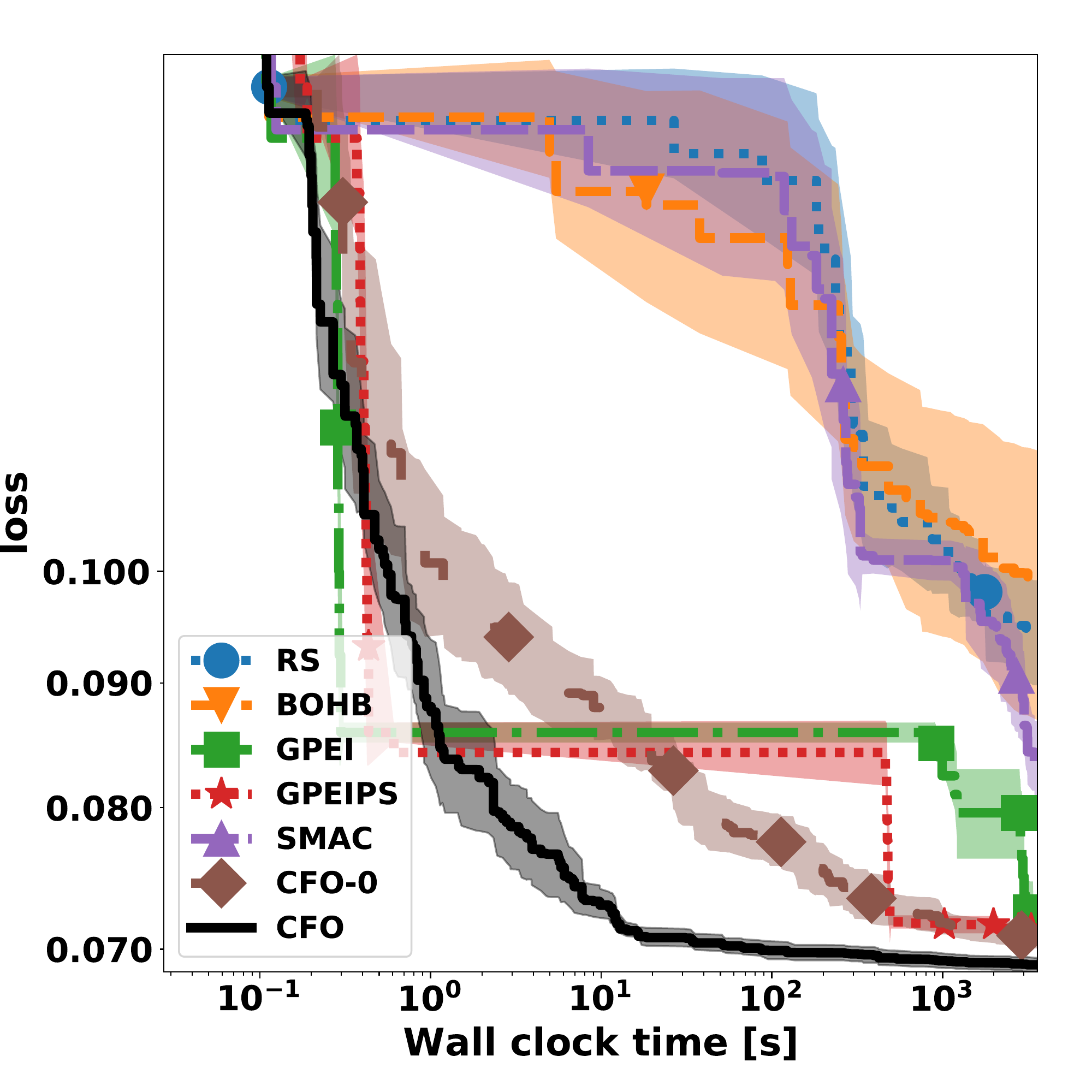}%
\caption{adult, loss = 1 - auc}%
\end{subfigure}\hfill%
\begin{subfigure}{0.26\paperwidth}
\includegraphics[width=\columnwidth]{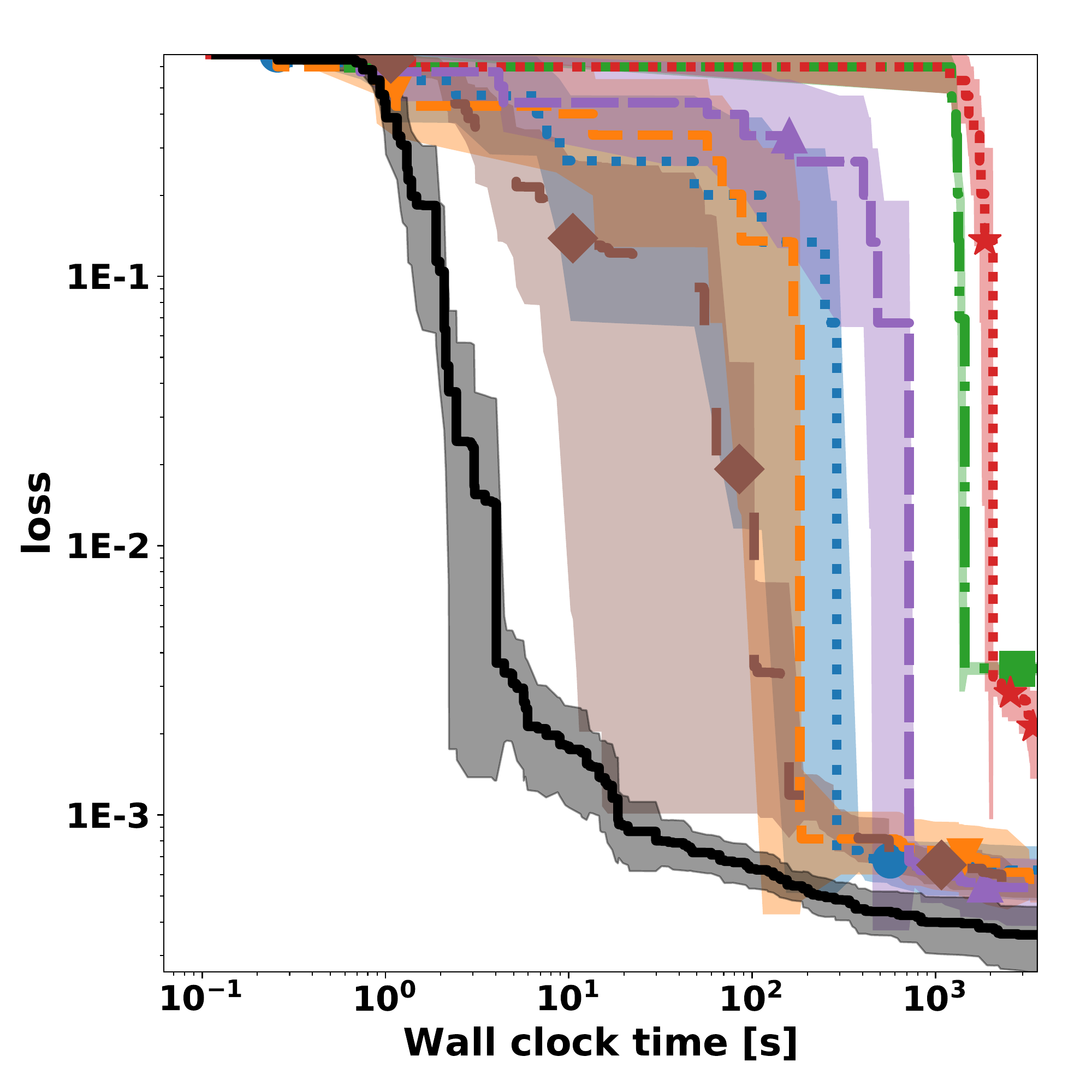}%
\caption{shuttle, loss = log-loss}%
\end{subfigure}\hfill%
\begin{subfigure}{0.26\paperwidth}
\includegraphics[width=\columnwidth]{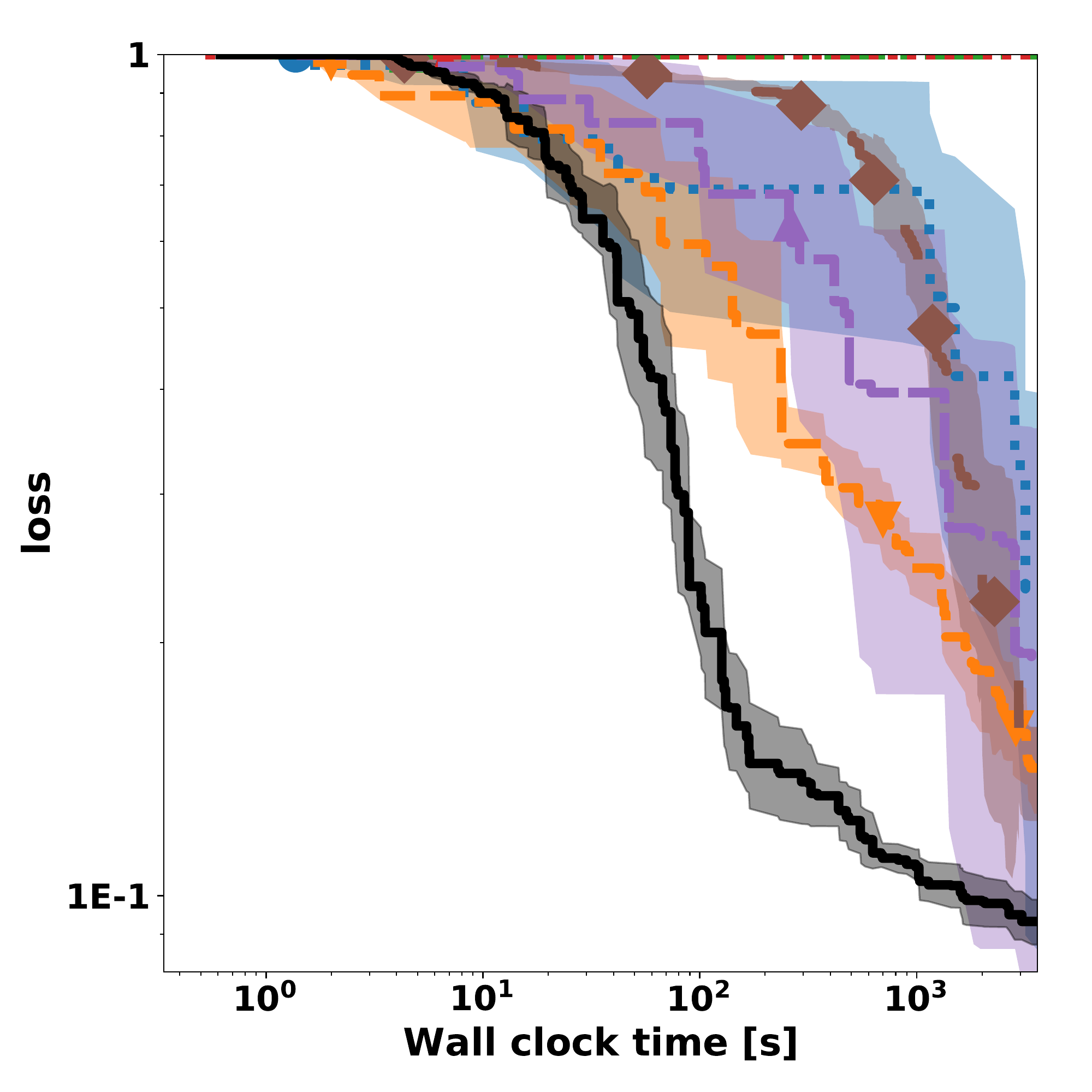}%
\caption{poker, loss = 1 - r2}%
\end{subfigure}\hfill%
\caption{Performance over time. Lines correspond to mean loss over 10 folds, and shades correspond to 95\% confidence intervals}
\label{fig:lc_xgboost}
\end{figure*} 

\subsection{Practical adjustments} \label{sec:CEO}

\begin{algorithm}[h]
\caption{\HPOname}\label{alg:hpo}
\begin{algorithmic}[1]
\State \textbf{Inputs:} 1. Feasible search space $\cX$, the dimensionality of which is $d$. 2. Initial low-cost configuration $\bx_0$.
\State \textbf{Initialization:} Set initial value $\bx = \bx_0$
(and get $f(\bx)$), $\delta = \delta_{init}$, $k = k'= n = r =0$, $l^{\text{r-best}} = +\text{inf}$
\While{Budget allows} 
    \State Sample $\bu$ uniformly at random from $\bbS$
    \State $\bx^{+} \gets \Proj_{\cX}(\bx + \delta \bu)$, $\bx^{-} \gets \Proj_{\cX}(\bx - \delta \bu)$
    \If{$f(\bx^{+}) < f(\bx)$}
        $\bx \gets \bx^{+}$ 
    \ElsIf{$f(\bx^{-}) < f(\bx)$}
            $\bx \gets \bx^{-}$ 
        \Else
            \space $n \gets n +1$
    \EndIf
   
     \If{$f(\bx) < l^{\text{r-best}}$}             $l^{\text{r-best}} \gets f(\bx) $ and $ k' \gets k$  
    \EndIf
    \State $k \gets k+1$
    \If{$n=2^{d-1}$}
        \State $n\gets 0$, $\delta \gets \delta \frac{1}{\sqrt{\eta}}, $ in which $\eta = \frac{k}{k'}$
        \If{$\delta \leq \delta_{\text{lower}}$ }
            \State $k \gets 0$, $l^{\text{r-best}} \gets +\text{inf}$
            \State Reset $\bx \gets \mathbf{g}$, where $\mathbf{g} \sim N(\bx_0, \mathbf{I})$  
            \State $ r\gets r+1$ and $\delta \gets r + \delta_{init}$
        \EndIf
  \EndIf
\EndWhile
\end{algorithmic}
\end{algorithm}

Despite the theoretical guarantee on its good convergence rate and expected total cost, vanilla \OPTname presented in Algorithm~\ref{alg:flow2} is not readily applicable for HPO problems because (1) the possibility of getting stuck into local optima; (2) stepsize is needed as a hyperparameter of \OPTname; and (3) the existence of discrete hyperparameters. Fortunately, those limitations can be effectively addressed using commonly used practical techniques in optimization.  
By adopting the practical adjustments listed below, we turn \OPTname\ into an off-the-shelf HPO solution, which is named as \HPOname(short for Cost-Frugal Optimization) and presented in Algorithm~\ref{alg:hpo}. It is also worth mentioning that in our empirical evaluations, we used the same practical adjustments on an existing zeroth-order optimization method in order to verify the unique advantages of \OPTname in terms of frugality.




\emph{Randomized restart of \OPTname.} 
Similar to most of the other local search algorithms, \OPTname\ may suffer from getting trapped in a local optimum. One common solution to relieve this pain is to restart the algorithm when no progress is observed~\cite{marti2010advanced,zabinsky2010stopping,gyorgy2011efficient}.
Following the same spirit, we restart \HPOname\ from a randomized starting point when no progress is made in it. Specifically, in our work, the `no progress' signal is determined by the number of consecutive no improvement interactions and the value of stepsize. The randomized starting point is generated by adding a Gaussian noise $\mathbf{g}$ to the original initial point $\bx_0$.

\emph{Dynamic adjustments of stepsize $\delta$}. 
To achieve the convergence rate proved in Section~\ref{sec:flow2_convergence}, the stepsize of \OPTname, i.e., $\delta$ needs to be set as a constant that is proportionally to  $1/\sqrt{{K^*}}$, which is difficult to be specified beforehand. 
Adaptive stepsize is prevalent in iterative optimization and search algorithms~\cite{boyd2004convex}. In our work, we propose a self-adjustable stepsize rule, which shares the same spirit with the adaptive rule in \cite{konnov2018conditional}. The stepsize is initially $\delta_{init}$, which we set to be $\sqrt{d}$. It will be decreased when the number of consecutively no improvement iterations is larger than $2^{d-1}$. Specifically, $\mu$ is discounted by a factor of $\frac{1}{\sqrt{\eta}}$, in which the reduction ratio $\eta>1$ is intuitively the ratio between the total number of iterations taken since the last restart and the total number of iterations taken to find the best configuration since the last restart. 
By doing so, the stepsize reduction ratio $\eta$ does not need to be pre-specified but is self-adjustable to the progress made by the search.   
 In order to prevent $\delta$ from becoming too small, we also impose a lower bound on it and stop decreasing $\delta$ once it reaches the lower bound $\deltalb$, which is also designed in a self-adjustable manner\techreport{ (detailed in Algorithm~\ref{alg:delta_lower} in Appendix~\ref{subsec:appendix_ceo}))}.

 \emph{Projection of the proposed configuration.} In practice, the newly proposed configuration $\bx \pm \bu$ is not necessarily in the feasible hyperparameter space $\cX$, especially when discrete hyperparameters exist. In such scenarios, we use a projection function $\text{Proj}_{\cX}(\cdot)$ to map it to the feasible space $\cX$. 

\papertext{We provide the detailed justifications on the design rationale of \HPOname and a variant of \HPOname where normalization of the search space is performed in~\cite{wu2020cost}.}
\techreport{We provide the detailed justifications on the design rationale of the practical adjustments in Appendix~\ref{subsec:appendix_ceo}.}

\subsection{Discussions}
\textbf{(1) No tuning needed.} We realized the practical adjustments in \HPOname in a way that is self-adjustable and does not require on any tuning. \textbf{(2) Low-cost initialization.} The low-cost initialization is fairly easy to perform in practice because we do not have any requirement on the performance (in terms of loss) of it. For example, the user can specify the initial point to be the point which has the lowest possible cost in the search space. \textbf{(3) Parallelization.} Our method is easy to be parallelized. When extra resource is available, instead of doing purely sequential random restarts, we can start multiple \OPTname search threads with different initial points and run them in parallel. \textbf{(4) Categorical hyperparameters.} Our current method \OPTname\ is primarily designed for the optimization of numerical hyperparameters with solid theoretical guarantees. In practice, it is possible to extend it to handle categorical hyperparameters. For example, we can encode categorical choices as integers. But instead of using a fixed mapping between the integer encoding and the categorical choices, we randomly reassign the categorical choices when the projected integer for a categorical dimension changes.
Study of extensions on parallelization and the handling of categorical hyperparameters is left for future work.

\section{Experiment} \label{sec:exp}

We perform an extensive experimental study using a latest open source AutoML benchmark~\cite{Gijsbers2019benchmark}, which includes 39 classification tasks. We enriched it with 14 regression tasks\footnote{Among the 120 regression datasets in PMLB, we selected the ones whose \# of instances are larger than 10K.} from  PMLB~\cite{Olson2017PMLB}. 
All the datasets are available on OpenML. 
Each task consists of a dataset in 10 folds, and a metric to optimize:  Roc-auc for binary tasks, log-loss for multi-class tasks, and r2 score for regression tasks.  

We include 5 representative HPO methods as baselines, including random search (RS)~\cite{Bergstra2012rs}, 
 Bayesian optimization with Gaussian Process and expected improvement (GPEI) and expected improvement per second (GPEIPS) as acquisition functions respectively~\cite{snoek2012practical}, SMAC~\cite{hutter2011}, and BOHB~\cite{falkner2018}. The latter four are all based on Bayesian optimization. Among them, BOHB was shown to be the state of the art multi-fidelity method. We consider the training sample size to be the resource dimension required in BOHB.  GPEIPS considers cost by building a probabilistic model about the configuration cost and using expected improvement per second as the acquisition function.
In addition to these existing HPO methods, we also include an additional method \HPOnameO, which uses the same framework as \HPOname\ but replaces \OPTname\ with the zeroth-order optimization method ZOGD~\cite{nesterov2017random}. Notice that like \OPTname, ZOGD is not readily applicable to the HPO problem, while the \HPOname\ framework permits ZOGD to be used as an alternative local search method. 
The comparison between \HPOnameO\ and \HPOname\ would allow us to evaluate the contribution of \OPTname\ in controlling the cost in \HPOname.
All the methods start from the same initial point as the one used in \HPOname to ensure that the performance boost in \HPOname is not only caused by the low-cost initialization.

We compare their performance in tuning 9 hyperparameters for XGBoost, which is one of the most commonly used libraries in many machine learning competitions and applications.  
In addition to XGBoost, we also evaluated all the methods on deep neural networks. 
Since the experiment setup and comparison conclusion are similar to those in the XGBoost experiment, we include the detailed experiment setup and most of the results on deep neural networks in Appendix~\ref{subsec:appendix_exp}. 


\newdimen\figrasterwd
\figrasterwd\textwidth
\begin{figure*} [h]
  \centering
  \parbox{.29\figrasterwd}{
  \begin{subfigure}{0.24\paperwidth} 
\includegraphics[width=\columnwidth]{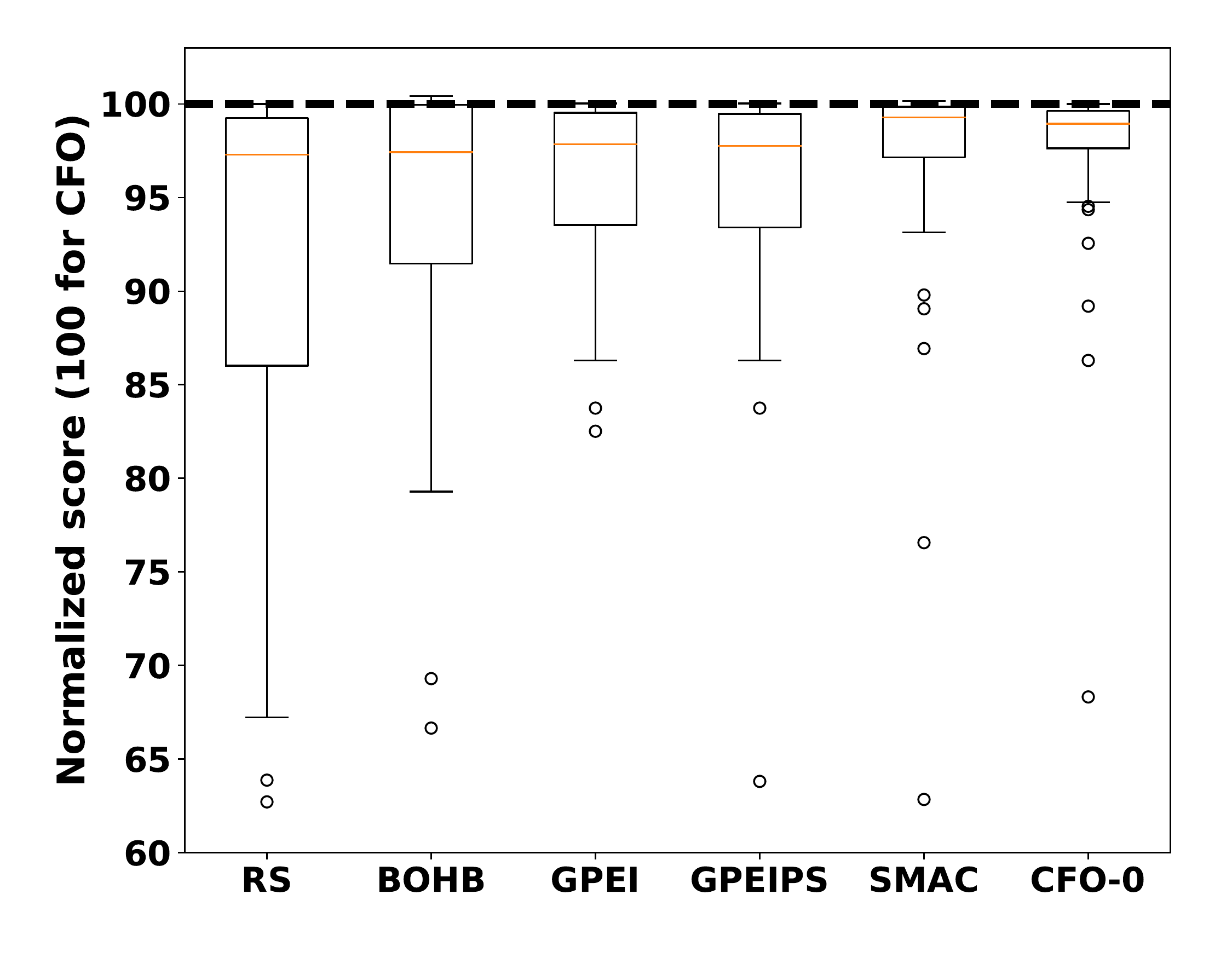}%
\caption{Scores in tuning XGBoost}  \label{fig:xgboost_box} 
\end{subfigure}\vfill%
  \begin{subfigure}{0.24\paperwidth} 
\includegraphics[width=\columnwidth]{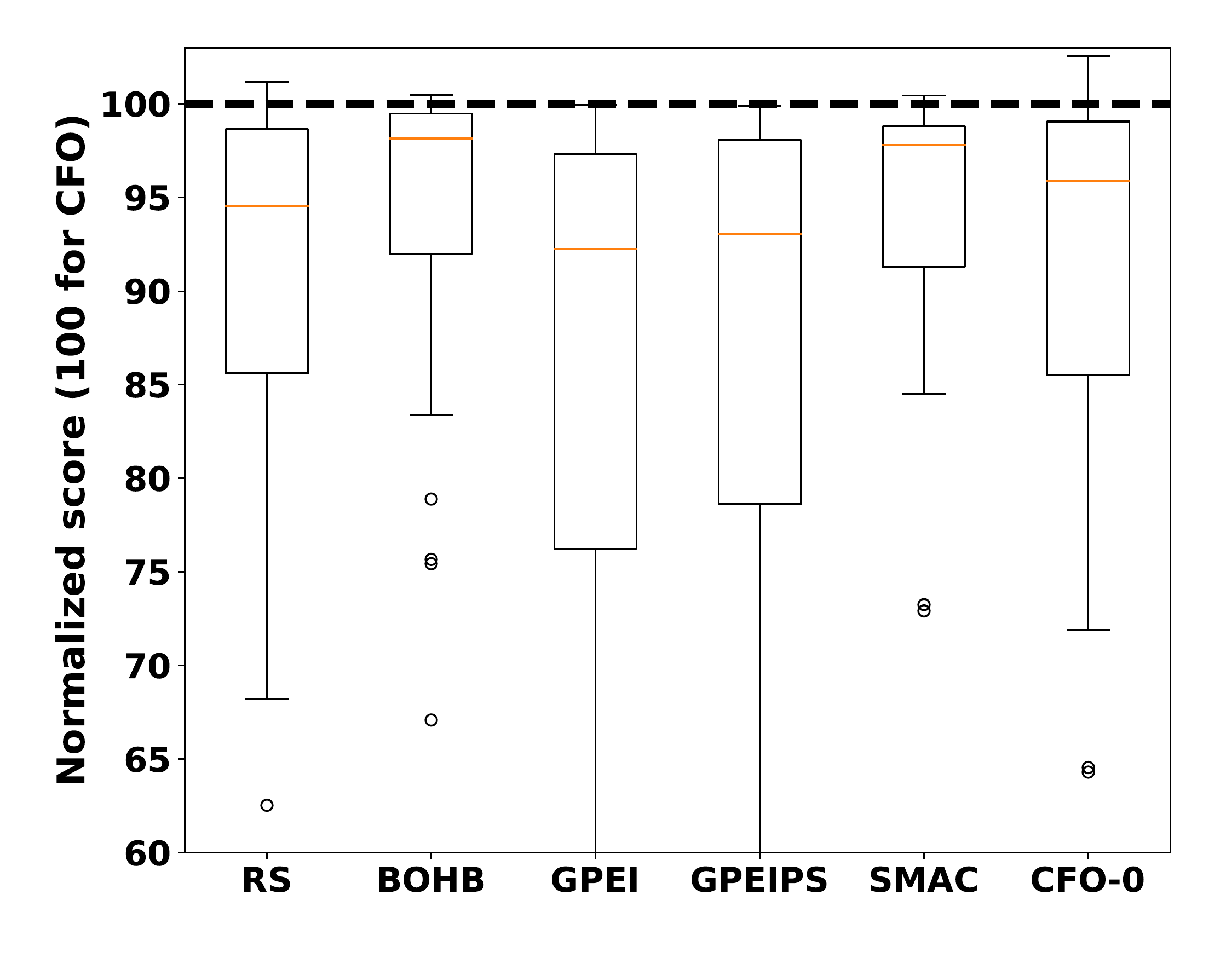}%
\caption{Scores in tuning DNN}  \label{fig:dnn_box}
\end{subfigure} 
  } \hfill
  \parbox{.69\figrasterwd}{
\begin{subfigure}{0.27\paperwidth} 
\includegraphics[width=\columnwidth]{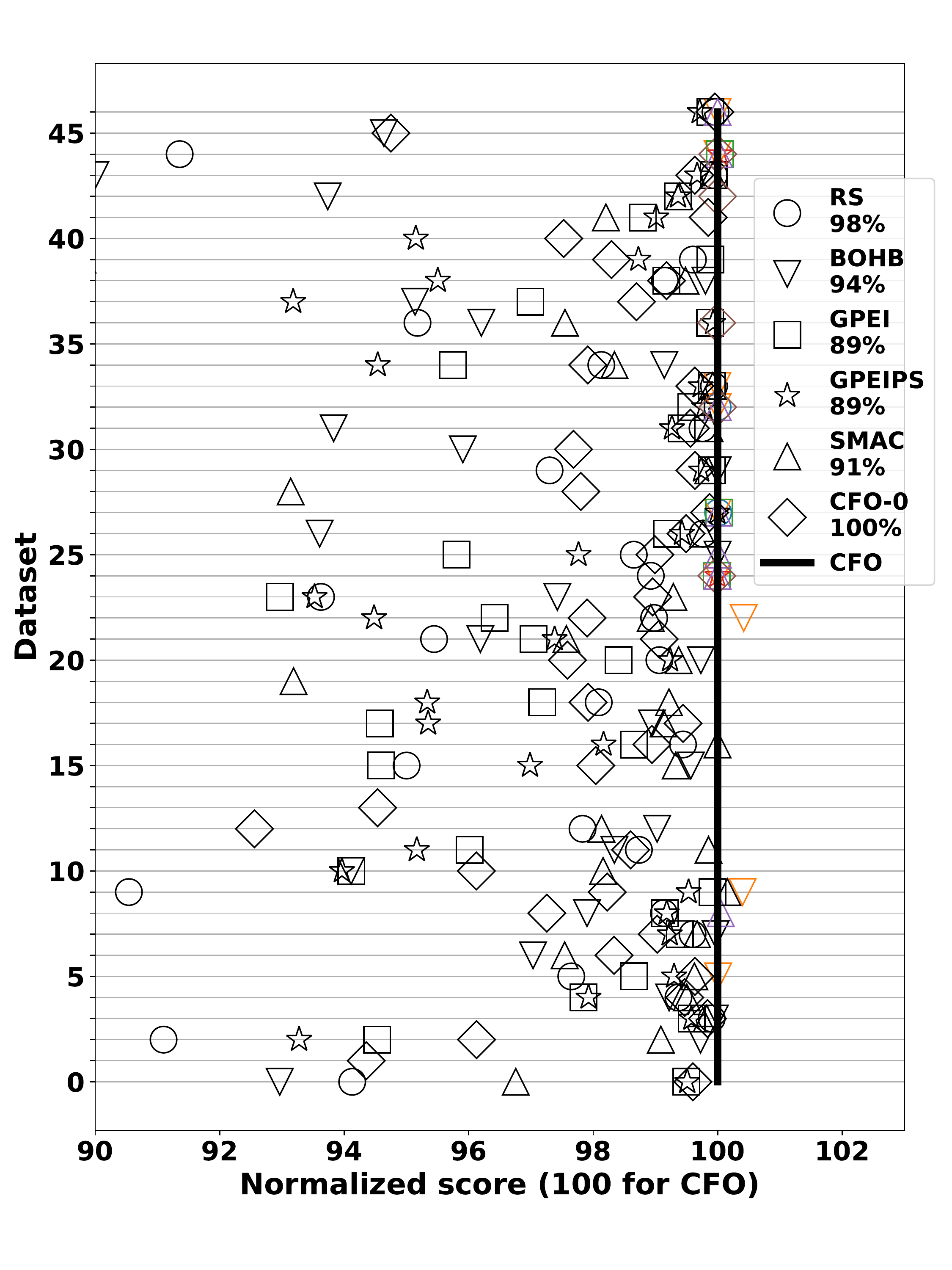}%
\caption{Scores in tuning XGBoost, the higher the better. The legends display the fraction of datasets on which the scores are over 90 for each method} \label{fig:xgboost_score}
\end{subfigure} \hfill
\begin{subfigure}{0.27\paperwidth} 
\includegraphics[width=\columnwidth]{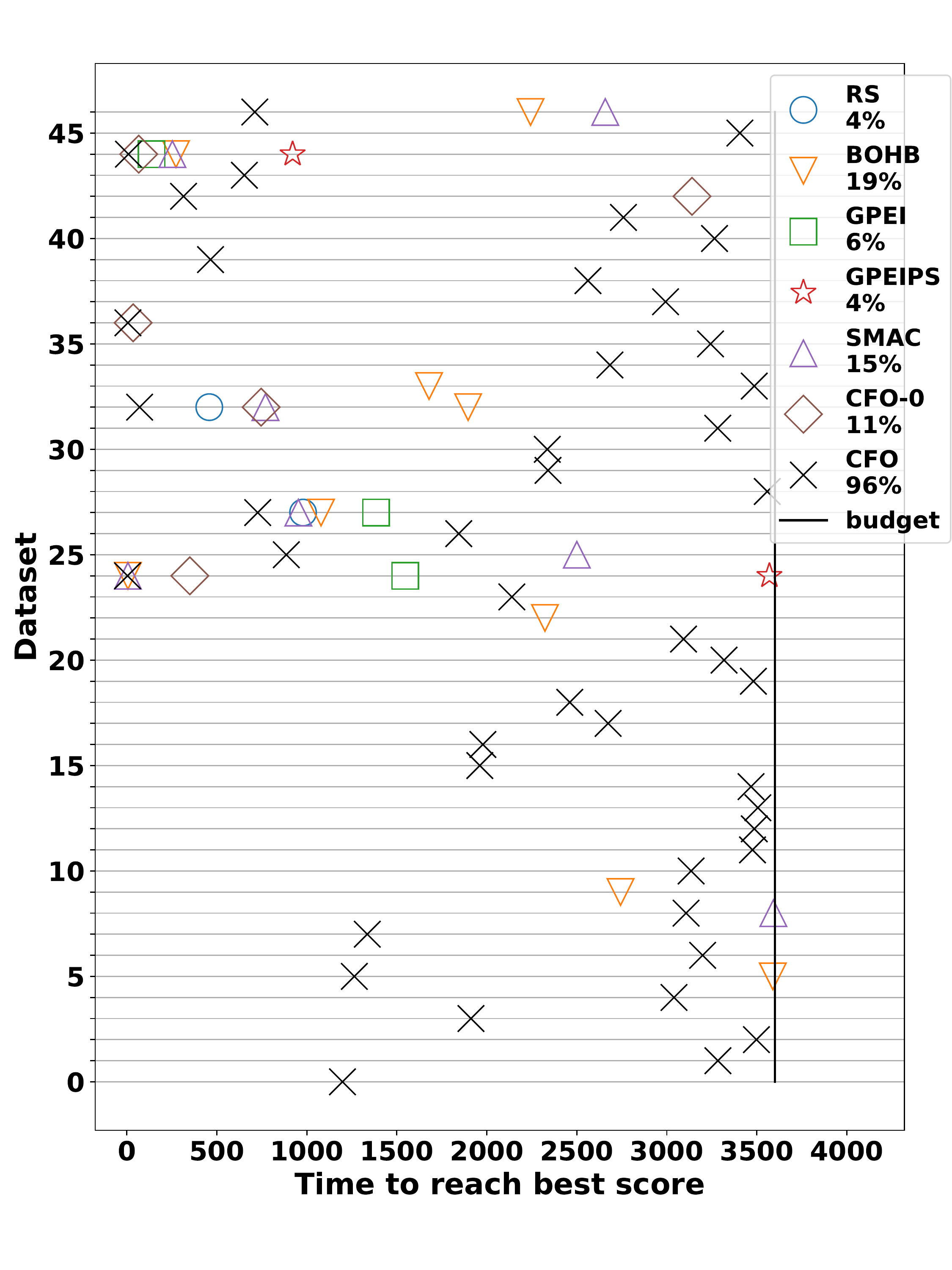}%
\caption{Time used in reaching best score for XGBoost. The legends display the fraction of datasets on which the best score can be reached within 1h} \label{fig:xgboost_time}
\end{subfigure}
}
\caption{
Scores in (a)-(c) are normalized loss (100 for \HPOname, higher score for lower loss).
The colored markers in (c) and (d) correspond to the cases where a baseline reaches score 99.95 within the time budget. The black markers in (c) do not have corresponding points in (d)}
\label{fig:xgboost_score_time}
\end{figure*}

\textbf{Performance curve.}
To investigate the effectiveness of \HPOname's cost control, we visualize the performance curve in terms of validation loss of all the methods over an one-hour wall clock time period. We include the performance curves in tuning XGBoost on 6 datasets in Figure~\ref{fig:lc_xgboost} and put the rest in the appendix. These 6 datasets represent a diverse group: The two rows in Figure~\ref{fig:lc_xgboost} include three small datasets and large datasets respectively. In each row, the three datasets are for binary classification, multi-class classification, and regression task respectively. The curves show that overall it takes RS and classical BO-based methods much longer time to reach low loss, because they are prone to trying unnecessarily expensive configurations. 
In our experiments, GPEIPS outperforms GPEI in some cases (for example, \emph{adult}) but tends to underperform GPEI on small datasets (for example, \emph{car}) probably due to its penalization on good but expensive configurations even when budget is adequate. 
\HPOnameO\ shares a similar spirit with \HPOnameO\ because ZOGD can also be considered as a randomized direct search method. However, due to the unique designs in \OPTname, \HPOname\ still maintains its leading performance.  
\HPOname\ demonstrates strong anytime performance, showing its effectiveness in controlling the evaluation cost incurred during the optimization process. It achieves up to three orders of magnitude speedup comparing to other methods for reaching any loss level.

\textbf{Overall optimality of loss and cost.} 
Figure~\ref{fig:xgboost_box} and \ref{fig:dnn_box} present boxplots of normalized scores (100 for \HPOname, the higher score for lower loss) obtained by all the baselines on all the datasets in tuning XGBoost and DNN within the required time budget. We can observe that \HPOname has dominating performance in terms of loss on the large collection of datasets under the same time budget with the others. To investigate the methods' optimality on each dataset, we show the normalized scores obtained on all the datasets for each method in tuning XGBoost within the required time budget in Figure~\ref{fig:xgboost_score}. We can see that \HPOname\ is able to find the best loss on almost all the datasets with a large margin comparing to the baselines. RS or each BO method has 19\%-32\% datasets (diversely distributed) with more than 10\% gap in score compared to \HPOname.
There are only two cases where \HPOname's score is lower than a baseline, i.e., BOHB, by 0.4\%.  
Figure~\ref{fig:xgboost_time} shows the time for reaching 
the best loss of each dataset, which means reaching up to 0.05\% below 
the highest score by all the compared methods within the required time budget in tuning XGBoost. These results show that (1) overall \HPOname\ has the highest ratio (96\%) of reaching the best loss in contrast to the very low ratios in baselines (less than 19\%); (2) on a very large fraction of the datasets, \HPOname\ can reach the best loss within a small amount of time while the others cannot reach the same loss within time budget. (3) in the cases where there are other methods reaching the best loss, \HPOname\ almost always uses the least amount of time.



\section{Conclusion and Future Work}\label{sec:conclusion}


In this work, we take a novel path of addressing the HPO problem from the cost aspect, which is under-explored in existing literature but especially important. We consider our work one of the initial successful attempts to make HPO cost-frugal, having both dominating empirical performance and provable theoretical guarantee on the total cost.  Our analysis of cost control is the first of its kind. It is a good starting point of better understandings of the cost behaviors of HPO solutions, including what conditions are needed for the HPO methods to be cost-frugal. As future work, it is worth studying the theoretical guarantees of our method under weaker conditions of the cost function.  It is also worth studying how to effectively incorporate cost observations into the HPO algorithms to make it more frugal. We also plan to study scenarios where categorical hyperparameters exist and also test the parallelization of our method. 

\section*{Ethical Impact}
Our work can help reduce the cost of doing hyperparameter optimization. It can be used to build more efficient automated machine learning (AutoML) solutions, which can save the effort and time of data scientists, and allow non-experts to make use of machine learning models and techniques. In a broader sense, we consider our work as an important attempt to make machine learning more economically and environmentally friendly. The current trend of massive consumption on computation resources in training and tuning machine learning models brings a tremendous burden to the environment. In fact, a recent study~\cite{Strubell_Ganesh_McCallum_2020} quantified the approximate financial and environmental costs of training deep neural networks, which calls for methods that can reduce costs and improve equity in machine learning research and practice. Given this consideration, machine learning solutions should be designed to be cost-effective even if the computational budget is not a bottleneck for a specific task or a specific group of people. The cost-effective design of our method well aligns with this principle.

\section*{Acknowledgement}
\papertext{The authors would like to thank the anonymous reviewers for their insightful and constructive comments.}  The authors would like to thank John Langford, Hongning Wang and Markus Weimer for their advice.

\medskip

{
\small

\bibliography{cfo_ref}
}
\newpage
\renewcommand\thesection{\Alph{section}}
\setcounter{section}{0}

\onecolumn
\title{Supplementary Material for: Cost Effective Optimization for Cost-related Hyperparameters}

\vskip 0.3in




\aaai{
\small
\begin{table*}[t]
\caption{Important notions and notations used in the main paper}
\begin{center}
\begin{tabular}{ |c|c||c|c| } 
 \hline
$\cX$ & search space of $d$ hyperparameters & $\bx$ & instantiated $d$-dimensional hyperparameters, $\bx\in \cX$ \\
\hline
$f(\bx)$ & validation loss of $\bx$ & $g(\bx)$ & evaluation cost of $\bx$ \\
\hline
$\bx^*$ & $\bx$ with the smallest $f(\bx)$ & $\pi$ & HPO algorithm  \\
\hline
$\tilde \bx^*$ & locally optimal point of $f$ & $\bx_0$  & initial hyperparameter configuration  \\
\hline
$k^*_{\pi}$ & number of iterations used by $\pi$ to find $\bx^*$ & $G(\pi)$ & total cost to identify $\bx^*$ by Algorithm $\pi$  \\
\hline 
$\bbS$ & $(d-1)$-unit sphere $\bx^*$  & $f'_{\bu}(\bx)$ & directional derivative of $f(\bs)$ on direction $\bu$ \\
\hline
$\mu$ & step size in Algorithm~\ref{alg:flow2}  & $\epsilon$ & $\epsilon$-convergence from first order stationary point $\tilde \bx^*$  \\
\hline
$K^*$ & expected iteration \# by \OPTname until $\epsilon$-convergence & $\tilde G(\OPTname)$ & expected total evaluation cost for $\epsilon$-convergence \\
\hline
\end{tabular}
\end{center}
\end{table*}
}

\section{More Details about Convergence Analysis}
 \label{sec:appendix_proof_convergence}
\subsection{Facts and definitions}
\begin{fact}[Directional derivative]
Here we list several facts about the directional derivative $f'_{\bu}(\bx)$ that will be used in our following proof.

Definition of $f'_{\bu}(\bx)$: $f'_{\bu}(\bx) = \lim_{\mu \to 0} \frac{f(\bx+\mu\bu) - f(\bx)}{\mu}$

Properties: $f'_{\bu}(\bx) = \nabla f(\bx)^{\mt} \bu$,  $f'_{-\bu}(\bx) = -f'_{\bu}(\bx)$
\end{fact}

\begin{definition} \label{def:unit_sphere_subspace}
Given $\bx \in \cX$, we define,
{\small{
\begin{align*}
 \tilde \bbS_{\bx}^{-} := \{ \bu \in \bbS|  f'_{\bu}(\bx) \leq - \frac{L \mu}{2} \},~~~\tilde \bbS_{\bx}^{+} := \{ \bu \in \bbS| f'_{\bu}(\bx) \geq \frac{L \mu}{2} \},~~~\tilde \bbS_{\bx}^{\#} := \{ \bu \in \bbS| |f_{\bu}'(\bx)| \leq \frac{L\mu}{2} \} 
\end{align*}
}
\small{
\begin{align*}
    \bbS_{\bx}^{-} := \{ \bu \in \bbS|  f(\bx + \mu \bu) -f(\bx) < 0 \},~~\bbS_{\bx}^{+-} := \{ \bu \in \bbS|  f(\bx + \mu \bu) -f(\bx) > 0 , f(\bx - \mu \bu) -f(\bx) < 0 \} 
\end{align*}
}
}
\end{definition}

\subsection{Technical lemmas}
\begin{lemma} \label{lemma:symmetric_dir_derivative}
  $\tilde \bbS_{\bx}^+$ and $\tilde \bbS_{\bx}^{-}$ are symmetric, i.e. $\tilde \bbS_{\bx}^+ = \{ -\bu| \bu \in \tilde \bbS_{\bx}^{-} \}$ and $\tilde \bbS_{\bx}^- = \{ -\bu| \bu \in \tilde \bbS_{\bx}^{+} \}$.
\end{lemma}
\begin{proof} [Proof of Lemma~\ref{lemma:symmetric_dir_derivative}]
According to the definition of directional derivative, we have $f'_{-\bu}(\bx) = -f'_{\bu}(\bx)$. Then $\forall \bu \in \tilde \bbS_{\bx}^{-}$, we have $f'_{\bu}(\bx) = a_{\bu} \leq -\frac{L\mu}{2}$, then we have $f'_{-\bu}(\bx) = -f'_{\bu}(\bx) = -a_{\bu} \geq \frac{L \mu}{2}$, i.e. $-\bu \in \tilde \bbS_{\bx}^{+}$. Similarly, we can prove when $\forall \bu \in \tilde \bbS_{\bx}^{+}$, $-\bu \in \tilde \bbS_{\bx}^{-}$. Then the conclusion in this lemma can be obtained. 

Intuitively it means that if we walk (in the domain) in one direction to make $f(\bx)$ go up, then walking in the opposite direction should make it go down. 
\end{proof}

\begin{lemma} \label{lemma:funcvalue_diff_directional_derivative}
Under Condition~\ref{assumption_smooth}, we have,

(1) $|f'_{\bu}(\bx)| > \frac{L \mu}{2}\Rightarrow \sign(f(\bx + \mu \bu) - f(\bx)) = \sign(f'_{\bu}(\bx))$.

(2) $\tilde \bbS_{\bx}^{-} \subseteq \bbS_{\bx}^{-}$. ~~~~~~~(3) $\tilde\bbS_{\bx}^{+} \subseteq   \bbS_{\bx}^{+ -}$.
\end{lemma} 
\begin{proof} [Proof of Lemma~\ref{lemma:funcvalue_diff_directional_derivative}]

\textbf{Proof of (1):}
According to the smoothness assumption specified in Condition~\ref{assumption_smooth}:
\begin{align} 
    |f(\by) - f(\bx) - \nabla f(\bx)^\mt (\by - \bx) )| \leq \frac{L}{2} \lVert \by - \bx \rVert^2 \nonumber
\end{align}
By letting $\by = \bx + \mu \bu$, we have $|f(\bx + \mu \bu) - f(\bx) - \mu  \nabla f(\bx)^\mt \bu )| \leq \frac{L\mu^2}{2}$, which is $ \left |\frac{f(\bx + \mu \bu) - f(\bx)}{\mu}  - f'_{\bu}(\bx)  )\right| \leq \frac{L\mu}{2}$, i.e.,

\begin{align*}
   f'_{\bu}(\bx) - \frac{L\mu}{2} \leq     \frac{f(\bx + \mu \bu) - f(\bx)}{\mu}   \leq  f'_{\bu}(\bx) + \frac{L\mu}{2}
\end{align*}

So $f'_{\bu}(\bx) > \frac{L\mu}{2}\Rightarrow \frac{f(\bx + \mu \bu) - f(\bx)}{\mu}    > 0$, and $f'_{\bu}(\bx) < - \frac{L\mu}{2}\Rightarrow \frac{f(\bx + \mu \bu) - f(\bx)}{\mu}  < 0$. Combinning them, we have when $|f'_{\bu}(\bx)| > \frac{L\mu}{2}$,  $  \sign(f'_{\bu}(\bx)) = \sign(\frac{f(\bx + \mu \bu) - f(\bx)}{\mu} )  = \sign(f(\bx + \mu \bu) - f(\bx) ) $, which proves conclusion (1).

\textbf{Proof of (2):} When $\bu \in \tilde \bbS_{\bx}^{-}$,  $f'_{\bu}(\bx) < -\frac{L\mu}{2}$, according to conclusion (1), $\sign(f(\bx + \mu \bu) - f(\bx)) = \sign(f'_{\bu}(\bx)) < 0$, i.e. $\bu \in \bbS_{\bx}^{-}$. So we have $\tilde \bbS_{\bx}^{-} \subseteq  \bbS_{\bx}^{-}$.

\textbf{Proof of (3):} Similarly, when $\bu \in \tilde \bbS_{\bx}^{+}$, i.e., $f'_{\bu}(\bx) > \frac{L \mu}{2}$, according to conclusion in (1), we have $ \sign(f(\bx + \mu \bu) - f(\bx)) =   \sign(f'_{\bu}(\bx)) >0$ and $ - \sign(f(\bx - \mu \bu) - f(\bx)) =  - \sign(f'_{-\bu}(\bx)) = \sign(f'_{\bu}(\bx)) >0$, which means that $\bu \in \bbS_{\bx}^{+-}$. So we have $\tilde \bbS_{\bx}^{+} \subseteq \bbS_{\bx}^{+-}$. 
\end{proof}

\begin{lemma} \label{lemma:dir_derivative_exp}
For any $\bx \in \mathcal{X}$, 
\begin{align} 
  \bbE_{\bu \in \bbS}[|f_{\bu}'(\bx)|] 
 =\frac{2 \Gamma(\frac d 2)}{(d-1)\Gamma(\frac{d-1}{2})\sqrt\pi} \lVert  \nabla f(\bx) \rVert_2
\end{align}
\end{lemma}
\begin{proof} [Proof of Lemma~\ref{lemma:dir_derivative_exp}]
According to the definition of directional derivative
\begin{align} \label{eq:exp_abs_dir_deriv}
  \bbE_{\bu \in \bbS}[|f_{\bu}'(\bx_{k})|] = \bbE_{\bu \in \bbS} [|\nabla f(\bx_k)^\mt \bu|] 
 & = \lVert  \nabla f(\bx_k) \rVert_2 \bbE_{\bu \in \bbS} [| \cos(\theta_{\bu}) |]  \\ \nonumber
 & = \lVert  \nabla f(\bx_k) \rVert_2 \bbE_{\bu \in \bbS} [|u_1|]  \\ \nonumber
 & =\frac{2 \Gamma(\frac d 2)}{(d-1)\Gamma(\frac{d-1}{2})\sqrt\pi} \lVert  \nabla f(\bx_k) \rVert_2
\end{align}

where $\theta_{\bu}$ is the angle between $\bu$ and $\nabla f(\bx_k)$, and $\Gamma(.)$ is the gamma function. The last equality can be derived by calculating the expected absolute value of a coordinate $|u_1|$ of a random unit vector\footnote{[stackexchange305888]: S. E. Average absolute value of a coordinate of a random unit vector? Cross Validated. \url{https://stats.stackexchange.com/q/305888} (version:2018-09-06).}.
\end{proof}

\subsection{Full proof of Proposition~\ref{proposition:flow_gradient_prop}, Theorem \ref{theorem:flow_convergence} and Theorem~\ref{theorem:flow_convergence_convex}}

\begin{proof}[Proof of Proposition~\ref{proposition:flow_gradient_prop}]
To simplify notations, we denote $\bz_{k+1} = \nabla f(\bx_{k})^\mt (\bx_{k+1} - \bx_{k}) + \frac{L}{2} \lVert \bx_{k+1} - \bx_{k} \rVert^2$. Denote the volume of a particular surface area as $\text{Vol}(\cdot)$. 

According to Condition~\ref{assumption_smooth}, we have, $f(\bx_{k+1}) - f(\bx_k) \leq  \bz_{k+1}$. By taking expectation on both sides of this inequality and using the fact that $f(\bx_{k+1}) - f(\bx_k)$ is always non-positive according to our update rule, we have,
\begin{align} \label{eq:flow_evol_diff_1}
 & \bbE_{\bu_k \in \bbS} [f(\bx_{k+1}) - f(\bx_k)]    
    \leq \frac{\Vol(\tilde \bbS_{\bx_k}^{-})}{\Vol(\bbS)} \bbE_{\bu \in \tilde \bbS_{\bx_k}^{-}} [\bz_{k+1}] +  \frac{\Vol(\tilde \bbS_{\bx_k}^{+})}{\Vol(\bbS)} \bbE_{\bu \in \tilde \bbS_{\bx_k}^{+}} [\bz_{k+1}]  
\end{align}
The update rules in line 5 and line 8 of Alg~\ref{alg:flow2} can be equivalently written into the following equations respectively,
\begin{align}
   & \bx_{k+1} = \bx_{k} - \delta \sign(f(\bx_k + \delta \bu_k) - f(\bx_k)) \bu_k \\
    & \bx_{k+1} = \bx_{k} - \delta \sign(f(\bx_k - \delta \bu_k) - f(\bx_k)) (-\bu_k)
\end{align}

According to conclusion (2) in Lemma~\ref{lemma:funcvalue_diff_directional_derivative},  $\tilde \bbS_{\bx_k}^{-} \subseteq \bbS_{\bx_k}^{-} $. So when $\bu_k \in \tilde \bbS_{\bx_{k}}^{-}$, line 5 of Alg~\ref{alg:flow2} will be triggered, and we have,
\begin{align} \label{eq:z_exp_neg}
 \bbE_{\bu \in \tilde \bbS_{\bx_k}^{-}} [\bz_{k+1}] 
  & = \bbE_{\bu \in \tilde \bbS_{\bx_k}^{-}} [ -\delta \sign(f(\bx_{k} + \delta \bu) - f(\bx_k)) \nabla f(\bx_{k})^\mt \bu]  + \frac{L\delta^2}{2} \\ \nonumber
  & = \bbE_{\bu \in \tilde \bbS_{\bx_k}^{-}} [ -\delta \sign(f(\bx_{k} + \delta \bu) - f(\bx_k)) f'_{\bu}(\bx_k)  ]  + \frac{L\delta^2}{2} \\ \nonumber
  & = \bbE_{\bu \in \tilde \bbS_{\bx_k}^{-}} [ -\delta \sign(f'_{\bu}(\bx_k)) f'_{\bu}(\bx_k)  ]  + \frac{L\delta^2}{2} \\ \nonumber
  & = -\delta \bbE_{\bu \in \tilde \bbS_{\bx_k}^{-}} [ | f'_{\bu}(\bx_k)|]  + \frac{L\delta^2}{2}
\end{align}

According to conclusion (3) in Lemma~\ref{lemma:funcvalue_diff_directional_derivative}, $\tilde \bbS_{\bx_{k}}^{+} \subseteq \bbS_{\bx_k}^{+-}$, so when $\bu \in \tilde \bbS_{\bx_{k}}^{+}$, line 8 of Alg~\ref{alg:flow2} will be triggered, and we have,
\begin{align} \label{eq:z_exp_pos}
 \bbE_{\bu \in \tilde \bbS_{\bx_k}^{+}} [\bz_{k+1}] 
 & = \bbE_{\bu \in \tilde \bbS_{\bx_k}^{+}} [\delta \sign(f(\bx_{k} - \delta \bu) - f(\bx_k)) \nabla f(\bx_{k})^\mt \bu ] + \frac{L\delta^2}{2} \\ \nonumber
 & = \bbE_{\bu \in \tilde \bbS_{\bx_k}^{+}} [ - \delta \sign(f(\bx_{k} - \delta \bu) - f(\bx_k))  f'_{-\bu}(\bx)  ]  + \frac{L\delta^2}{2} \\ \nonumber
  & = \bbE_{\bu \in \tilde \bbS_{\bx_k}^{+}} [ - \delta \sign( f'_{-\bu}(\bx)) f'_{-\bu}(\bx)  ]  + \frac{L\delta^2}{2} \\ \nonumber
  & = -\delta \bbE_{\bu \in \tilde \bbS_{\bx_k}^{-}} [ | f'_{\bu}(\bx)|]  + \frac{L\delta^2}{2} 
\end{align}
in which the last equality used the fact that $\tilde \bbS_{\bx_k}^{+}$ and $\tilde \bbS_{\bx_k}^{-}$ are symmetric according to Lemma~\ref{lemma:symmetric_dir_derivative}. 

By substituting Eq~\eqref{eq:z_exp_neg} and \eqref{eq:z_exp_pos} 
into Eq~\eqref{eq:flow_evol_diff_1}, we have,
\begin{align} \label{eq:flow_evol_diff_2}
    &  \bbE_{\bu_k \in \bbS} [f(\bx_{k+1}) - f(\bx_k)]  
    \leq  2\frac{\Vol(\tilde \bbS_{\bx_k}^{-})}{\Vol(\bbS)} (  - \delta \bbE_{\bu \in \tilde \bbS_{\bx_k}^{-}} [ | f_{\bu}'(\bx) |]   + \frac{L\delta^2}{2} ) \\ \nonumber
\end{align}

Based on the symmetric property of $\tilde \bbS_{\bx_k}^{+}$ and $\tilde \bbS_{\bx_k}^{-}$, we have,
\begin{align} \label{eq:flow_exp_S}
  \bbE_{\bu_k \in \bbS}[|f_{\bu_k}'(\bx_{k})|] 
  & =  2 \frac{\text{Vol}(\tilde \bbS_{\bx_k}^{-})}{\text{Vol}(\bbS)} \bbE_{\bu \in  \tilde \bbS_{\bx_k}^{-}} [|f_{\bu}'(\bx_{k})| ] +  \frac{\Vol(\tilde \bbS^{\#}_{\bx_k})}{\text{Vol}(\bbS)} \bbE_{\bu \in \tilde \bbS^{\#}_{\bx_k}} [|f_{\bu}'(\bx_{k})]   \\ \nonumber
& \leq 2 \frac{\text{Vol}( \tilde \bbS_{\bx_k}^{-})}{\text{Vol}(\bbS)} \bbE_{\bu \in  \tilde \bbS_{\bx_k}^{-}} [|f_{\bu}'(\bx_{k})| ] +  (1- 2 \frac{\text{Vol}( \tilde \bbS_{\bx_k}^{-})}{\text{Vol}(\bbS)}) \frac{L\mu}{2}  \\ \nonumber 
& =  2 \frac{\text{Vol}( \tilde \bbS_{\bx_k}^{-})}{\text{Vol}(\bbS)} (\bbE_{\bu \in  \tilde \bbS_{\bx_k}^{-}} [|f_{\bu}'(\bx_{k})| ] - \frac{L\mu}{2}  ) +  \frac{L\mu}{2}  \\ \nonumber 
\end{align}

Combining Eq~\eqref{eq:flow_evol_diff_2} and Eq~\eqref{eq:flow_exp_S}, we have,
\begin{align} \label{eq:flow_evol_diff_3}
     \bbE_{\bu_k \in \bbS} [f(\bx_{k+1})|\bx_k] - f(\bx_k) & \leq - \delta \bbE_{\bu_k \in \bbS}[|f_{\bu_k}'(\bx_{k})|] +    \frac{L\delta^2}{2}  = - \delta  c_d \lVert \nabla f(\bx_k) \rVert_2  + \frac{1}{2}L\delta^2
\end{align}

in which the last inequality is based on Lemma~\ref{lemma:dir_derivative_exp}.
\end{proof}

\begin{proof} [Proof of Theorem~\ref{theorem:flow_convergence}]
Denote by $\cU_k = (\bu_0, \bu_i, ..., 
\bu_k)$ a random vector composed by independent and identically distributed (i.i.d.) variables $\{\bu_k\}_{k \geq 0}$
attached to each iteration of the scheme up to iteration $k$. Let $\phi_0 = f(\bx_0)$ and $\phi_k := \bbE_{\cU_{k-1}}[f(\bx_k)], k \geq 1$ (i.e., taking expectation over randomness in the trajectory).

According to Proposition~\ref{proposition:flow_gradient_prop}, we have,
\begin{align} \label{eq:prop_flow_gradient_bound}
    & f(\bx_k) - \bbE_{\bu_k \in \bbS} [f(\bx_{k+1})|\bx_k]   \geq \mu  c_d \lVert \nabla f(\bx_k) \rVert_2 - \frac{L\mu^2}{2}
\end{align}

By taking expectation on $\cU_{k}$ for both sides of Eq~\eqref{eq:prop_flow_gradient_bound}, we have,
\begin{align}
  \delta c_d  \bbE_{\cU_k}[\lVert \nabla f(\bx_k) \rVert_2] \leq  \bbE_{\cU_k} [f(\bx_k)] - \bbE_{\cU_{k}}[ \bbE_{\bu_k \in \bbS} [f(\bx_{k+1})|\bx_k]] + \frac{1}{2} L\delta^2    = \phi_{k} - \phi_{k+1}  + L\delta^2
\end{align}

By taking telescoping sum we have,

\begin{align}
    \sum_{k=0}^{K-1} \delta c_d  \bbE_{\cU_k}[\lVert \nabla f(\bx_k) \rVert_2] & \leq \phi_0 - \phi_{K} +  \frac{1}{2} L \sum_{k=0}^{K-1}  \delta^2   \leq f(\bx_0) -  f(\bx^*) + \frac{1}{2} L \sum_{k=0}^{K-1}  \delta^2  
\end{align}

So,
\begin{align}
     \min_{k\in[K]} \bbE[\lVert \nabla f(\bx_k) \rVert_2 ] \leq  \frac{f(\bx_0) -  f(\bx^*)  + \frac{1}{2}L \sum_{k=0}^{K-1} \delta^2}{ c_d \sum_{k=0}^{K-1} \delta}
\end{align}

in which $c_d= \frac{2\Gamma(\frac d 2)}{(d-1)\Gamma(\frac{d-1}{2})\sqrt\pi}$, so $\frac{1}{c_d} = O(\sqrt{d})$. By letting $\delta = \frac{1}{\sqrt{K}}$, we have $\min_{k \in [K]} \bbE[ \lVert \nabla f(\bx_k) \rVert_2 ]= O(\frac{\sqrt{d}}{\sqrt{K}})$. 

\end{proof}

\begin{proof}[Proof of Theorem~\ref{theorem:flow_convergence_convex}]
According to Proposition~\ref{proposition:flow_gradient_prop}, under Condition~\ref{assumption_smooth},
\begin{align} \label{eq:flow_convex_evol_diff}
    & \delta  c_d \lVert \nabla f(\bx_k) \rVert_2   \leq  f(\bx_k) - \bbE_{\bu_k \in \bbS} [f(\bx_{k+1})|\bx_k]   + \frac{1}{2}L\delta^2
\end{align}
Under the convex condition, we have:
\begin{align} \label{eq:flow_convex}
 & f(\bx_k) -  f(\bx^*)  \leq   \nabla f(\bx_k) ( \bx_k - \bx^* )  \leq \lVert \nabla f(\bx_k) \rVert_2   \lVert \bx_k - \bx^* \rVert_2  \leq R  \lVert \nabla f(\bx_k) \rVert_2    
\end{align}

Combining Eq~\eqref{eq:flow_convex} and Eq~\eqref{eq:flow_convex_evol_diff}, we have:
\begin{align}
   & \frac{\delta c_d}{R} ( f(\bx_k) -  f(\bx^*) ) 
   \leq f(\bx_k) - \bbE_{\bu_k \in \bbS} [f(\bx_{k+1})|\bx_k]   + \frac{1}{2}L\delta^2 
\end{align}
By taking expectation over $\cU_{k}$ on both sides, we have:
\begin{align} \label{eq:loss_diff_two_step}
   \frac{\delta c_d}{R} ( \phi_k -  f(\bx^*) ) \leq  \phi_k -  \phi_{k+1}   + \frac{1}{2}L\delta^2 
\end{align}
By rearranging the above equation, we get:
\begin{align} \label{eq:convex_diff}
  \phi_{k+1} - f(\bx^*) \leq (1 - \frac{\delta c_d}{R}) ( \phi_k -  f(\bx^*) )  +  \frac{1}{2}L\delta^2 
\end{align}
To simplify the notations, we define $\alpha =  1-\frac{\delta c_d}{R} \in (0,1)$,  $r_{k} = \phi_{k} - f(\bx^*) $. Then we have,
\begin{align} \label{eq:convex_final}
   r_{K} \leq \alpha r_{K-1}  +  \frac{1}{2}L\delta^2  
    \leq \alpha^{K} r_{0}  +  \frac{1}{2}L\delta^2 \sum_{j=0}^{K-1} \alpha^i 
    \leq  e^{-\frac{\delta c_d K}{R}} r_0 +  \frac{L\delta^2}{2(1-\alpha)}  =  e^{-\frac{\delta c_d K}{R}} r_0 +  \frac{L\delta R}{2c_d}
\end{align}
in which the last inequality is based on the fact that $\ln(1+x) < x$ when $x \in (-1,1)$ and geometric series formula. 

\end{proof}

\subsection{Remarks about convergence conclusion}

\begin{remark}[Comparison with zeroth-order optimization]\label{remark:convergence_compare_zo}
The best convergence rate so far for zeroth-order optimization methods that only use function evaluations is $O(\frac{d}{K})$~\cite{nesterov2017random}. It has better dependency on $K$ but worse dependency on $d$. And that convergence rate requires step-size to be dependent on extra unknowns besides the total number of iterations, which makes it hard to achieve in practice.
\end{remark}

 \begin{remark}[Comparison with Bayesian Optimization] \label{remark:convergence_compare_bo}
 The convergence results in Theorem~\ref{theorem:flow_convergence} and Theorem~\ref{theorem:flow_convergence_convex} show that \OPTname\ can achieve a convergence rate of $O(\frac{\sqrt{d}}{\sqrt{K}})$ in both non-convex and convex case when $\mu \propto \frac{1}{\sqrt{K}}$.
 The best known convergence result for the commonly used Bayesian optimization methods in HPO problems is $O(K^{-v/d})$ \cite{bull2011convergence}. It requires a Gaussian prior with a smoothness parameter $v$. The comparison between our convergence result and their convergence result with respect to $K$ depends on the smoothness parameter $v$ of the Gaussian prior. When $v < 2d$, our result has a better dependency on $K$. And our convergence result has a better dependency on $d$. 
 \end{remark}

\section{More Details about Cost Analysis} \label{sec:appendix_cost}

\subsection{Cost analysis of \OPTname with factorized form of cost function} \label{sec:appendix_cost_factorized}
In this subsection, we consider a particular form of cost function and provide the cost analysis of \OPTname with such a cost function. Specifically, we consider the type of cost which satisfies Condition~\ref{assu:factorized_cost}.    

\begin{condition}\label{assu:factorized_cost}
 The cost function $g(\cdot)$ in terms of $\bx$ can be written into a factorized form
$g(\bx) = \prod_{i \in D'} e^{\pm x_i}$, where $D'$ is the cost-related subset of the coordinates.  
\end{condition}
We acknowledge that such an assumption on the function is not necessarily always true in all the cases. We will illustrate how to apply proper transformation functions on the original hyperparameter variables to realize Condition~\ref{assu:factorized_cost} later in this subsection.

With the factorized form specified in Condition~\ref{assu:factorized_cost}, the following fact is true. 

\begin{fact}[Invariance of cost ratio]\label{ass:cost_ratio} If Condition~\ref{assu:factorized_cost} is true,
$\frac{g(\bx + \Delta)}{g(\bx)} = c(\Delta)$, in which 
$c(\Delta) = e^{\sum_{i \in D'} \pm \Delta_i}$.
\end{fact}

Implication of Fact~\ref{ass:cost_ratio}: according to Fact~\ref{ass:cost_ratio}, $\log (g(\bx_1)) - \log (g(\bx_2)) = \sum_{i \in D'} ( (\pm x_{1,i}) - (\pm x_{2,i}))$. It means that if the cost function can be written into a factorized form, the Lipschitz continuous assumption in Condition~\ref{assu:cost_Lipsch} is true in the log space of the cost function. 

Using the same notation as that in Fact~\ref{ass:cost_ratio}, we define $C := \max_{\bu\in\bbS} c(\mu \bu)$.
\begin{condition} [Local monotonicity between cost and loss] \label{ass:monotonicity}
 $\forall$ $\bx_1,\bx_2 \in \cX$, if $ \max\{2D, (C^2-1) g(\tilde \bx^*)\} +  g(\tilde \bx^*) \geq g(\bx_1) > g(\bx_2) > g(\tilde \bx^*)$, then $f(\bx_1) \geq f(\bx_2)$.
\end{condition}
This assumption is similar in nature to Condition~\ref{assu:cost_loss_monotonicity}, but has a slightly stronger condition on the local monotonicity region.

\begin{proposition}[Bounded cost change in \OPTname] \label{prop:cost_diff}
Under Condition~\ref{assu:factorized_cost}, $g(\bx_{k+1}) \leq g(\bx_k)C $, $\forall k$. 
\end{proposition}

\begin{proposition}[Bounded evaluation cost for any function evaluation of \OPTname]\label{proposition:g(bc_k)}
Under Condition~\ref{assu:factorized_cost}, Condition~\ref{ass:monotonicity} and $g(\bx_0) < g(\tilde \bx^*)$, we have $g(\bx_k) \leq g(\tilde \bx^*)C $, $\forall k$. 
\end{proposition}
Proposition~\ref{prop:cost_diff} and \ref{proposition:g(bc_k)} are similar to Proposition~\ref{prop:cost_diff_add} and \ref{prop:cost_bound_add} respectively, but have a multiplicative form on the cost bound. 

\begin{theorem} [Expected total evaluation cost of \OPTname] \label{thm:FLOWcost_twophase}
Under Condition~\ref{prop:cost_diff} and  Condition~\ref{ass:monotonicity},
if $K^* \leq \ceil{\frac{\log \gamma}{\log C}}$, $T\leq  g(\tilde \bx^*)\cdot\frac{2(\gamma-1)C}{\gamma( C -1)}$; else,
$T \leq  g(\tilde \bx^*)\cdot 2C \big( K^*C  +   \frac{\gamma-1}{\gamma( C -1)}  - \frac{\log \gamma}{\log C} C + C \big)$, 
in which $\gamma=\frac{g(\tilde \bx^*)}{g(\bx_0)}>1$.
\end{theorem}

\paragraph{Proof idea}
The proof of Theorem~\ref{thm:FLOWcost_twophase} is similar to that of Theorem~\ref{thm:FLOWcost_twophase_add} and is provided in Appendix~\ref{sec:appendix_proof_cost}.

\begin{remark}[Intuitions and implications of Theorem~\ref{thm:FLOWcost_twophase}]
In Theorem~\ref{thm:FLOWcost_twophase}, the factorized form of the cost function ensures that the cost bound sequence begins as a geometric series (in contrast to the Arithmetic series in Theorem~\ref{thm:FLOWcost_twophase_add}) until it gets close to the optimal configuration's cost, such that the summation of this subsequence remains a constant factor times the optimal configuration's cost. It suggests that the expected total cost for obtaining a $\epsilon$-approximation of the loss is $O(1)$ times that minimal cost compared to $ O(d\epsilon^{-2})$ in Theorem~\ref{thm:FLOWcost_twophase_add} when $K^* \leq \min \{  \ceil{\frac{\log (g(\tilde \bx^*)/g(\bx_0))}{\log C}},  \ceil{\frac{(g(\tilde \bx^*) - g(\bx_0))}{D}} \}$. 
\end{remark}

\begin{remark}[Comparison between Theorem~\ref{thm:FLOWcost_twophase_add} and \ref{thm:FLOWcost_twophase}]
By comparing Theorem~\ref{thm:FLOWcost_twophase} to Theorem~\ref{thm:FLOWcost_twophase_add}, we can see that the factorized form does not affect the asymptotic form of the overall cost bound. It improves a constant term in the overall cost bound, and should be used if the relation of the cost function with respect to hyperparameter is known and can be written into the factorized form required in Condition~\ref{assu:factorized_cost}. We emphasize that even when such information is unknown and the factorized form is not available, the asymptotic cost bound of \OPTname still applies because of Theorem~\ref{thm:FLOWcost_twophase_add}.
\end{remark}

Our analysis above shows that a factorized form of cost function stated in Condition~\ref{assu:factorized_cost} brings extra benefits in controlling the expected total cost. Although Condition~\ref{assu:factorized_cost} not necessarily always holds for an arbitrary cost function, it can be easily realized as long as the computational complexity of the ML algorithm with respect to the hyperparameters is partially known (in $\Theta$ notation). Here we illustrate how can we use proper transformation functions to map the original problem space to a space where the factorized form in Condition~\ref{assu:factorized_cost} is satisfied in case such assumption is not satisfied in terms of the original hyperparameter variable.
\begin{remark}[Use transformation functions to realize Condition~\ref{assu:factorized_cost}]\label{remark:factorized_cost} \label{remark:trans_func} Here we introduce a new set of notations corresponding to the original hyperparameter variables that need to be transformed. Denote $\bc \in \cC$ as the original hyperparameter variable in the hyperparameter space $\cC$.  Denote $l(\bc)$ as  the validation loss using hyperparameter $\bc$ and denote $t(\bc)$ as the evaluation cost using $\bc$. Mapping back to our previous notations, $f(\bx) = l(\bc) $, and $ t(\bc) = g(\bx)$. Let's assume the dominant part of the cost $t$ can be written into: $t(\bc) = \prod_{i=1}^{d'} S_i(\bc)$, where $d'\le d$ and $S_i(\cdot): \mathbb{R}^d \rightarrow \mathbb{R^+} $, and $S(\bc) = (S_1(\bc), S_2(\bc), ..., S_d(\bc)): \mathbb{R}^d \rightarrow \mathbb{R^+}^d $ has a reverse mapping $\tS$, i.e., $\forall \bc \in \cC$,  $\tS( S(\bc)) = \bc$. 

Such transformation function pairs can be easily constructed if the computational complexity of the ML algorithm with respect to the hyperparameters is known. For example, if $\bc = (h_1, h_2, ..., h_5)$, and the complexity of the training algorithm is 
$t(\bc)=\Theta(h_1^2h_2^{-1/2}2^{h_3}\log h_4)$, 
where $h_1$ and $h_2$ are positive numbers, and $h_4 >1$,
we can let $S_1(\bc)=h_1^2, S_2(\bc)=h_2^{-1/2}, S_3(\bc)=2^{h_3}, S_4(\bc)=\log h_4, S_5(\bc)=h_5, d'=4$.
With the transformation function pair $S$ and $\tilde S$ introduced, we define $ \bx := \log S(\bc)$. Then we have $\bc = \tilde S(e^{\bx})$, $f(\bx) = l(\bc) = l(\tS(e^\bx))$ and $g(\bx)=t(\bc) = t(\tS(e^\bx)) = \prod_{i=1}^d e^{x_i}$, which realize Condition~\ref{assu:factorized_cost}. 

Let us use hyperparameter tuning for XGBoost as a more concrete example in the HPO context. In XGBoost, cost-related hyperparameter include $h_1 =\text{tree num}, h_2 = \text{leaf num}, h_3 = \text{min child weight}$. The correlation between evaluation cost and hyperparameter is $t(\bc) = \Theta(h_1 \times h_2 / h_3)$. Using a commonly used log transformation, $x_i  = \log h_i$ for $i \in [d]$ (i.e., $S_i(\bc) = h_i$ for $i \in [d]$), we have $g(\bx) = \Theta(e^{x_1}\times e^{x_2} \times e^{-x_3})$, which again satisfies Condition~\ref{assu:factorized_cost}. 

\end{remark}

\subsection{Proofs of propositions and theorems in cost analysis} \label{sec:appendix_proof_cost}

\begin{proof}[Proof of Proposition \ref{prop:cost_bound_add}]
We prove by induction. Because of the low-cost initialization, we have $g(\bx_0) \leq g(\tilde \bx^*)$. Since $D \geq 0$, we naturally have $g(\bx_0) \leq g(\tilde \bx^*)  + D$. Next, let's assume $g(\bx_k) \leq g(\tilde \bx^*)  + D$, we would like to prove $g(\bx_{k+1})  \leq g(\tilde \bx^*)  + D$ also holds. We consider the following two cases,

{\em (a)} $ g(\bx_k) \leq g(\tilde \bx^*)$;~~~~~{\em (b)} $g(\tilde \bx^*) < g(\bx_k)  \leq g(\tilde \bx^*) + D$

In case {\em (a)}, according to \OPTname, $\bx_{k+1} = \bx_{k} + \delta \bu_k$ or $\bx_{k+1} = \bx_{k} - \delta \bu_k$, or $\bx_{k+1} = \bx_{k} $. 
Then according to the Lipschitz continuity in Condition~\ref{assu:cost_Lipsch}, we have $g(\bx_{k+1}) \leq g(\bx_k)+  U \times  z(\delta \bu_k)$ or $g(\bx_{k+1}) \leq g(\bx_k)+  U \times  z(-\delta \bu_k)$, or $g(\bx_{k+1}) = g(\bx_k)$. So we have $g(\bx_{k+1}) \leq g(\bx_k)+  U \times \max_{\bu \in \bbS} z(\delta \bu) \leq  g(\tilde \bx^*) + U \times \max_{\bu \in \bbS} z(\delta \bu) =  g(\tilde \bx^*)  + D$.

In case {\em (b)}, where $g(\tilde \bx^*)
      \leq g(\bx_k)  \leq g(\tilde \bx^*) + D$, if $g(\bx_{k+1}) \leq g( \tilde  \bx^*)$, then $g(\bx_{k+1}) <  g(\tilde \bx^*) + D$ naturally holds; if $g(\bx_{k+1}) >   g(\tilde \bx^*)$, we can prove  $g(\bx_{k+1}) < g(\bx_k)$ by contradiction as follows.
      
      Assume $g(\bx_{k+1}) > g(\bx_k)$. Since $g(\bx_{k}) <  g(\tilde \bx^*) + D$, and $\bx_{k+1} = \bx_{k} + \delta \bu_k$ or $\bx_{k+1} = \bx_{k} - \delta \bu_k$, or $\bx_{k+1} = \bx_{k} $, we have $g(\bx_{k+1}) < g(\bx_k) + U \times \max_{\bu \in \bbS}  z(\delta \bu) \leq  g(\tilde \bx^*) + 2U \times \max_{\bu \in \bbS} z(\mu \bu) $. Thus we have $g(\tilde \bx^*) < g(\bx_{k}) < g(\bx_{k+1}) \leq g(\tilde \bx^*) + 2U \times \max_{\bu \in \bbS} z(\mu \bu)$. Then based on Condition~\ref{assu:cost_loss_monotonicity}, $f(\bx_{k+1}) > f(\bx_k)$. However, this contradicts to the execution of \OPTname, in which $f(\bx_{k+1}) \leq f(\bx_{k})$. Hence we prove $g(\bx_{k}) \leq g(\tilde \bx^*) + D$

\end{proof}

\begin{proof} [Proof of Theorem~\ref{thm:FLOWcost_twophase_add}]
According to Proposition~\ref{prop:cost_diff_add}, we have $g(\bx_{k+1}) \leq D + g(\bx_{k})$, with $D = U \times \max_{\bu \in \bbS} z(\mu \bu)$. When $k \leq \Bar{k}=\ceil{\frac{\gamma}{D}}-1$, we have $g(\bx_{k}) \leq  g(\bx_{0}) + Dk\leq g(\bx^*)$.

For  $K^* \leq \Bar{k}$,
\begin{align} \label{eq:costsum_add_1}
   \sum_{k=1}^{K^*} & g(\bx_k) \leq \sum_{k=1}^{K^*} (g(\bx_0) + Dk)  \leq \frac{K^* (g(\bx_0) + g(\tilde \bx^*))}{2}
\end{align}

For $K^*>\Bar{k}$,

\begin{align}\label{eq:costsum_add_2}
    \sum_{k=1}^{K^*}  g(\bx_k) = \sum_{k=1}^{\Bar{k}} g(\bx_k) + \sum_{k=\Bar{k}+1}^{K^*} g(\bx_k) 
    & \leq  \sum_{k=1}^{\Bar{k}} (g(\bx_0) + Dk) + \sum_{k=\Bar{k} + 1}^{K^*} (g(\tilde \bx^*) + D) \\ \nonumber
    & \leq \Bar{k} g(\bx_0) + \frac{1}{2}D \Bar{k} (1+\Bar{k})  + (K^* - \Bar{k})(g(\tilde \bx^*) + D)\\ \nonumber
    & \leq \Bar{k} ( g(\tilde \bx^*) - D  \frac{(\Bar{k} -1) }{2}) + (K^* - \Bar{k})(g(\tilde \bx^*) + D)\\ \nonumber
     & = K^* g(\tilde \bx^*) + D (K^* - \frac{\Bar{k} (\Bar{k}+1)}{2})  \\ \nonumber
     & \leq K^* g(\tilde \bx^*) + D K^* - \frac{ (\gamma/D -1)\gamma}{2}  \\
     \nonumber 
\end{align}

In each iteration $k$, \OPTname\ requires the evaluation of configuration $\bx_k + \mu \bu$ and potentially the evaluation of $\bx_k - \mu \bu$. Thus, $t_k \leq g(\bx_k + \mu \bu) + g(\bx_k - \mu \bu) \leq 2(g(\bx_k) + D)$. Then we have 
\begin{align}\label{eq:cost_all_add}
    \tilde G(\OPTname) & \leq 2 \sum_{k=1}^{K^*} (g({\bx_k}) + D) = 2 \sum_{k=1}^{K^*} g({\bx_k}) + 2K^* D \\ \nonumber
\end{align}
Substituting Eq~\eqref{eq:costsum_add_1}
 and Eq~\eqref{eq:costsum_add_2} into Eq~\eqref{eq:cost_all_add} respectively finishes the proof.
 \end{proof}

\begin{proof}[Proof of Proposition~\ref{proposition:g(bc_k)}]
We prove by induction similar to the proof of Proposition~\ref{prop:cost_bound_add}.  

According to the definition of $C$, we have $C > 1$.  Since the first point $x_0$ in \OPTname\ is initialized in the low-cost region, $g(\bx_0) \leq g(\tilde \bx^*)$. So $g(\bx_0) \leq g(\tilde \bx^*) C$ holds based on Fact~\ref{ass:cost_ratio}. Next, let us assume $g(\bx_k) \leq g(\tilde \bx^*) C$, and we would like to prove $g(\bx_{k+1}) \leq g(\tilde \bx^*) C$. We consider the following two cases:

{\em (a)} $g(\bx_k) \leq g(\tilde \bx^*)$;~~~~~{\em (b)} $g(\tilde \bx^*)< g(\bx_k) \leq g(\tilde \bx^*) C$

In case {\em (a)}, according to \OPTname,
$\bx_{k+1}= \bx_k + \mu \bu_k$ or $\bx_{k+1}= \bx_k - \mu \bu_k$ or $\bx_{k+1}= \bx_k$.  Accordingly, $g(\bx_{k+1}) = g(\bx_k) c(\mu \bu_k)$ or  $g(\bx_{k+1}) = g(\bx_k) c(-\mu \bu_k)$ or $g(\bx_{k+1}) = g(\bx_k)$ based on the fact Fact~\ref{ass:cost_ratio}, so we have $g(\bx_{k+1}) \leq g(\bx_k) \max_{\bu \in \bbS} c(\mu \bu) \leq g(\tilde \bx^*) \max_{\bu \in \bbS} c(\mu \bu)  = g(\tilde \bx^*) C$.


In case {\em (b)}, where $g( \tilde\bx^*) < g(\bx_k) \leq g(\tilde \bx^*) \max_{\bu\in\bbS} c(\mu \bu)$, if $g(\bx_{k+1})\leq g(\tilde \bx^*)$, then $g(\bx_{k+1}) \leq \max_{\bu\in \bbS} g(\tilde \bx^* + \mu \bu)$ follows; if $g(\bx_{k+1}) > g(\tilde \bx^*)$, we can prove $g(\bx_{k+1}) \leq g(\bx_{k})$ by contradiction as follows. 

Assume $g(\bx_{k+1}) > g(\bx_{k})$. Since $g(\bx_{k}) \leq \max_{\bu\in\bbS} g(\tilde \bx^* + \mu \bu)$, and $\bx_{k+1}= \bx_k + \mu \bu_k$ or $\bx_{k+1}= \bx_k - \mu \bu_k$ or $\bx_{k+1}= \bx_k$, we have $g(\bx_{k+1}) \leq g(\bx_k) \max_{\bu\in \bbS} c(\mu \bu) \leq g(\tilde \bx^*) (\max_{\bu\in\bbS} c(\mu \bu) )^2  =  g(\tilde \bx^*) C^2$.
Thus, we have $g(\tilde \bx^*) C^2 > g(\bx_{k + 1}) > g(\bx_{k}) > g(\tilde \bx^*)$, and then based on Condition~\ref{ass:monotonicity}, we have $f(\bx_{k+1}) > f(\bx_k)$. However, this contradicts to the execution of \OPTname, in which $f(\bx_{k+1}) \leq f(\bx_k)$. Hence, we prove $ g(\bx_{k+1}) < g(\bx_k) \leq g(\tilde \bx^*) C $.
\end{proof}



\begin{proof} [Proof of Theorem~\ref{thm:FLOWcost_twophase}]
According to Proposition~\ref{prop:cost_diff}, we have $g(\bx_{k+1}) \leq C g(\bx_{k})$, with $C = \max_{\bu \in \bbS} c(\mu\bu))$.

When $k \leq \Bar{k}=\ceil{\frac{\log \gamma}{\log C}}-1$, we have $g(\bx_{k}) \leq C^{k} g(\bx_{0}) \leq g(\bx^*)$.

For $K^*>\Bar{k}$,

\begin{align}\label{eq:costsum}
    \sum_{k=1}^{K^*} g(\bx_k) = \sum_{k=1}^{\Bar{k}} g(\bx_k) + \sum_{k=\Bar{k}+1}^{K^*} g(\bx_k) 
    & \leq  \sum_{k=1}^{\Bar{k}} g(\bx_0)C^k + \sum_{k=\Bar{k} + 1}^{K^*} g(\tilde \bx^*)C \\ \nonumber
    & \leq \frac{g(\tilde\bx^*) - g(\bx_0)}{C -1} + (K^* - \Bar{k})g(\tilde \bx^*)C \\ \nonumber
    & = g(\tilde \bx^*) \big( \frac{ \gamma -1}{\gamma( C -1)} + (K^* - \Bar{k})  C  \big) \\ \nonumber
     & = g(\tilde \bx^*) \left( K^*C+   \frac{\gamma-1}{\gamma( C -1)}  - \frac{\log \gamma}{\log C} C +C \right)  \\ \nonumber
\end{align}

Similar to the proof of Proposition~\ref{prop:cost_bound_add}, in each iteration $k$, $t_k \leq g(\bx_k + \mu \bu) + g(\bx_k - \mu \bu) \leq 2g(\bx_k)C$. Then we have 
\begin{align}\label{eq:cost_all}
   \tilde G(\OPTname) & \leq 2C \sum_{k=1}^{K^*} g({\bx_k}) \leq  g(\tilde \bx^*)\cdot 2C \left(  K^*C+   \frac{\gamma-1}{\gamma( C -1)}  - \frac{\log \gamma}{\log C} C +C \right)   \\ \nonumber
\end{align}

For $K^*\le \Bar{k}$, \(
    \sum_{k=1}^{K^*} g(\bx_k) \le \sum_{k=1}^{\Bar{k}} g(\bx_k) \),
    and the bound on $T$ is obvious from Eq~\eqref{eq:costsum} and \eqref{eq:cost_all}.

\end{proof}

\begin{remark}[Justifications for Condition~\ref{assu:cost_loss_monotonicity}]\label{remark:monotonicity}
Condition~\ref{assu:cost_loss_monotonicity} states that when the cost surpasses a locally optimal point $\tilde \bx^*$'s cost, i.e., $g(\bx) \geq g(\tilde \bx^*)$, 
with the increase of the evaluation cost in a local range, the loss does not decrease.  To give a concrete example, let us consider the number $K$ in $K$-nearest-neighbor. Condition~\ref{assu:cost_loss_monotonicity} essentially means that the further is $K$ above the optimal point, the larger is the validation loss, which is typically true in the surrounding space of the optimal point. We exhaustively evaluate $K$-nearest-neighbor models with every possible $K$ in a particular range. We visualize the relationship between loss, cost (i.e., training and validating time), and the cost-related hyperparameter $K$ (i.e., the number of neighbors) on two examples datasets in  Figure~\ref{fig:knn_cost_loss}.  From Figure~\ref{fig:knn_cost_loss_poker}, we can see that the optimal configuration $\tilde \bx^*$ is around 20, which correspond to a cost of 200 seconds, i.e. ($g(\tilde \bx^*) = 200$ seconds). Once the cost becomes larger than $g(\tilde \bx^*)$, the increase of cost (the blue curve) corresponds to a decrease of loss (reflected from the curve in red). This phenomenon is even more apparent on the mv dataset showed in Figure~\ref{fig:knn_cost_loss_mv}, where a larger range of $K$ is used. The results empirically verified the statement in Condition~\ref{assu:cost_loss_monotonicity}.

\end{remark}

\begin{remark}
To the best of our knowledge, the only theoretical analysis for HPO problems that considers cost appears in Hyperband~\cite{ICLR:li2017hyperband}. They derived a theoretical bound on the loss with respect to the input budget. 
However, as introduced in Section~\ref{sec:related_work}, their notion of `budget' is not suitable for modeling generic cost-related hyperparameters.  
\end{remark}

\begin{figure}[t]
\centering
\begin{subfigure}{0.45\columnwidth} 
\includegraphics[width=\columnwidth]{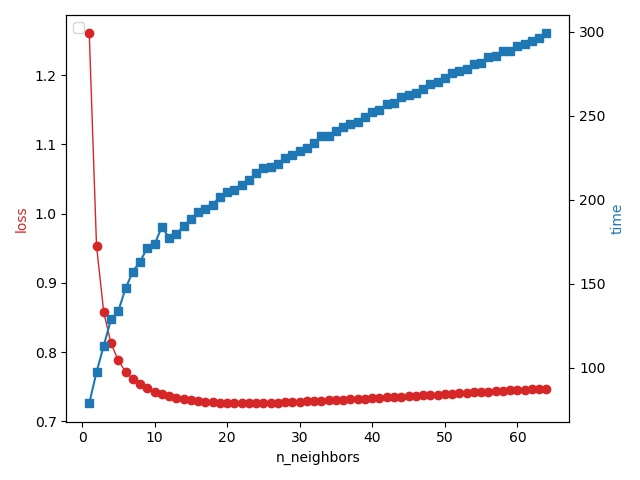}%
\caption{on poker dataset } \label{fig:knn_cost_loss_poker}
\end{subfigure}\hfill%
  \begin{subfigure}{0.45\columnwidth} 
\includegraphics[width=\columnwidth]{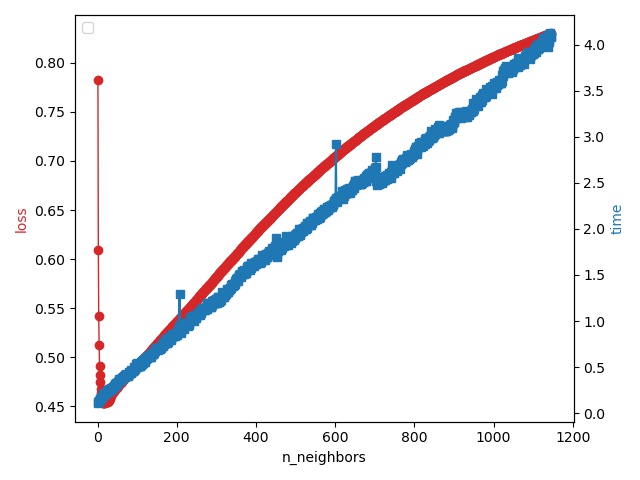}%
\caption{on mv dataset} \label{fig:knn_cost_loss_mv}
\end{subfigure}\hfill%
\caption{Relationship between cost, loss, and the number of neighbors in $K$-nearest-neighbor model. In both sub-figures, the x-axis shows the number of neighbors (i.e., $K$ in the $K$-nearest-neighbor model), the left-hand side of y-axis shows validation loss, which corresponds to the curves in red, and the left-hand side of y-axis shows the training and validating time (i.e., cost), which corresponds to the curves in blue.}
  \label{fig:knn_cost_loss}
\end{figure}

\section{More Details about \HPOname and Experiments} \label{sec:appendix_ceo_and_exp}
\subsection{More details about \HPOname} \label{subsec:appendix_ceo}

\begin{algorithm}[h]
\caption{Setting of $\delta_{\text{lower}}$}\label{alg:delta_lower}
\begin{algorithmic}[1]
\Statex \textbf{Input:} Initial stepsize $\delta_0$, search space $\mathcal{X} = \{ [x^{\min}_i, x^{\max}_i] \}_{i \in [d]}$
\State Define $\texttt{disc}$ as the dimension index set of the discrete hyperparameters in $\mathcal{X}$.
\If{$\texttt{disc} \neq \emptyset$ }
    \State $\delta_{\text{lower}} = \min\{ \frac{\log (1 + \Delta_i/x^{\text{best}}_i)}{\log (x^{\max}_i/x^{\min}_i)} \delta_0 \}_{i \in \texttt{disc}} $, in which $\Delta_i$ is the minimum possible difference on the $i$-dimension of hyperparameter configuration.
\Else
     \space $\delta_{\text{lower}} = 0.01$
\EndIf
\end{algorithmic}
\end{algorithm}

Complete pseudocode for our HPO solution \HPOname\ is presented in Algorithm~\ref{alg:hpo}. Below we explain several design choices made in \HPOname.

\begin{itemize}
    \item The setting of no improvement threshold is set to be $N=2^{d-1}$. This design is backed by the following rational:
    At point $\bx$, define $\mathbf{U} = \{\bu \in \bbS|\sign(\bu) =  \sign(\nabla f(\bx))\}$, if $\nabla f(\bx) \neq \mathbf{0}$, i.e., $\bx_k$ is not a first-order stationary point, $\forall \bu\in \mathbf{U}$, $f'_{\bu}(\bx) = \nabla f(\bx)^\mt \bu >0$, and $f'_{-\bu}(\bx) < 0$.   
    Even when L-smoothness is only satisfied in $\mathbf{U}\cup -\mathbf{U}$, we will observe a decrease in loss when $
    \bu\in \mathbf{U}\cup -\mathbf{U}$. The probability of sampling such $\bu$ is $\frac{2}{2^d}$. It means that we are expected to observe a decrease in loss after $2^{d-1}$ iterations even in this limited smoothness case.
    \item Initialization and lower bound of the stepsize: we consider the optimization iterations between two consecutive resets as one \textit{round} of \OPTname. At the beginning of each round $r$, $\delta$ is initialized to be a constant that is related to the round number: $\delta = r + \delta_0$ with $\delta_0 = \sqrt{d}$. Let $\delta_{\text{lower}}$ denotes the lower bound on the stepsize. Detailed setting of the lower bound is provided in Algorithm~\ref{alg:delta_lower}. It is computed according to the minimal possible difference on the change of each hyperparameter. For example, if a hyperparameter is an integer type, the minimal possible difference is 1. When no such requirement is posed, we set $\delta_{\text{lower}}$ to be a small value $0.01$.
    \item Projection function: as we mentioned in the main content of the paper, the newly proposed configuration $\bx \pm \bu$ is not necessarily in the feasible hyperparameter space. To deal with such scenarios, we use a projection function $\Proj_{\cX}(\cdot)$ to map the proposed new configuration to the feasible space. In \HPOname, the following projection function is used $\Proj_{\cX}(\bx') = \argmin_{\bx \in \cX} \lVert \bx - \bx' \rVert_1 $ (in line 6 of Algorithm~\ref{alg:hpo}), which is a straightforward way to handle discrete hyperparameters and bounded hyperparameters. After the projection operation, if the loss function is still non-differentiable with respect to $\bx$, smoothing techniques such as Gaussian smoothing~\cite{nesterov2017random}, can be used to make it differentiable. And ideally, our algorithm can operate on the smoothed version of the original loss function, denoted as $\tilde f(\Proj_{\cX}(\bx'))$. However, since the smoothing operation adds additional complexity to the algorithm and, in most cases, $\tilde f(\cdot)$ can be considered as a good enough approximation to $f(\cdot)$, the smoothing operation is omitted in our algorithm and implementation. 
\end{itemize}

\subsection{More details about experiments and results} \label{subsec:appendix_exp}



\begin{table}
\caption{Task ids on OpenML} \label{tab:openml}
\centering
\begin{tabular} {r| l|r |l} 
dataset  &  id & dataset & id \\ 
 \hline
 adult & 7592 & Airlines & 189354 \\
 Amazon\_employee & 34539 & APSFailure & 168868 \\
 Australian & 146818 & bank\_marketing & 14965 \\
 blood-transfusion & 10101 & car & 146821 \\
 christine & 168908 & cnae-9 & 9981 \\
 connect-4 & 146195 & credit-g & 31 \\
 fabert & 168852 & Fashion-MNIST & 146825 \\
 Helena & 168329 & higgs & 146606 \\
 Jannis & 168330 & jasmine & 168911 \\
 jungle\_chess\_2pcs & 167119 & kc1 & 3917 \\
 KDDCup09\_appe & 3945 & kr-vs-kp & 3 \\
 mfeat-factors & 12 & MiniBooNE & 168335 \\
 nomao & 9977 & numerai28.6 & 167120 \\
 phoneme & 9952 & riccardo & 178333 \\
 segment & 146822 & shuttle & 146212 \\
 sylvine & 168853 & vehicle & 53 \\
 volkert & 168810 & & \\
 bng\_echomonths* & 7323 & bng\_lowbwt* & 7320 \\
 bng\_breastTumor* & 7324 & bng\_pbc* & 7318 \\
 bng\_pharynx* & 7322 & bng\_pwLinear* & 7325 \\
 fried* & 4885 & houses* & 5165 \\
 house\_8L* & 2309 & house\_16H* & 4893 \\
 mv* & 4774 & pol* & 2292 \\
 poker* & 10102 & 2dplane* & 2306 
\end{tabular}
\end{table}

\subsubsection{Experiments on XGBoost}

The 9 hyperparameters tuned in XGBoost in our experiment are listed in Table~\ref{tab:xgboost_hyperparameters}. Logarithm transformation is performed for all hyperparameters in \HPOname, \HPOnameO, GPEI and GPEIPS, for all hyperparameters except \emph{colsample by level} and \emph{colsample by tree} in SMAC and BOHB. For BOHB, we consider the training sample size to be the resource, which is a common way to use BOHB. We set the minimum budget parameter to be 10,000 and max budget to be the full size of the training set. 

\textbf{Performance curve.} The performance curves for all datasets are displayed in Figure~\ref{fig:curve}. \HPOname show superior performance over all the compared baselines on the almost all the datasets. 

\textbf{Optimality of loss and cost.}  In addition to the overall optimality summary in Figure~\ref{fig:xgboost_score_time}, we report the final loss and the cost used to reach that loss per method per dataset in Table~\ref{tab:optimality_xgboost}. Here optimality is considered as the best performance that can be reached within the one-hour time budget by all the compared methods. Meanings of the numbers in Table~\ref{tab:optimality_xgboost} are explained in the table caption. On the one hand, we can compare all the methods' capability of finding best loss within the given time budget.  On the other hand, for those methods which can find the best loss within the time budget, we can compare the time cost taken to find it.  Specifically, we can see that \HPOname\ is able to find the best loss on almost all the datasets.
 For example, on \emph{Airlines}, \HPOname\ can find a configuration with a scaled score (high score corresponds to low loss) of 1.3996 in 3285 seconds, while existing RS and BO methods report more than 16\% lower score in the one-hour time budget and \HPOnameO\ is 5.6\% lower.
 For the cases where the optimality can be reached by both \HPOname\ and at least one baseline, \HPOname\ almost always takes the least amount of cost.
 For example, on \emph{shuttle}, 
 all the methods except GPEI and GPEIPS can reach a scaled score of 0.9999. It only takes \HPOname\ 72 seconds. Other methods show 6$\times$ to 26$\times$ slowdowns for reaching it.  Overall Table~\ref{tab:optimality_xgboost} shows that \HPOname\ has a dominating advantage over all the others in reaching the same or better loss within the least amount of time. 

\begin{table}
\caption{Hyperparameters tuned in XGBoost} \label{tab:xgboost_hyperparameters}
\centering
\begin{tabular} {r| l |l} 
hyperparameter  & type & range \\ 
 \hline
 tree num & int & [4,  min(32768, \# instance)]  \\ 
 leaf num & int & [4,  min(32768, \# instance)]  \\ 
 min child weight & float & [0.01, 20] \\
learning rate & float & [0.01, 0.1] \\
subsample & float & [0.6, 1.0]  \\
reg alpha & float & [1e-10, 1.0]   \\
reg lambda & float & [1e-10, 1.0]  \\
colsample by level & float & [0.6, 1.0]  \\
colsample by tree & float & [0.7, 1.0] \\
\end{tabular}
\end{table}

\begin{figure}[t]
\centering
\begin{subfigure}{0.45\columnwidth}
\includegraphics[width=\columnwidth]{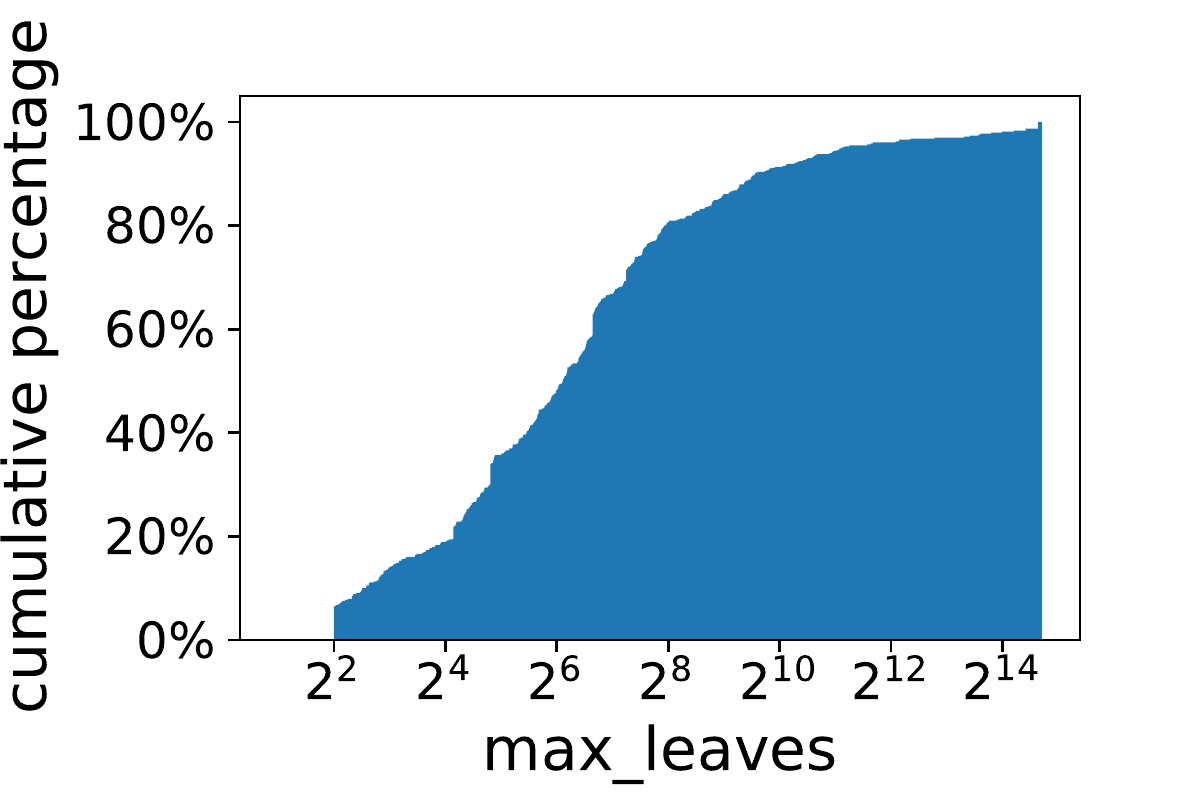}%
\caption{cumulative histogram of max\_leaves }%
\end{subfigure}\hfill%
  \begin{subfigure}{0.45\columnwidth}
\includegraphics[width=\columnwidth]{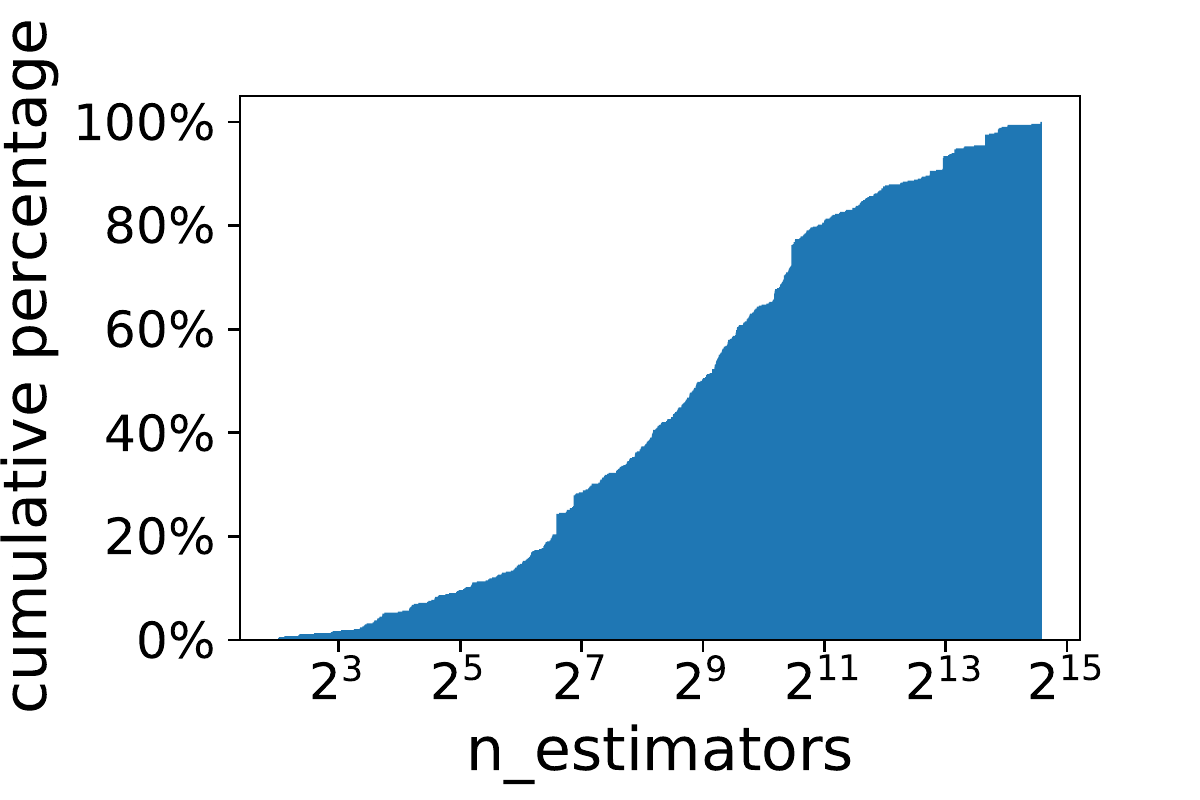}%
\caption{cumulative histogram of n\_estimators}%
\end{subfigure}\hfill%
\caption{Distribution of two cost-related hyperparameters in best configurations over all the datasets}
  \label{fig:dist_param}
\end{figure}

\begin{figure}
  \centering
  \begin{minipage}[t]{0.49\textwidth}
    \includegraphics[width=0.7\textwidth]{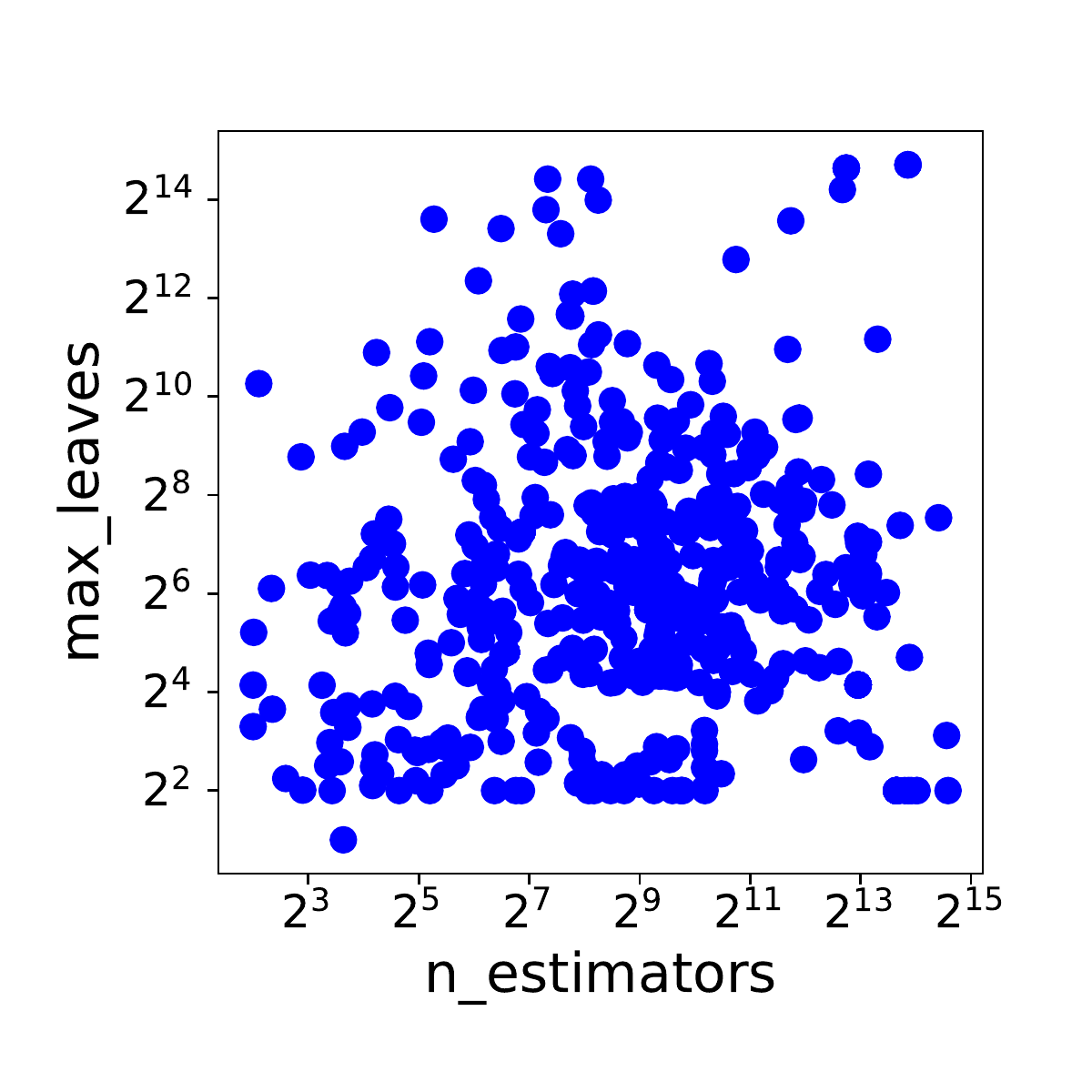}
    \caption{Scatter plot of two cost-related hyperparameters in best configurations over all the datasets}
    \label{fig:scatter}
  \end{minipage}
  \hfill
  \begin{minipage}[t]{0.49\textwidth}
    \includegraphics[width=\textwidth]{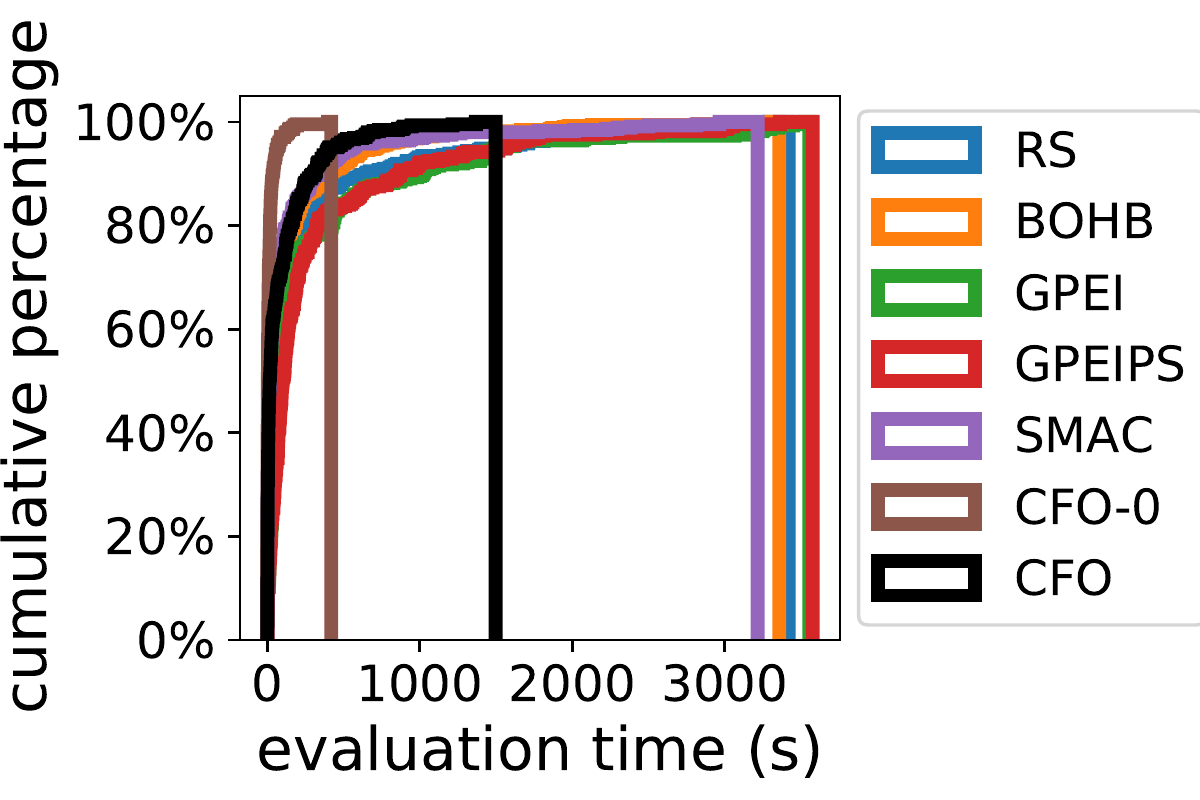}
    \caption{Distribution of evaluation time of the best XGBoost config found by each method}
    \label{fig:dist_time}
  \end{minipage}
\end{figure}


\begin{table*} 
\caption{Optimality of score and cost for XGBoost. We scale the original scores for classification tasks following the referred AutoML benchmark: 0 corresponds to a constant class prior predictor, and 1 corresponds to a tuned random forest. Bold numbers are best scaled score and lowest cost to find it among all the methods within given budget. For methods which cannot find best score after given budget, we show the score difference compared to the best score, e.g., -1\% means 1 percent lower score than best. For methods which can find best score with suboptimal cost, we show the cost difference compared to the best cost, e.g., 
        \textit{$\times$10} means 10 times higher cost than best  } \label{tab:optimality_xgboost}
\centering
\begin{tabular}{r | l l l l l  l  l }
\hline
Dataset & RS & BOHB & GPEI & GPEIPS & SMAC & \HPOnameO & \HPOname \\\hline
adult & -5.9\% & -7.1\% & -0.6\% & -0.5\% & -3.3\% & -0.4\% & \textbf{1.0524 (1199s)} \\
Airlines & -37.3\% & -33.3\% & -16.2\% & -16.2\% & -37.2\% & -5.6\% & \textbf{1.3996 (3285s)} \\
Amazon & -8.9\% & -0.3\% & -5.5\% & -6.7\% & -0.9\% & -3.9\% & \textbf{1.0008 (3499s)} \\
APSFailure & -0.1\% & -0.1\% & -0.5\% & -0.5\% & -0.2\% & -0.2\% & \textbf{1.0050 (1913s)} \\
Australian & -0.7\% & -0.8\% & -2.2\% & -2.1\% & -0.5\% & -0.6\% & \textbf{1.0695 (3041s)} \\
bank & -2.4\% & \textit{$\times$2.8} & -1.4\% & -0.7\% & -0.4\% & -0.4\% & \textbf{1.0091 (1265s)} \\
blood & -26.2\% & -3.0\% & -11.1\% & -10.6\% & -2.5\% & -1.7\% & \textbf{1.6565 (3200s)} \\
car & -0.5\% & -0.1\% & -0.7\% & -0.8\% & -0.4\% & -1.0\% & \textbf{1.0578 (1337s)} \\
christine & -0.9\% & -2.1\% & -0.9\% & -0.9\% & \textit{$\times$1.2} & -2.8\% & \textbf{1.0281 (3108s)} \\
cnae-9 & -9.9\% & \textbf{1.0803 (2744s) } & -0.5\% & -0.9\% & -0.3\% & -2.2\% & -0.4\% \\
connect-4 & -12.7\% & -5.9\% & -5.9\% & -6.0\% & -1.9\% & -3.9\% & \textbf{1.3918 (3136s)} \\
credit-g & -1.3\% & -1.7\% & -4.0\% & -4.8\% & -0.1\% & -1.4\% & \textbf{1.1924 (3477s)} \\
fabert & -2.2\% & -1.0\% & -17.5\% & -36.2\% & -1.9\% & -7.4\% & \textbf{1.0805 (3487s)} \\
Helena & -51.6\% & -54.0\% & -82.3\% & -82.3\% & -56.4\% & -31.7\% & \textbf{5.2707 (3469s)} \\
higgs & -5.0\% & -0.5\% & -5.4\% & -3.1\% & -0.7\% & -2.0\% & \textbf{1.0180 (1962s)} \\
Jannis & -36.1\% & -48.8\% & -73.0\% & -55.5\% & -6.8\% & -10.8\% & \textbf{1.2316 (3482s)} \\
jasmine & -0.9\% & -0.3\% & -1.6\% & -0.8\% & -0.6\% & -2.4\% & \textbf{1.0083 (3319s)} \\
jungle & -4.6\% & -3.8\% & -3.0\% & -2.6\% & -2.5\% & -1.0\% & \textbf{1.3209 (3094s)} \\
kc1 & -1.5\% & \textbf{0.9406 (2324s)} & -4.0\% & -5.9\% & -1.5\% & -2.5\% & -0.4\% \\
KDDCup09 & -6.4\% & -2.6\% & -7.0\% & -6.5\% & -0.7\% & -1.1\% & \textbf{1.1778 (2140s)} \\
kr-vs-kp & -1.1\% & \textit{$\times$1.4} & \textit{$\times$309.6} & \textit{$\times$714.2} & \textit{$\times$1.4} & \textit{$\times$70.2} & \textbf{1.0002 (5s)} \\
mfeat & -0.3\% & -6.4\% & -0.9\% & -0.6\% & -0.3\% & -0.6\% & \textbf{1.0676 (1845s)} \\
MiniBooNE & -1.4\% & -0.1\% & -4.2\% & -2.3\% & \textit{$\times$2.8} & -1.1\% & \textbf{1.0100 (887s)} \\
nomao & \textit{$\times$1.3} & \textit{$\times$1.5} & \textit{$\times$1.9} & -0,1\% & \textit{$\times$1.3} & -0.2\% & \textbf{1.0022 (728s)} \\
numerai28 & -13.8\% & -12.6\% & -11.4\% & -10.5\% & -6.9\% & -2.2\% & \textbf{1.5941 (3559s)} \\
phoneme & -2.7\% & -0.1\% & -0.1\% & -0.3\% & -0.2\% & -0.4\% & \textbf{0.9938 (2341s)} \\
segment & -0.3\% & -6.2\% & -0.6\% & -0.8\% & -0.2\% & -0.5\% & \textbf{1.0084 (3284s)} \\
shuttle & \textit{$\times$6.4} & \textit{$\times$26.3} & -0.5\% & -0.3\% & \textit{$\times$10.7} & \textit{$\times$10.4} & \textbf{0.9999 (72s)} \\
sylvine & -0.1\% & \textbf{1.0071 (1680s)} & -0.1\% & -0.3\% & -0.1\% & -0.4\% & \textit{$\times$2.1} \\
vehicle & -1.9\% & -0.9\% & -4.3\% & -5.5\% & -1.7\% & -2.1\% & \textbf{1.0892 (2685s)} \\
volkert & -30.8\% & -30.7\% & -91.1\% & -91.1\% & -43.4\% & -13.7\% & \textbf{1.2152 (3244s)} \\
\hline
2dplanes* & -4.9\% & -3.8\% & -0.2\% & -0.1\% & -2.5\% & \textit{$\times$5.1} & \textbf{0.9479 (7s)} \\
breast* & -32.8\% & -4.9\% & -3.0\% & -6.8\% & -13.1\% & -1.3\% & \textbf{0.1772 (2994s)} \\
echomonth* & -0.9\% & -0.2\% & -0.9\% & -4.5\% & -0.6\% & -0.9\% & \textbf{0.4739 (2563s)} \\
lowbwt* & -0.4\% & -10.8\% & -0.2\% & -1.3\% & -10.2\% & -1.8\% & \textbf{0.6175 (466s)} \\
pbc* & -32.0\% & -20.7\% & -10.5\% & -4.9\% & -23.5\% & -2.5\% & \textbf{0.4514 (3266s)} \\
pharynx* & -20.7\% & -12.2\% & -1.2\% & -1.0\% & -1.8\% & -0.2\% & \textbf{0.5140 (2760s)} \\
pwLinear* & -14.3\% & -6.3\% & -0.7\% & -0.7\% & -0.7\% & \textit{$\times$9.9} & \textbf{0.6222 (316s)} \\
fried* & -0.1\% & -10.0\% & -0.1\% & -0.4\% & -0.1\% & -0.4\% & \textbf{0.9569 (654s)} \\
houses* & -0.6\% & -10.5\% & -1.4\% & -1.9\% & -0.1\% & -1.1\% & \textbf{0.8526 (1979s)} \\
house\_8L* & -11.1\% & -1.1\% & -5.5\% & -4.7\% & -0.9\% & -0.6\% & \textbf{0.6913 (2675s)} \\
house\_16H* & -2.0\% & -12.0\% & -2.9\% & -4.7\% & -0.8\% & -2.1\% & \textbf{0.6702 (2462s)} \\
mv* & -8.7\% & \textit{$\times$27.4} & \textit{$\times$13.8} & \textit{$\times$92.2} & \textit{$\times$25.4} & \textit{$\times$6.7} & \textbf{0.9995 (10s)} \\
poker* & -15.3\% & -5.4\% & -100.0\% & -100.0\% & -10.9\% & -5.2\% & \textbf{0.9068 (3407s)} \\
pol* & -0.1\% & \textit{$\times$3.2} & -0.2\% & -0.3\% & \textit{$\times$3.7} & -0.1\% & \textbf{0.9897 (712s)} \\
\hline
\multicolumn{8}{l}{* for regression datasets} 
\end{tabular}
\end{table*}



The total evaluation cost of RS and BO-based methods can be largely affected by the setting of the cost-related hyperparamters' range. In this work, we used $\text{min}(32768, \# \text{instance})$ as the upper bound of tree number and leaf number. To show that this range is not unnecessarily large, we visualize the distribution of these two hyperparameters in best configurations over all the datasets through a scatter plot in Figure~\ref{fig:dist_param}. In this scatter plot, each point represents the best configurations of tree and leaf number found (among all the methods) on a particular dataset. This result shows that our settings of the upper bound of tree and leaf number are not unnecessarily large because we need to ensure that the configuration of the best models on all of the datasets can be covered by this range. 

We compare the distribution of evaluation time for the best configuration found by each method (which also includes the computational time of that method at the corresponding iteration) in Figure~\ref{fig:dist_time}.  Since that \HPOname\ has the leading final performance on almost the all the datasets, we consider the best configuration found by \HPOname\ as the configuration that has the appropriate complexity. From this figure we can observe that the best configurations found by BO methods tend to have unnecessarily large evaluation cost and \HPOnameO\ tends to return configurations with insufficient complexity.

\subsubsection{Experiments on DNN}

\begin{table}
\caption{Hyperparameters tuned in DNN (fully connected Relu network with 2 hidden layers)} \label{tab:dnn_hyperparameters}
\centering
\begin{tabular} {r| l |l} 
hyperparameter  & type & range \\ 
 \hline
neurons in first hidden layer & int &  [4,  2048]  \\ 
neurons in second hidden layer & int & [4, 2048]  \\ 
 batch size & int & [8, 128] \\
dropout rate & float & [0.2, 0.5] \\
learning rate & float & [1e-6, 1e-2]  \\
\end{tabular}
\end{table}

The 5 hyperparameters tuned in DNN in our experiment are listed in Table~\ref{tab:dnn_hyperparameters}. 
Table~\ref{tab:optimality_DNN} shows the final loss and the cost used to reach that loss per method within two hours budget. Since experiments for DNN are more time consuming, we removed five large datasets comparing to the experiment for XGBoost, and all the results reported are averaged over 5 folds (instead of 10 as in XGBoost).  

\textbf{Performance curve.} The performance curves for all datasets are displayed in Figure~\ref{fig:curve_dnn}. In these experiments, \HPOname\ again has superior performance over all the compared baselines similar to the experiments for XGBoost.

\textbf{Optimality of score and cost.} Similar to the experiments for XGBoost, we studied the optimality of score and cost in DNN tuning and reported the results in Figure~\ref{fig:dnn_score_time} and Table~\ref{tab:optimality_DNN}. Similar conclusions can be drawn about the optimality of cost and loss in tuning DNN to that in XGBoost.

\begin{figure}[h]
  \centering
  \begin{subfigure}{0.5\columnwidth} 
\includegraphics[width=\columnwidth]{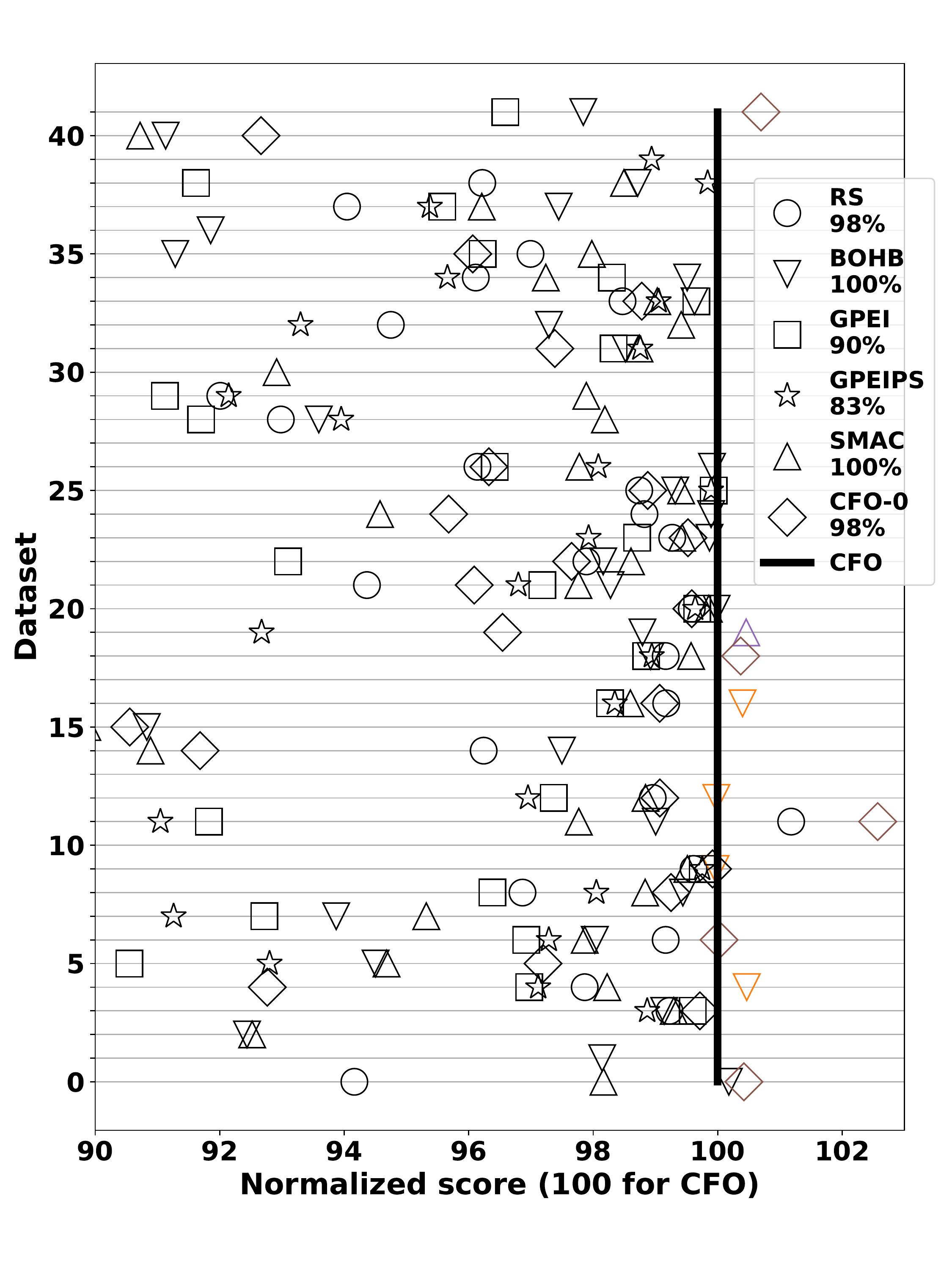}%
\caption{Scores in tuning DNN, the higher the better}  \label{fig:dnn_score}
\end{subfigure}\hfill%
\begin{subfigure}{0.5\columnwidth} 
\includegraphics[width=\columnwidth]{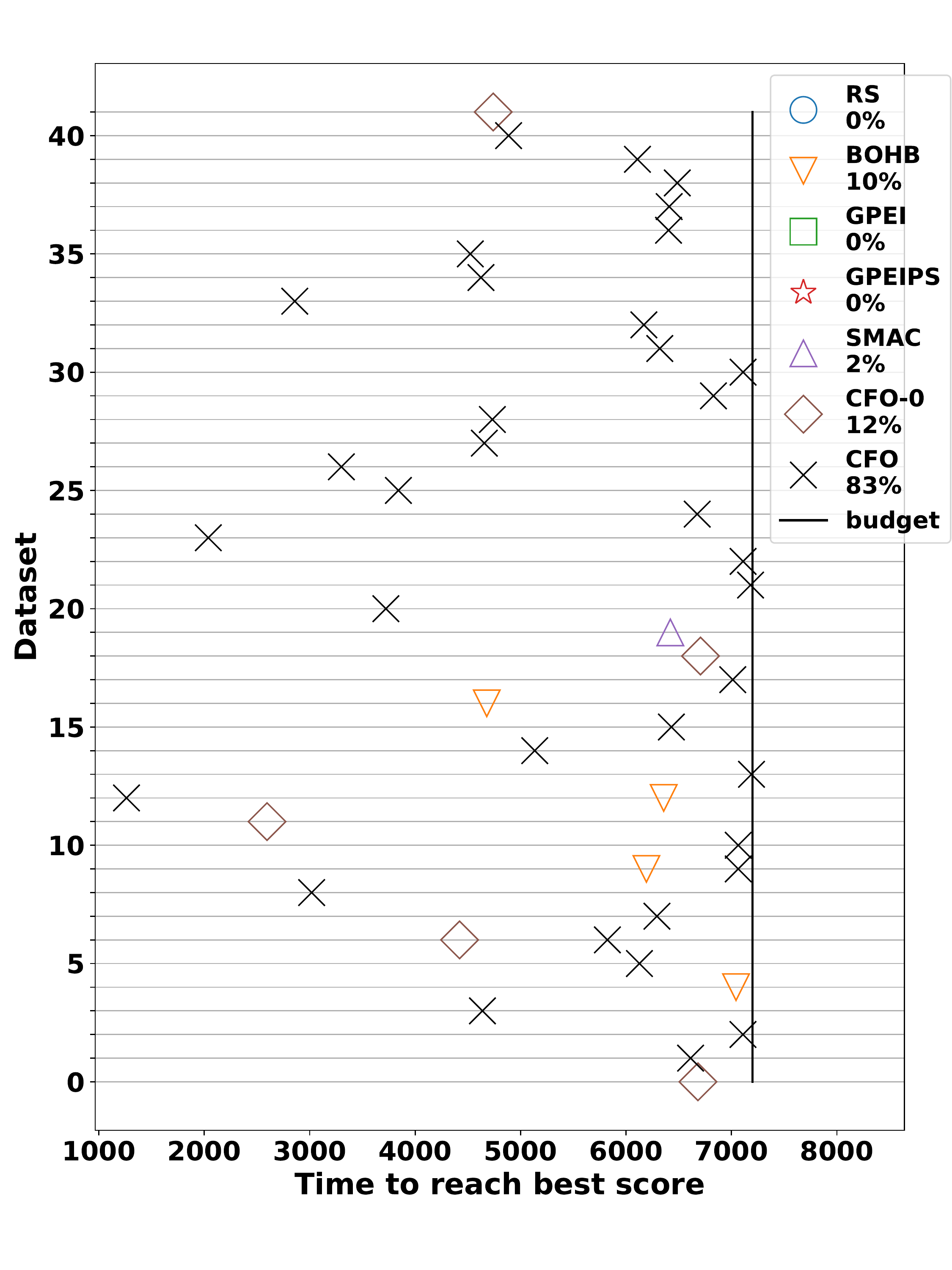}%
\caption{Time used in tuning DNN, the lower the better} \label{fig:dnn_time}
\end{subfigure}\hfill
\caption{Y-axis: indexes of the evaluated datasets. In (a), the legends display the fraction of datasets on which the scores are over 90. In (b), the legends display the fraction of datasets on which the best loss can be reached within the 2h time budget}
\label{fig:dnn_score_time}
\end{figure}

\begin{table*} 
\caption{Optimality of score and cost for DNN tuning. The meaning of each cell in this table is the same as that in Table \ref{tab:optimality_xgboost}} \label{tab:optimality_DNN}
\centering
\begin{tabular}{r | l l l  l l l  l  l }
\hline
Dataset & RS & BOHB & GPEI & GPEIPS & SMAC & \HPOnameO & \HPOname \\ 
\hline
adult & -20.8\% & -0.2\% & -32.0\% & -47.6\% & -2.2\% & \textbf{0.7072 (6670s)} & -0.4\% \\
Airlines & -21.2\% & -1.9\% & -42.0\% & -40.6\% & -11.3\% & -10.5\% & \textbf{0.5663 (6613s)} \\
Amazon & -31.8\% & -7.6\% & -67.2\% & -40.0\% & -7.5\% & -14.5\% & \textbf{0.2247 (7110s)} \\
APSFailure & -0.8\% & -0.9\% & -0.4\% & -1.2\% & -0.8\% & -0.3\% & \textbf{0.9963 (4639s)} \\
Australian & -2.6\% & \textbf{0.8920 (7043s)} & -3.5\% & -3.3\% & -2.2\% & -7.7\% & -0.5\% \\
bank & -63.1\% & -5.5\% & -9.5\% & -7.2\% & -5.3\% & -2.8\% & \textbf{0.8085 (6128s)} \\
blood & -0.9\% & -2.0\% & -3.1\% & -2.7\% & -2.2\% & \textbf{1.4288 (4421s)} & \textit{$\times$1.3} \\
car & -13.4\% & -6.1\% & -7.3\% & -8.7\% & -4.7\% & -16.7\% & \textbf{0.6416 (6294s)} \\
christine & -21.2\% & -0.6\% & -3.6\% & -2.0\% & -1.2\% & -0.8\% & \textbf{0.9739 (3019s)} \\
cnae-9 & -0.4\% & \textbf{1.1148 (6192s)} & -0.2\% & -0.2\% & -0.5\% & -0.1\% & \textit{$\times$1.1} \\
connect-4 & -37.5\% & -21.1\% & -77.0\% & -78.2\% & -14.1\% & -35.5\% & \textbf{0.5843 (7065s)} \\
credit-g & -1.4\% & -3.5\% & -10.5\% & -11.2\% & -4.7\% & \textbf{0.7247 (2595s)} & -2.5\% \\
fabert & -1.1\% & \textit{$\times$5.0} & -2.7\% & -3.1\% & -1.2\% & -0.9\% & \textbf{1.0070 (1263s)} \\
Helena & -58.0\% & -30.7\% & -42.0\% & -56.1\% & -27.1\% & -26.7\% & \textbf{1.9203 (7191s)} \\
higgs & -3.8\% & -2.5\% & -39.5\% & -66.8\% & -9.1\% & -8.3\% & \textbf{0.6615 (5135s)} \\
Jannis & -11.9\% & -9.2\% & -24.2\% & -23.4\% & -10.2\% & -9.5\% & \textbf{0.7691 (6431s)} \\
jasmine & -1.2\% & \textbf{0.8978 (4678s)} & -2.1\% & -2.0\% & -1.8\% & -1.3\% & -0.4\% \\
jungle & -14.8\% & -32.9\% & -23.9\% & -37.4\% & -14.0\% & -35.7\% & \textbf{0.4403 (7013s)} \\
kc1 & -1.2\% & -20.7\% & -1.5\% & -1.4\% & -0.8\% & \textbf{0.8618 (6705s)} & -0.4\% \\
KDDCup09 & -14.4\% & -1.7\% & -17.4\% & -7.7\% & \textbf{0.3759 (6421s)} & -3.9\% & -0.5\% \\
kr-vs-kp & -0.5\% & -0.1\% & -0.4\% & -0.4\% & -0.2\% & -0.5\% & \textbf{0.9917 (3723s)} \\
mfeat & -5.6\% & -1.7\% & -2.8\% & -3.2\% & -2.2\% & -3.9\% & \textbf{0.9739 (7182s)} \\
MiniBooNE & -2.1\% & -1.8\% & -6.9\% & -14.7\% & -1.4\% & -2.3\% & \textbf{0.9716 (7111s)} \\
nomao & -0.8\% & -0.2\% & -1.3\% & -2.1\% & -0.6\% & -0.5\% & \textbf{0.9839 (2039s)} \\
numerai28. & -1.2\% & -0.1\% & -25.7\% & -29.0\% & -5.4\% & -4.3\% & \textbf{1.3909 (6676s)} \\
phoneme & -1.3\% & -0.7\% & -0.1\% & -0.1\% & -0.6\% & -1.1\% & \textbf{0.9016 (3842s)} \\
segment & -3.9\% & -0.1\% & -3.6\% & -1.9\% & -2.2\% & -3.7\% & \textbf{0.9088 (3301s)} \\
shuttle & -27.9\% & -16.1\% & -23.4\% & -52.1\% & -26.8\% & -28.1\% & \textbf{0.8719 (4657s)} \\
sylvine & -7.0\% & -6.4\% & -8.3\% & -6.0\% & -1.8\% & -11.1\% & \textbf{0.7935 (4733s)} \\
vehicle & -8.0\% & -24.8\% & -8.9\% & -7.9\% & -2.1\% & -18.4\% & \textbf{0.3195 (6830s)} \\
volkert & -16.3\% & -11.2\% & -27.1\% & -21.9\% & -7.1\% & -14.4\% & \textbf{0.8716 (7111s)} \\
\hline
2dplanes* & -21.4\% & -1.5\% & -1.7\% & -1.2\% & -1.3\% & -2.6\% & \textbf{0.9357 (6321s)} \\
breast* & -5.3\% & -2.7\% & -11.2\% & -6.7\% & -0.6\% & -13.0\% & \textbf{0.0739 (6170s)} \\
echomonth* & -1.6\% & -0.4\% & -0.4\% & -1.0\% & -1.0\% & -1.2\% & \textbf{0.4252 (2859s)} \\
lowbwt* & -3.9\% & -0.5\% & -1.7\% & -4.4\% & -2.8\% & -11.9\% & \textbf{0.5631 (4625s)} \\
pwLinear* & -3.0\% & -20.9\% & -3.8\% & -11.9\% & -2.1\% & -4.0\% & \textbf{0.6005 (4524s)} \\
fried* & -24.8\% & -8.1\% & -21.1\% & -19.8\% & -15.0\% & -25.2\% & \textbf{0.6383 (6404s)} \\
house\_16H* & -6.0\% & -2.6\% & -4.4\% & -4.6\% & -3.8\% & -16.2\% & \textbf{0.1678 (6410s)} \\
house\_8L* & -3.8\% & -1.3\% & -8.4\% & -0.2\% & -1.5\% & -15.5\% & \textbf{0.1424 (6487s)} \\
houses* & -24.8\% & -16.6\% & -27.8\% & -1.1\% & -15.5\% & -45.9\% & \textbf{0.2146 (6109s)} \\
mv* & -11.4\% & -8.9\% & -22.8\% & -11.0\% & -9.3\% & -7.3\% & \textbf{0.8856 (4887s)} \\
pol* & -14.8\% & -2.8\% & -4.1\% & -12.6\% & -11.3\% & \textbf{0.6860 (4736s)} & -0.7\% \\
\hline
\multicolumn{8}{l}{* for regression datasets} 
\end{tabular}
\end{table*}

\renewcommand*{\thesubfigure}{\arabic{subfigure}}
\begin{figure*}[h]
  \centering
  \begin{subfigure}{0.32\columnwidth}
\includegraphics[width=\columnwidth]{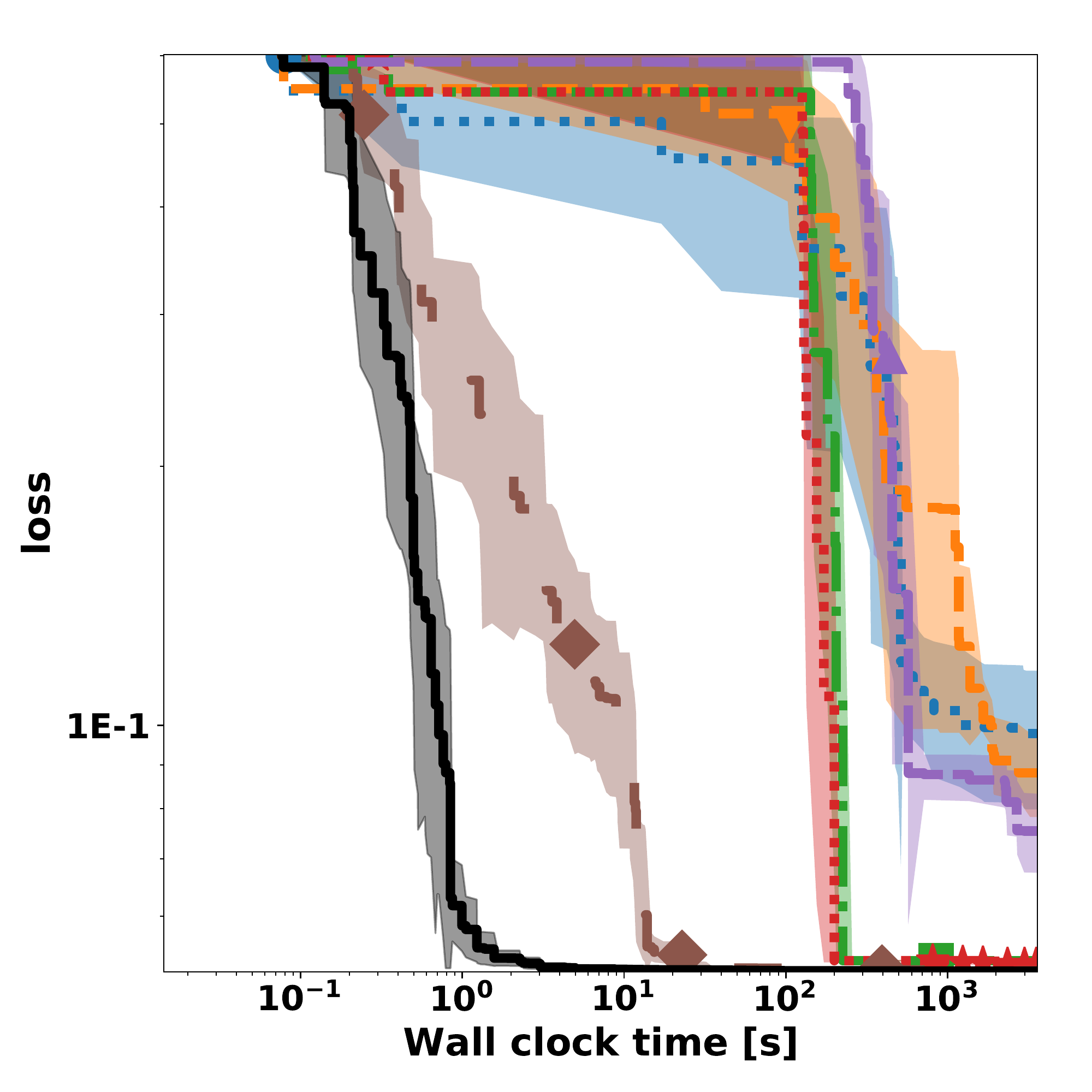}%
\caption{2dplanes}%
\end{subfigure}\hfill%
\begin{subfigure}{0.32\columnwidth}
\includegraphics[width=\columnwidth]{figures/xgboost_adult.pdf}%
\caption{adult}%
\end{subfigure}\hfill%
\begin{subfigure}{0.32\columnwidth}
\includegraphics[width=\columnwidth]{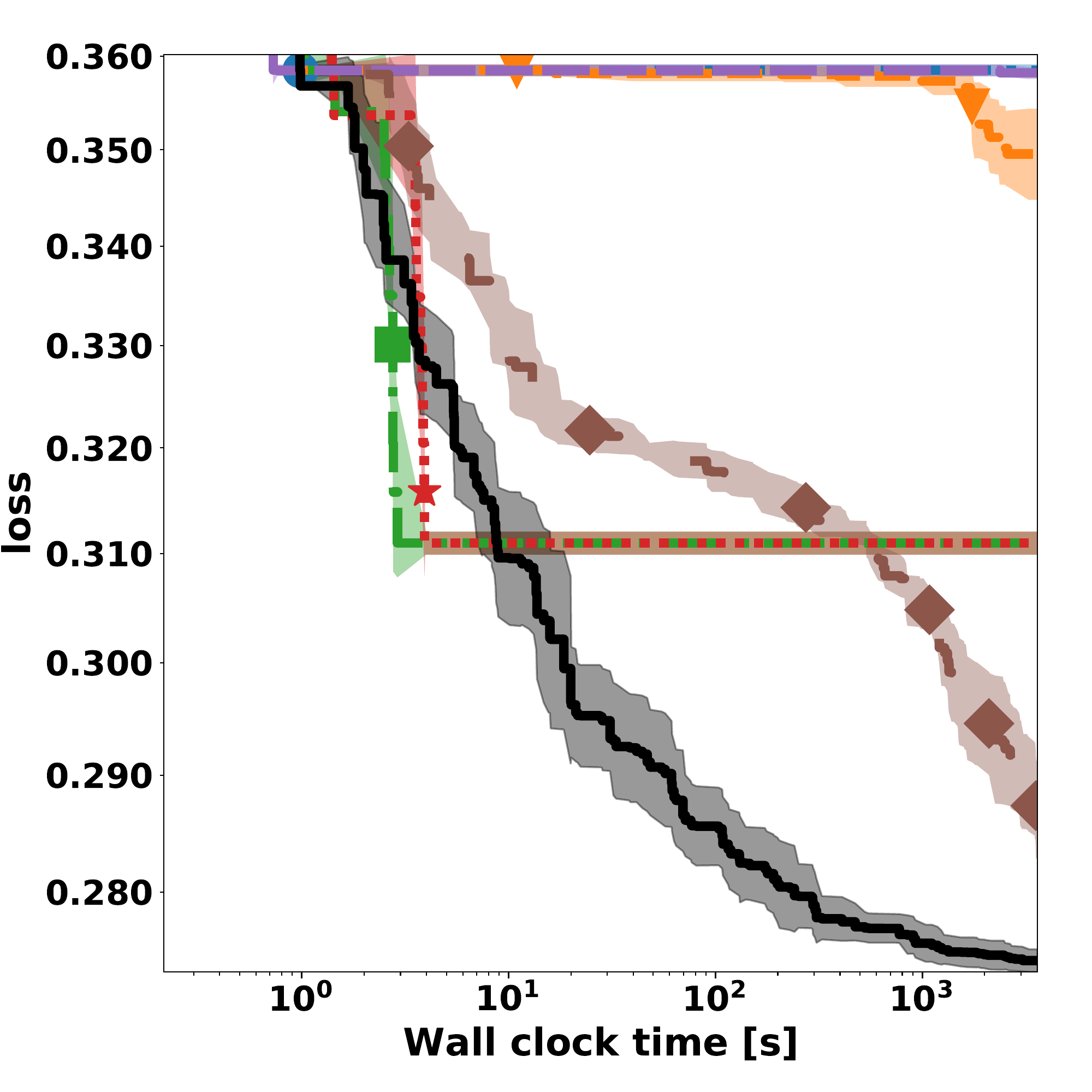}%
\caption{Airlines}%
\end{subfigure}\hfill%
  \begin{subfigure}{0.32\columnwidth}
\includegraphics[width=\columnwidth]{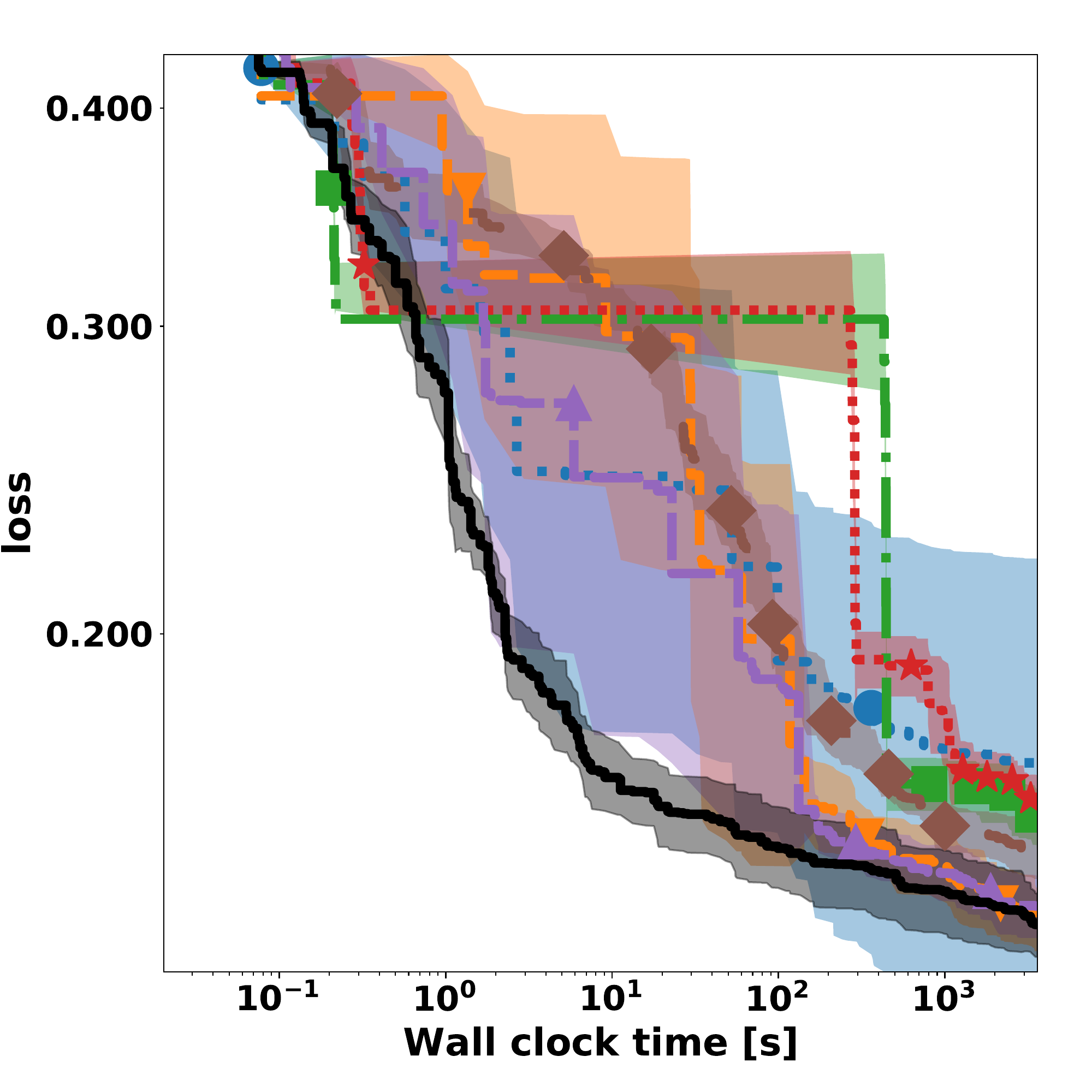}%
\caption{Amazon\_employee\_access}%
\end{subfigure}\hfill%
\begin{subfigure}{0.32\columnwidth}
\includegraphics[width=\columnwidth]{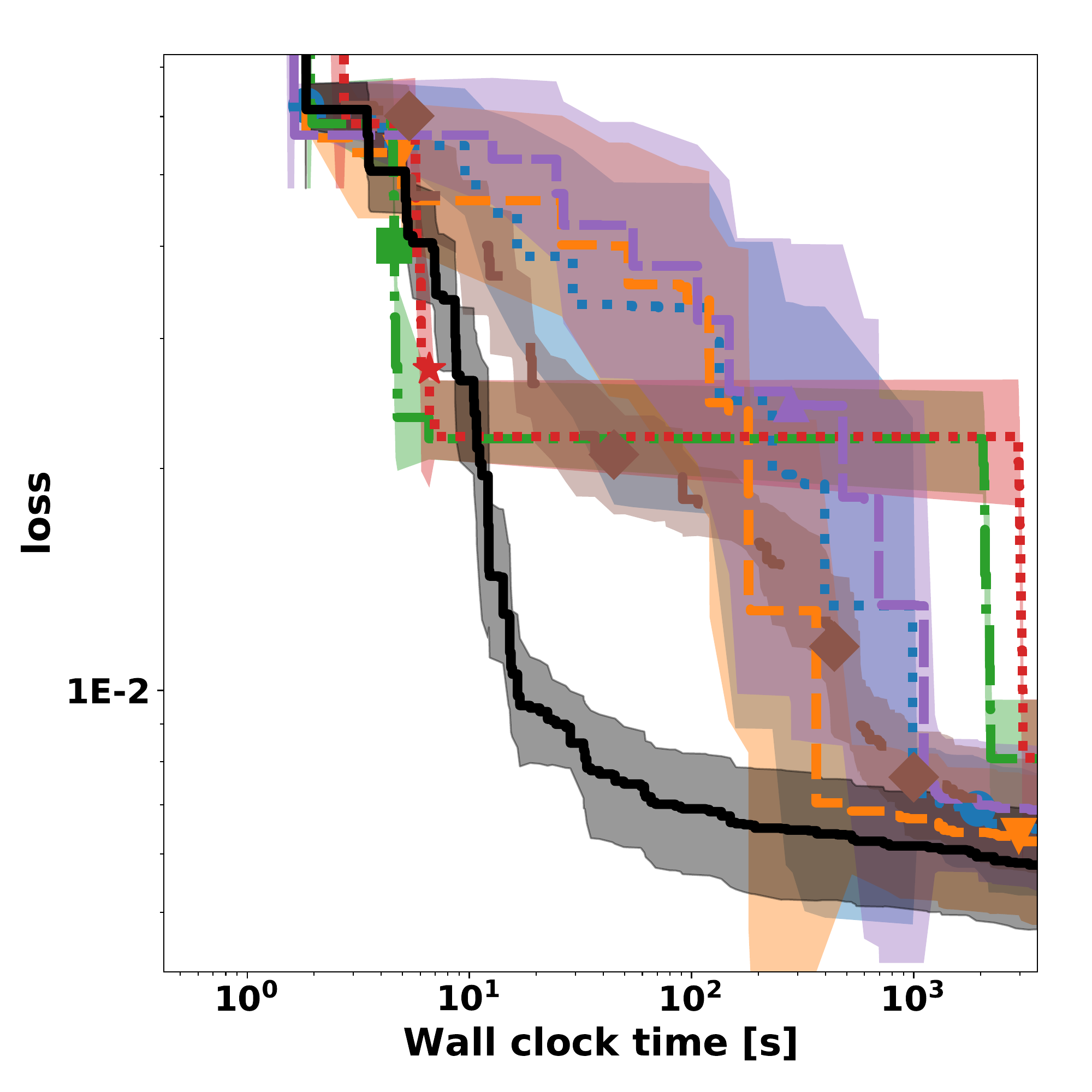}%
\caption{APSFailure}%
\end{subfigure}\hfill%
\begin{subfigure}{0.32\columnwidth}
\includegraphics[width=\columnwidth]{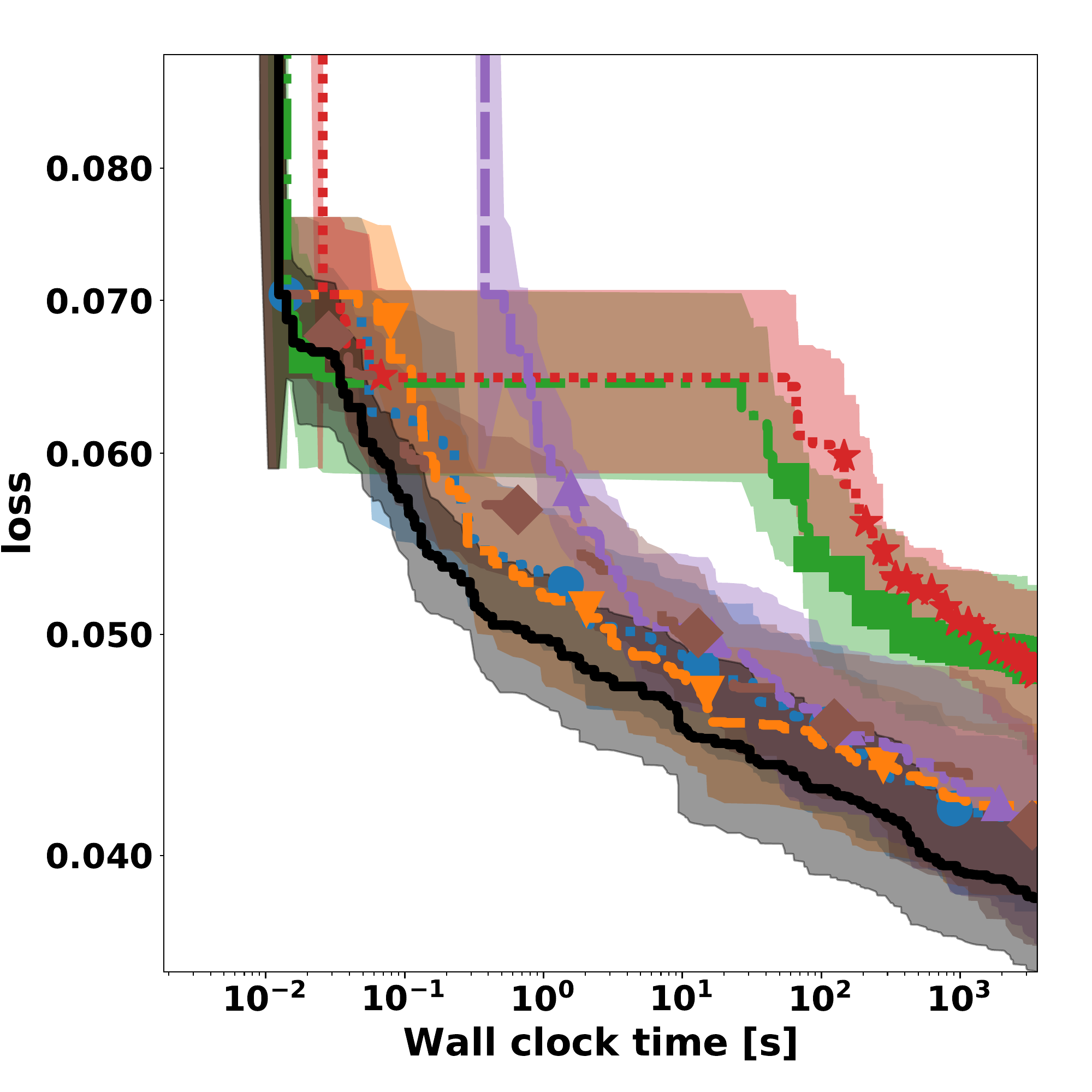}%
\caption{Australian}%
\end{subfigure}\hfill%
\begin{subfigure}{0.32\columnwidth}
\includegraphics[width=\columnwidth]{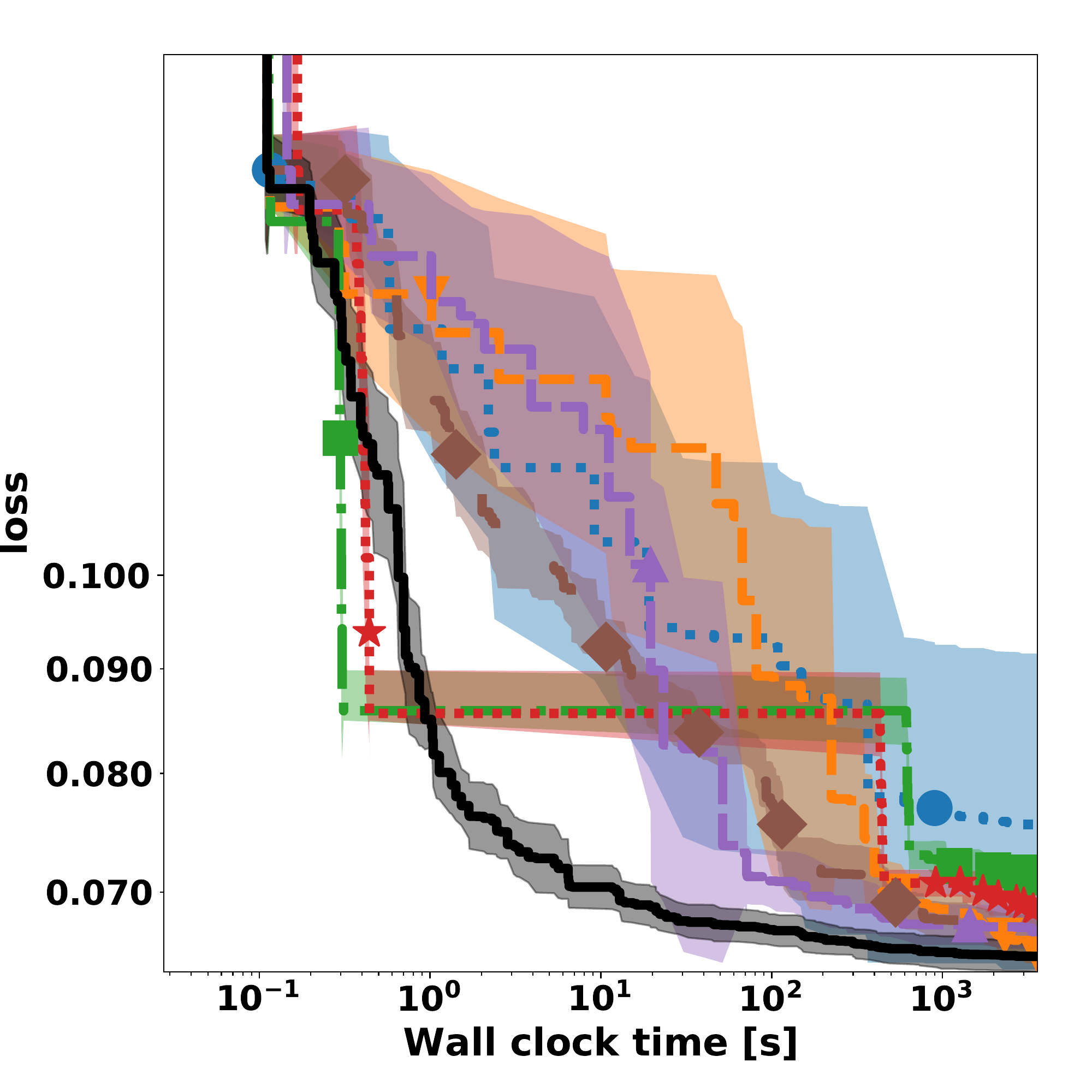}%
\caption{bank\_marketing}%
\end{subfigure}\hfill%
  \begin{subfigure}{0.32\columnwidth}
\includegraphics[width=\columnwidth]{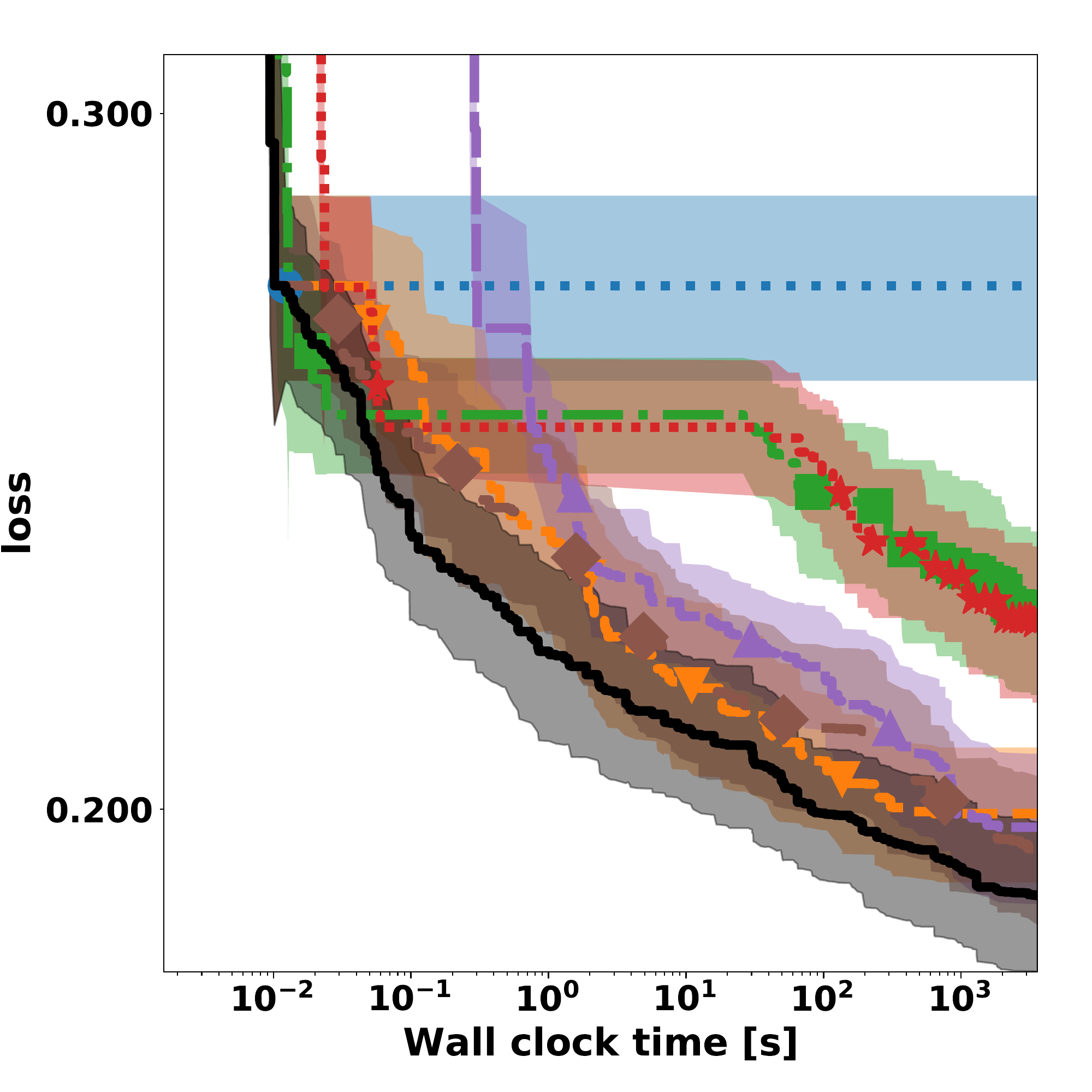}%
\caption{blood-transfusion}%
\end{subfigure}\hfill%
\begin{subfigure}{0.32\columnwidth}
\includegraphics[width=\columnwidth]{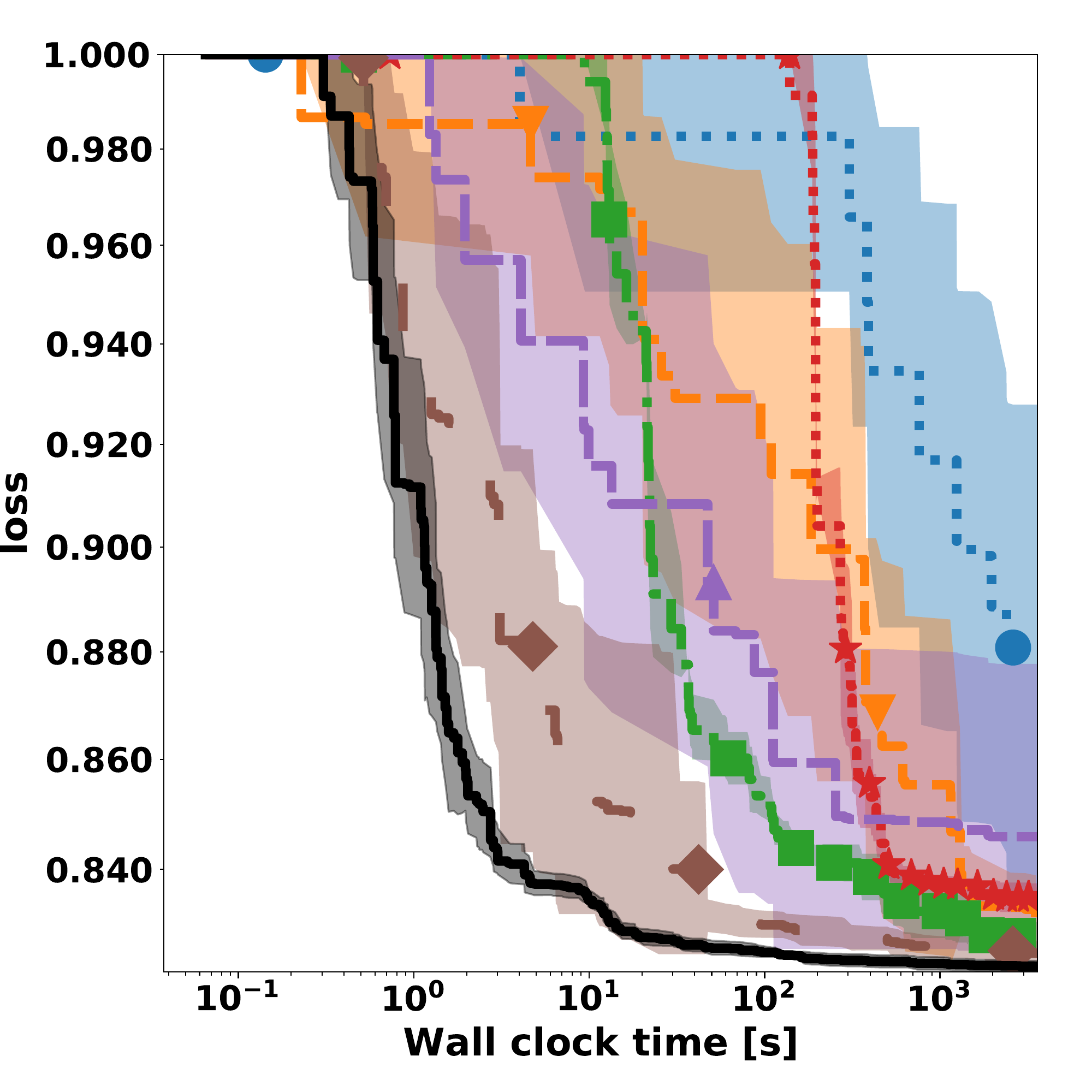}%
\caption{bng\_breastTumor}%
\end{subfigure}\hfill%
\begin{subfigure}{0.32\columnwidth}
\includegraphics[width=\columnwidth]{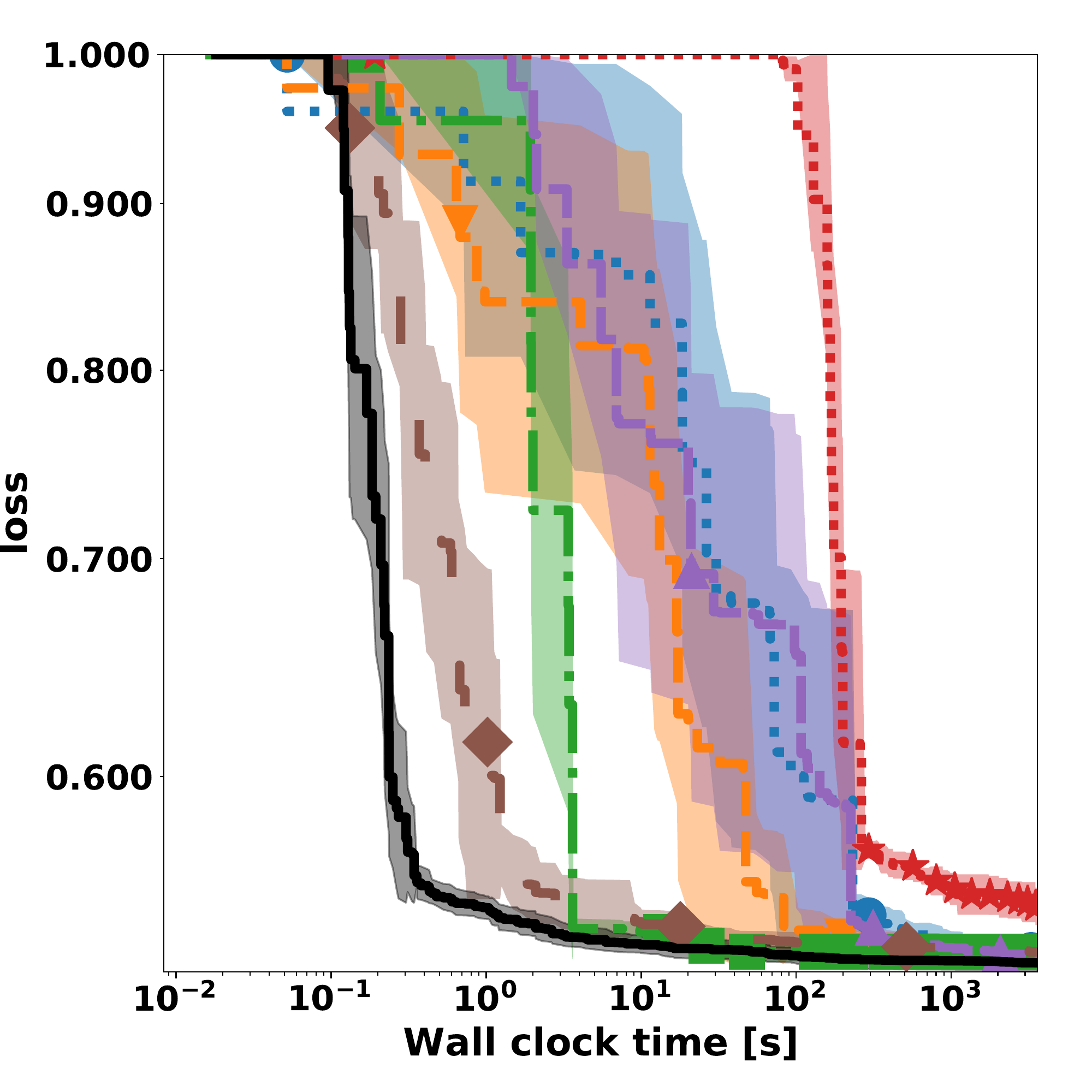}%
\caption{bng\_echomonths}%
\end{subfigure}\hfill%
  \begin{subfigure}{0.32\columnwidth}
\includegraphics[width=\columnwidth]{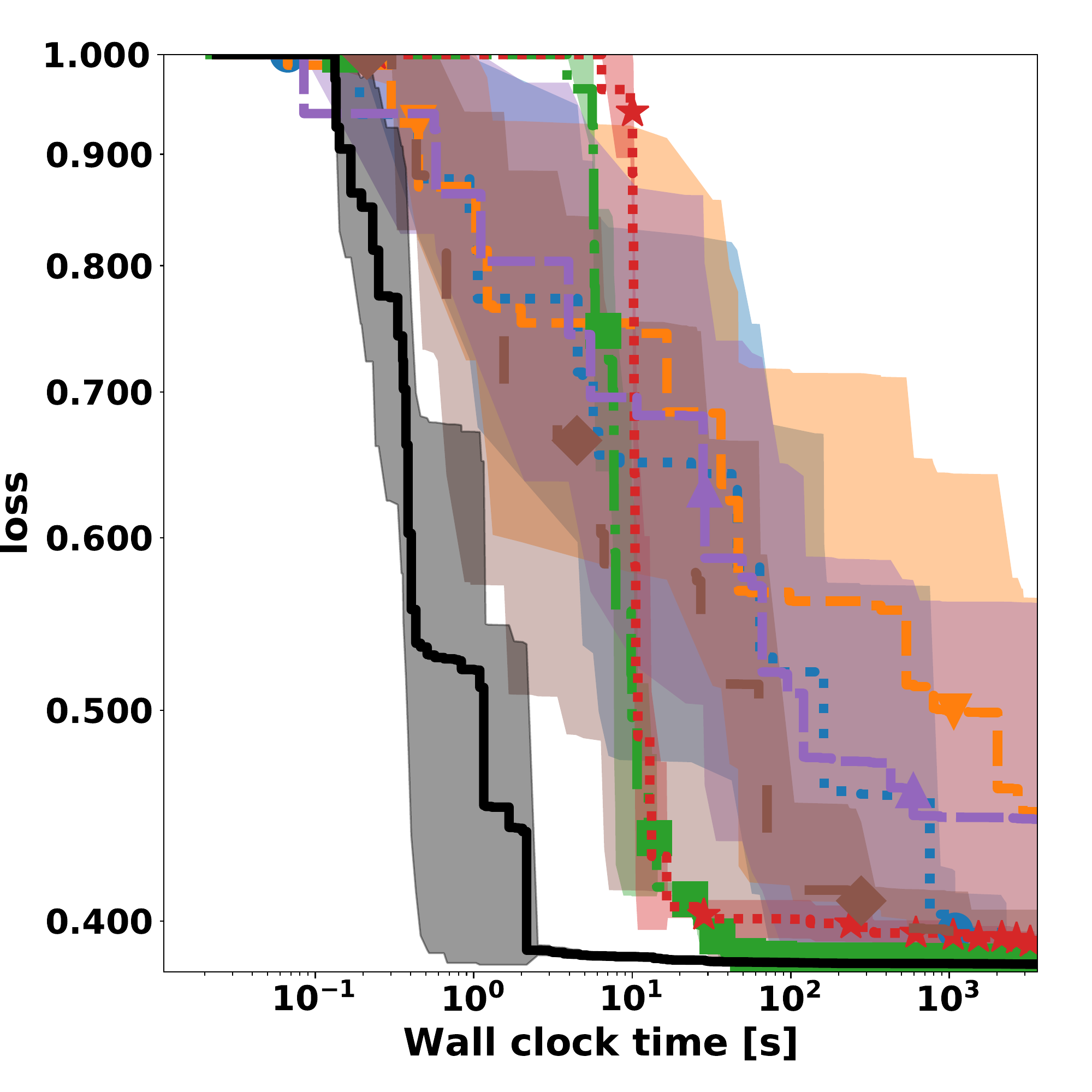}%
\caption{bng\_lowbwt}%
\end{subfigure}\hfill%
\begin{subfigure}{0.32\columnwidth}
\includegraphics[width=\columnwidth]{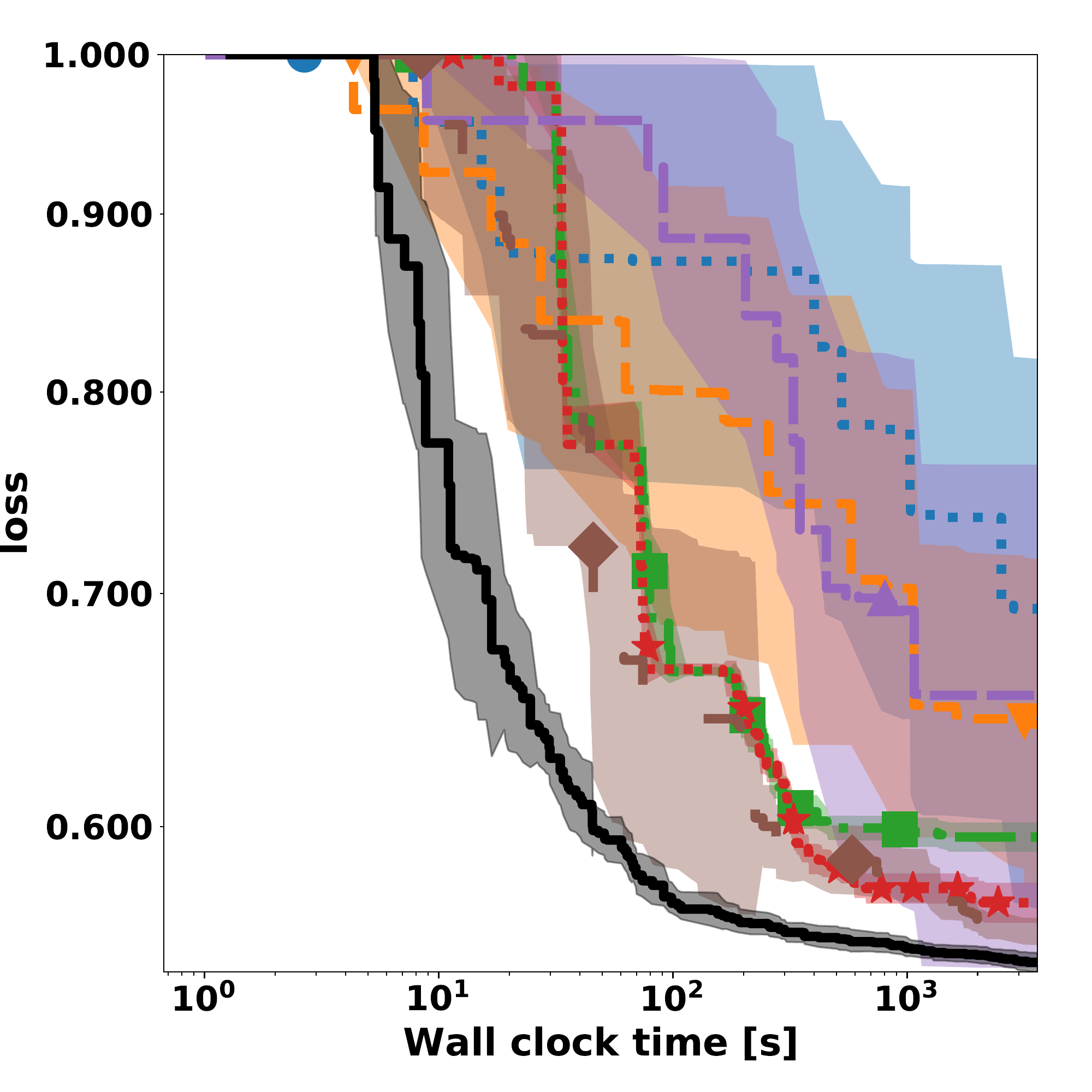}%
\caption{bng\_pbc}%
\end{subfigure}\hfill%
\caption{Optimization performance curve for XGBoost (pt 1/4)}
\end{figure*}
\begin{figure*}[h]
\ContinuedFloat
\begin{subfigure}{0.32\columnwidth}
\includegraphics[width=\columnwidth]{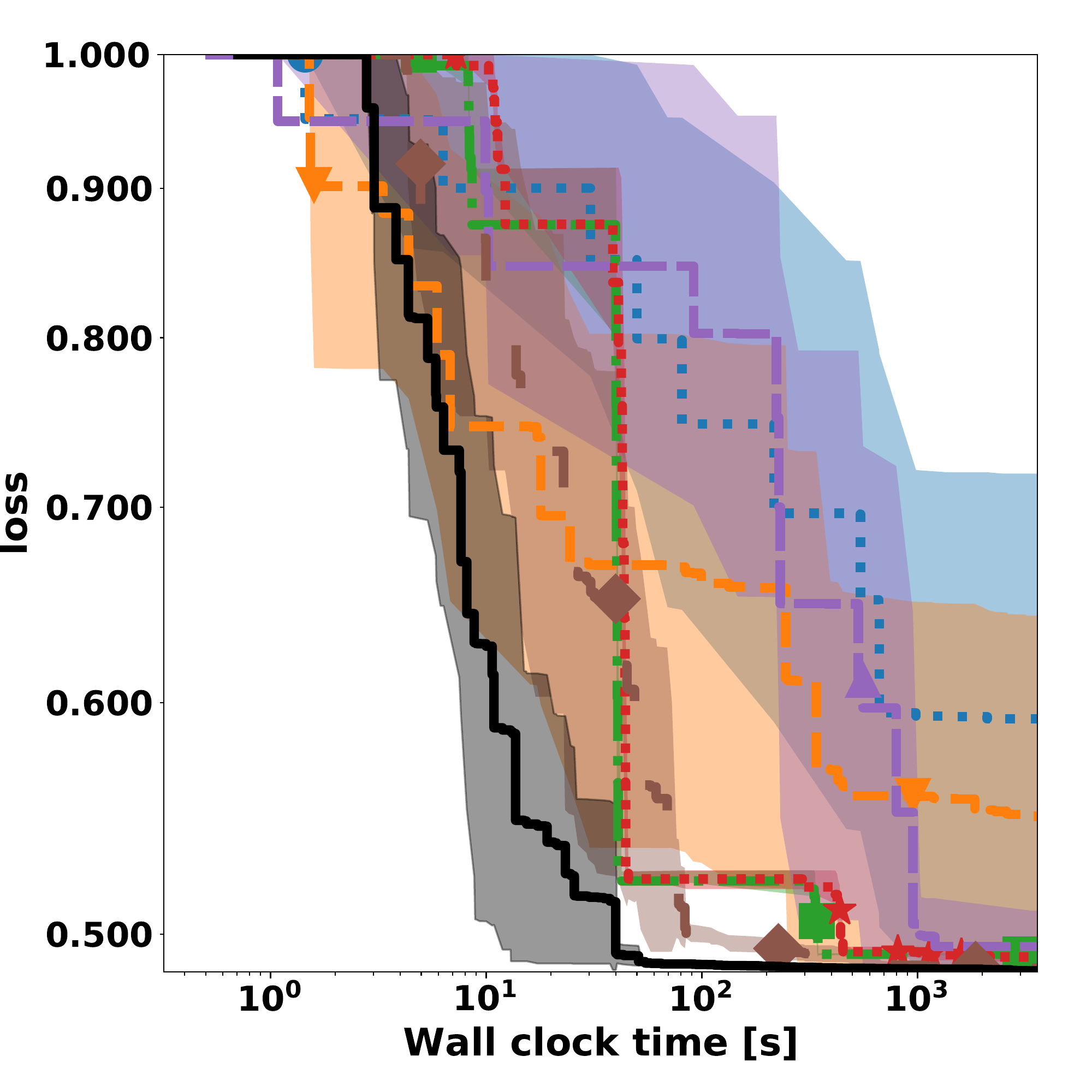}%
\caption{bng\_pharynx}%
\end{subfigure}\hfill%
  \begin{subfigure}{0.32\columnwidth}
\includegraphics[width=\columnwidth]{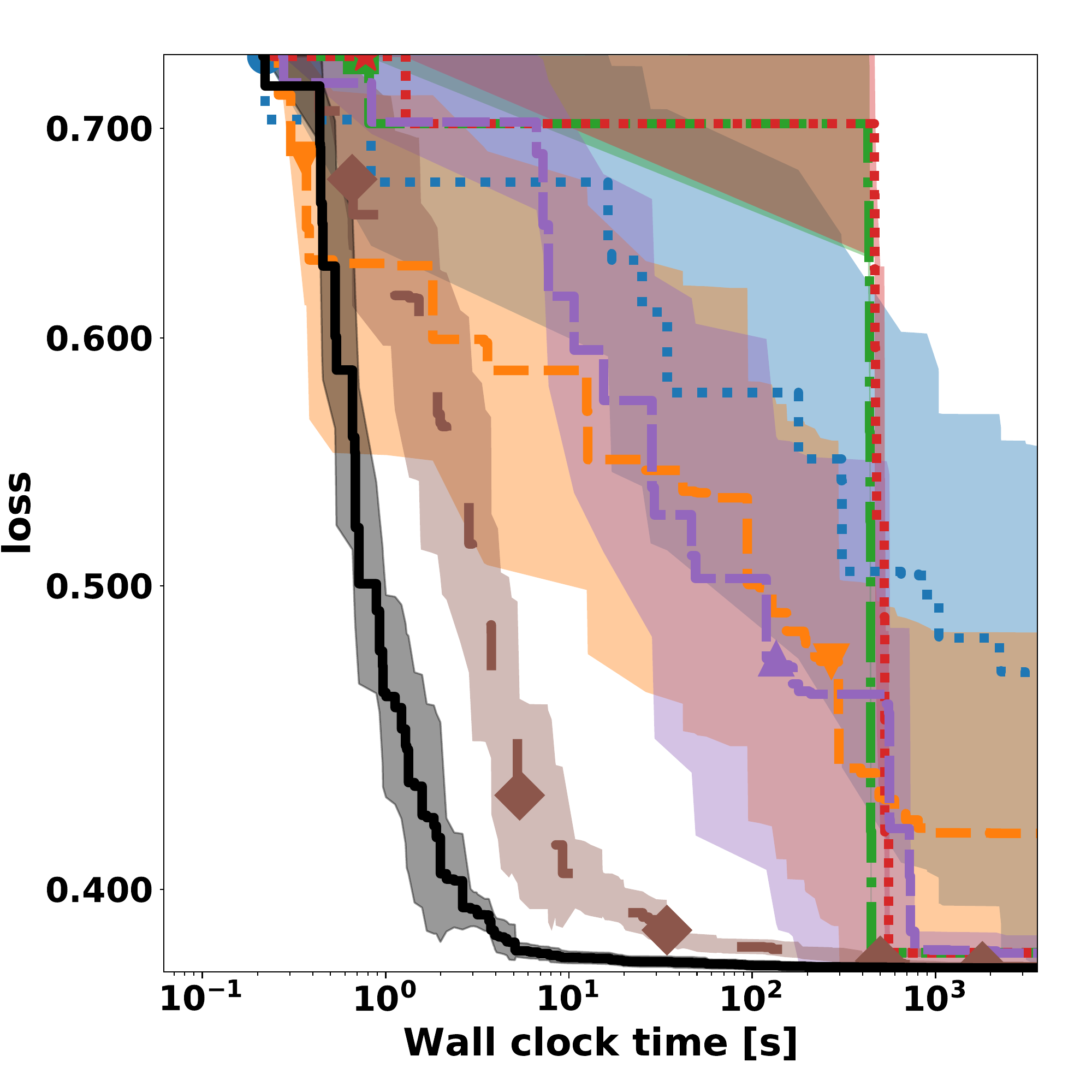}%
\caption{bng\_pwLinear}%
\end{subfigure}\hfill%
\begin{subfigure}{0.32\columnwidth}
\includegraphics[width=\columnwidth]{figures/xgboost_car.pdf}%
\caption{car}%
\end{subfigure}\hfill%
\begin{subfigure}{0.32\columnwidth}
\includegraphics[width=\columnwidth]{figures/xgboost_christine.pdf}%
\caption{christine}%
\end{subfigure}\hfill
  \begin{subfigure}{0.32\columnwidth}
\includegraphics[width=\columnwidth]{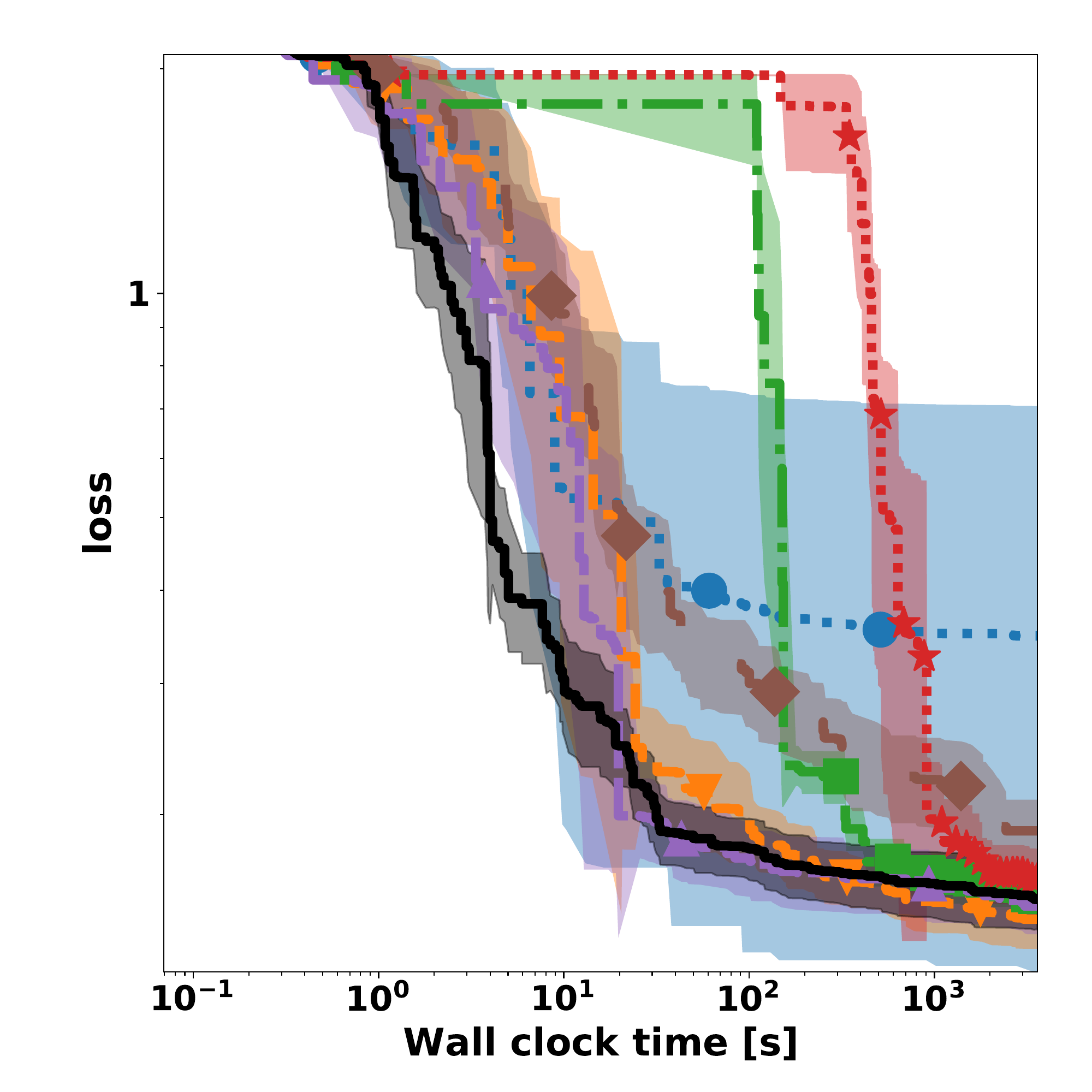}%
\caption{cnae-9}%
\end{subfigure}\hfill%
\begin{subfigure}{0.32\columnwidth}
\includegraphics[width=\columnwidth]{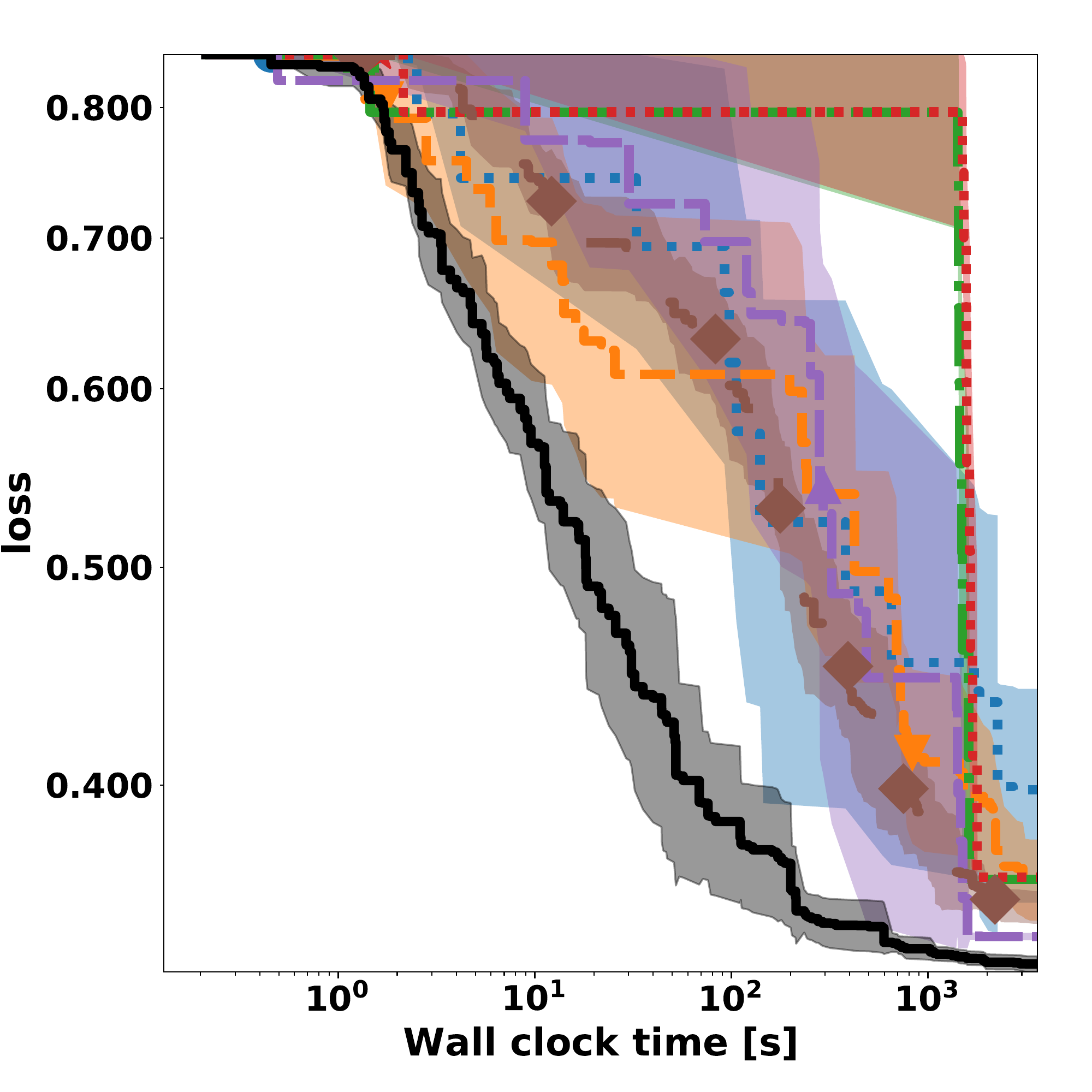}%
\caption{connect-4}%
\end{subfigure}\hfill%
\begin{subfigure}{0.32\columnwidth}
\includegraphics[width=\columnwidth]{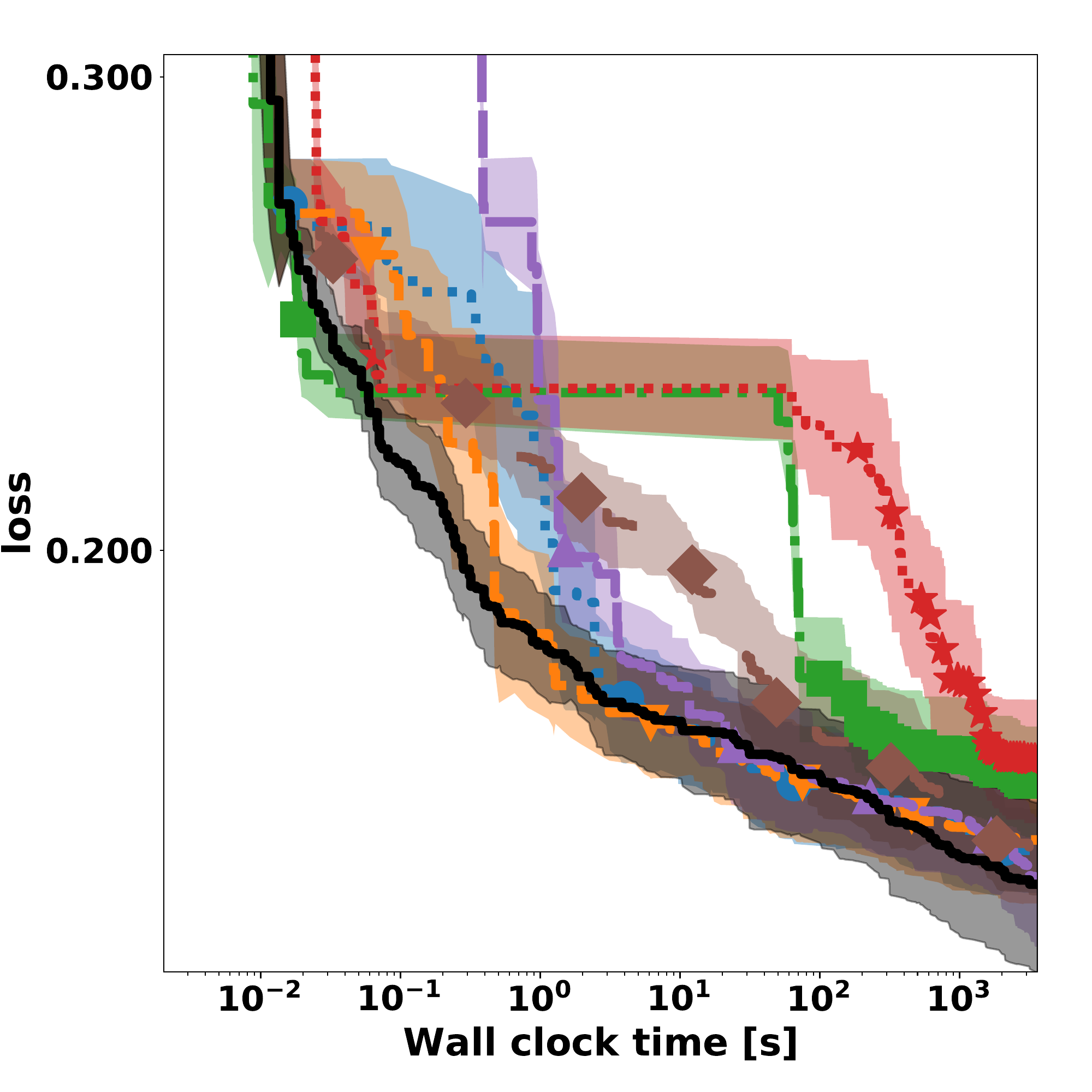}%
\caption{credit-g}%
\end{subfigure}\hfill%
  \begin{subfigure}{0.32\columnwidth}
\includegraphics[width=\columnwidth]{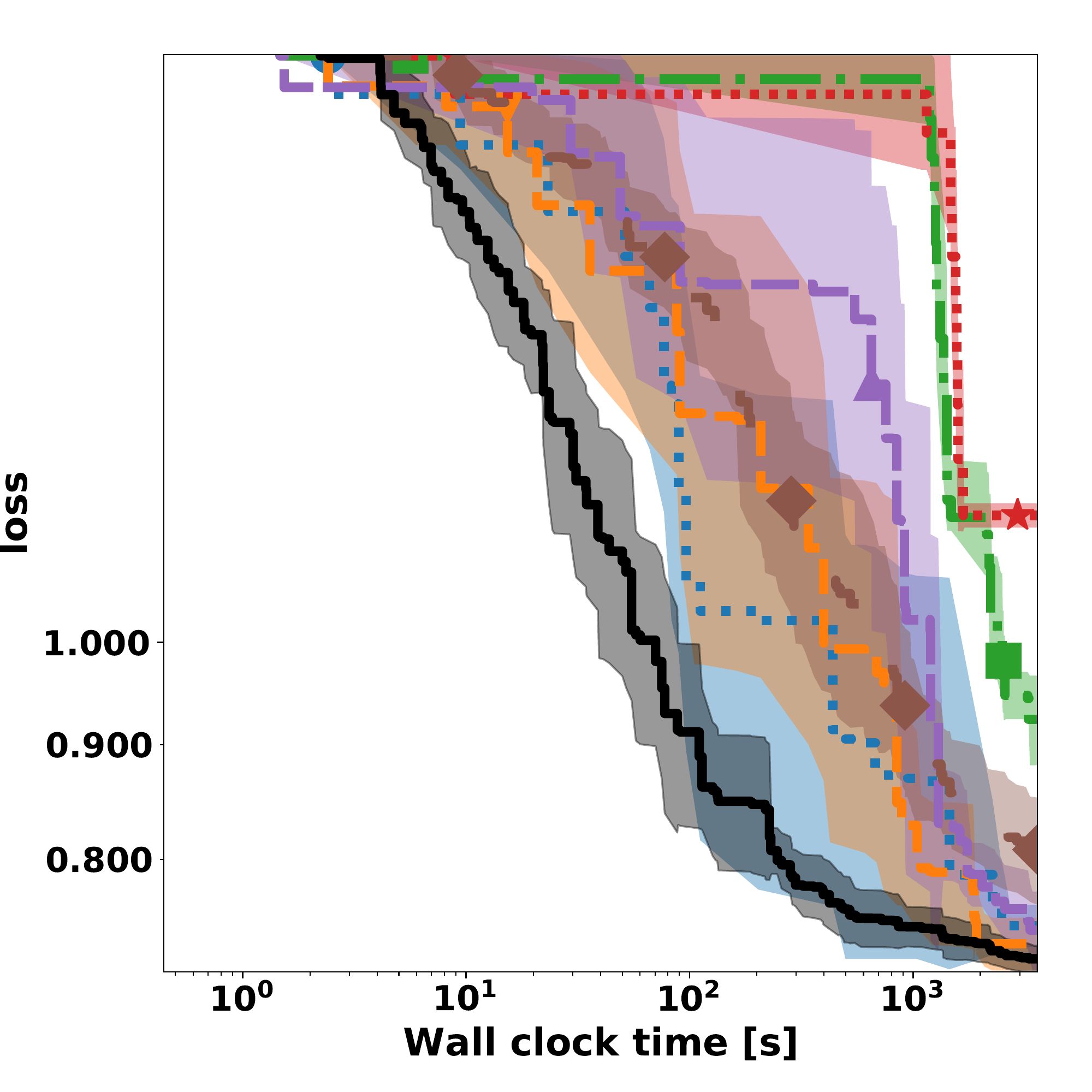}%
\caption{fabert}%
\end{subfigure}\hfill%
\begin{subfigure}{0.32\columnwidth}
\includegraphics[width=\columnwidth]{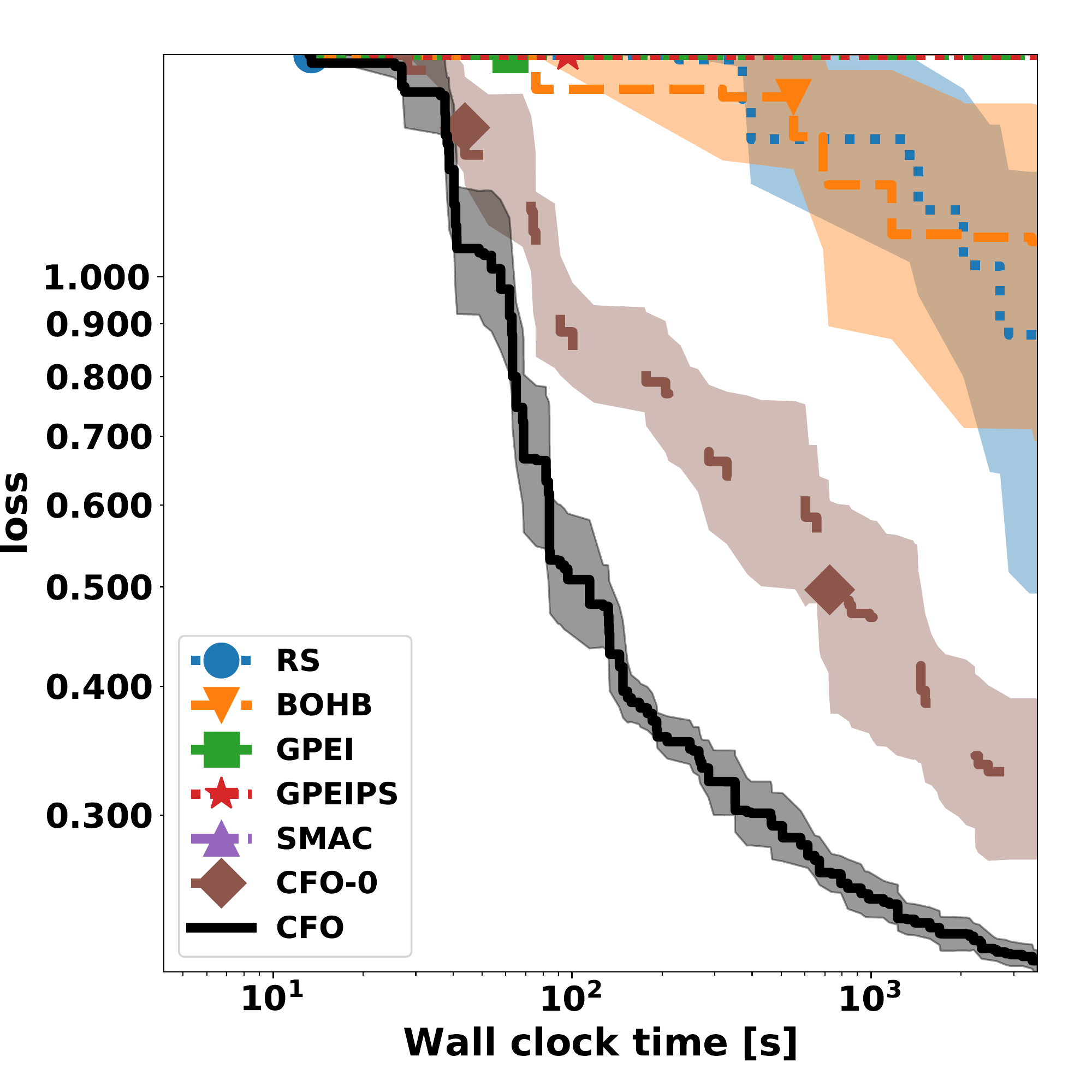}%
\caption{Fashion-MNIST}%
\end{subfigure}\hfill%
\begin{subfigure}{0.32\columnwidth}
\includegraphics[width=\columnwidth]{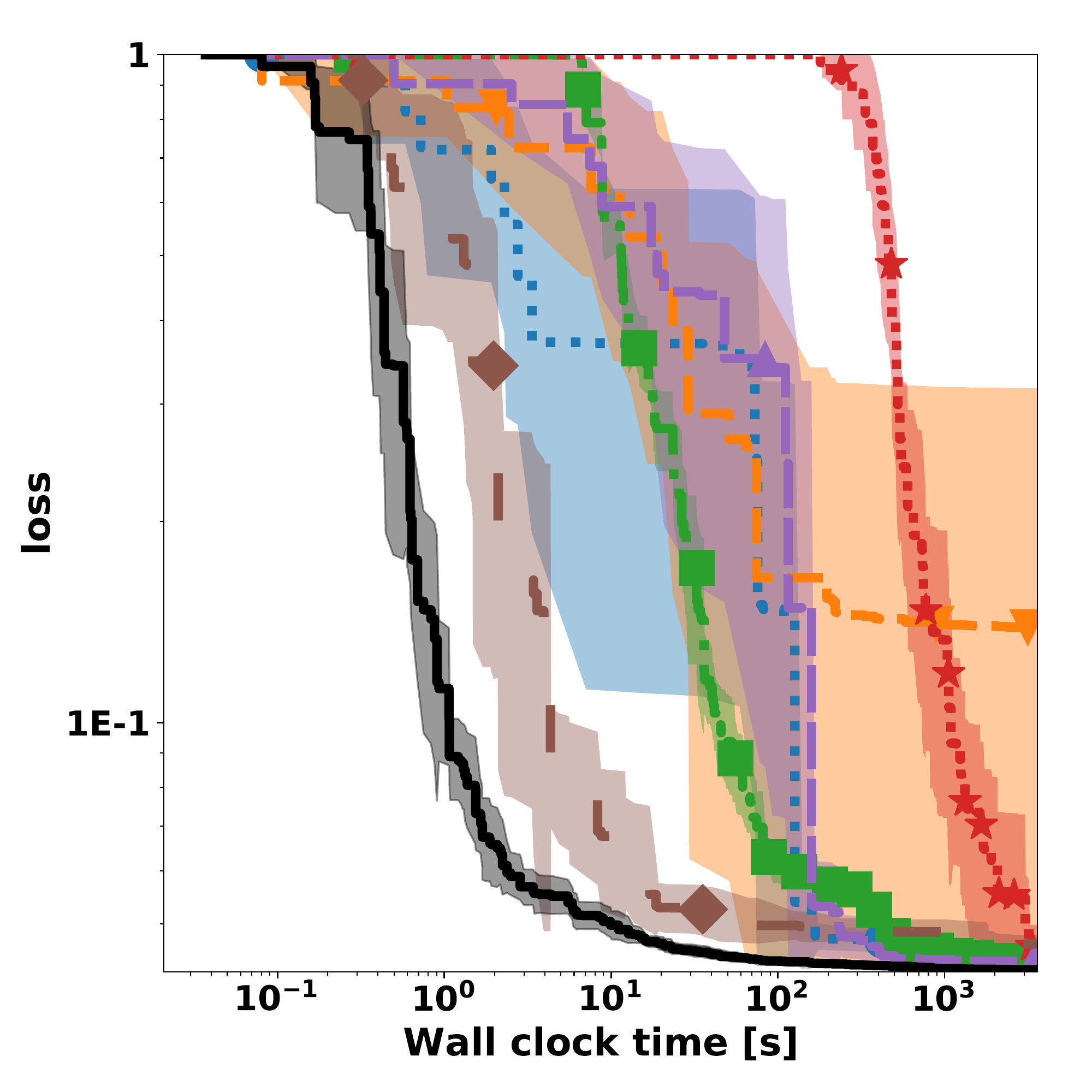}%
\caption{fried}%
\end{subfigure}\hfill%
\begin{subfigure}{0.32\columnwidth}
\includegraphics[width=\columnwidth]{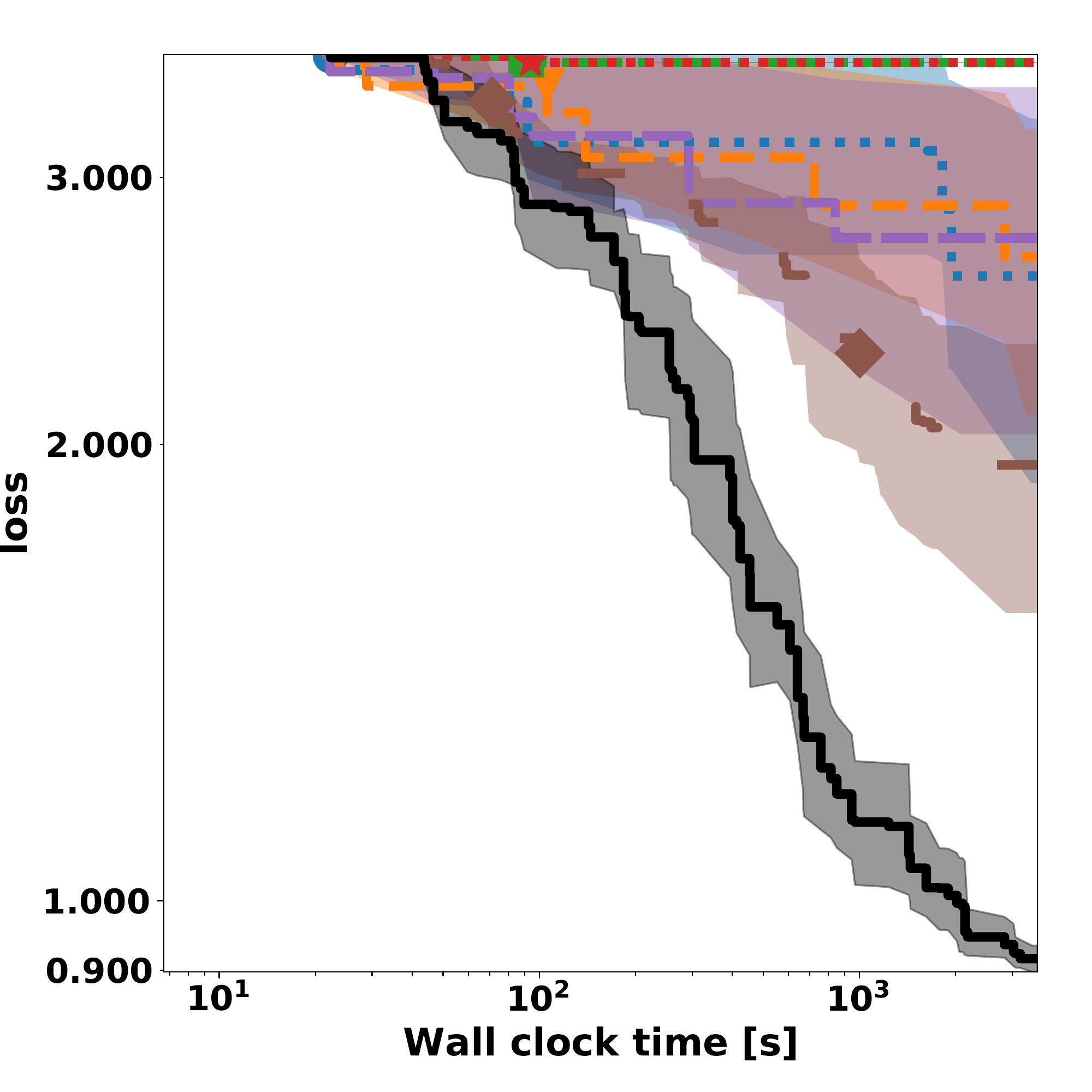}%
\caption{Helena}%
\end{subfigure}\hfill%
  \begin{subfigure}{0.32\columnwidth}
\includegraphics[width=\columnwidth]{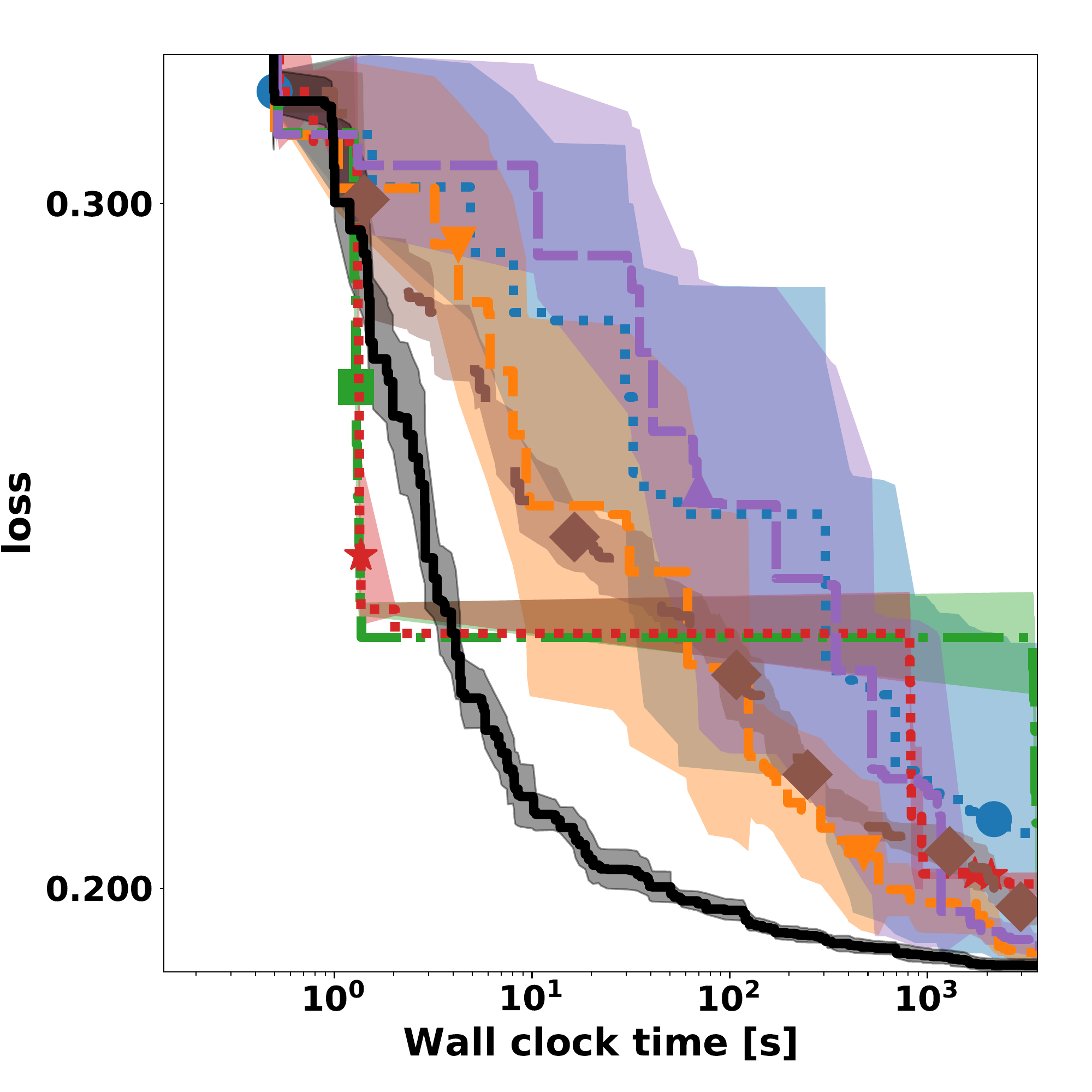}%
\caption{higgs}%
\end{subfigure}\hfill%
\caption{Optimization performance curve for XGBoost (pt 2/4)}
\end{figure*}
\begin{figure*}[h]
\ContinuedFloat
\begin{subfigure}{0.32\columnwidth}
\includegraphics[width=\columnwidth]{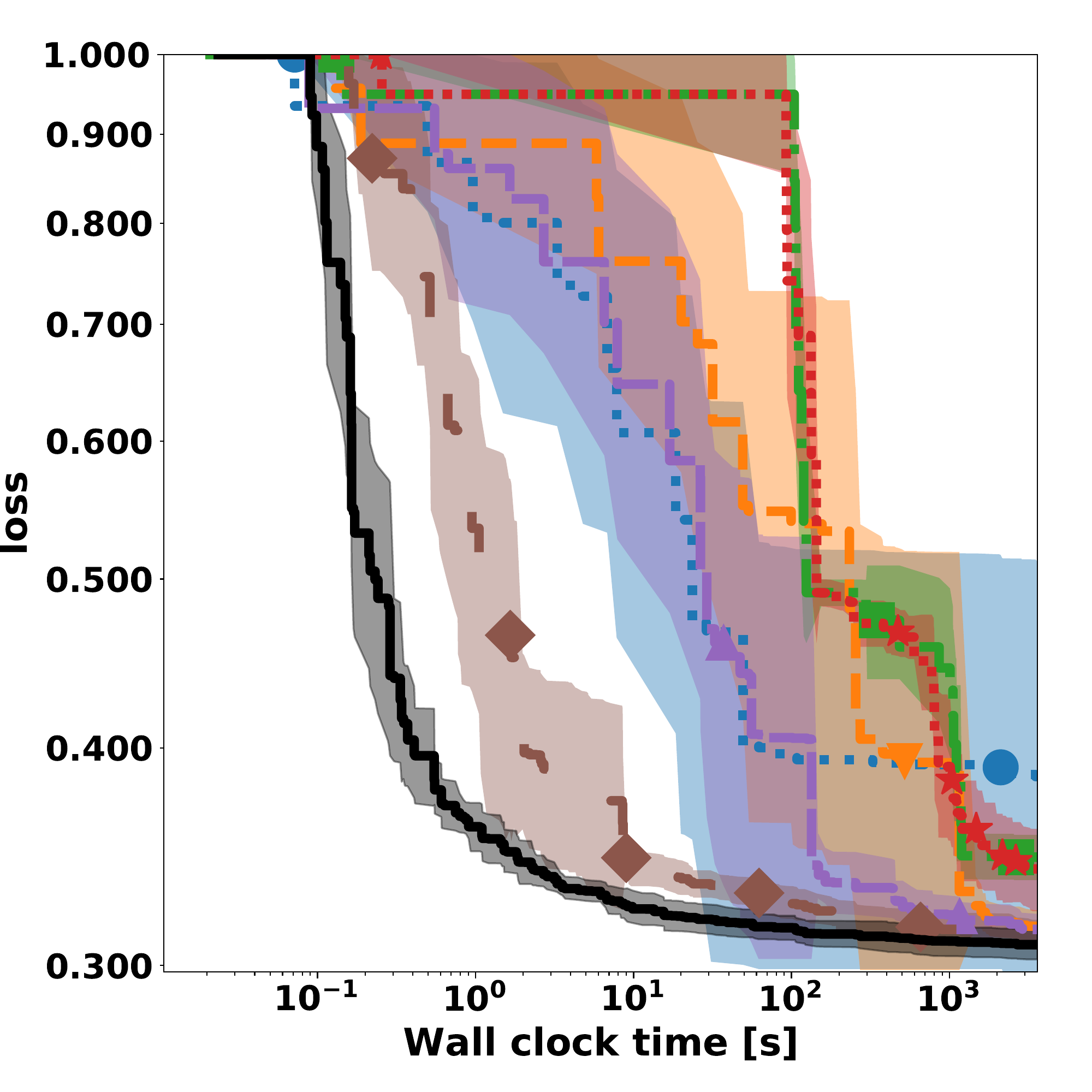}%
\caption{house\_8L}%
\end{subfigure}\hfill
\begin{subfigure}{0.32\columnwidth}
\includegraphics[width=\columnwidth]{figures/xgboost_house_16H.pdf}%
\caption{house\_16H}%
\end{subfigure}\hfill%
  \begin{subfigure}{0.32\columnwidth}
\includegraphics[width=\columnwidth]{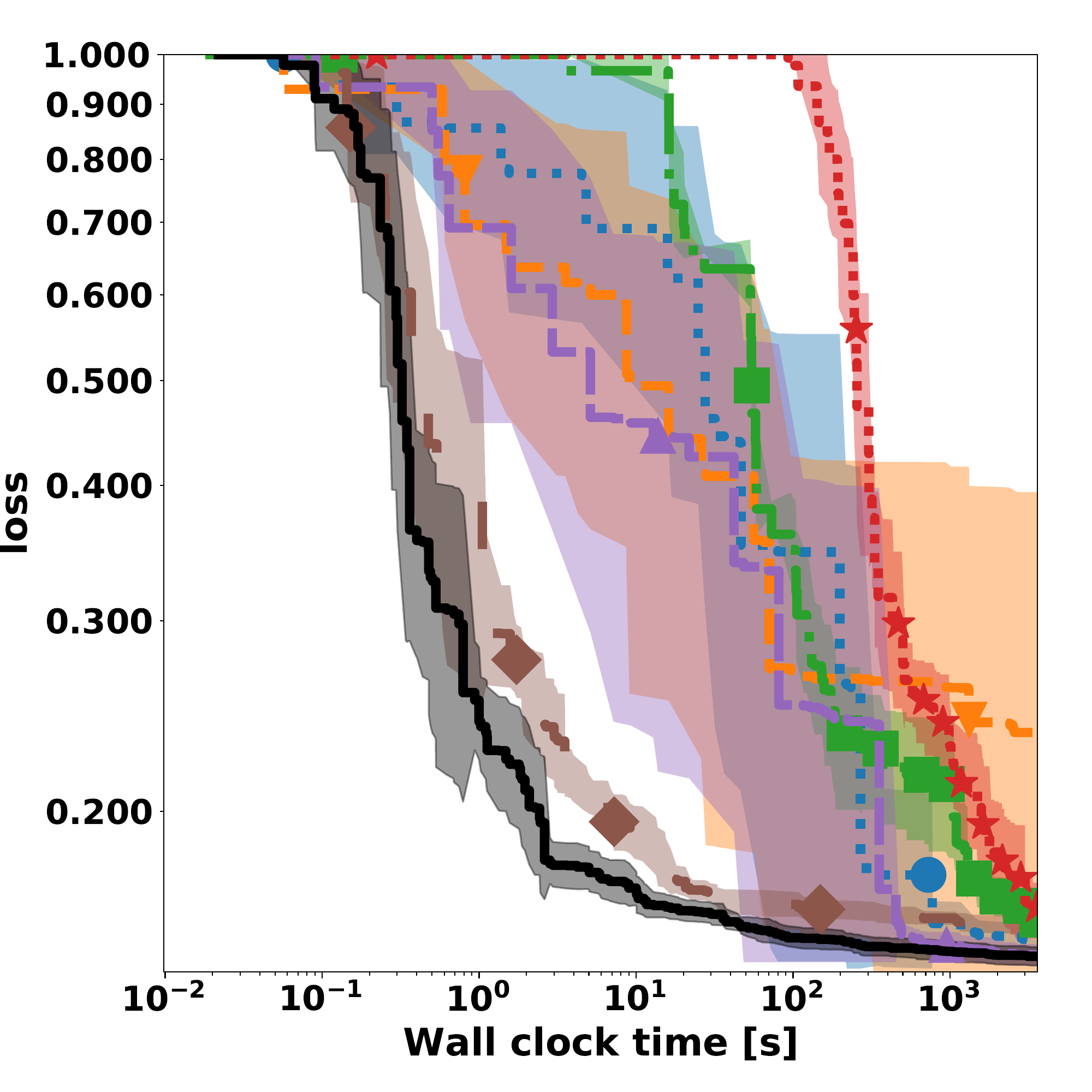}%
\caption{houses}%
\end{subfigure}\hfill%
\begin{subfigure}{0.32\columnwidth}
\includegraphics[width=\columnwidth]{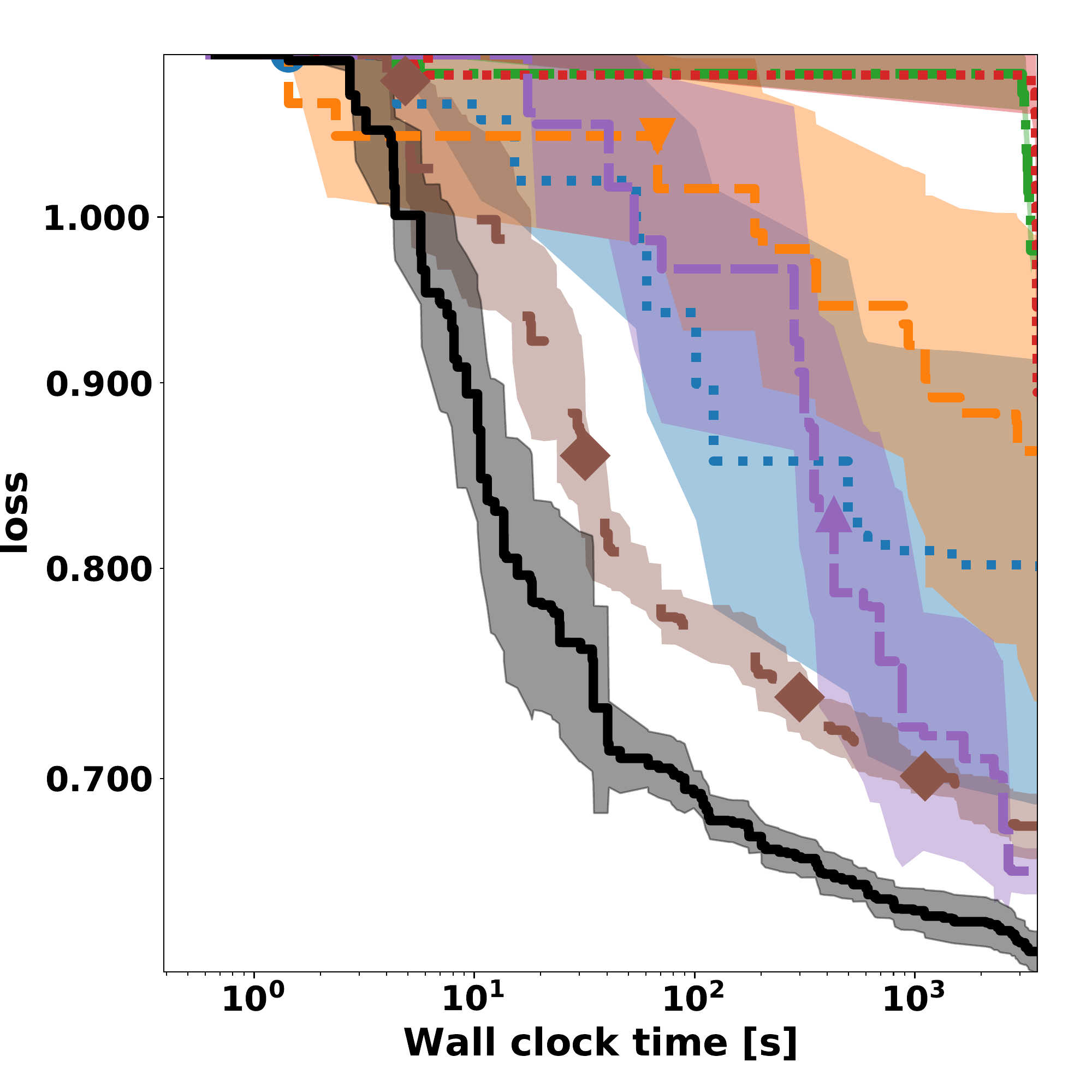}%
\caption{Jannis}%
\end{subfigure}\hfill%
\begin{subfigure}{0.32\columnwidth}
\includegraphics[width=\columnwidth]{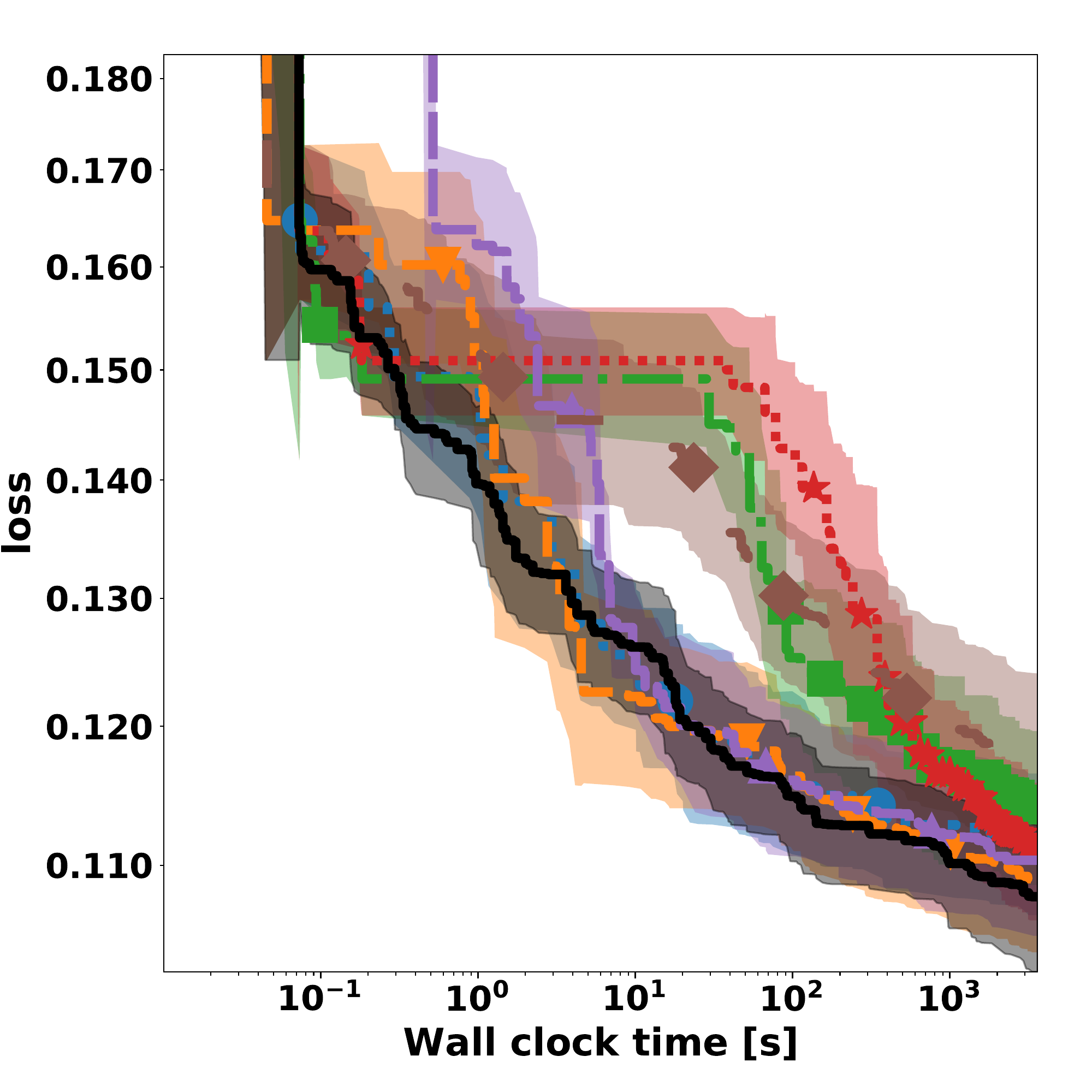}%
\caption{jasmine}%
\end{subfigure}\hfill%
  \begin{subfigure}{0.32\columnwidth}
\includegraphics[width=\columnwidth]{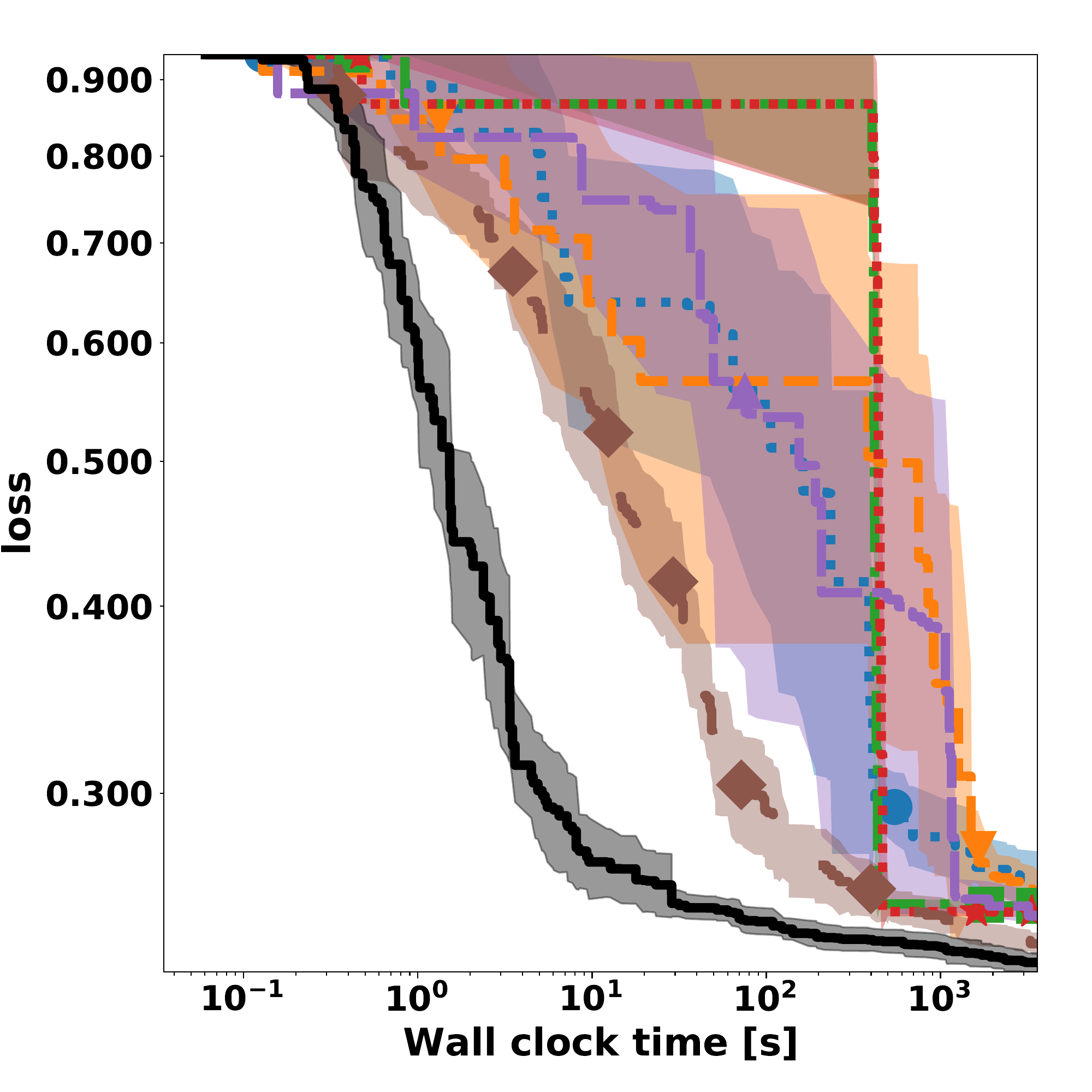}%
\caption{jungle\_chess\_2pcs\_raw\_endgame}%
\end{subfigure}\hfill%
\begin{subfigure}{0.32\columnwidth}
\includegraphics[width=\columnwidth]{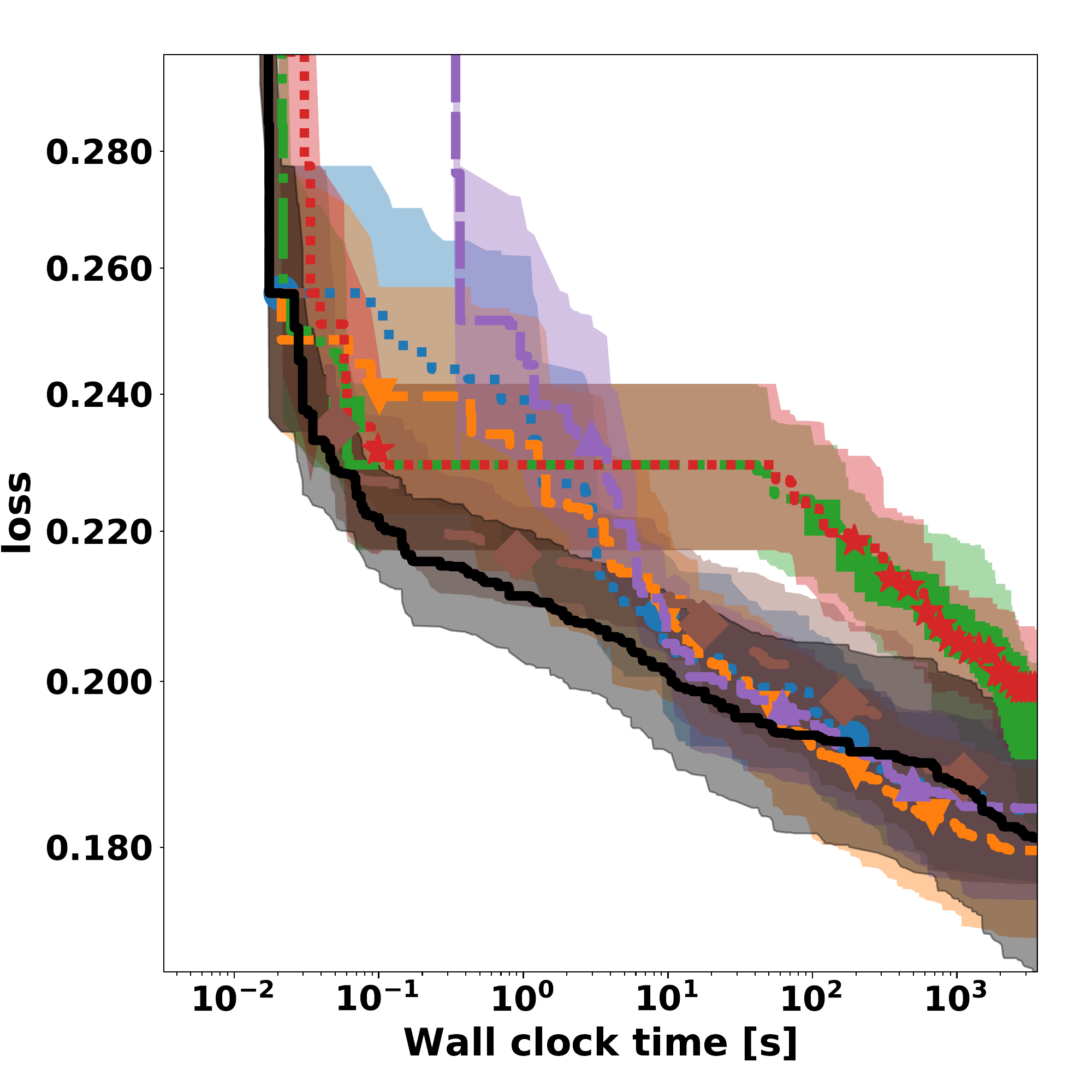}%
\caption{kc1}%
\end{subfigure}\hfill%
\begin{subfigure}{0.32\columnwidth}
\includegraphics[width=\columnwidth]{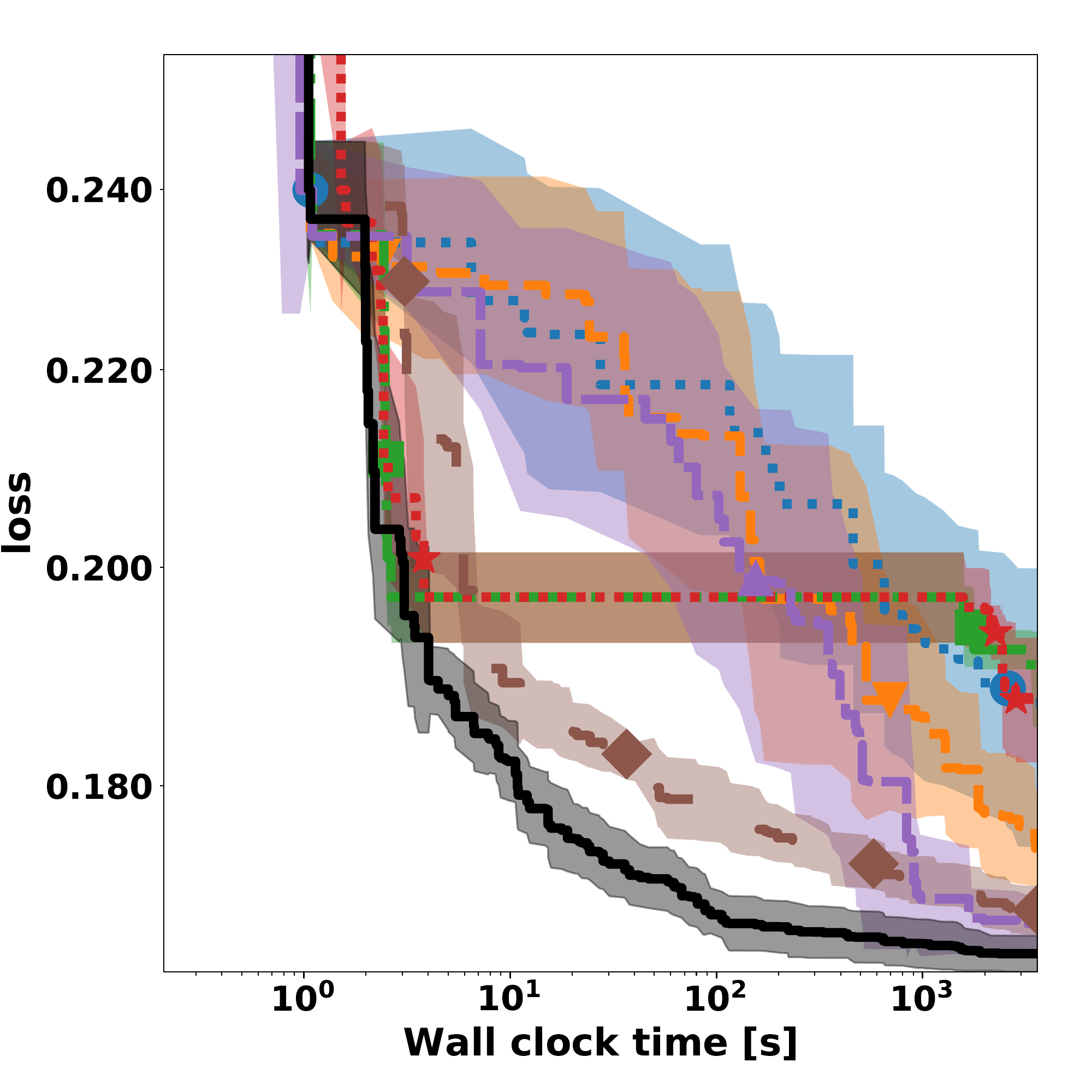}%
\caption{KDDCup09\_appetency}%
\end{subfigure}\hfill%
  \begin{subfigure}{0.32\columnwidth}
\includegraphics[width=\columnwidth]{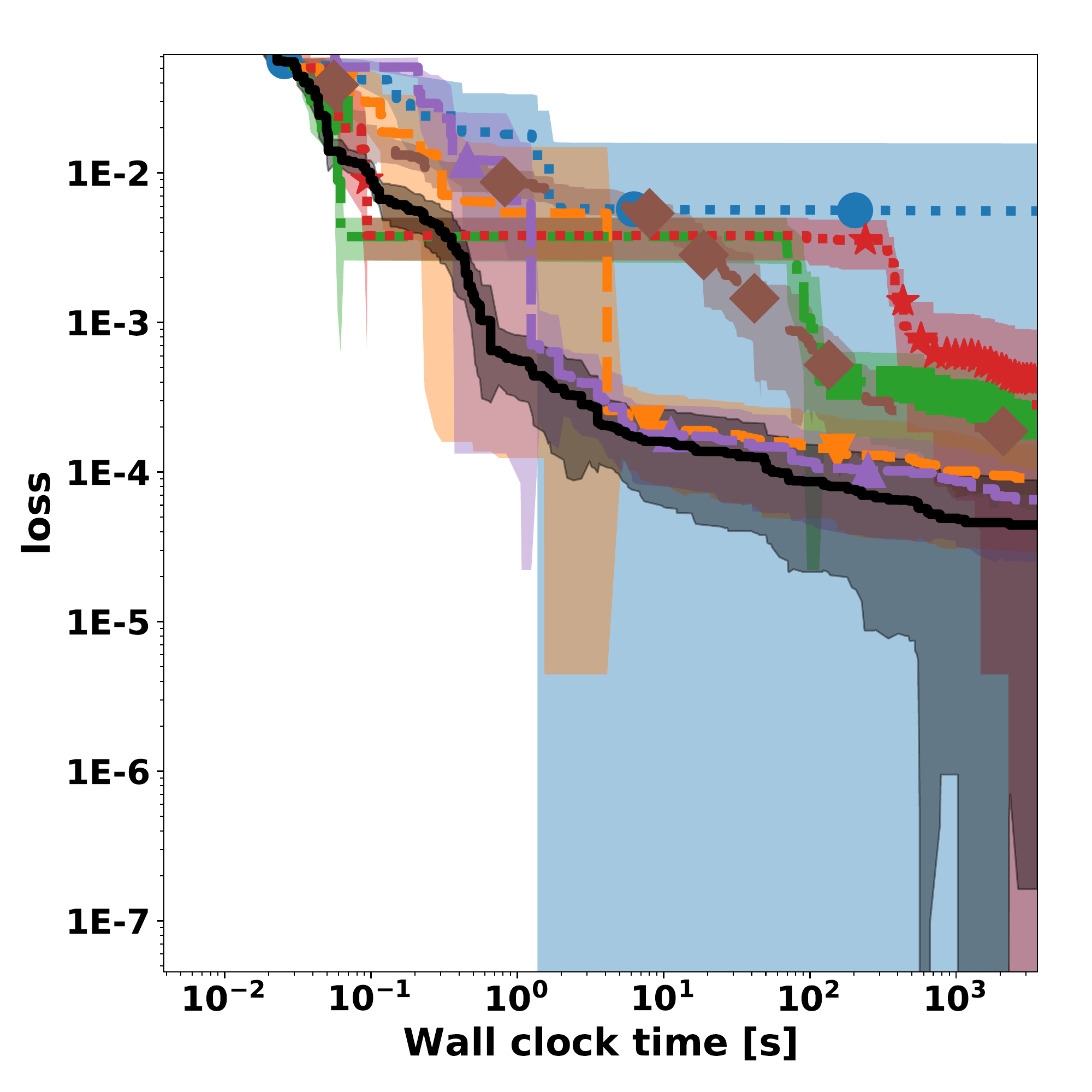}%
\caption{kr-vs-kp}%
\end{subfigure}\hfill%
\begin{subfigure}{0.32\columnwidth}
\includegraphics[width=\columnwidth]{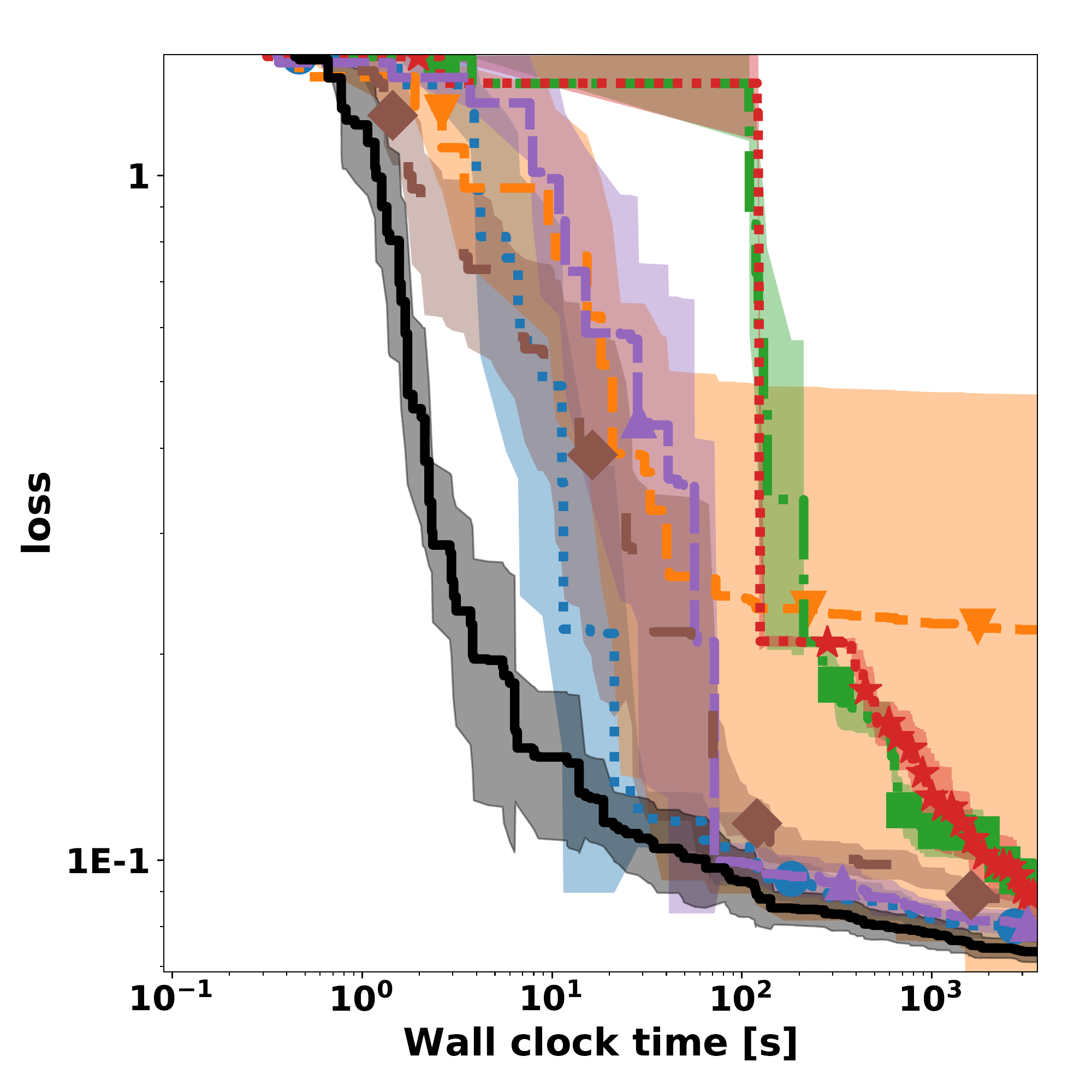}%
\caption{mfeat-factors}%
\end{subfigure}\hfill%
\begin{subfigure}{0.32\columnwidth}
\includegraphics[width=\columnwidth]{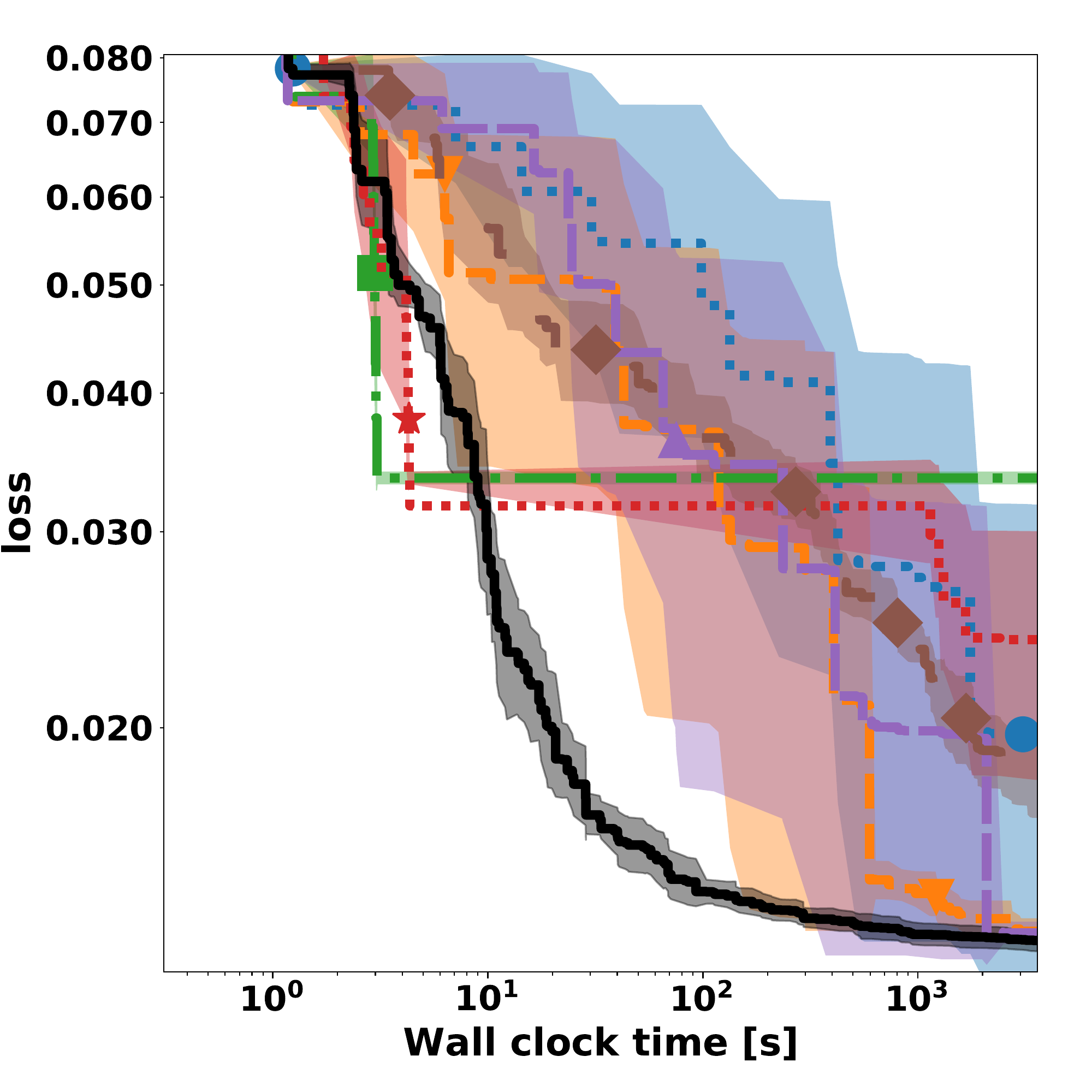}%
\caption{MiniBooNE}%
\end{subfigure}\hfill%
  \begin{subfigure}{0.32\columnwidth}
\includegraphics[width=\columnwidth]{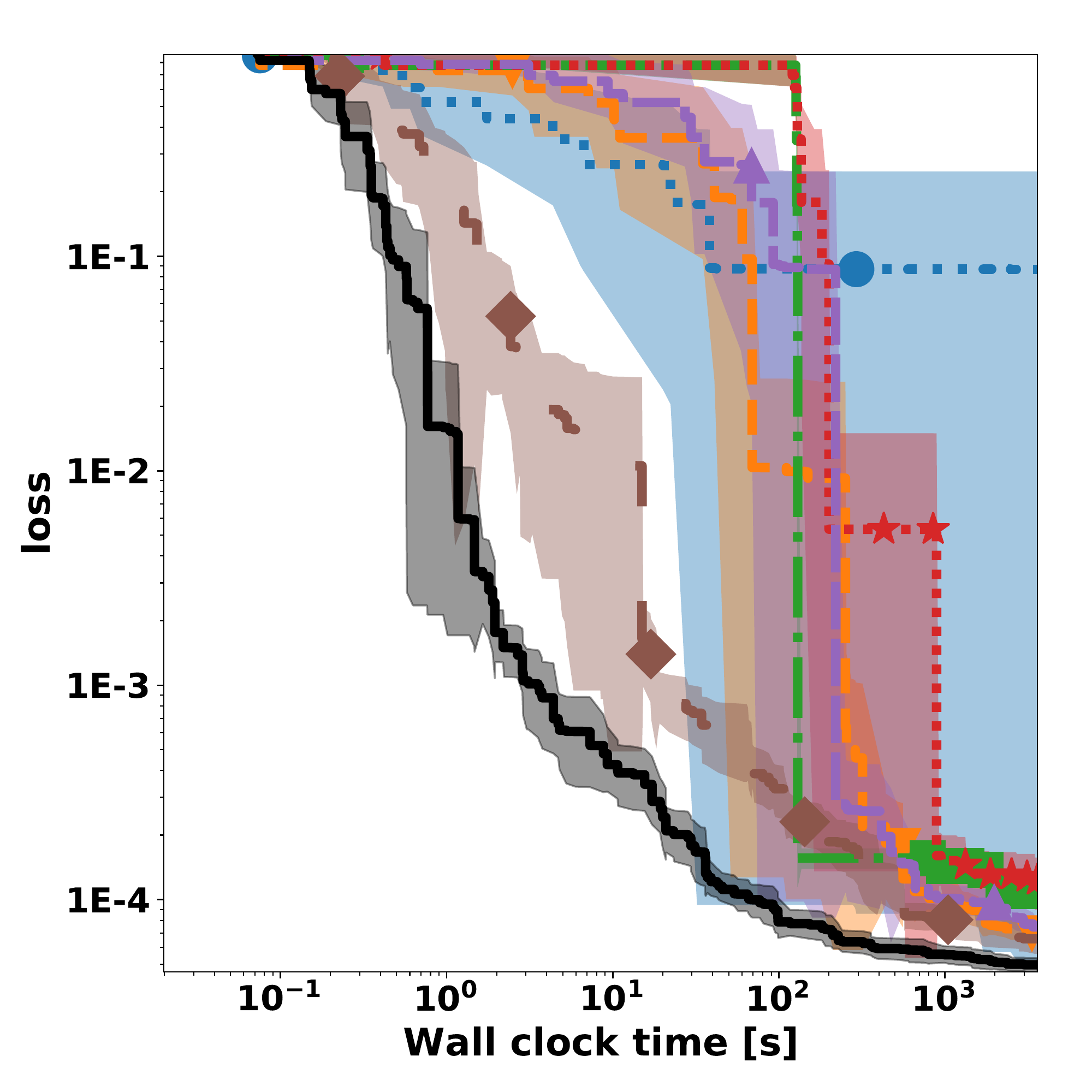}%
\caption{mv}%
\end{subfigure}\hfill%
\caption{Optimization performance curve for XGBoost (pt 3/4)}
\end{figure*}
\begin{figure*}[h]
\ContinuedFloat
\begin{subfigure}{0.32\columnwidth}
\includegraphics[width=\columnwidth]{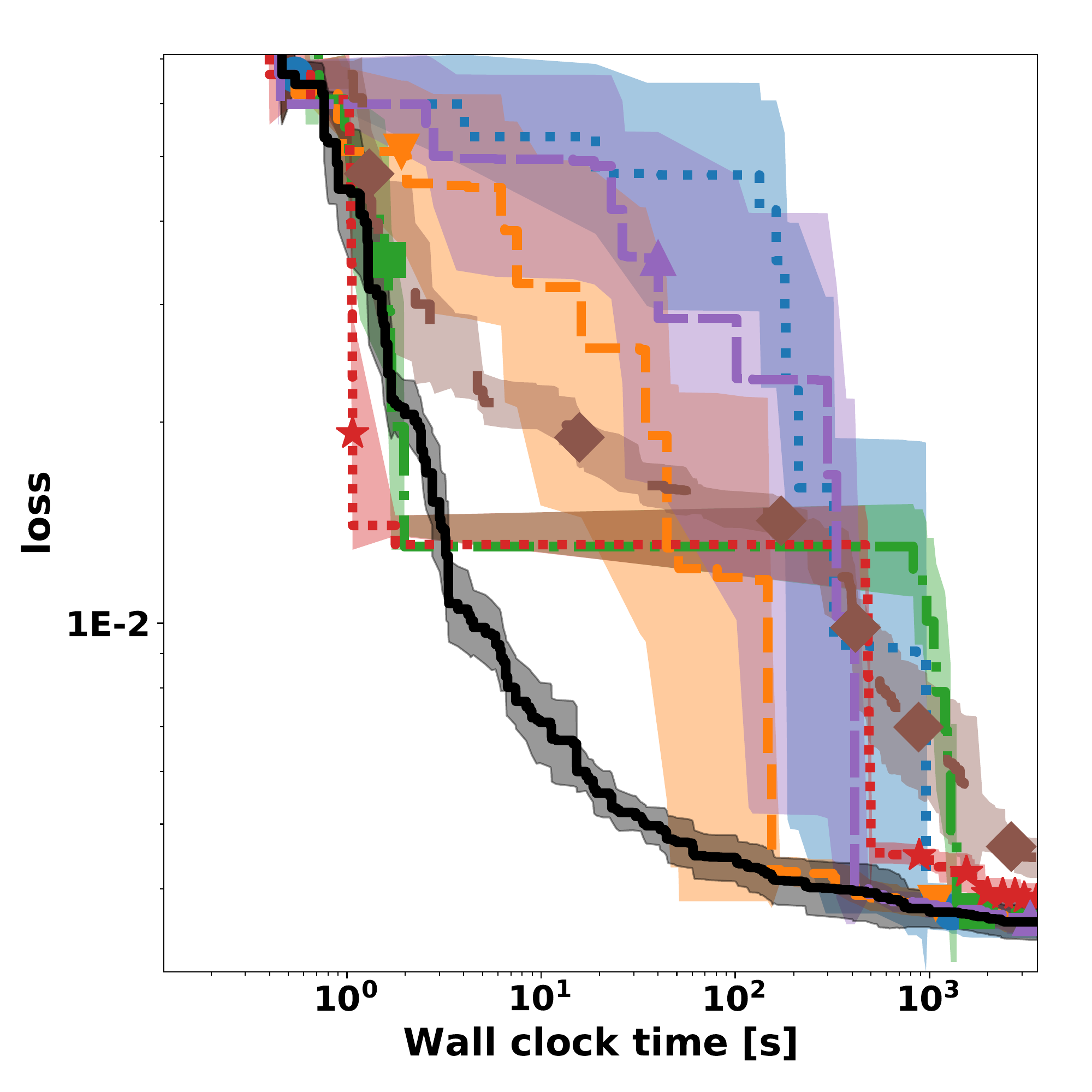}%
\caption{nomao}%
\end{subfigure}\hfill%
\begin{subfigure}{0.32\columnwidth}
\includegraphics[width=\columnwidth]{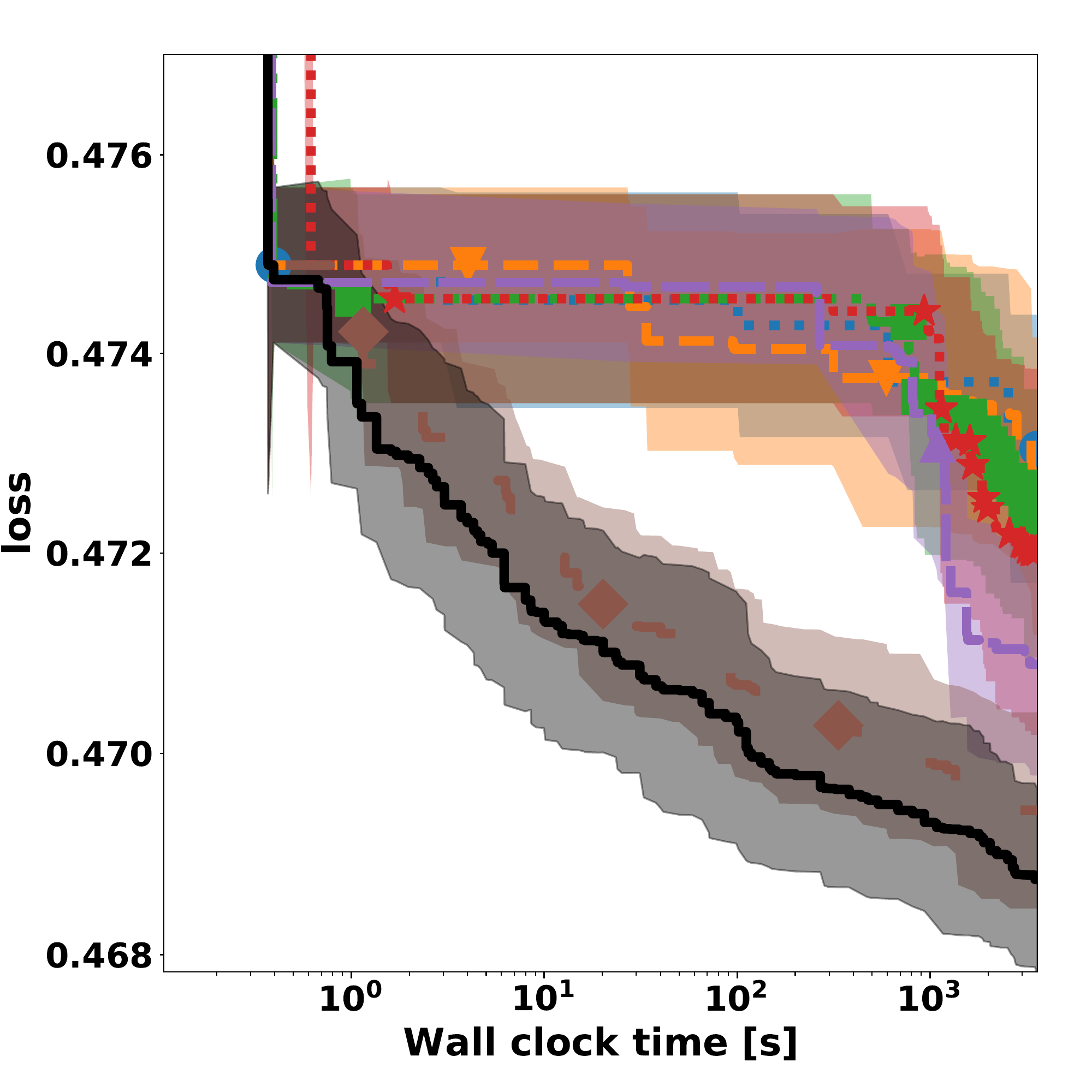}%
\caption{numerai28.6}%
\end{subfigure}\hfill%
  \begin{subfigure}{0.32\columnwidth}
\includegraphics[width=\columnwidth]{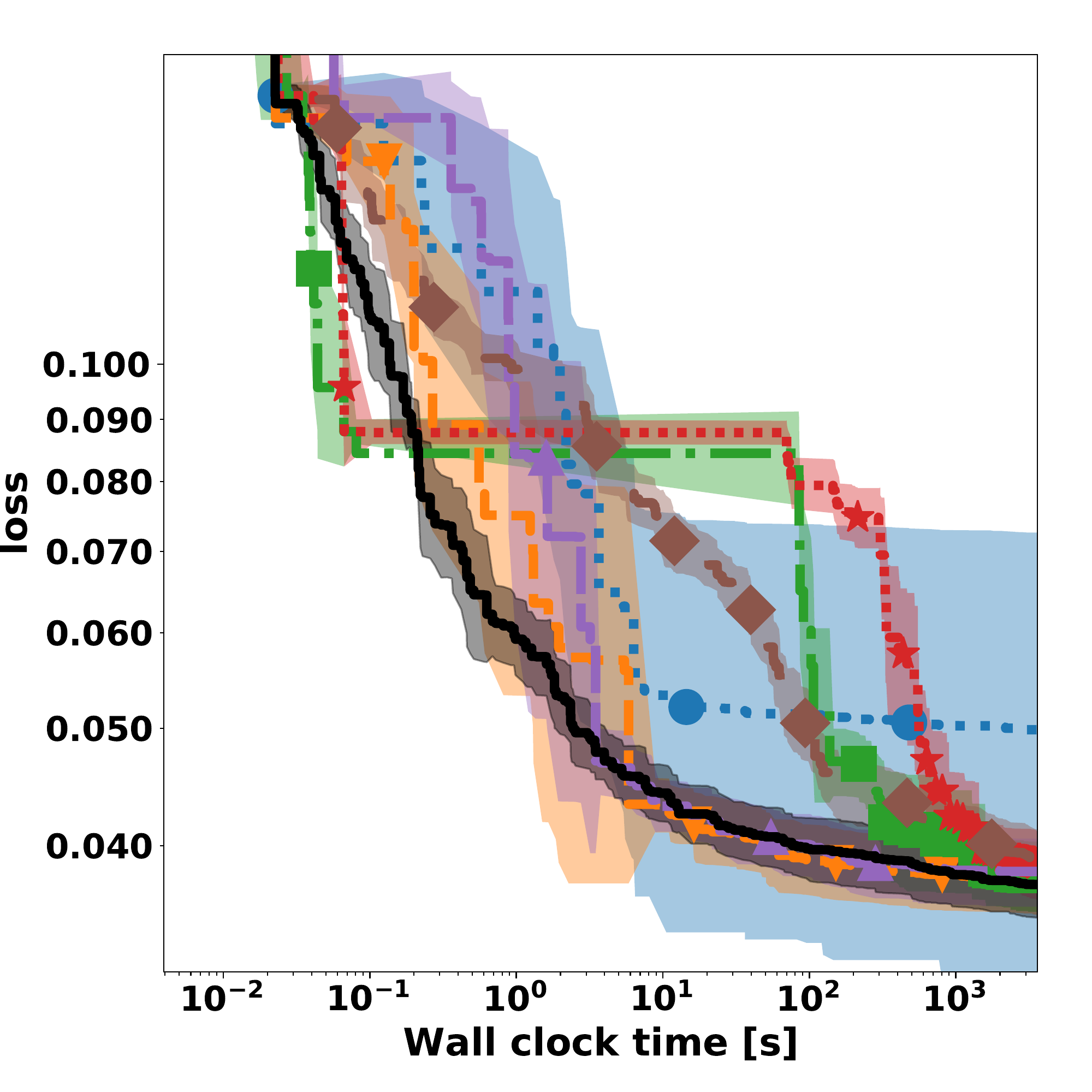}%
\caption{phoneme}%
\end{subfigure}\hfill%
\begin{subfigure}{0.32\columnwidth}
\includegraphics[width=\columnwidth]{figures/xgboost_poker.pdf}%
\caption{poker}%
\end{subfigure}\hfill%
\begin{subfigure}{0.32\columnwidth}
\includegraphics[width=\columnwidth]{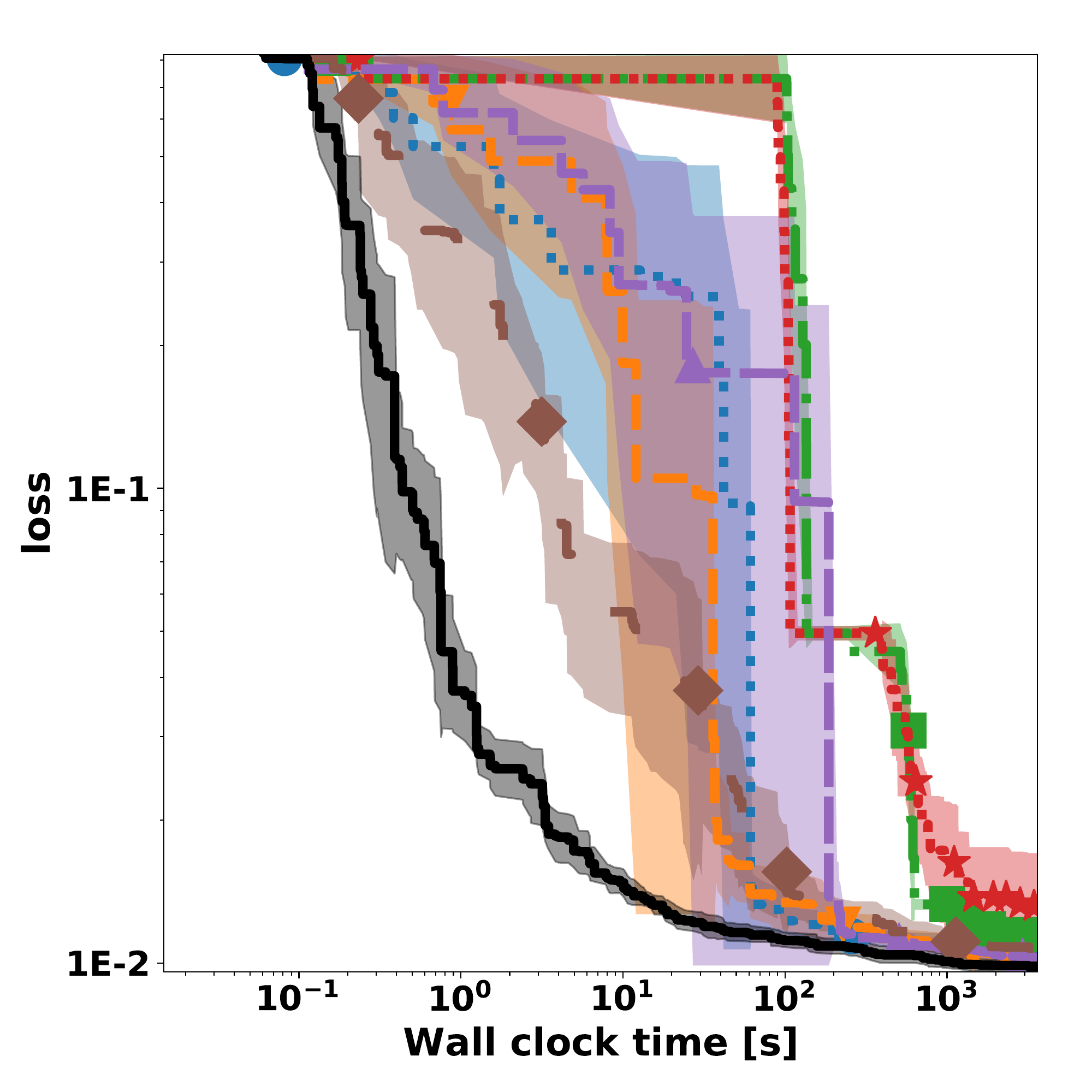}%
\caption{pol}%
\end{subfigure}\hfill%
  \begin{subfigure}{0.32\columnwidth}
\includegraphics[width=\columnwidth]{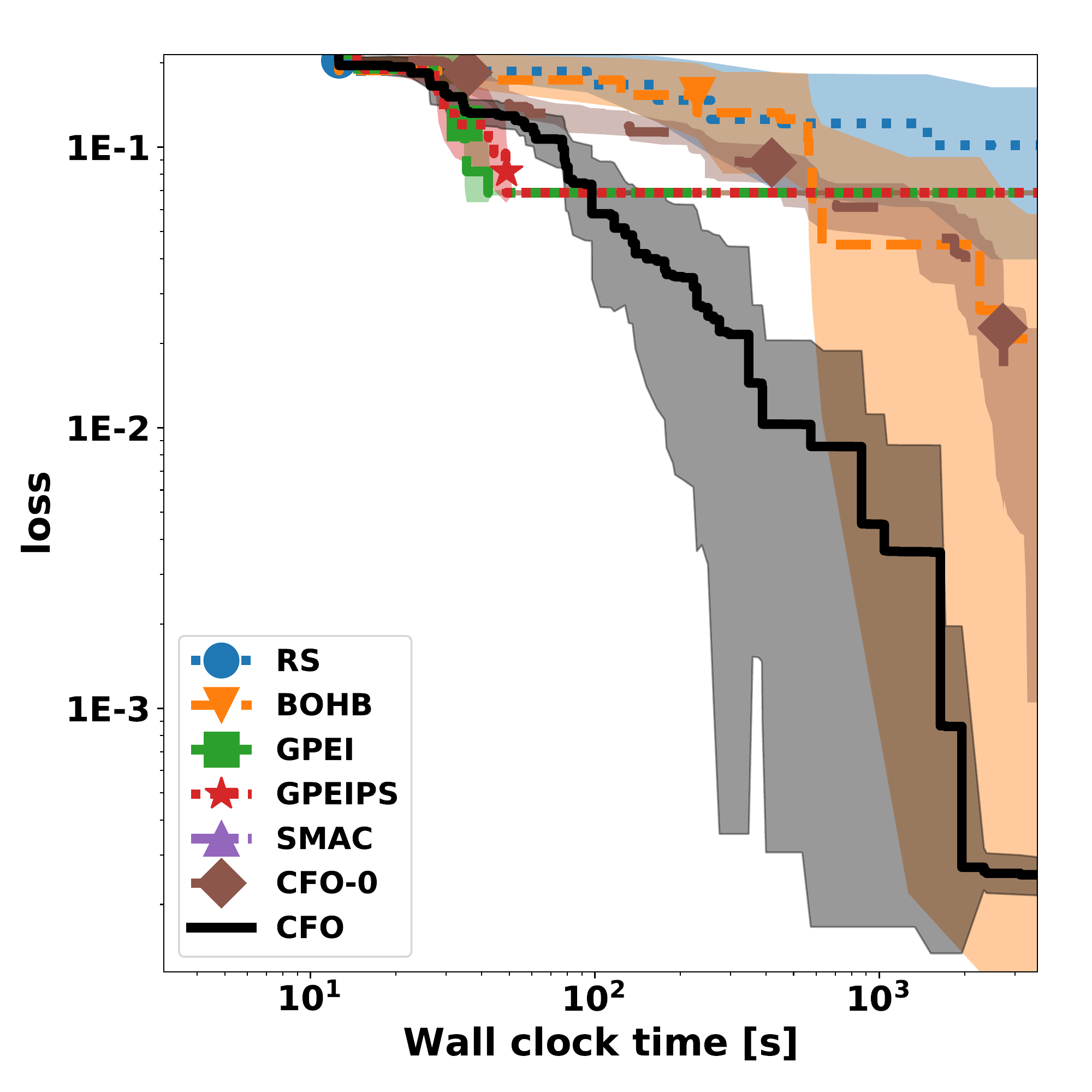}%
\caption{riccardo}%
\end{subfigure}\hfill%
  \begin{subfigure}{0.32\columnwidth}
\includegraphics[width=\columnwidth]{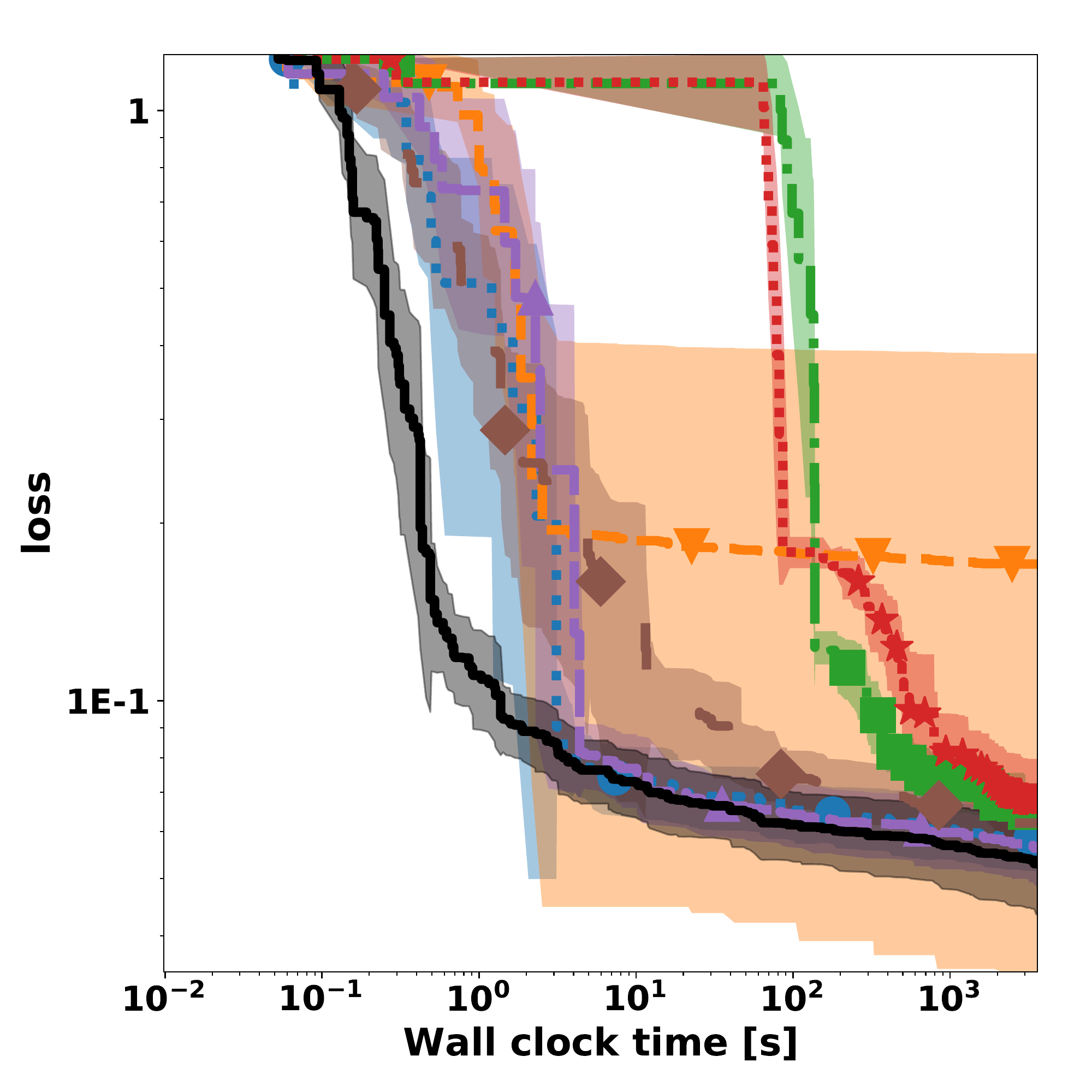}%
\caption{segment}%
\end{subfigure}\hfill%
\begin{subfigure}{0.32\columnwidth}
\includegraphics[width=\columnwidth]{figures/xgboost_shuttle.pdf}%
\caption{shuttle}%
\end{subfigure}\hfill%
\begin{subfigure}{0.32\columnwidth}
\includegraphics[width=\columnwidth]{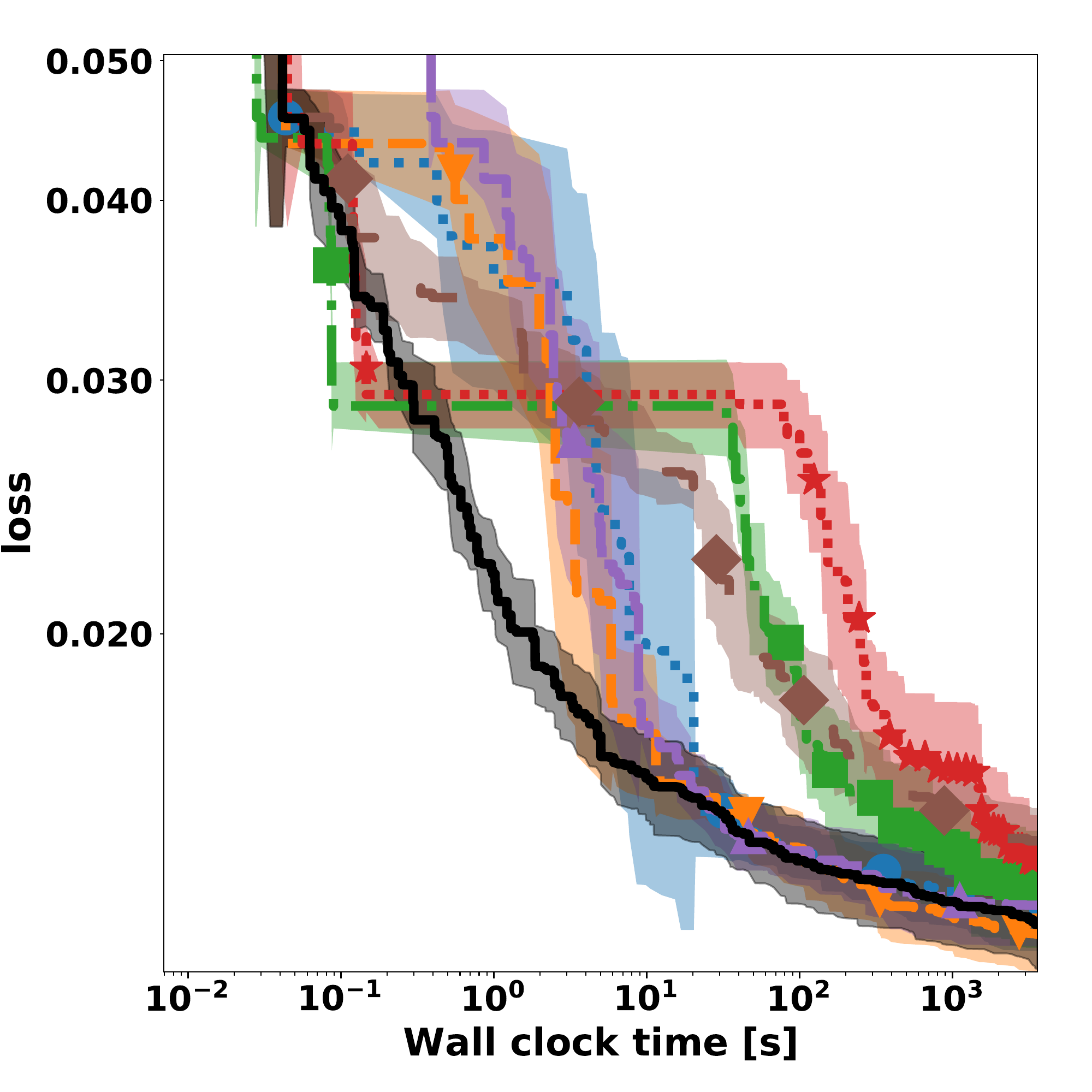}%
\caption{sylvine}%
\end{subfigure}\hfill%
  \begin{subfigure}{0.32\columnwidth}
\includegraphics[width=\columnwidth]{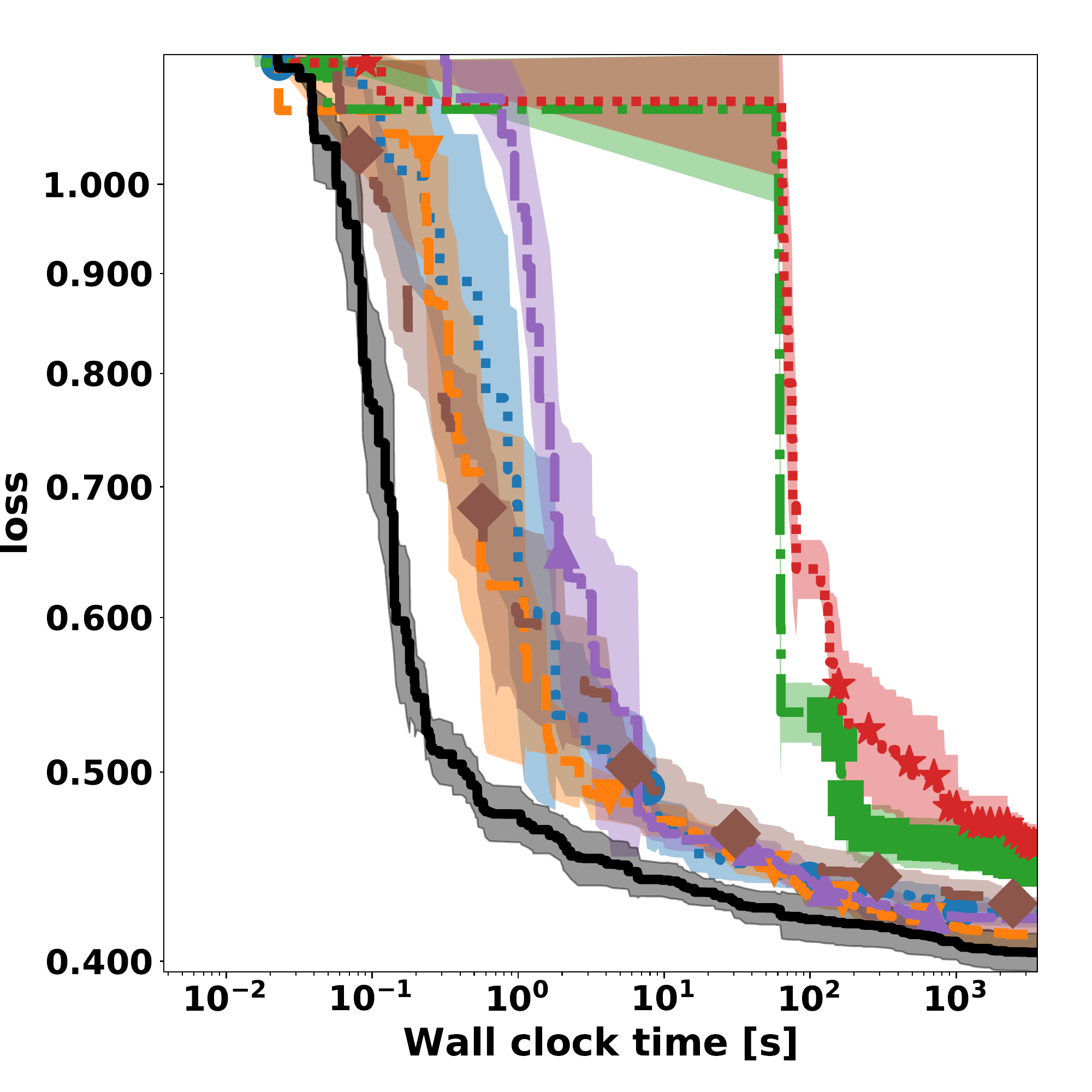}%
\caption{vehicle}%
\end{subfigure}\hfill%
\begin{subfigure}{0.32\columnwidth}
\includegraphics[width=\columnwidth]{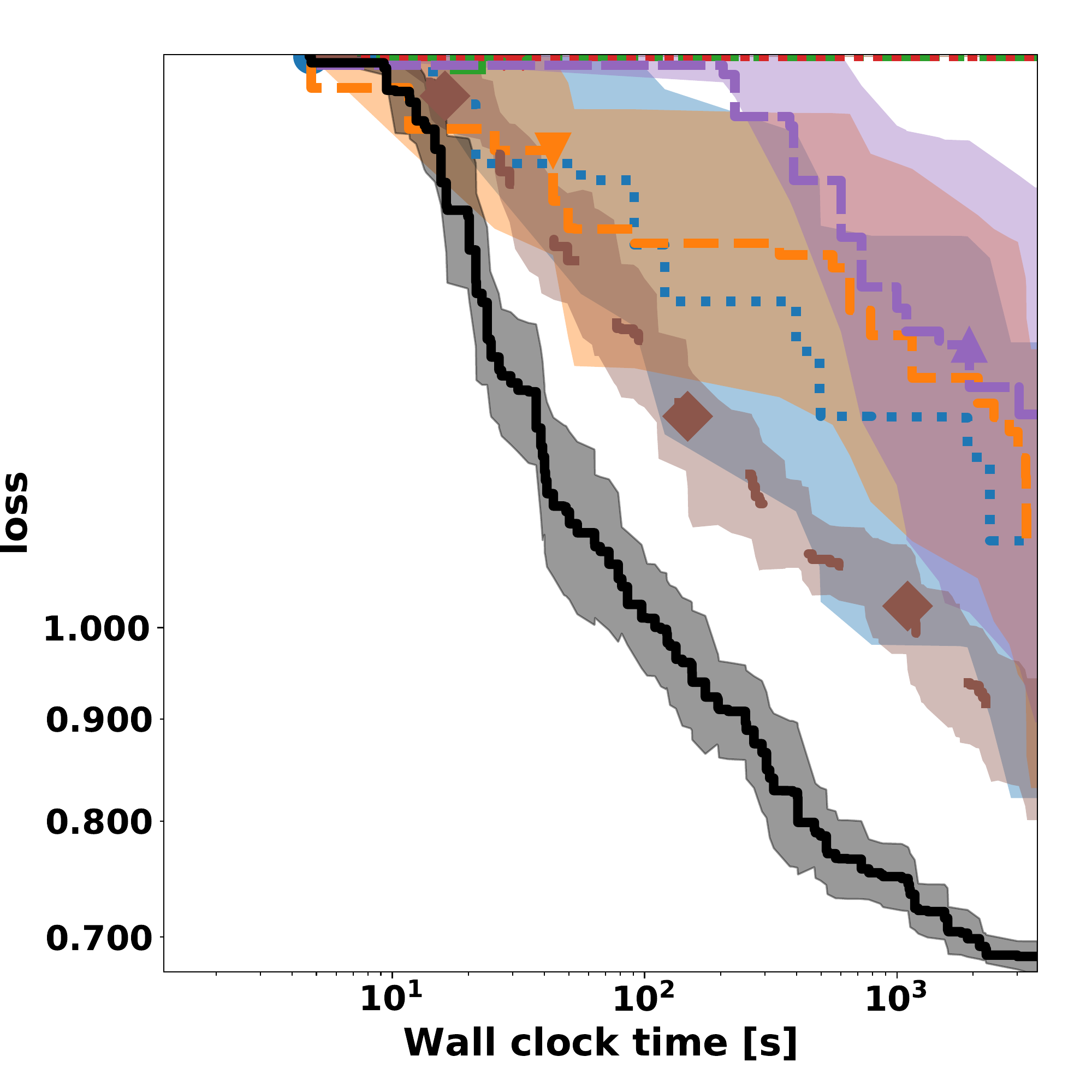}%
\caption{volkert}%
\end{subfigure}\hfill%
\caption{Optimization performance curve for XGBoost (pt 4/4)}
  \label{fig:curve}
\end{figure*}

\begin{figure*}[h]
  \centering
\begin{subfigure}{0.3\columnwidth}
\includegraphics[width=\columnwidth]{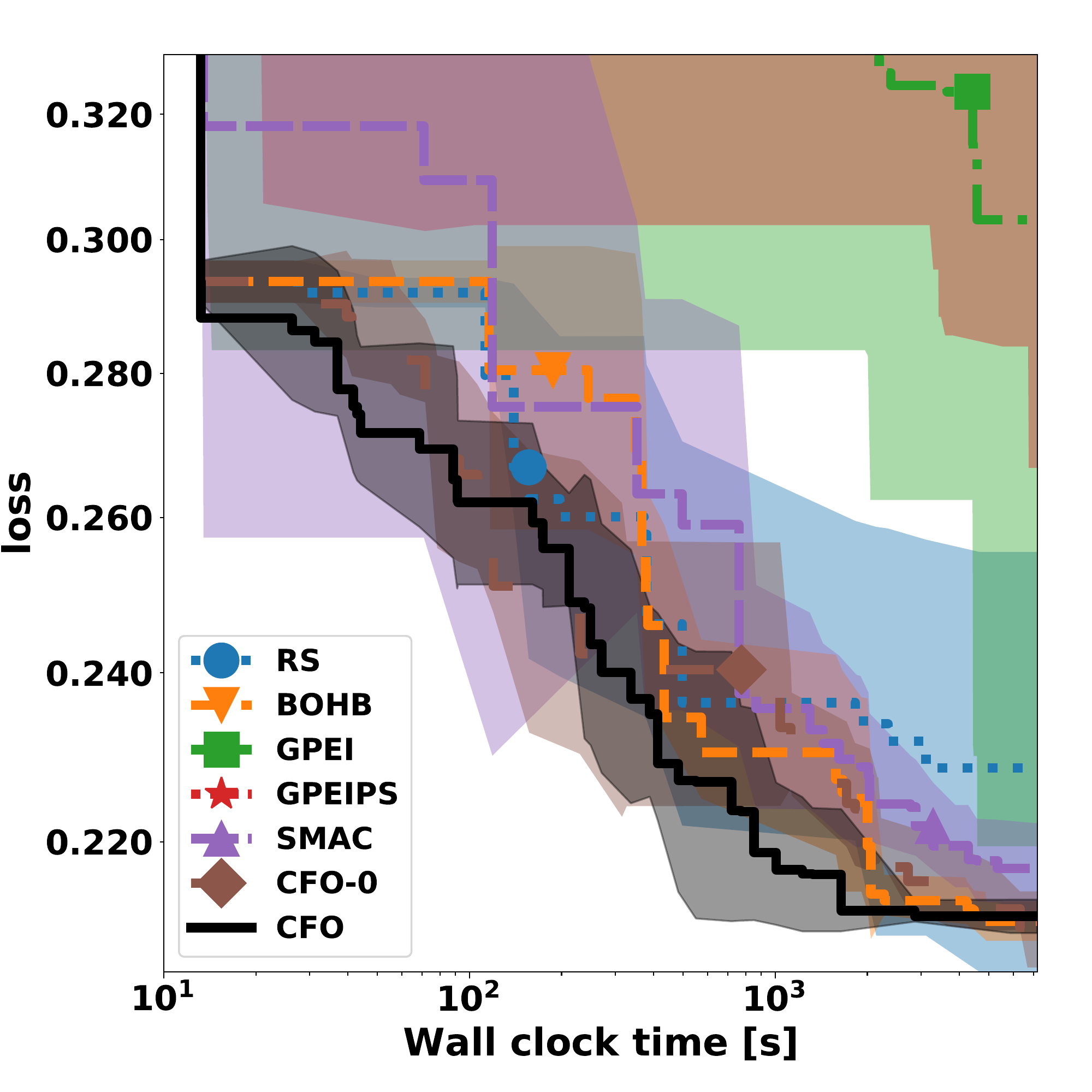}%
\caption{adult}%
\end{subfigure}\hfill%
\begin{subfigure}{0.32\columnwidth}
\includegraphics[width=\columnwidth]{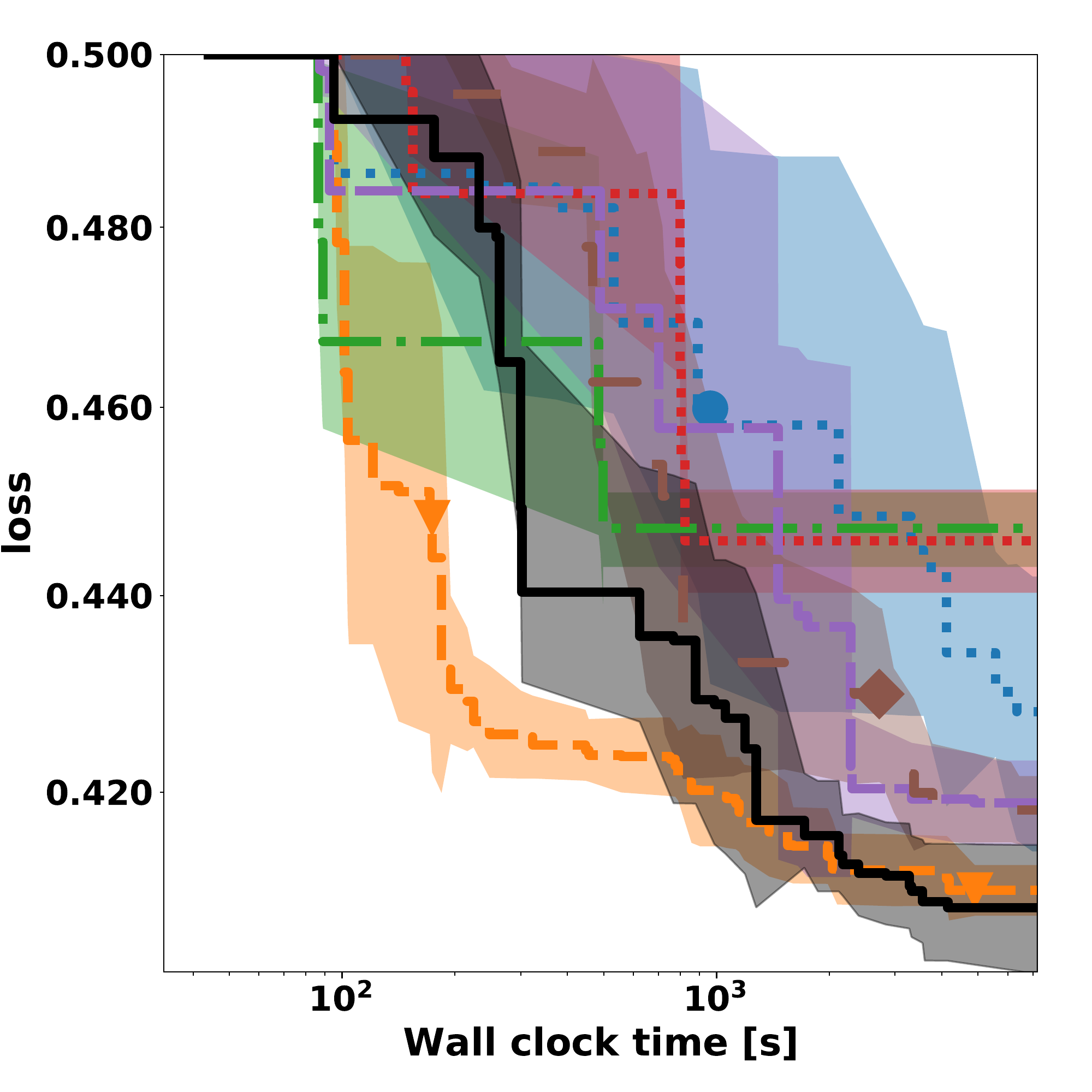}%
\caption{Airlines}%
\end{subfigure}\hfill%
  \begin{subfigure}{0.32\columnwidth}
\includegraphics[width=\columnwidth]{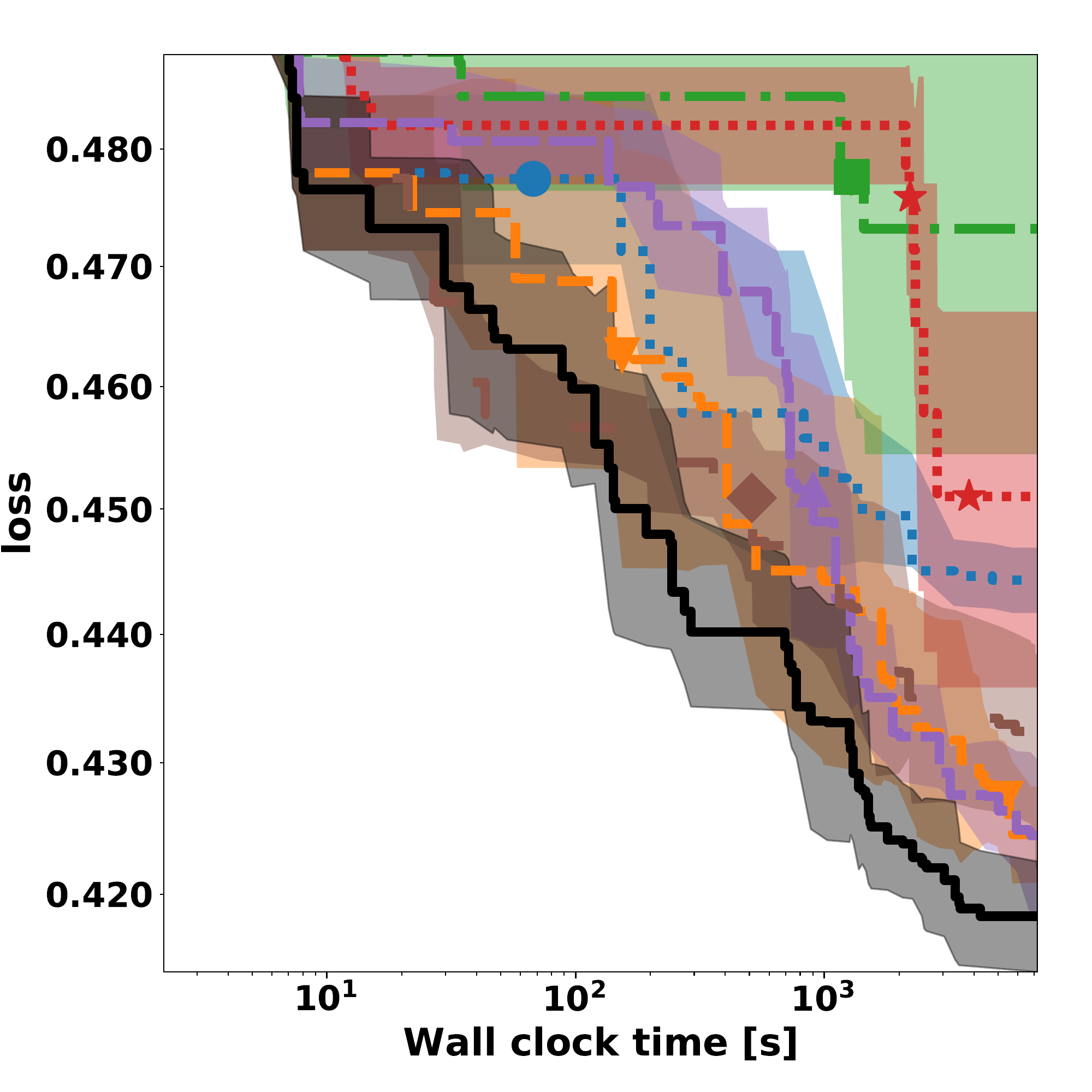}%
\caption{Amazon\_employee\_access}%
\end{subfigure}\hfill%
\begin{subfigure}{0.32\columnwidth}
\includegraphics[width=\columnwidth]{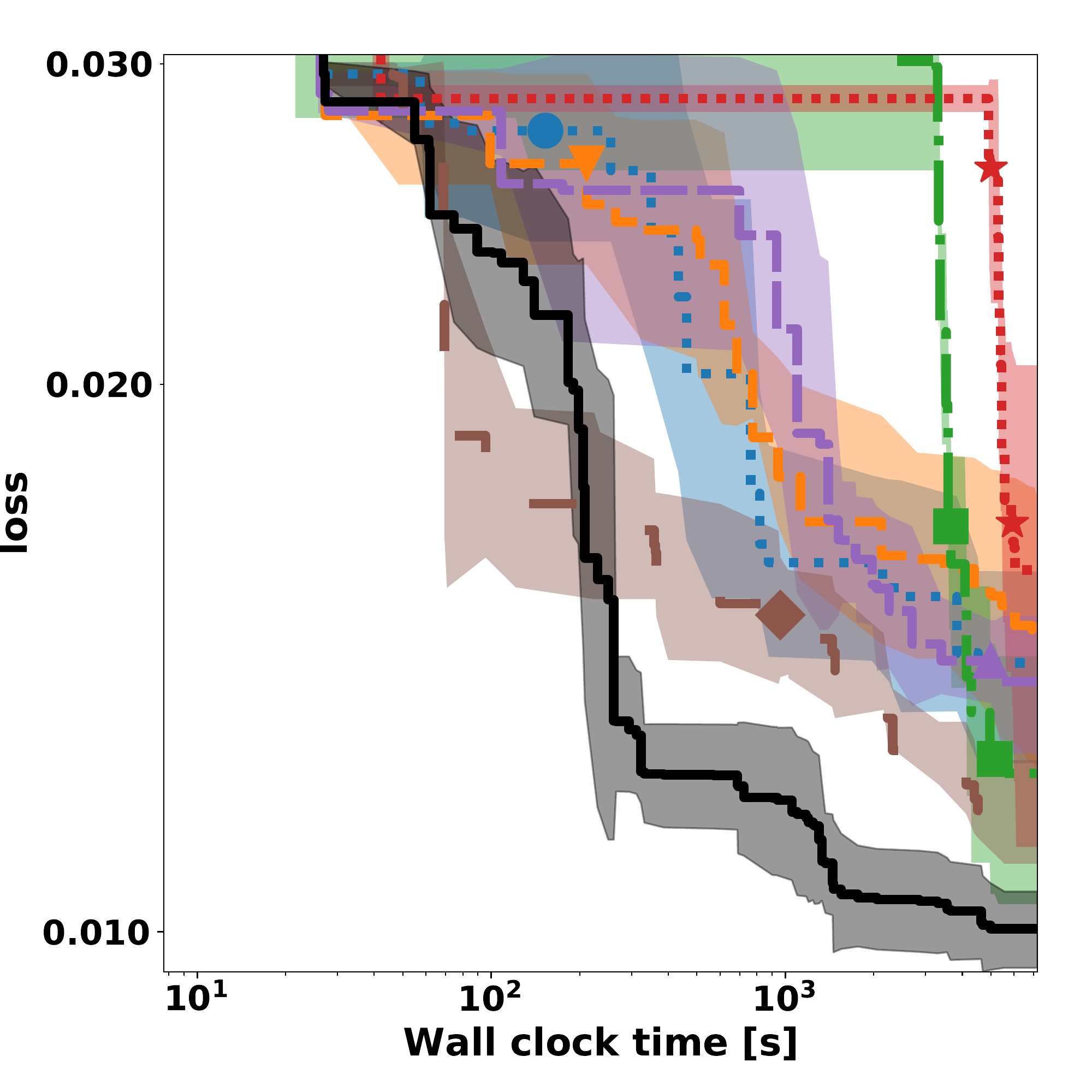}%
\caption{APSFailure}%
\end{subfigure}\hfill%
\begin{subfigure}{0.32\columnwidth}
\includegraphics[width=\columnwidth]{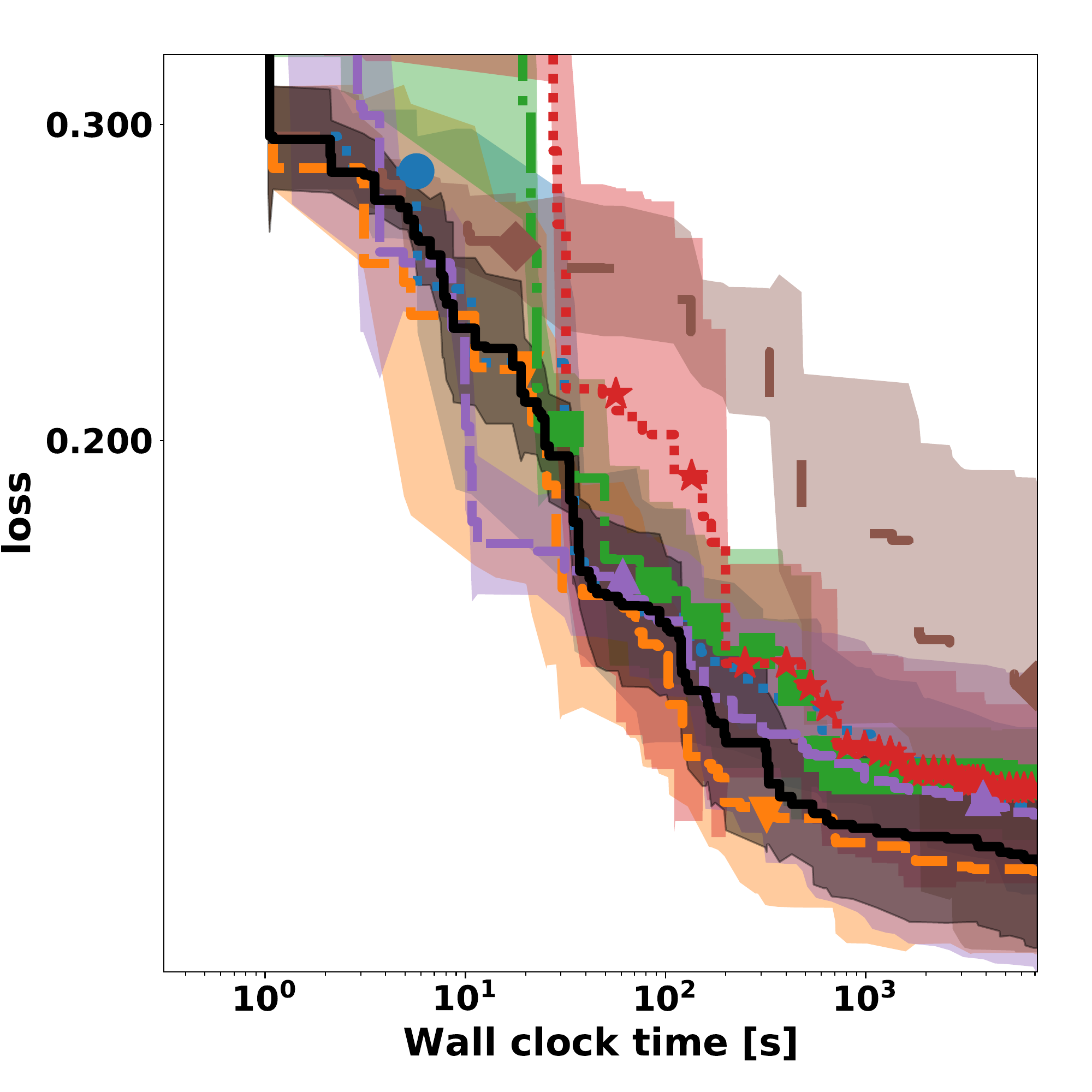}%
\caption{Australian}%
\end{subfigure}\hfill%
\begin{subfigure}{0.32\columnwidth}
\includegraphics[width=\columnwidth]{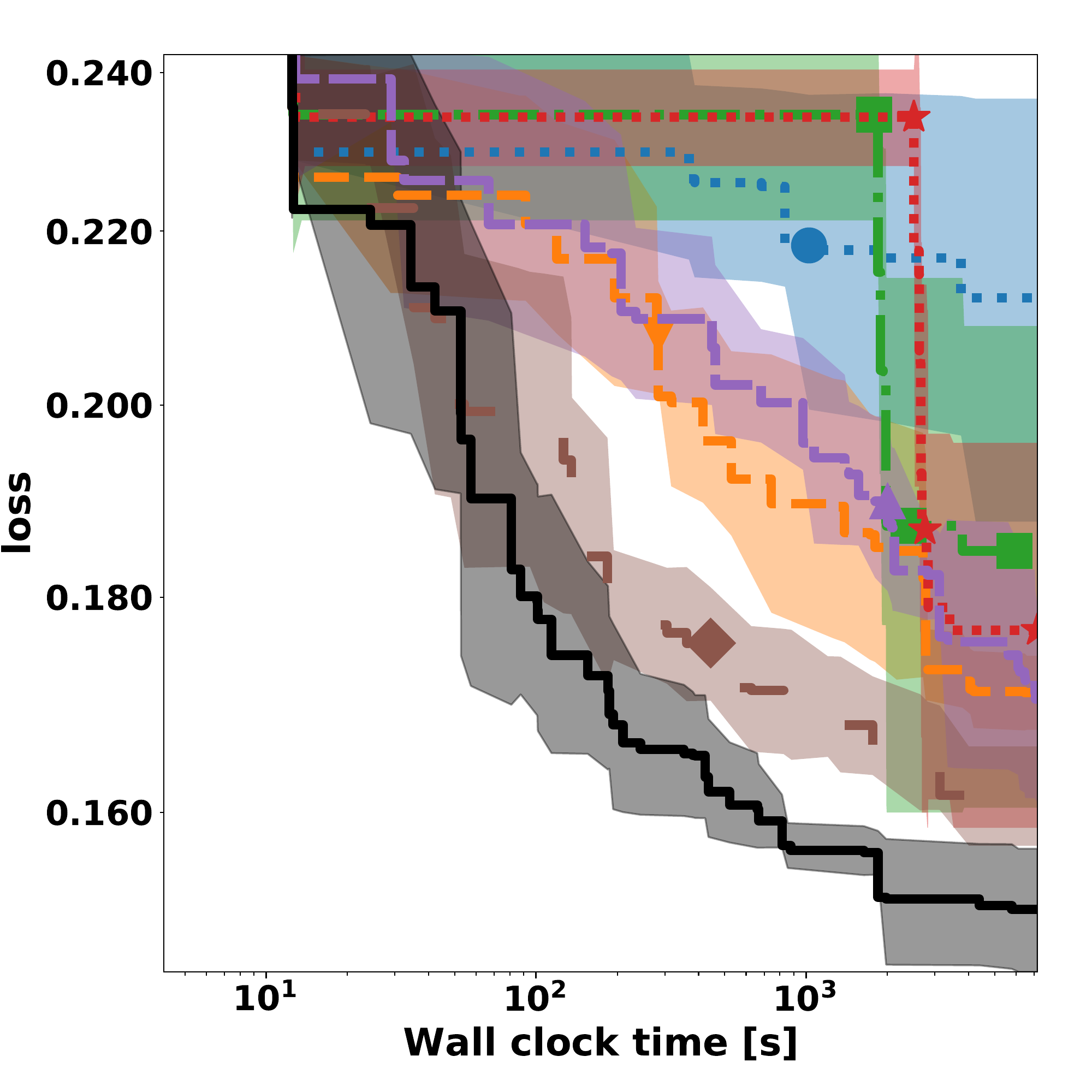}%
\caption{bank\_marketing}%
\end{subfigure}\hfill%
  \begin{subfigure}{0.32\columnwidth}
\includegraphics[width=\columnwidth]{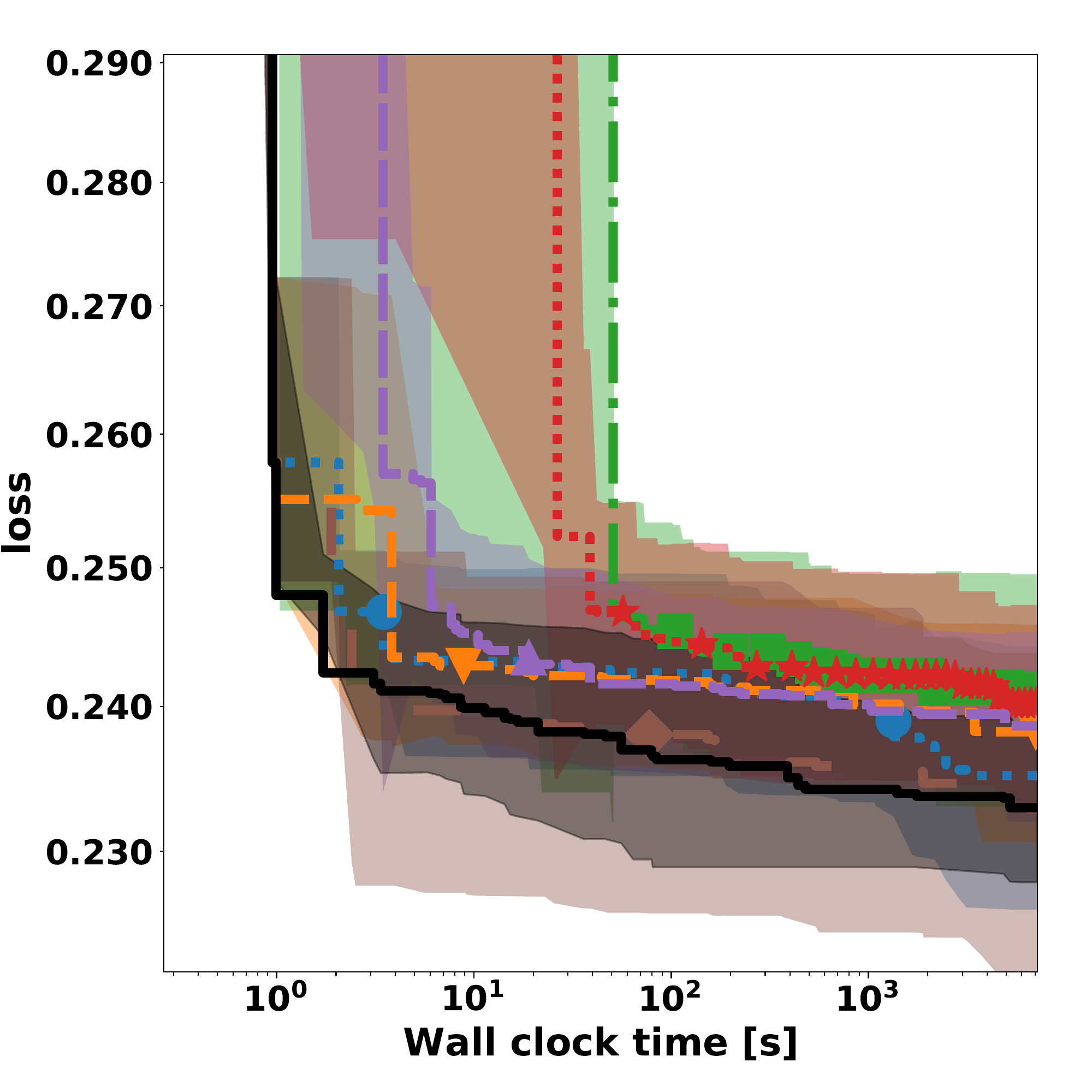}%
\caption{blood-transfusion}%
\end{subfigure}\hfill%
\begin{subfigure}{0.32\columnwidth}
\includegraphics[width=\columnwidth]{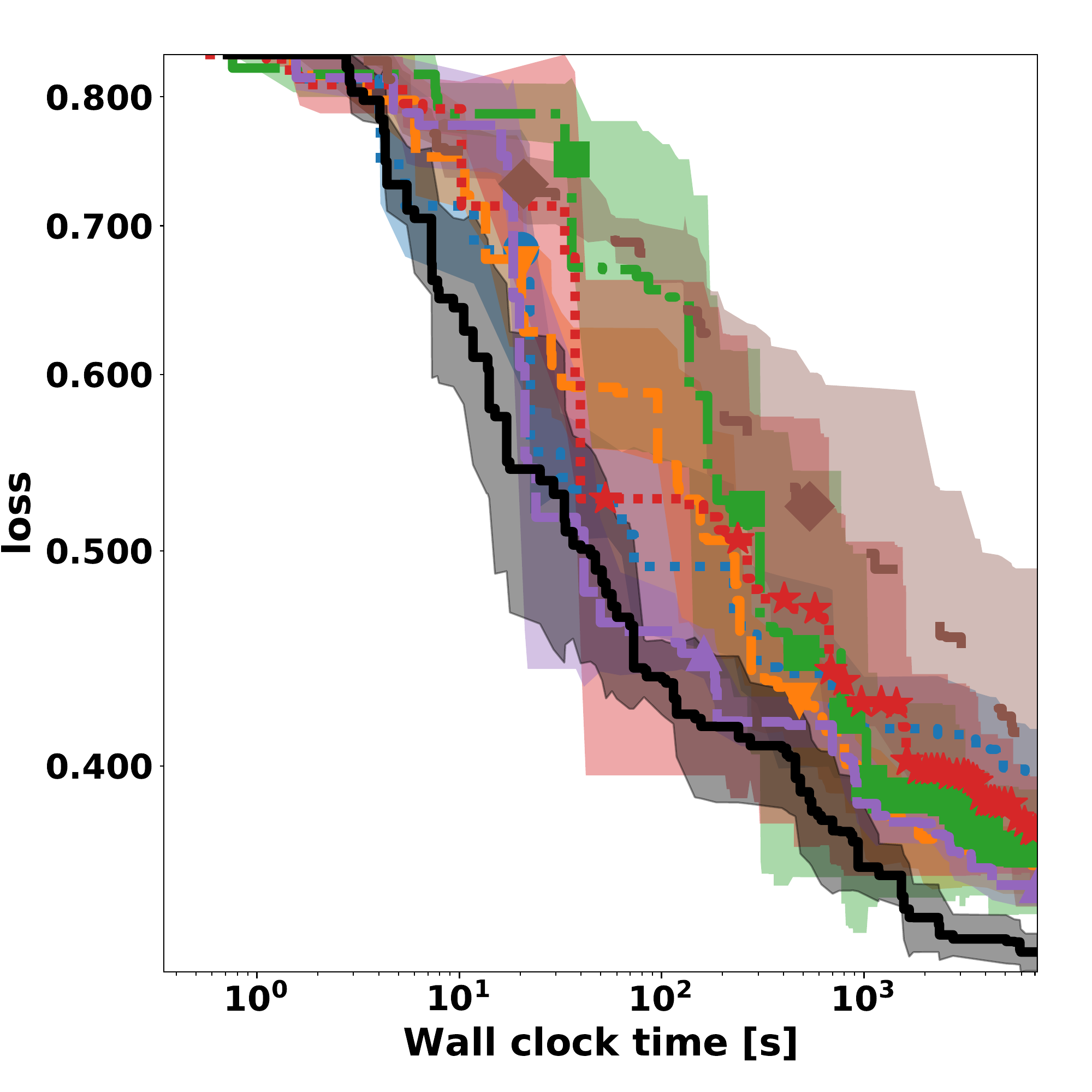}%
\caption{car}%
\end{subfigure}\hfill%
\begin{subfigure}{0.32\columnwidth}
\includegraphics[width=\columnwidth]{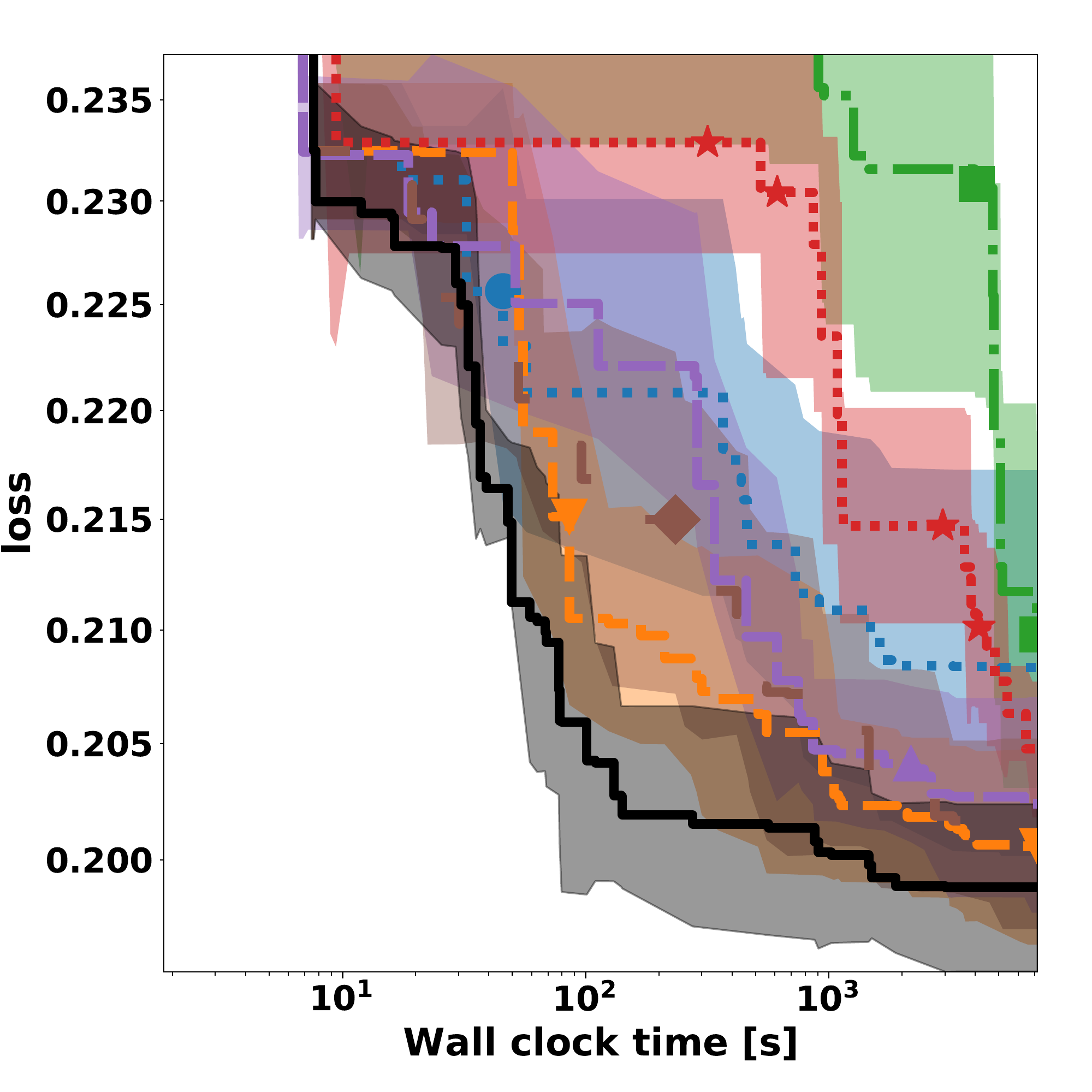}%
\caption{christine}%
\end{subfigure}\hfill
  \begin{subfigure}{0.32\columnwidth}
\includegraphics[width=\columnwidth]{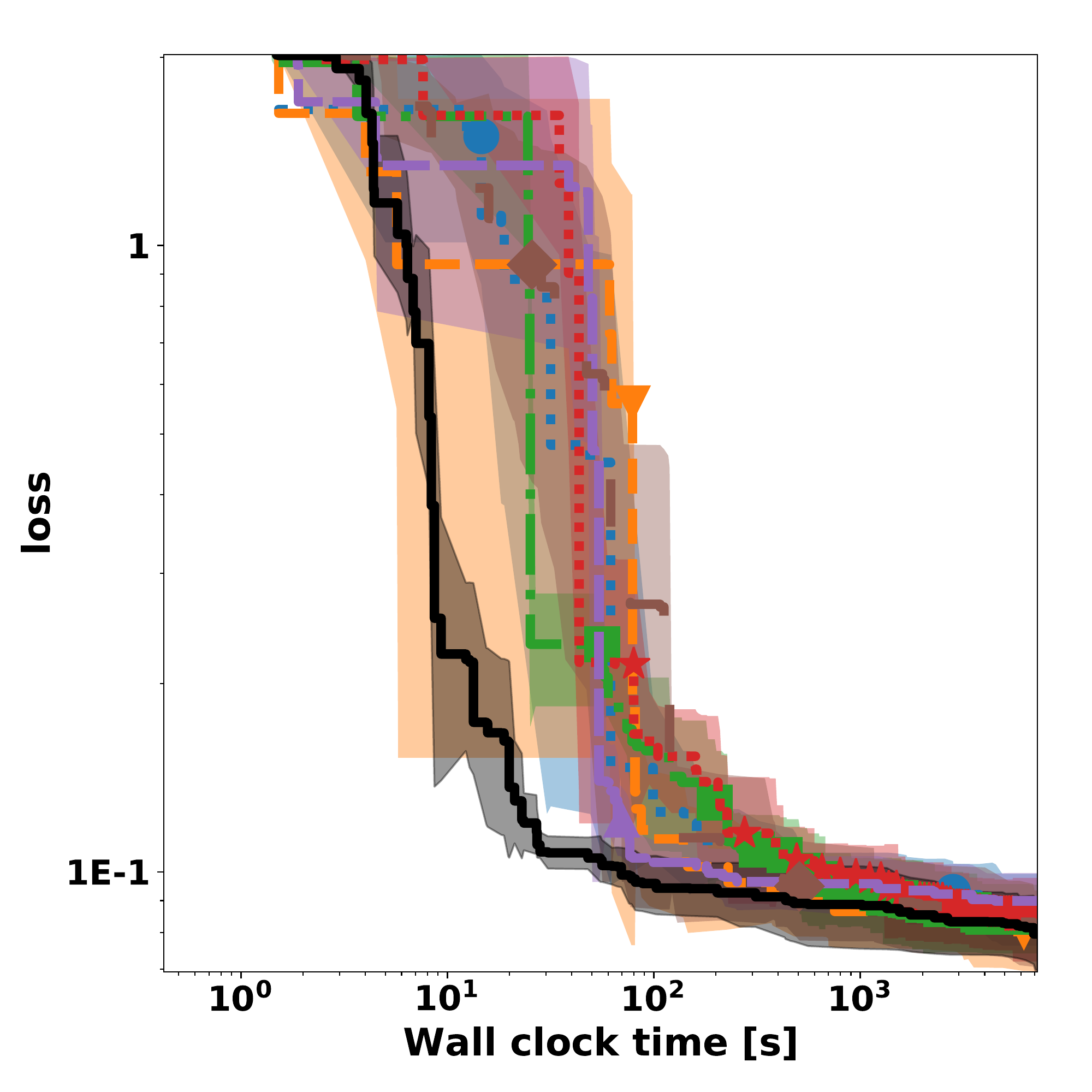}%
\caption{cnae-9}%
\end{subfigure}\hfill%
\begin{subfigure}{0.32\columnwidth}
\includegraphics[width=\columnwidth]{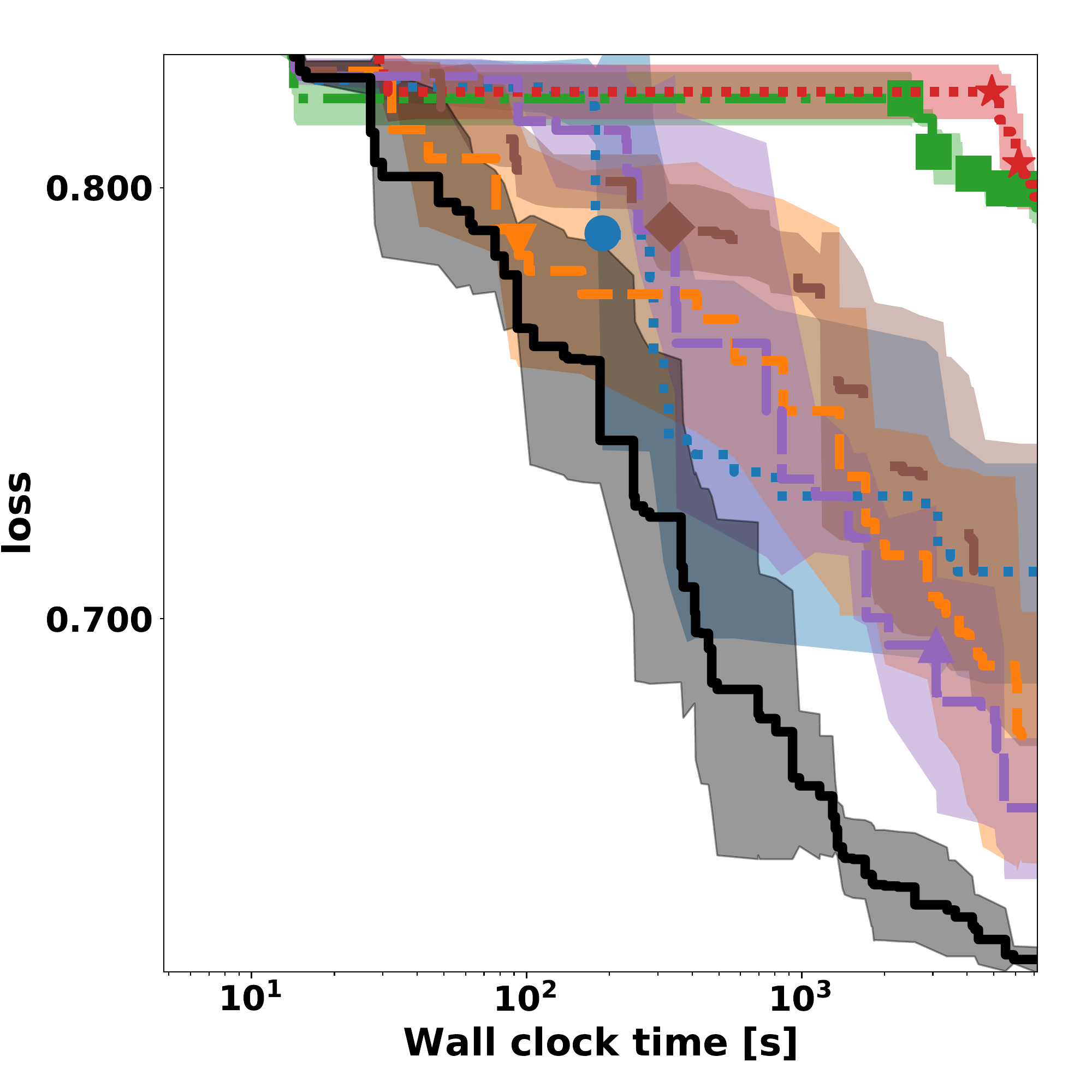}%
\caption{connect-4}%
\end{subfigure}\hfill%
\begin{subfigure}{0.32\columnwidth}
\includegraphics[width=\columnwidth]{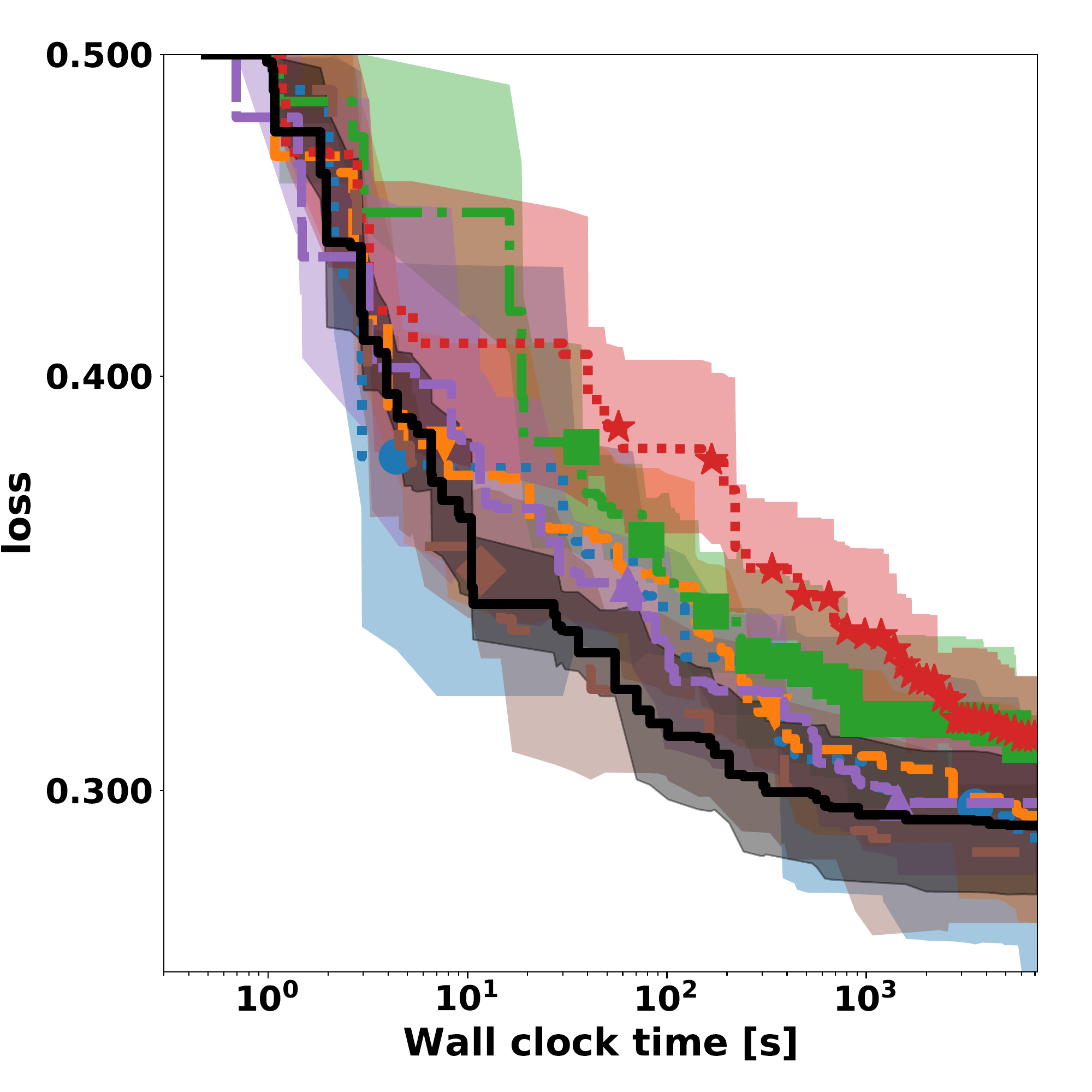}%
\caption{credit-g}%
\end{subfigure}\hfill%
\caption{Optimization performance curve for DNN (pt 1/4)}
\end{figure*}
\begin{figure*}[h]
\ContinuedFloat
  \begin{subfigure}{0.32\columnwidth}
\includegraphics[width=\columnwidth]{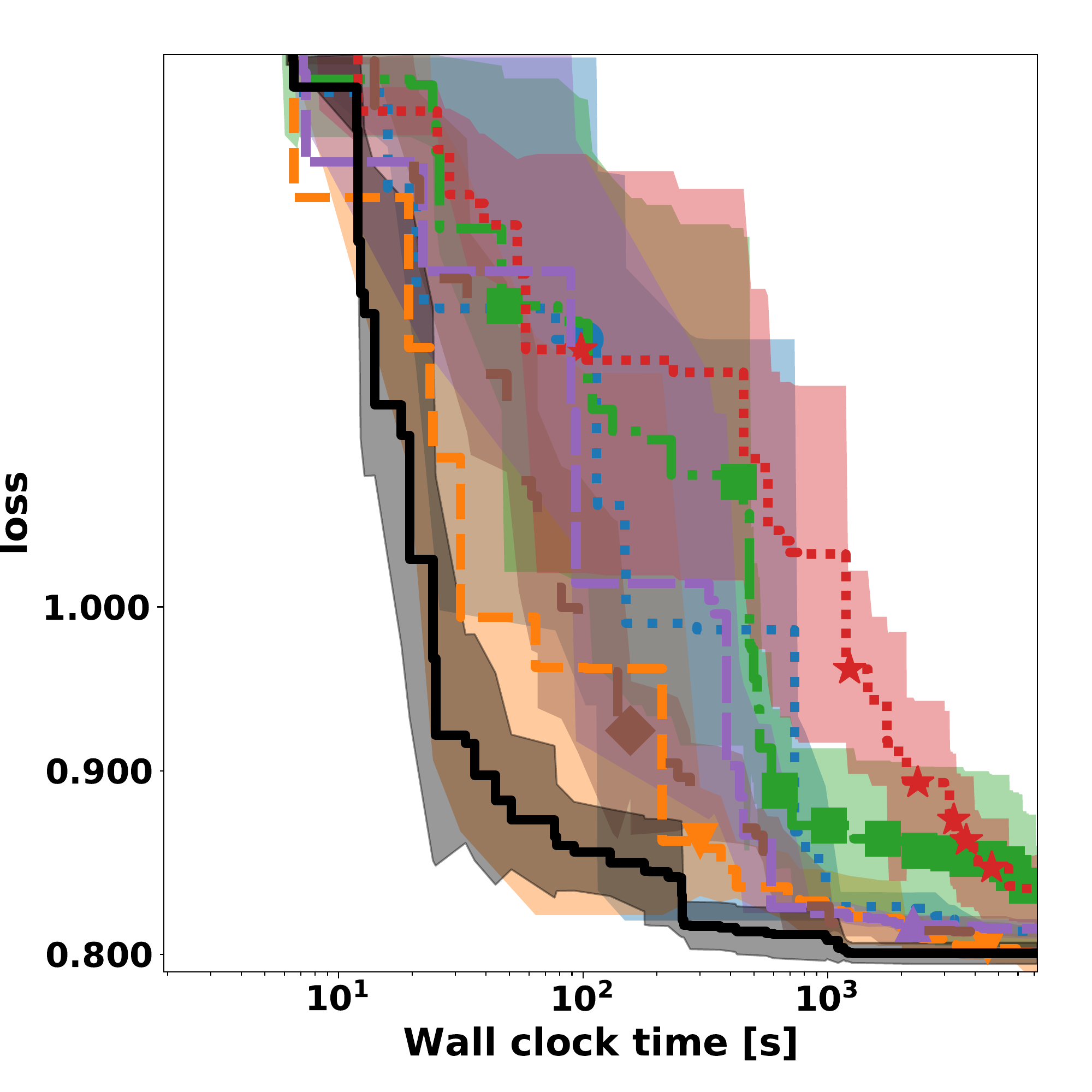}%
\caption{fabert}%
\end{subfigure}\hfill%
\begin{subfigure}{0.32\columnwidth}
\includegraphics[width=\columnwidth]{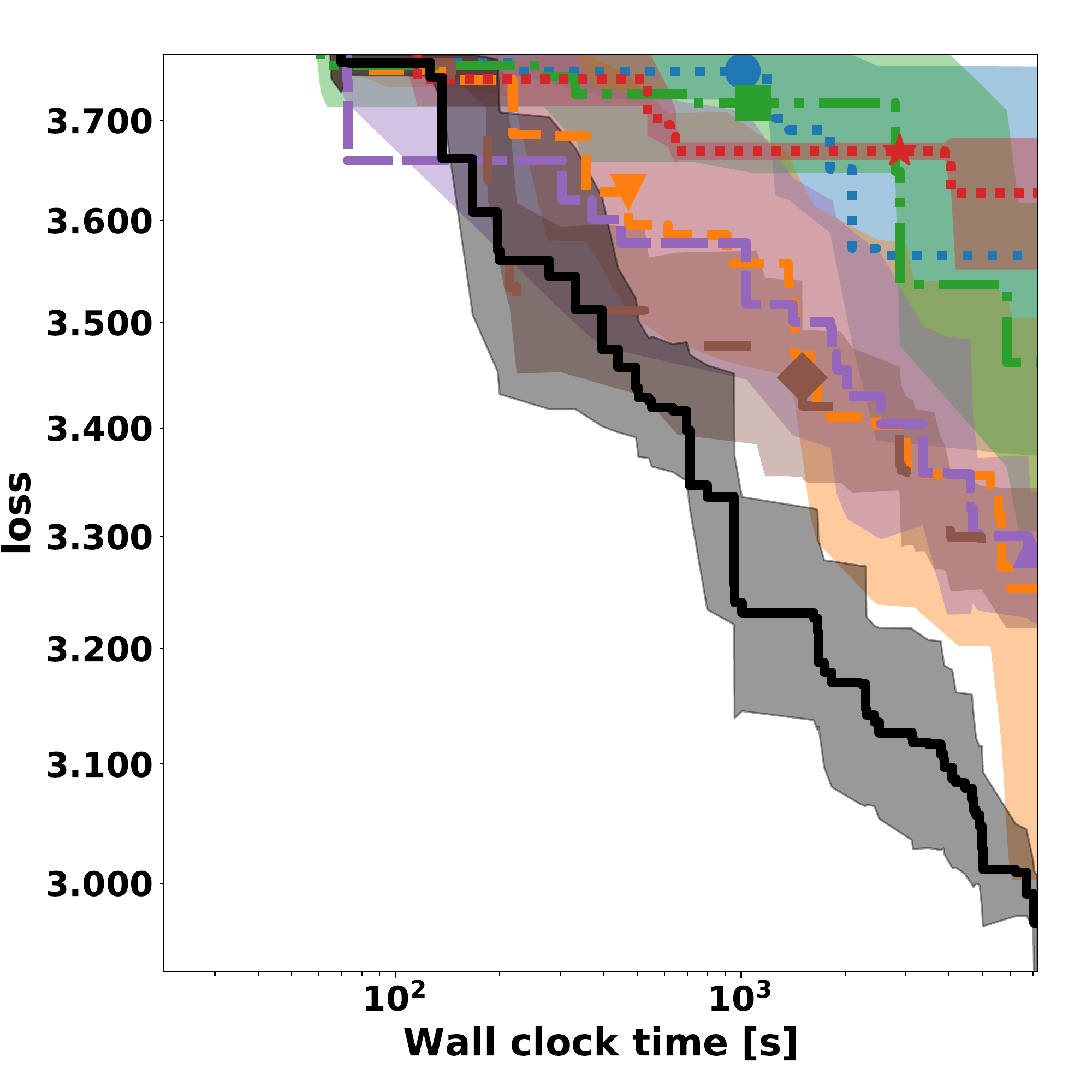}%
\caption{Helena}%
\end{subfigure}\hfill%
  \begin{subfigure}{0.32\columnwidth}
\includegraphics[width=\columnwidth]{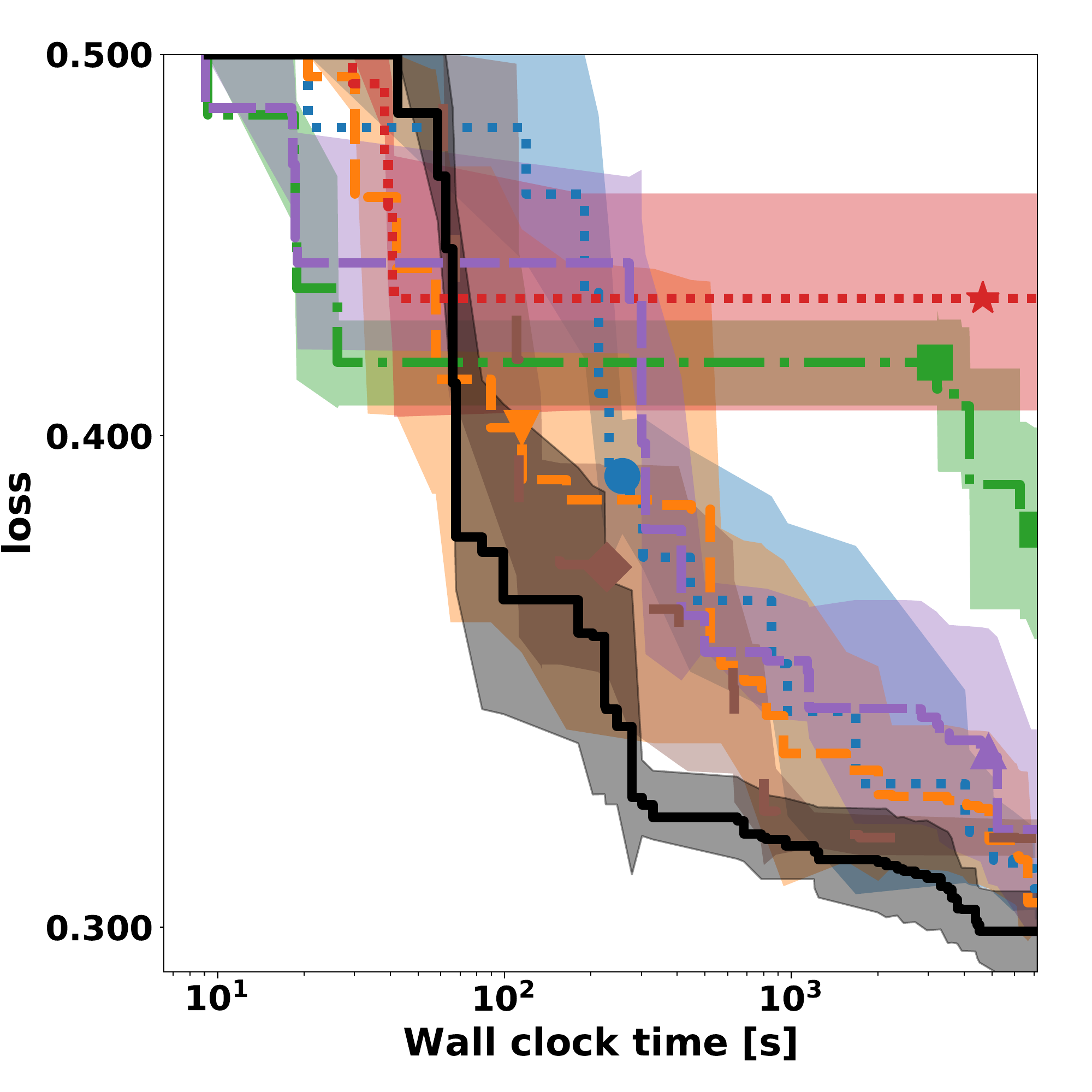}%
\caption{higgs}%
\end{subfigure}\hfill%
\begin{subfigure}{0.32\columnwidth}
\includegraphics[width=\columnwidth]{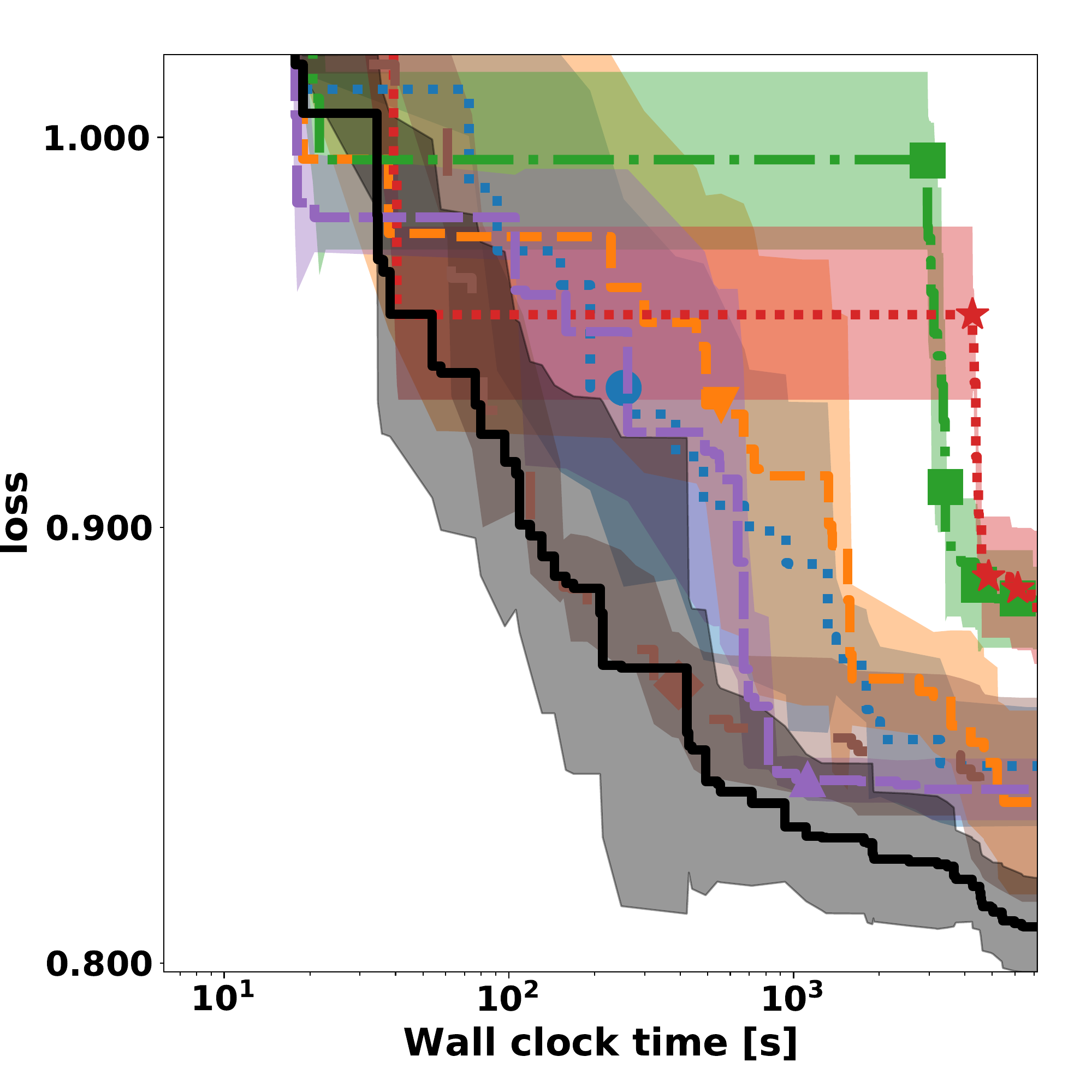}%
\caption{Jannis}%
\end{subfigure}\hfill%
\begin{subfigure}{0.32\columnwidth}
\includegraphics[width=\columnwidth]{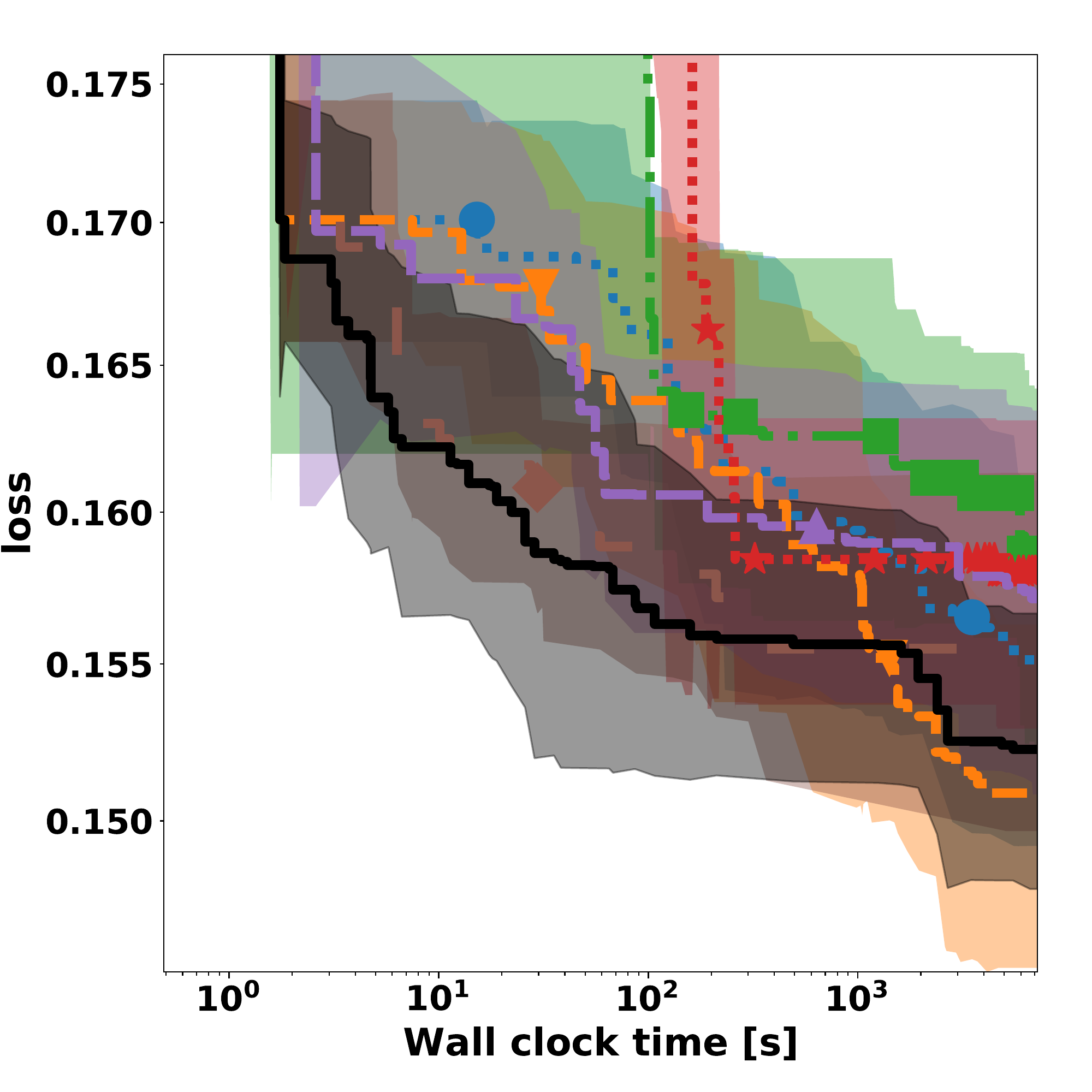}%
\caption{jasmine}%
\end{subfigure}\hfill%
  \begin{subfigure}{0.32\columnwidth}
\includegraphics[width=\columnwidth]{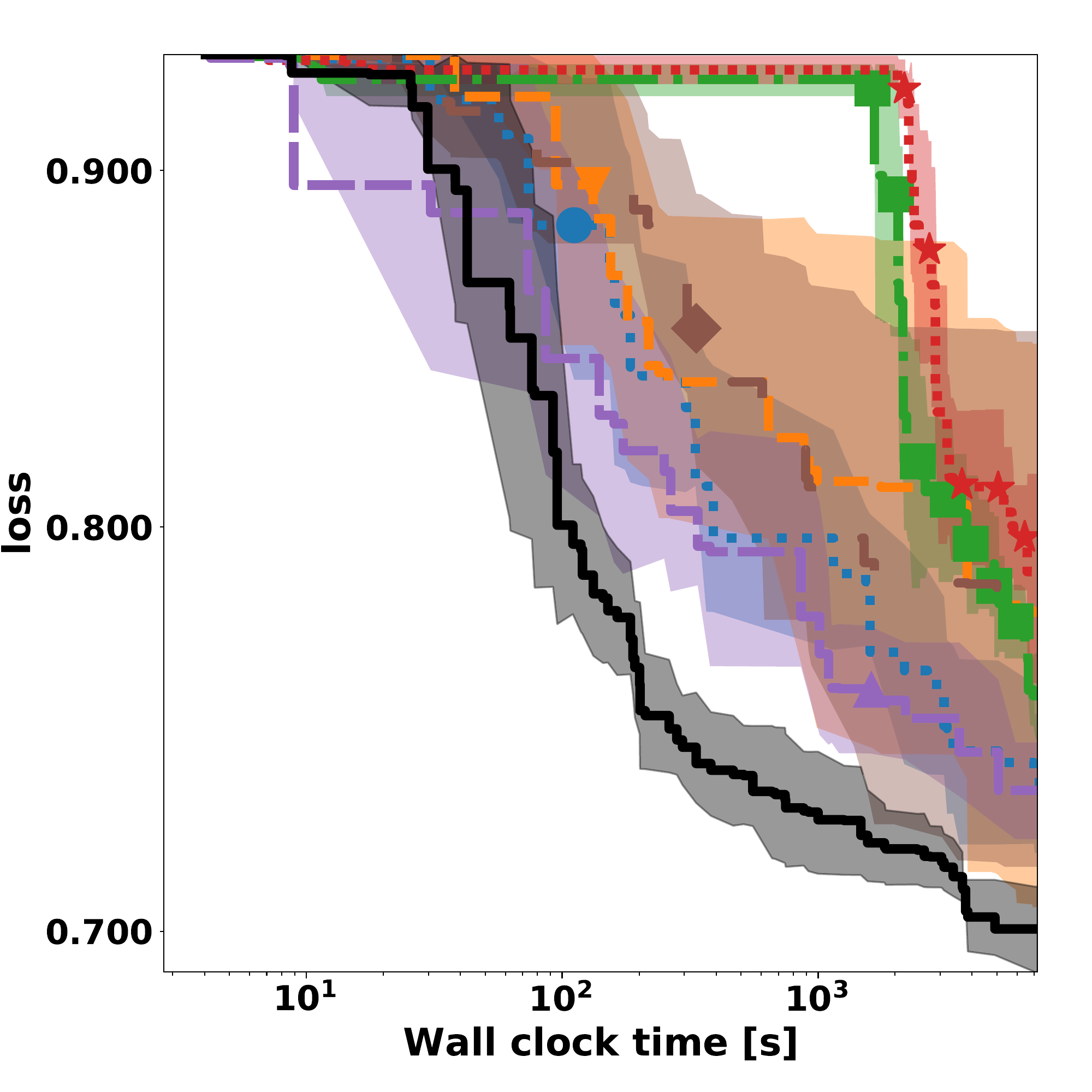}%
\caption{jungle\_chess\_2pcs\_raw\_endgame}%
\end{subfigure}\hfill%
\begin{subfigure}{0.32\columnwidth}
\includegraphics[width=\columnwidth]{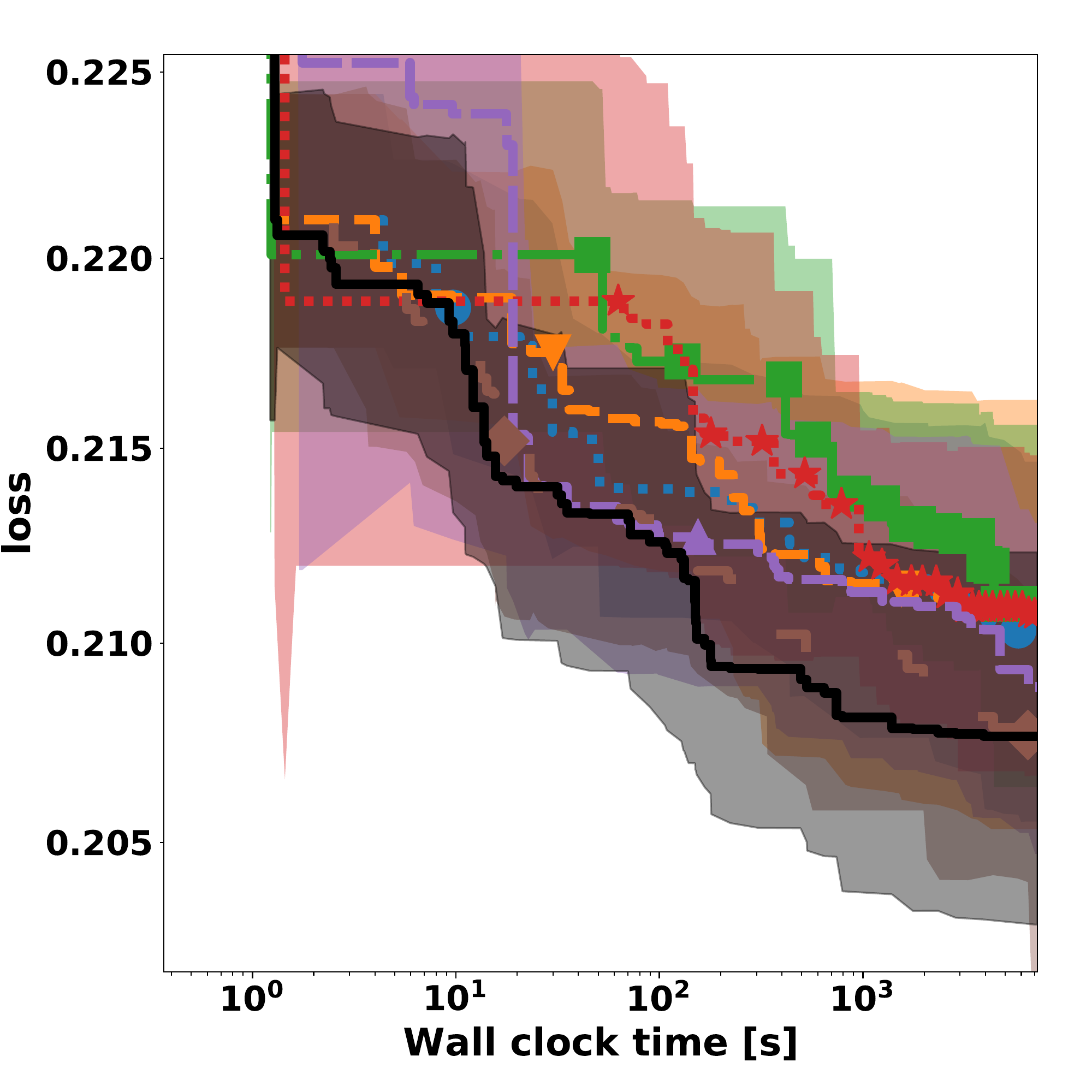}%
\caption{kc1}%
\end{subfigure}\hfill%
\begin{subfigure}{0.32\columnwidth}
\includegraphics[width=\columnwidth]{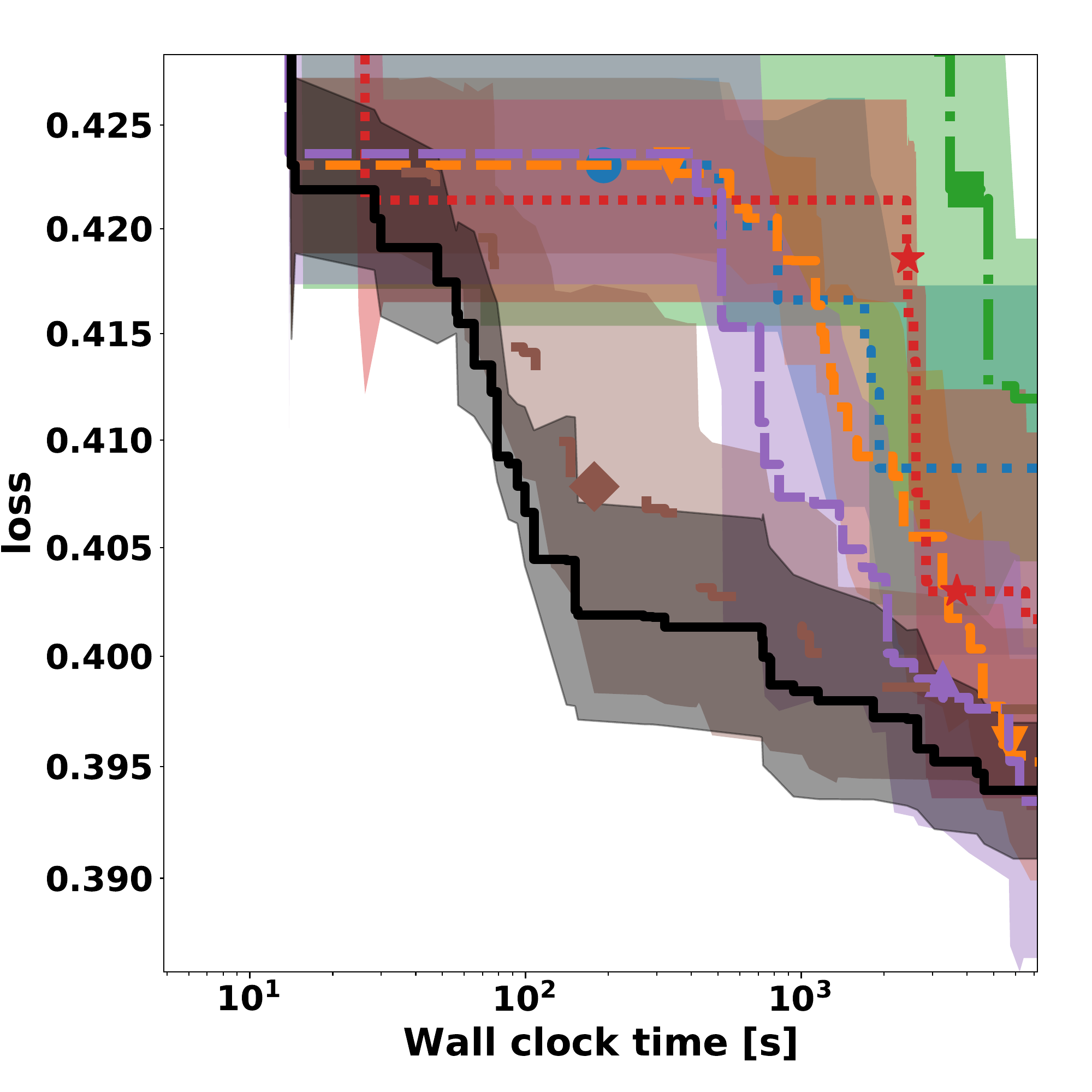}%
\caption{KDDCup09\_appetency}%
\end{subfigure}\hfill%
  \begin{subfigure}{0.32\columnwidth}
\includegraphics[width=\columnwidth]{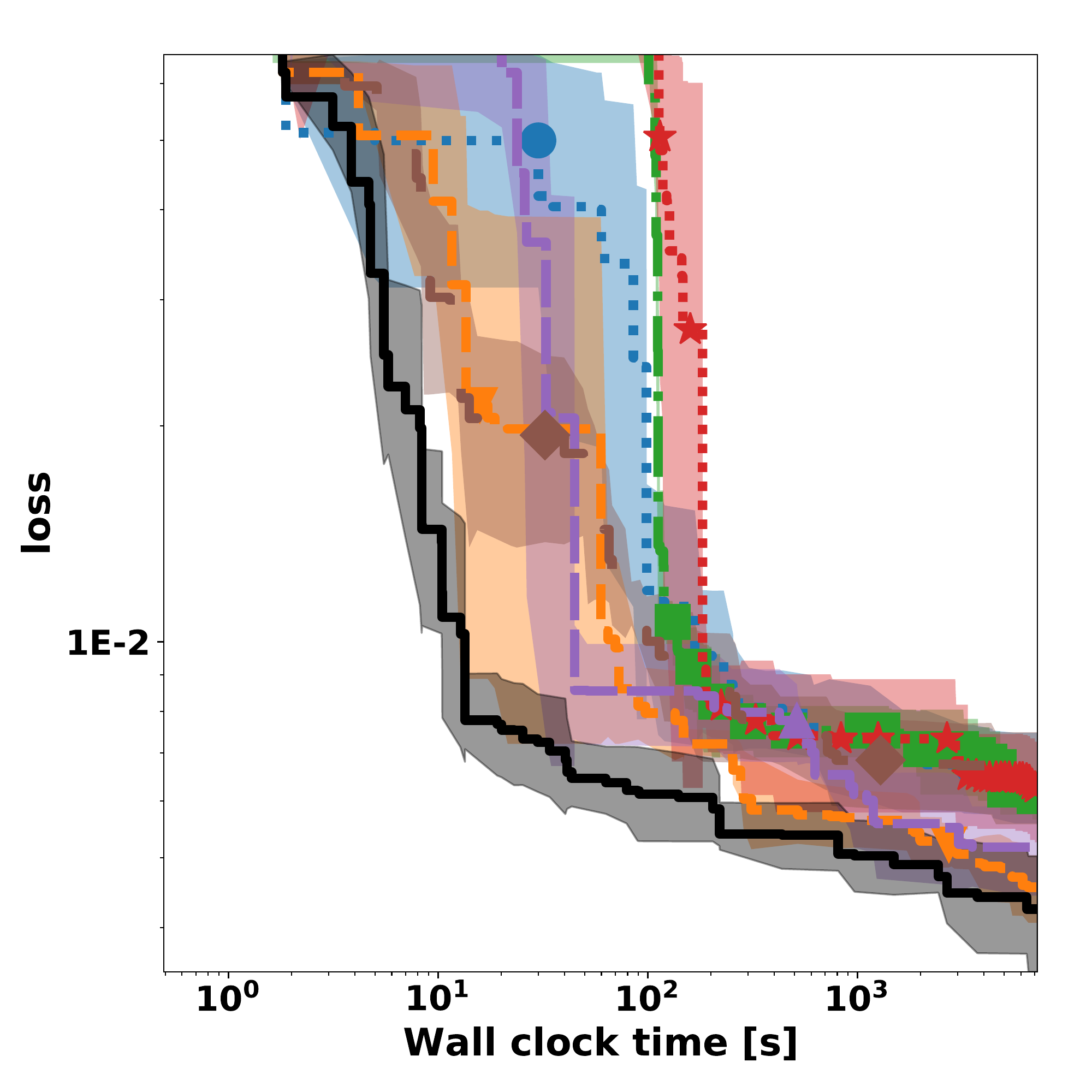}%
\caption{kr-vs-kp}%
\end{subfigure}\hfill%
\begin{subfigure}{0.32\columnwidth}
\includegraphics[width=\columnwidth]{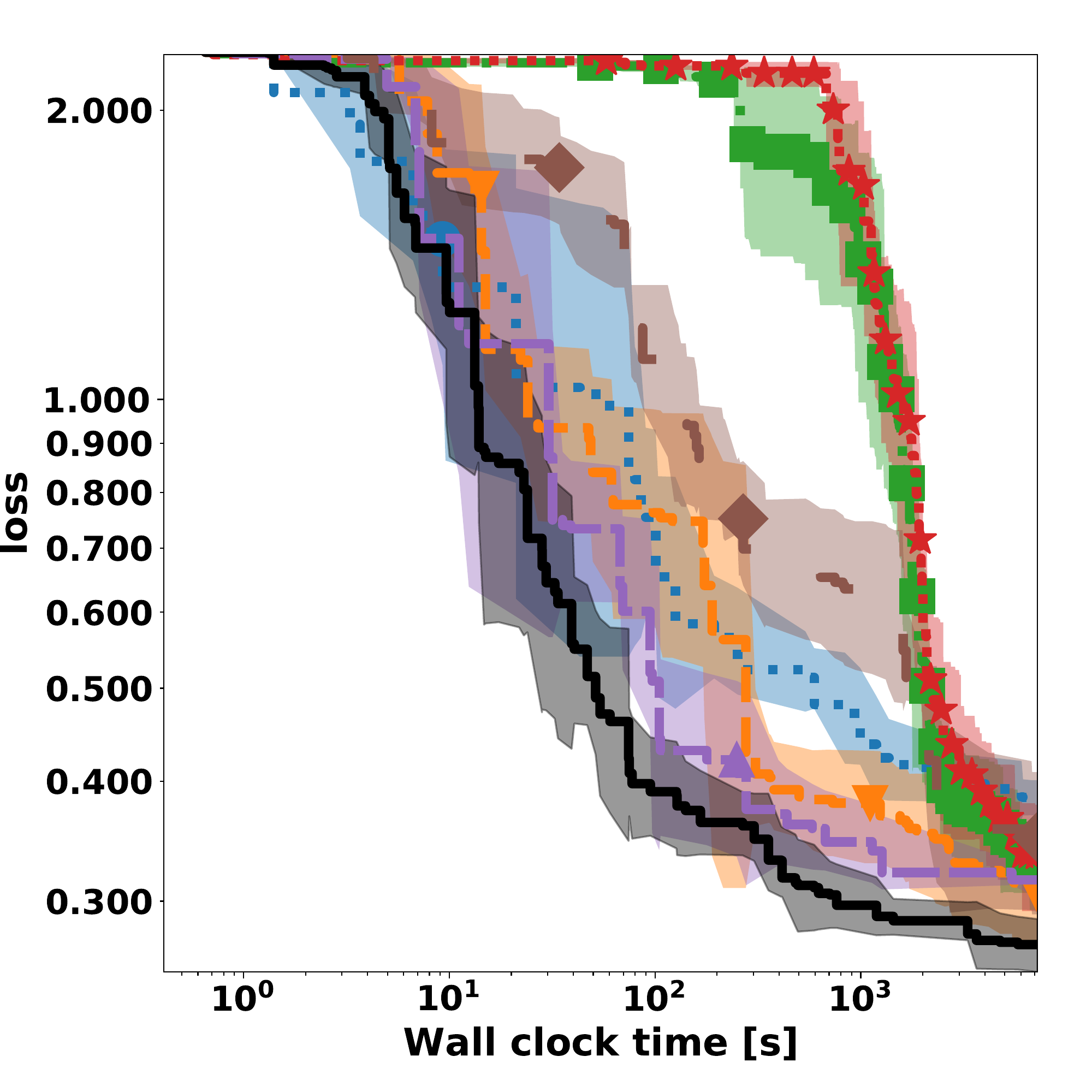}%
\caption{mfeat-factors}%
\end{subfigure}\hfill%
\begin{subfigure}{0.32\columnwidth}
\includegraphics[width=\columnwidth]{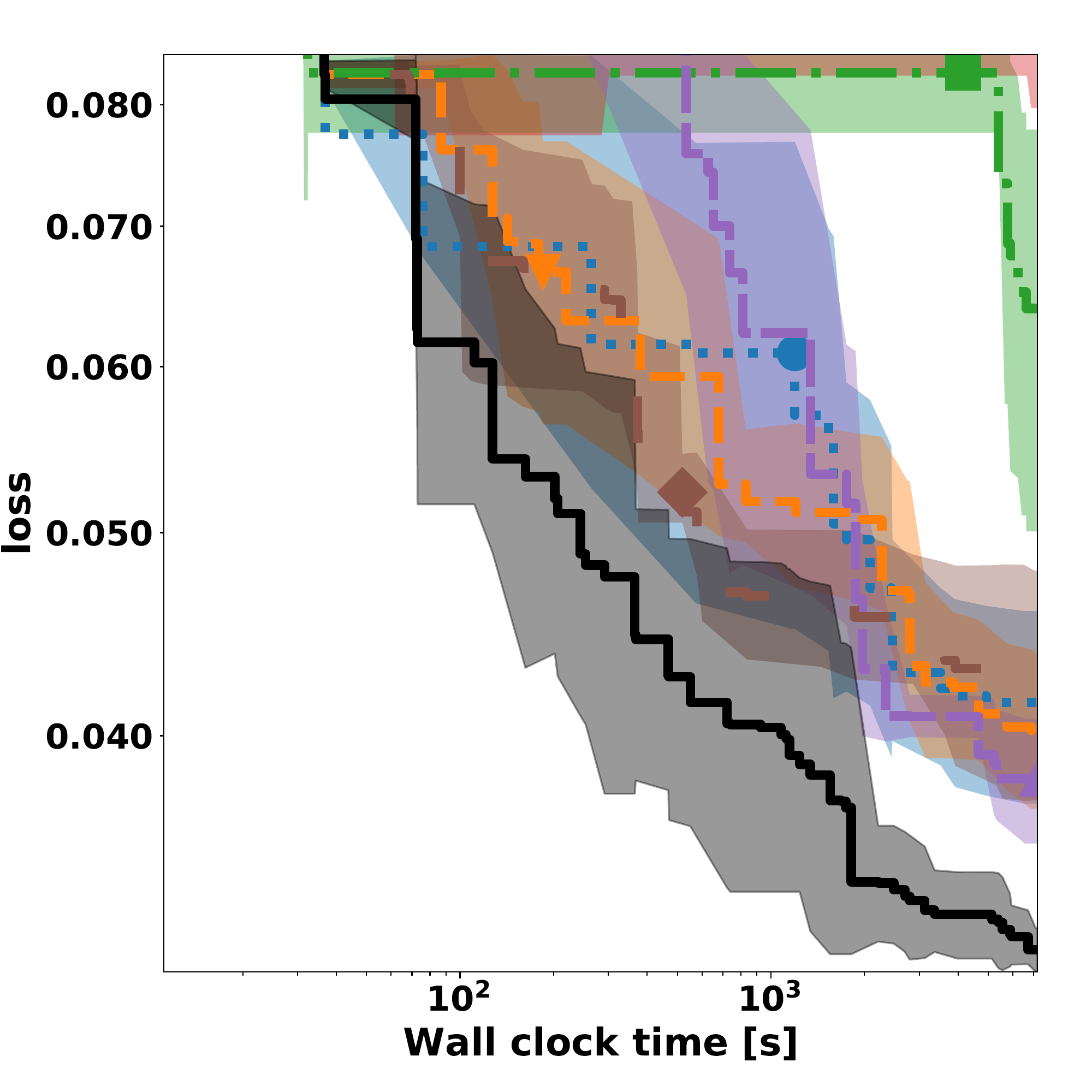}%
\caption{MiniBooNE}%
\end{subfigure}\hfill%
\begin{subfigure}{0.32\columnwidth}
\includegraphics[width=\columnwidth]{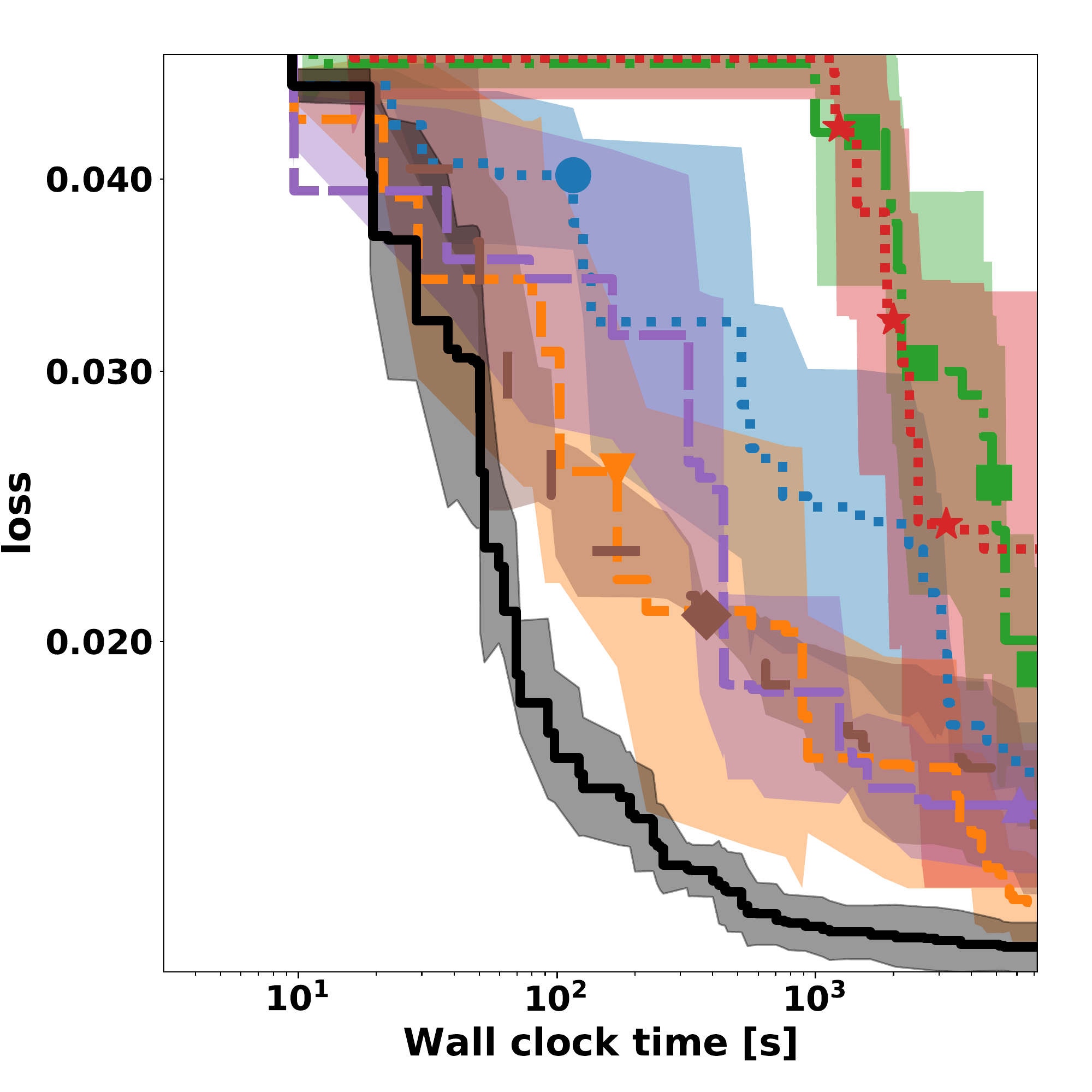}%
\caption{nomao}%
\end{subfigure}\hfill%
\caption{Optimization performance curve for DNN (pt 2/4)}
\end{figure*}
\begin{figure*}[h]
\ContinuedFloat
\begin{subfigure}{0.32\columnwidth}
\includegraphics[width=\columnwidth]{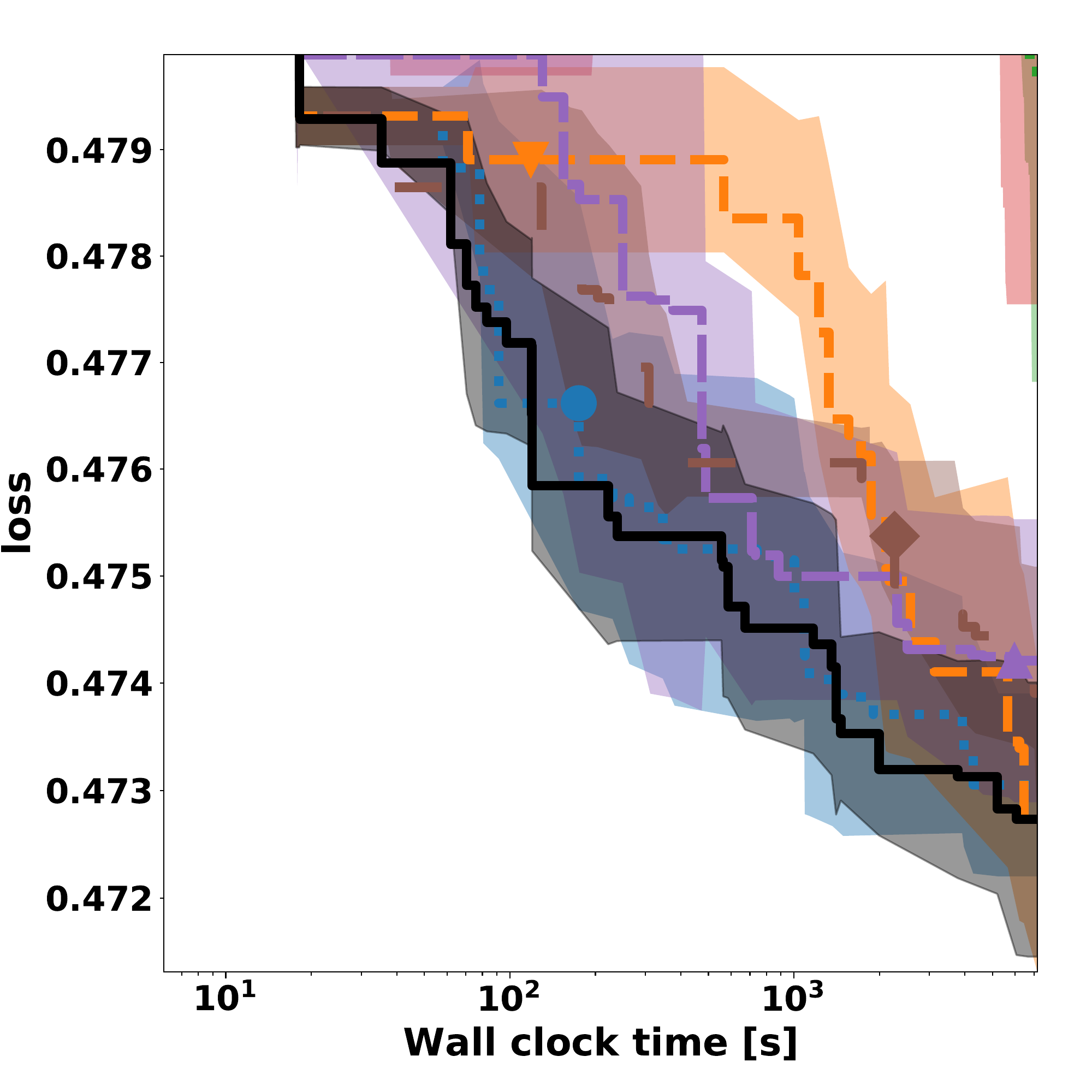}%
\caption{numerai28.6}%
\end{subfigure}\hfill%
  \begin{subfigure}{0.32\columnwidth}
\includegraphics[width=\columnwidth]{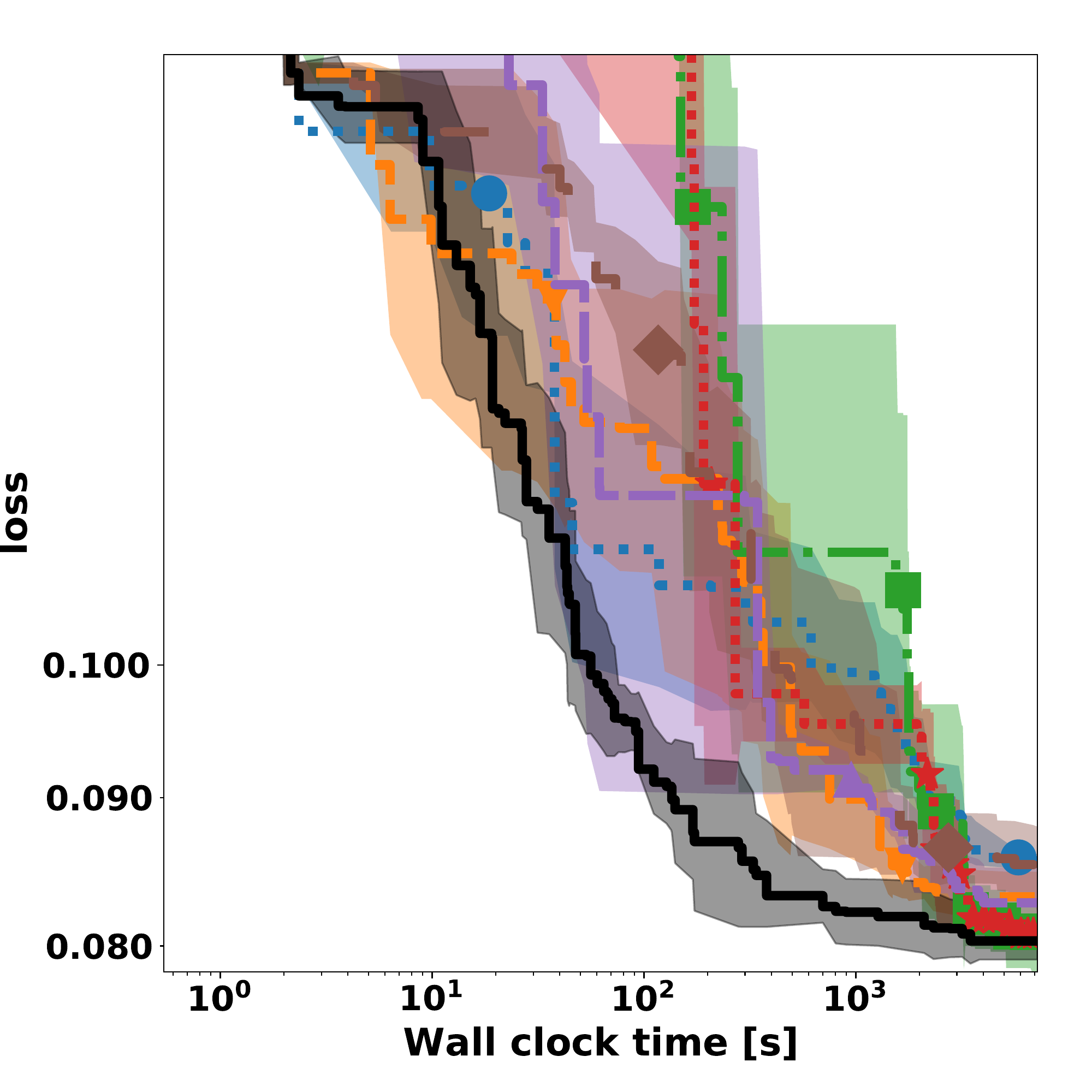}%
\caption{phoneme}%
\end{subfigure}\hfill%
  \begin{subfigure}{0.32\columnwidth}
\includegraphics[width=\columnwidth]{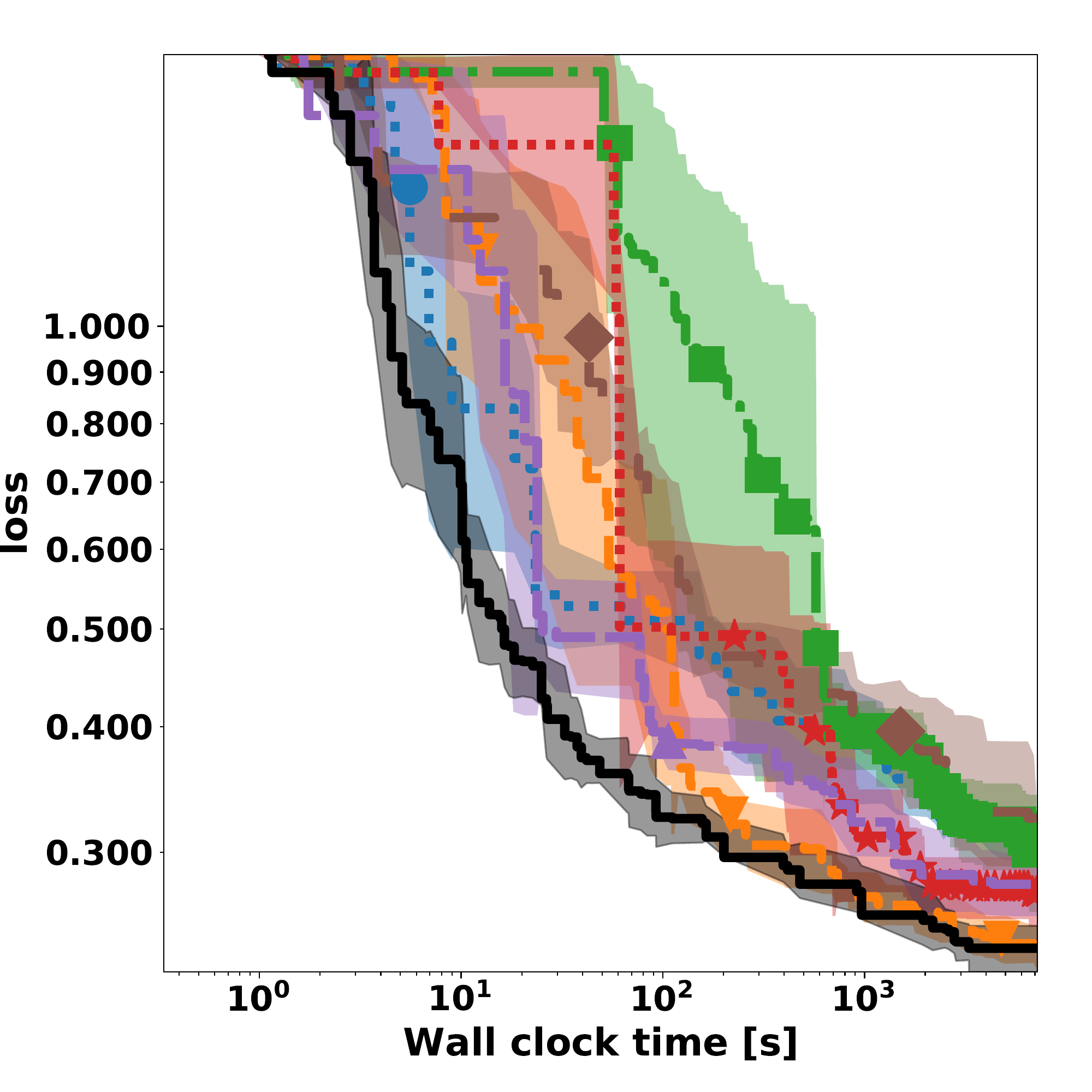}%
\caption{segment}%
\end{subfigure}\hfill%
\begin{subfigure}{0.32\columnwidth}
\includegraphics[width=\columnwidth]{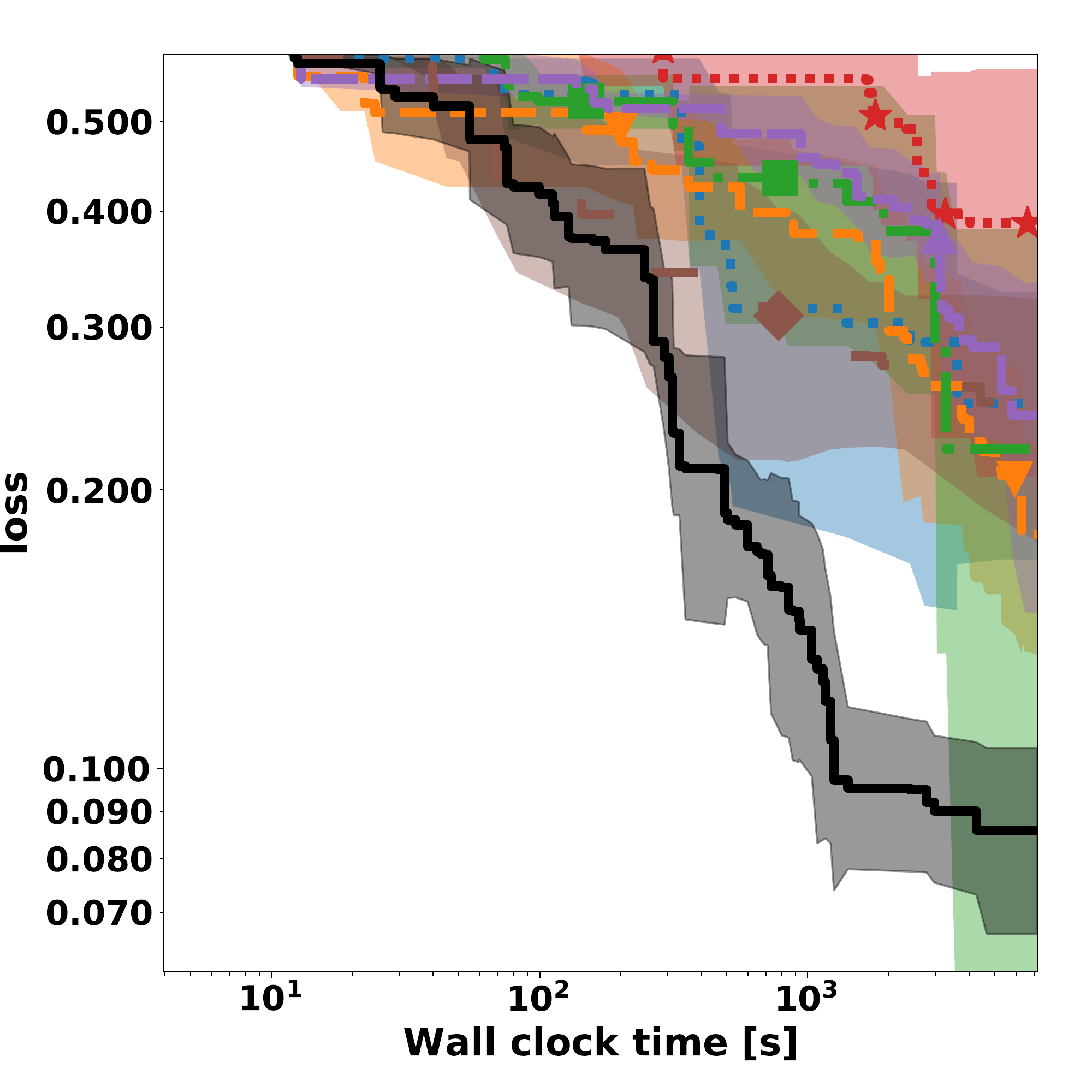}%
\caption{shuttle}%
\end{subfigure}\hfill%
\begin{subfigure}{0.32\columnwidth}
\includegraphics[width=\columnwidth]{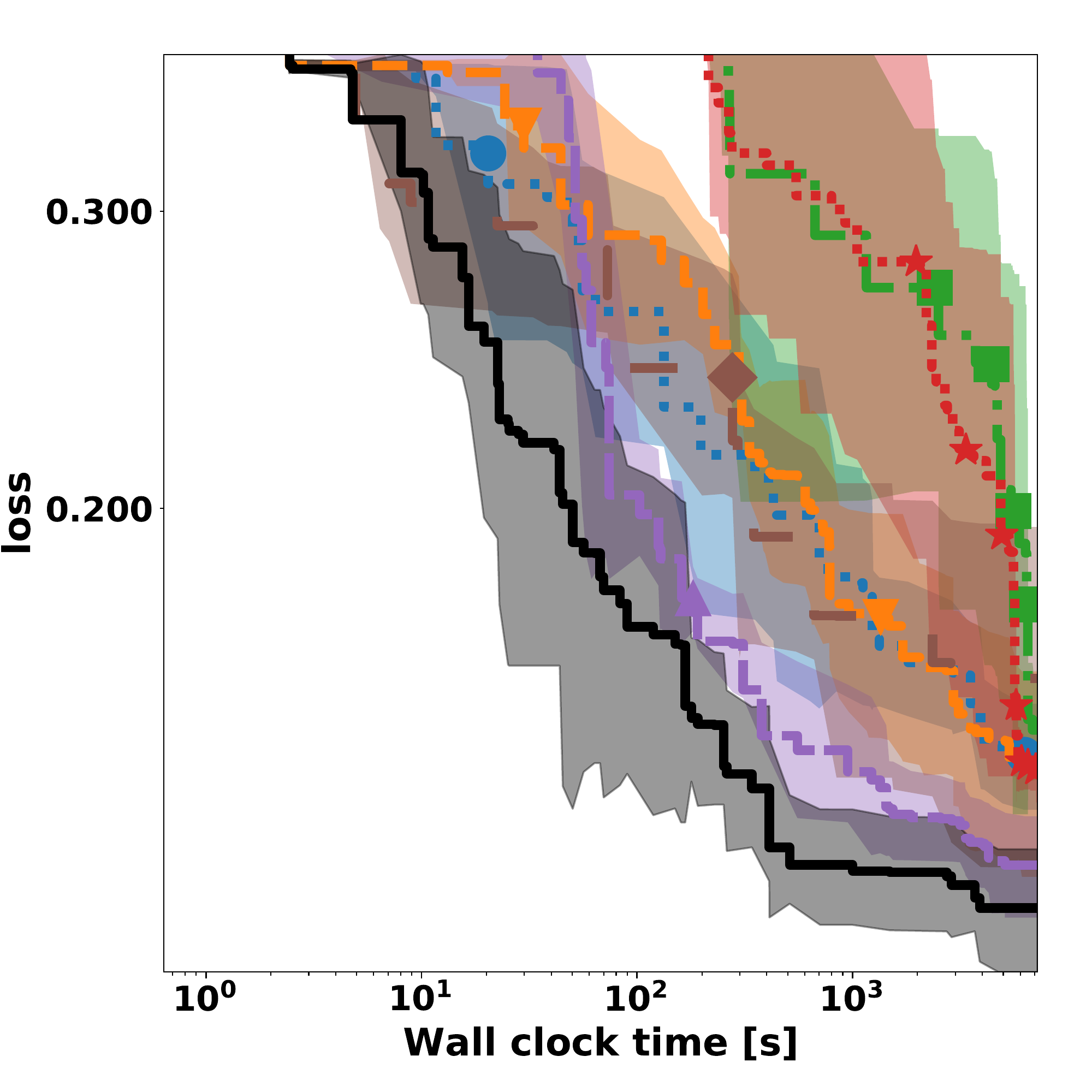}%
\caption{sylvine}%
\end{subfigure}\hfill%
  \begin{subfigure}{0.32\columnwidth}
\includegraphics[width=\columnwidth]{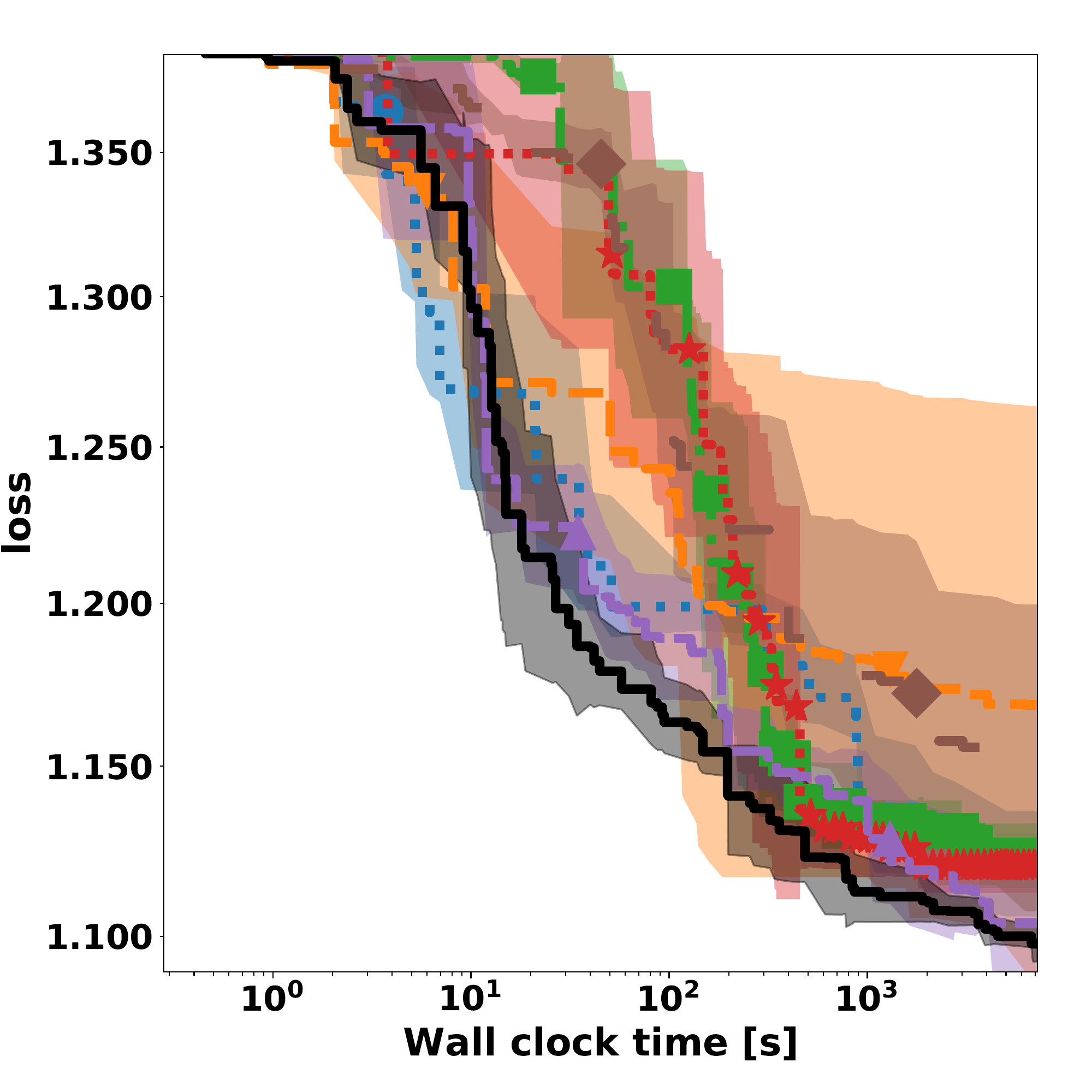}%
\caption{vehicle}%
\end{subfigure}\hfill%
\begin{subfigure}{0.32\columnwidth}
\includegraphics[width=\columnwidth]{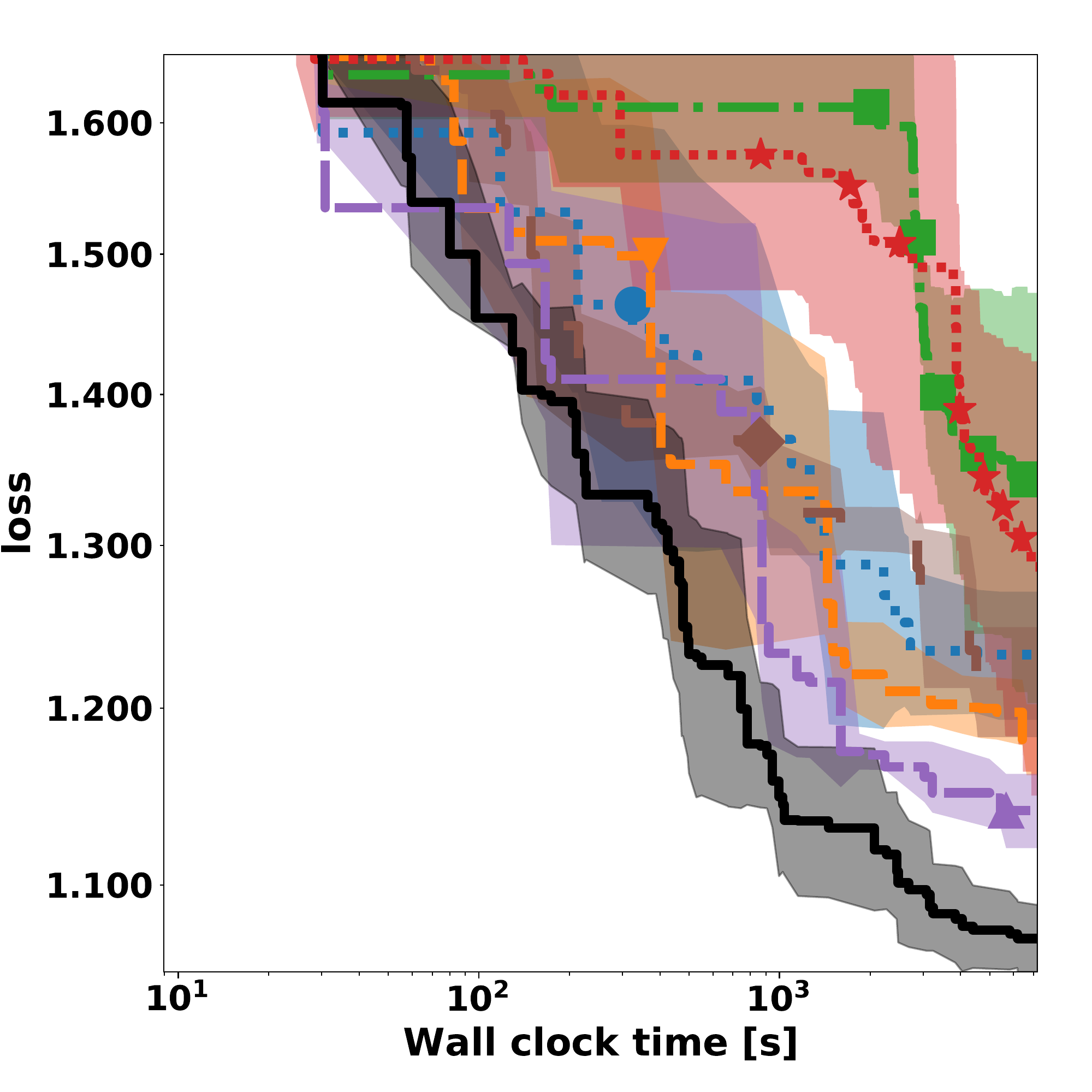}%
\caption{volkert}%
\end{subfigure}\hfill%
  \begin{subfigure}{0.32\columnwidth}
\includegraphics[width=\columnwidth]{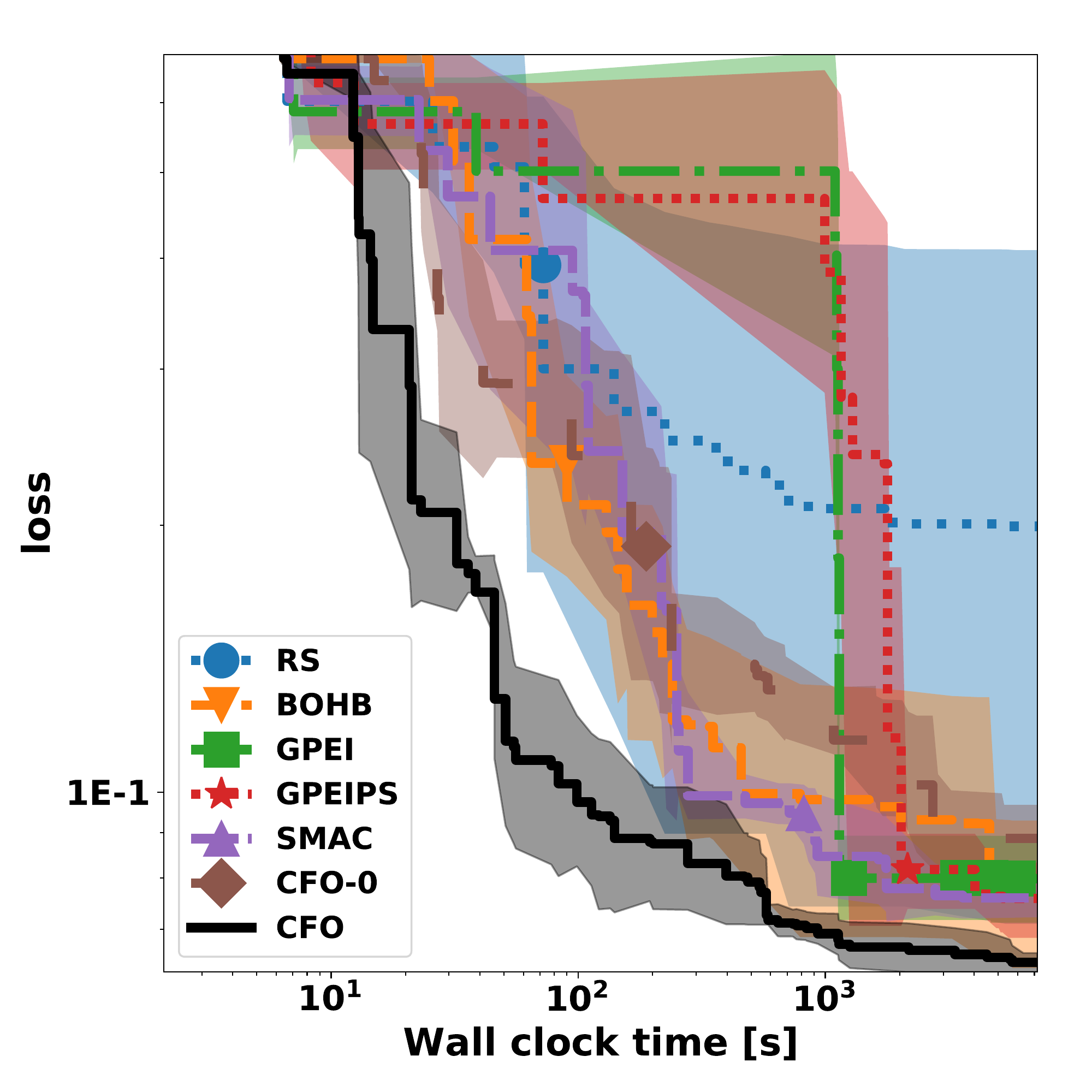}%
\caption{2dplanes}%
\end{subfigure}\hfill%
\begin{subfigure}{0.32\columnwidth}
\includegraphics[width=\columnwidth]{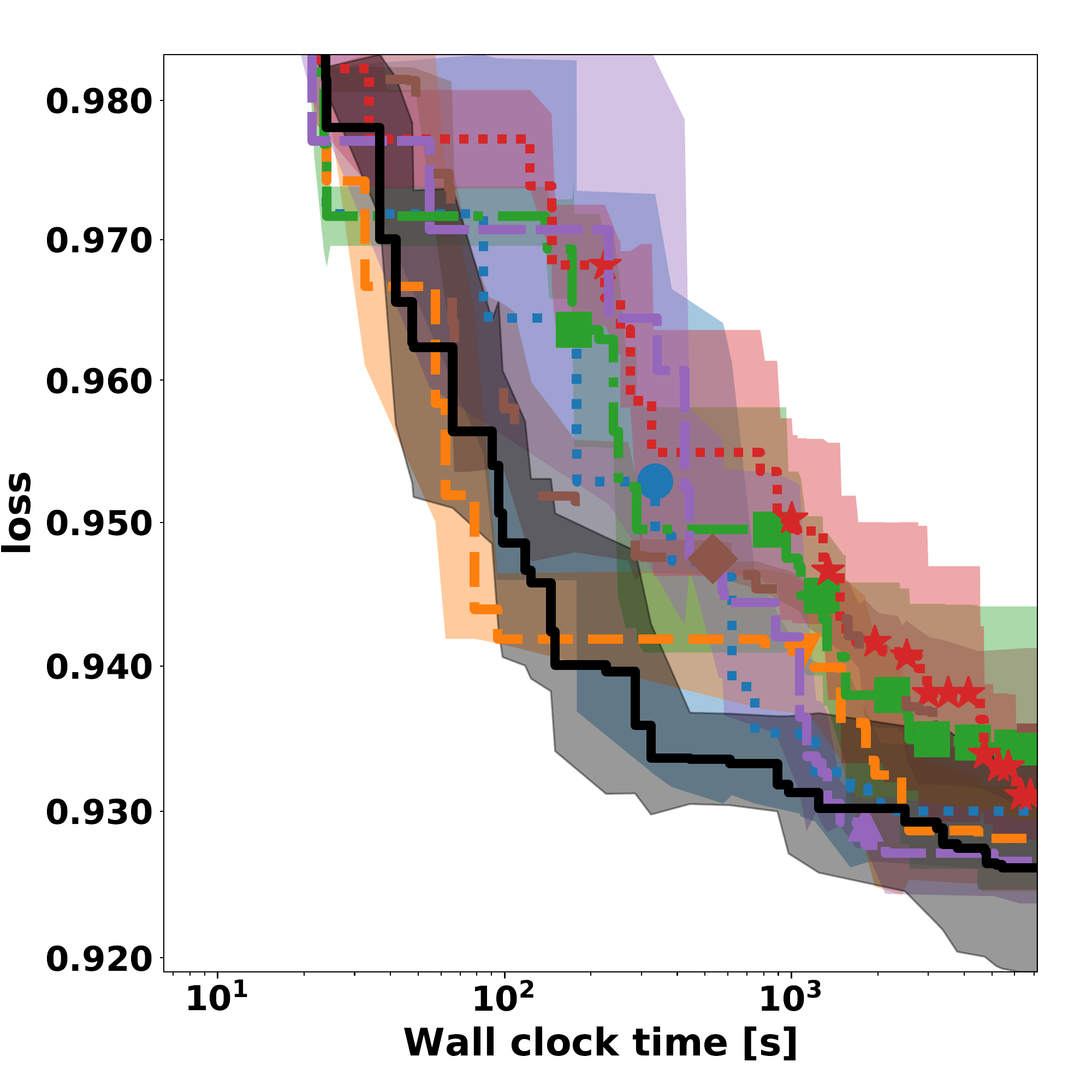}%
\caption{bng\_breastTumor}%
\end{subfigure}\hfill%
\caption{Optimization performance curve for DNN (pt 3/4)}
\end{figure*}
\begin{figure*}[h]
\ContinuedFloat
\begin{subfigure}{0.32\columnwidth}
\includegraphics[width=\columnwidth]{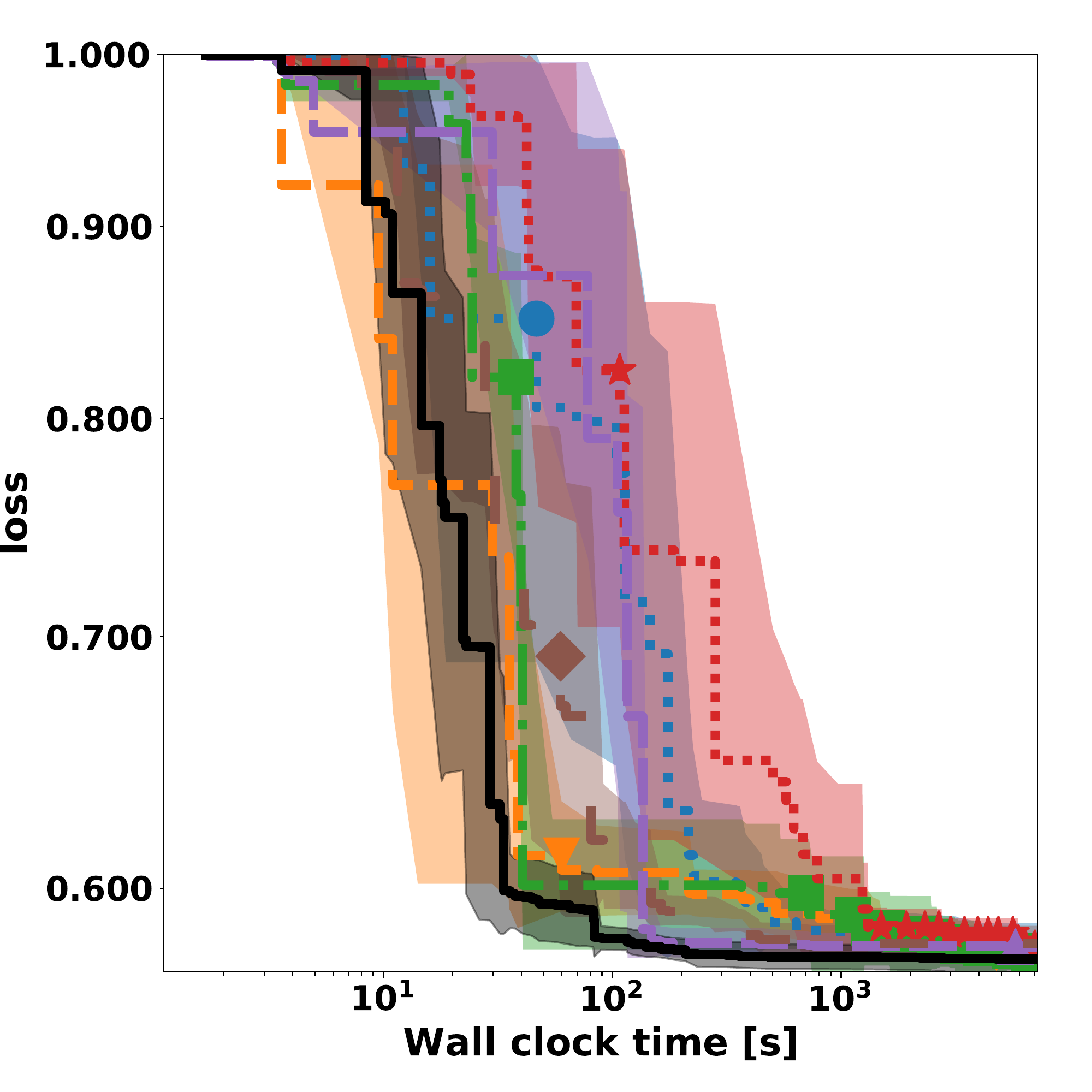}%
\caption{bng\_echomonths}%
\end{subfigure}\hfill%
  \begin{subfigure}{0.32\columnwidth}
\includegraphics[width=\columnwidth]{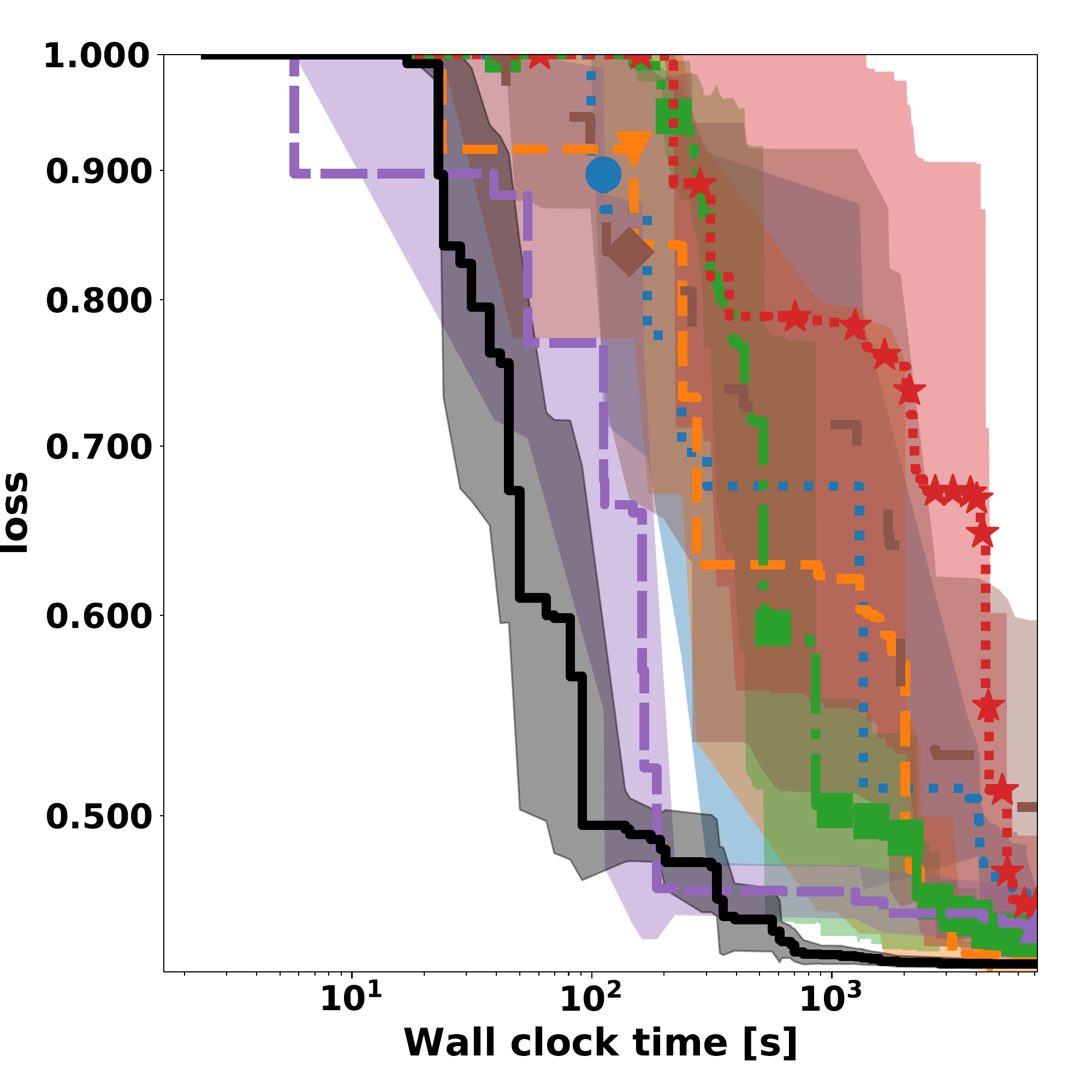}%
\caption{bng\_lowbwt}%
\end{subfigure}\hfill%
\begin{subfigure}{0.32\columnwidth}
\includegraphics[width=\columnwidth]{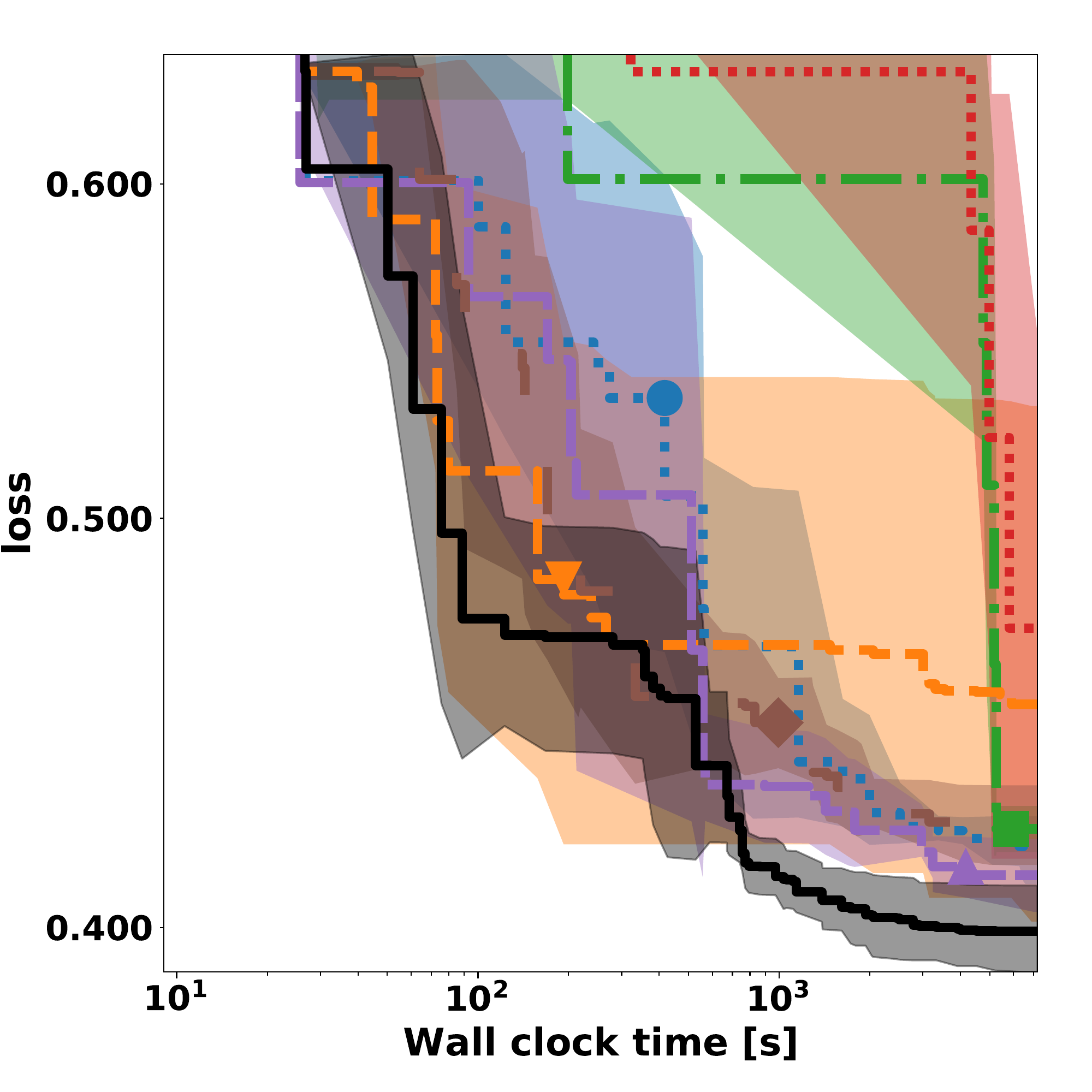}%
\caption{bng\_pwLinear}%
\end{subfigure}\hfill%
\begin{subfigure}{0.32\columnwidth}
\includegraphics[width=\columnwidth]{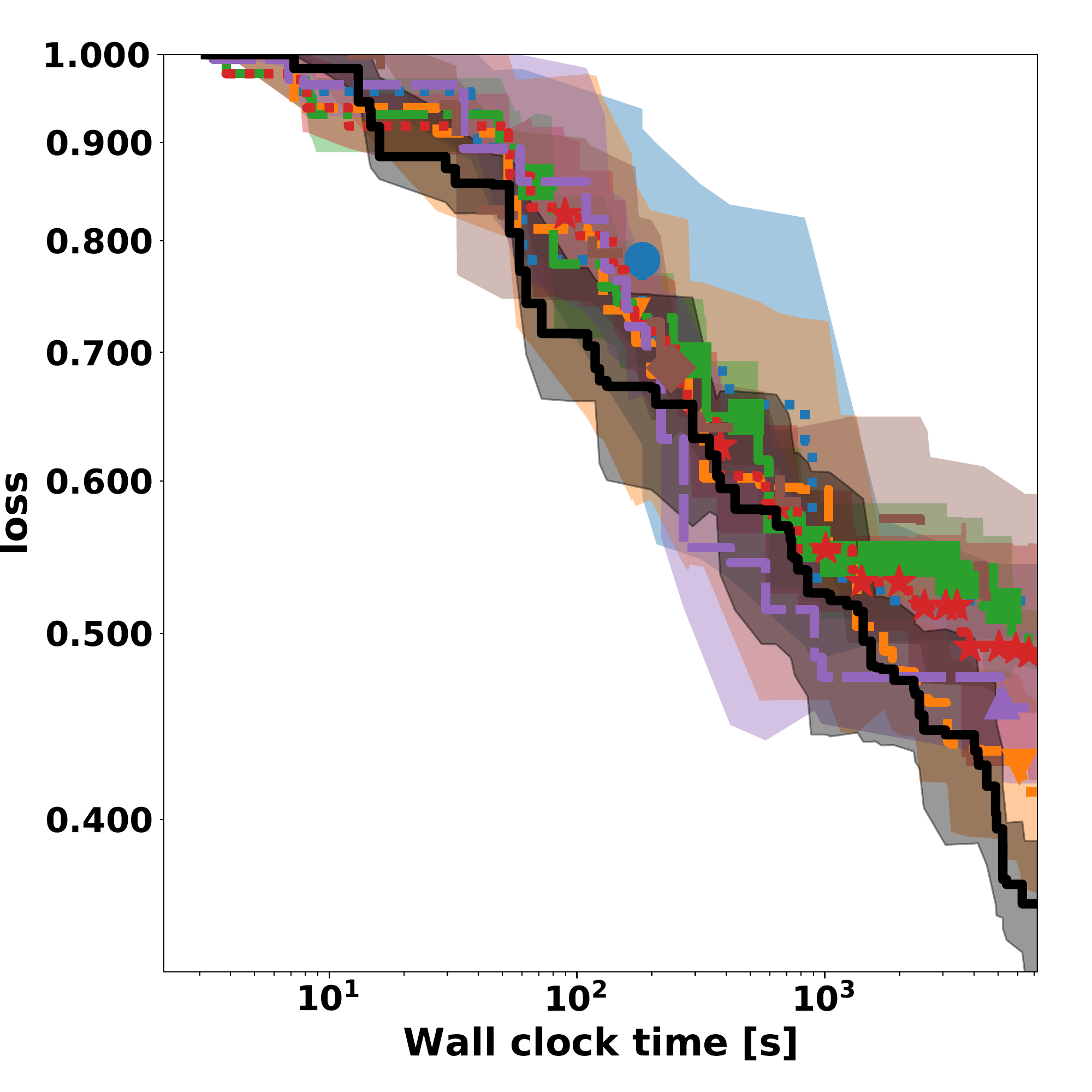}%
\caption{fried}%
\end{subfigure}\hfill%
  \begin{subfigure}{0.32\columnwidth}
\includegraphics[width=\columnwidth]{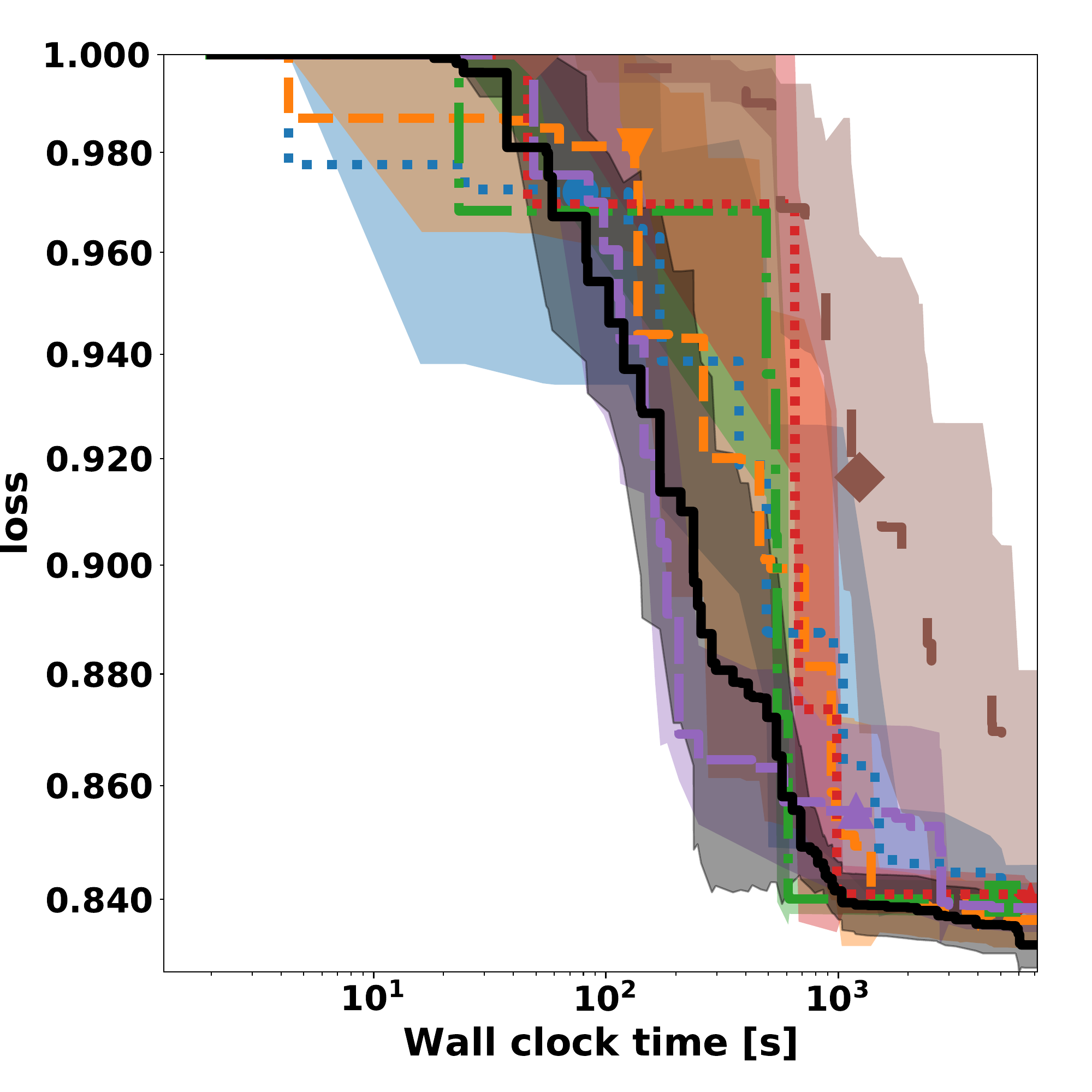}%
\caption{house\_16H}%
\end{subfigure}\hfill%
  \begin{subfigure}{0.32\columnwidth}
\includegraphics[width=\columnwidth]{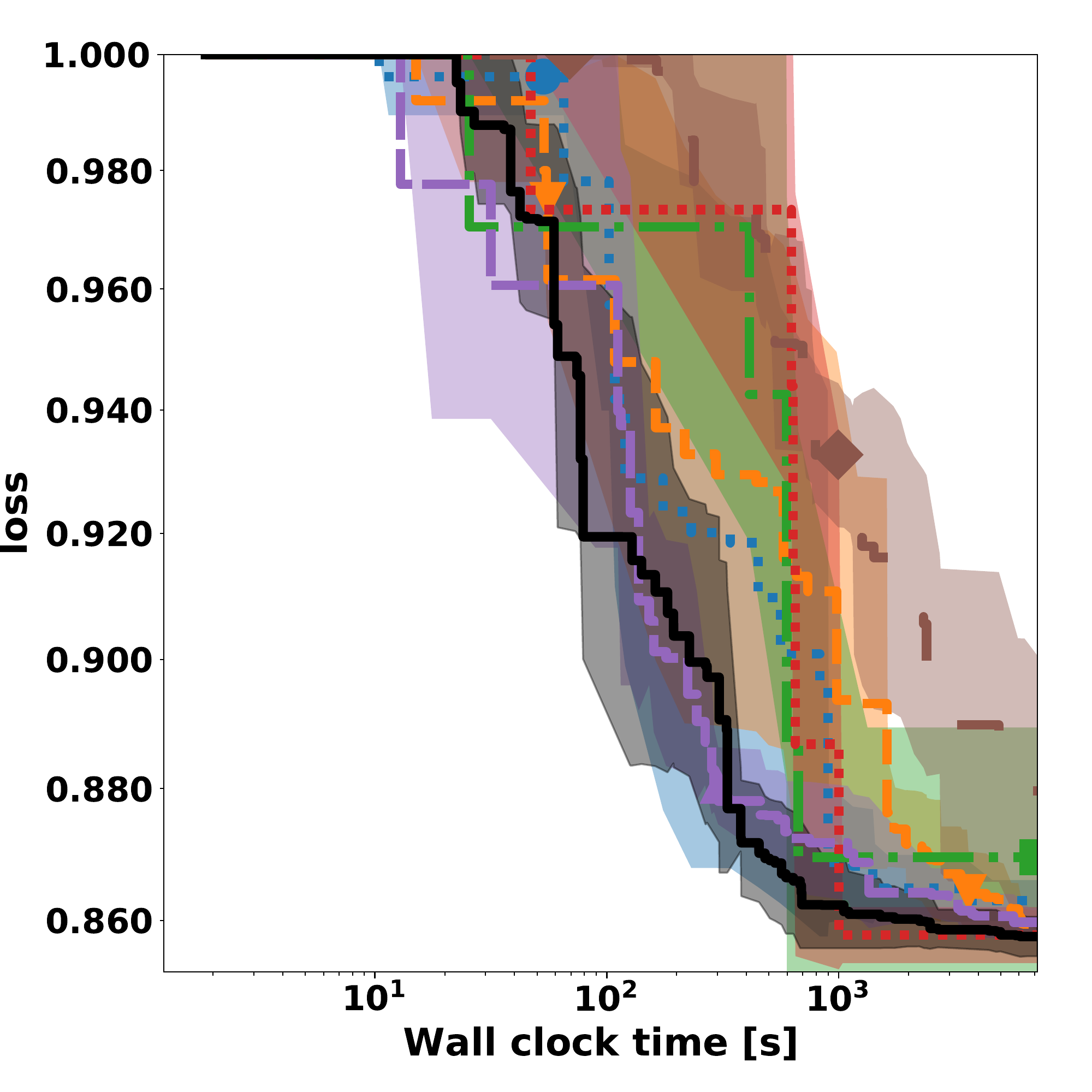}%
\caption{house\_8L}%
\end{subfigure}\hfill%
\begin{subfigure}{0.32\columnwidth}
\includegraphics[width=\columnwidth]{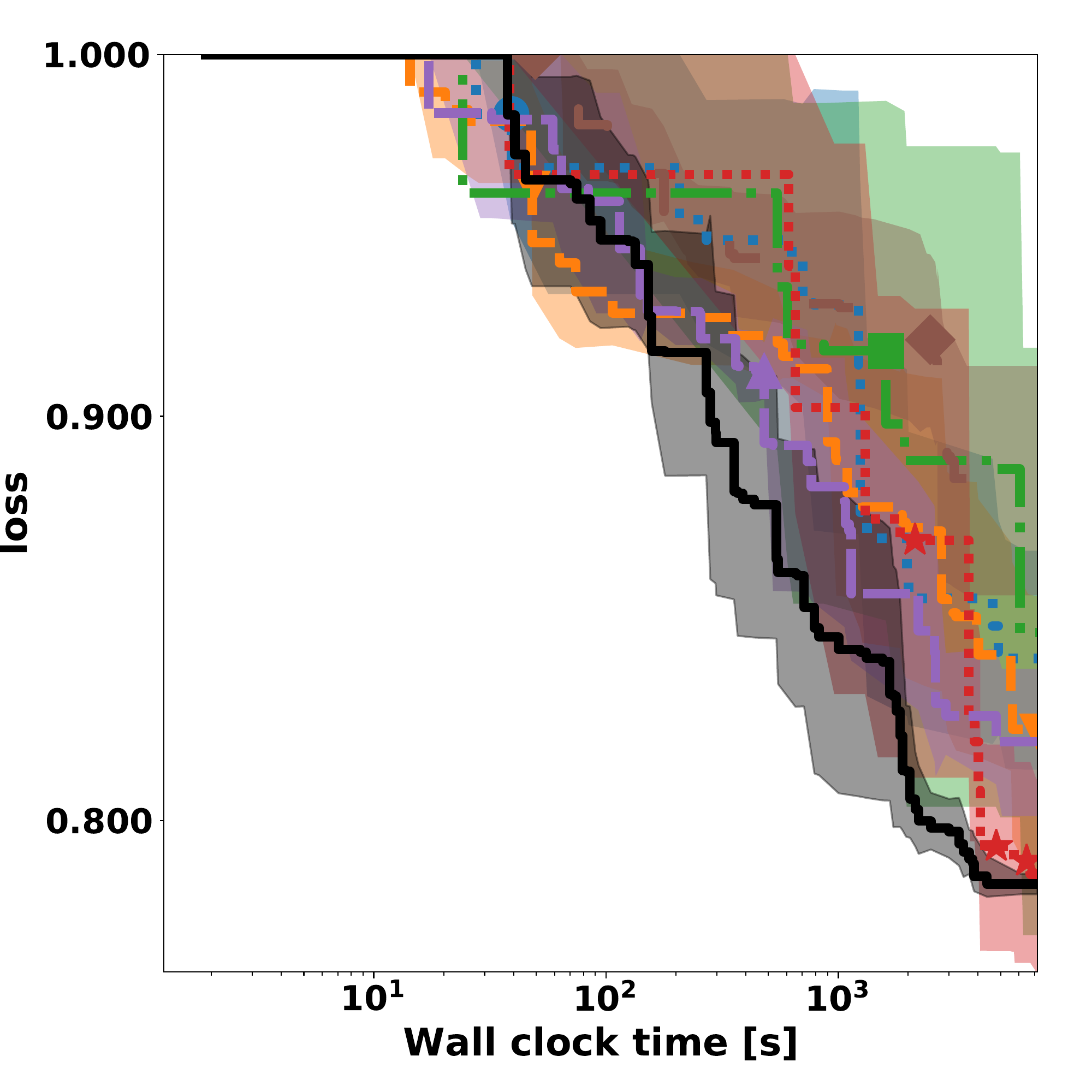}%
\caption{houses}%
\end{subfigure}\hfill%
\begin{subfigure}{0.32\columnwidth}
\includegraphics[width=\columnwidth]{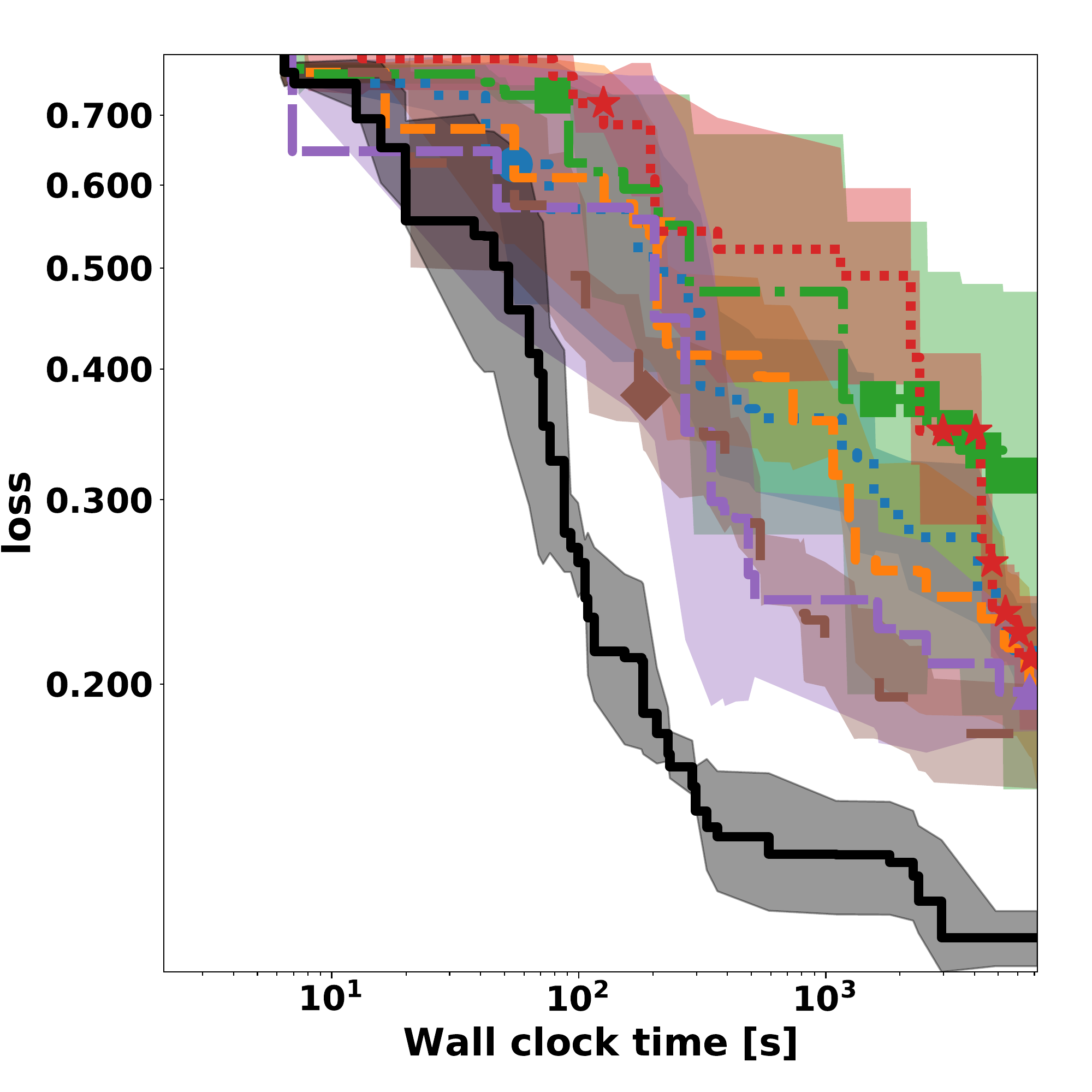}%
\caption{mv}%
\end{subfigure}\hfill%
  \begin{subfigure}{0.32\columnwidth}
\includegraphics[width=\columnwidth]{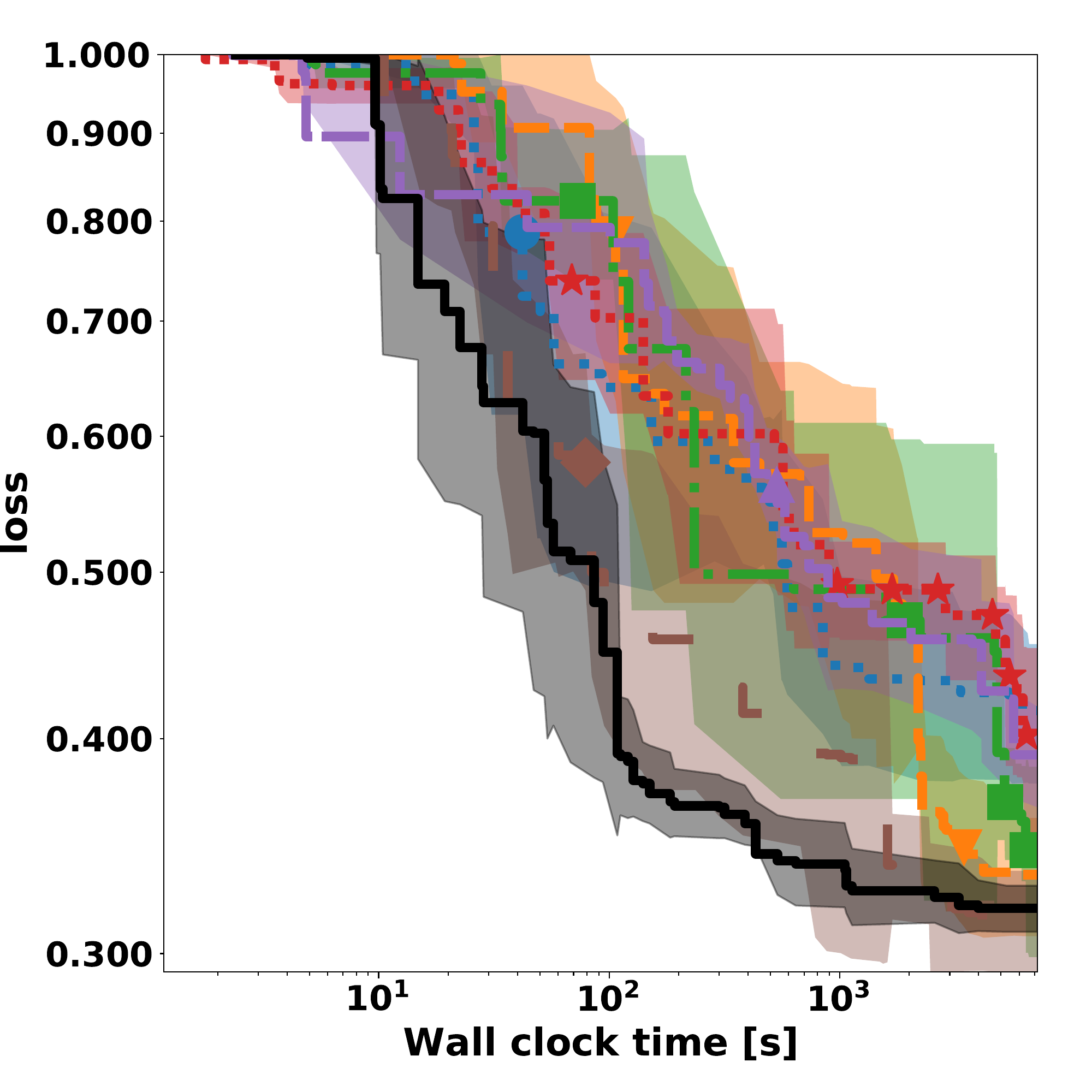}%
\caption{pol}%
\end{subfigure}\hfill%
\caption{Optimization performance curve for DNN (pt 4/4)}
  \label{fig:curve_dnn}
\end{figure*}

\end{document}